%% file: paper.tex
\documentclass[twoside,11pt]{article}

\usepackage{blindtext}
\usepackage{jmlr2e}

\usepackage{url}            
\usepackage{booktabs}       
\usepackage{amsfonts}       
\usepackage{nicefrac}       
\usepackage{microtype}      
\usepackage{enumitem}
\usepackage{amsmath}
\usepackage{mathtools}
\usepackage{mathrsfs}
\mathtoolsset{showonlyrefs}
\usepackage{graphicx}
\usepackage{subcaption}
\graphicspath{ {images/} }
\usepackage{bbm}
\usepackage{algorithm}
\usepackage[noend]{algpseudocode}
\usepackage{sidecap}
\usepackage{wrapfig}
\usepackage{xcolor}
\usepackage{pgf}
\usepackage{tikz}
\usepackage{dsfont}
\usepackage[export]{adjustbox}

\newcommand{\argmax}{\mathop{\mathrm{arg max}}}
\newtheorem{hypothesis}{\textbf{Hypothesis}}

\newcommand{\defeq}{\mathrel{\overset{\makebox[0pt]{\mbox{\normalfont\tiny\sffamily def}}}{=}}}
\newcommand{\EE}{\mathbb{E}}
\newcommand{\RR}{\mathbb{R}}

\newcommand{\States}{\mathcal{S}}
\newcommand{\Actions}{\mathcal{A}}
\newcommand{\Goals}{\mathcal{G}}
\newcommand{\augGoals}{\bar{\mathcal{G}}}
\newcommand{\sterm}{s_{\text{terminal}}}

\newcommand{\mparams}{{\boldsymbol{\theta}}}

\newcommand{\subroutine}[1]{\texttt{#1}}
\newcommand{\qparams}{{\mathbf{w}}}

\newcommand{\vsg}{r_\gamma}
\newcommand{\vgg}{\tilde{r}_\gamma}
\newcommand{\Gammasg}{\Gamma}
\newcommand{\Gammagg}{\tilde{\Gamma}}
\newcommand{\indsg}{m}
\newcommand{\relsg}{d}
\newcommand{\relgg}{\tilde{d}}
\newcommand{\vsub}{v_{g^\star}}
\newcommand{\vsubt}[1]{v_{g^\star,{#1}}}
\newcommand{\rparams}{{\boldsymbol{\theta}^r}}
\newcommand{\gamparams}{{\boldsymbol{\theta}^\Gamma}}
\newcommand{\vggparams}{{\tilde{\boldsymbol{\theta}}^r}}
\newcommand{\gamggparams}{{\tilde{\boldsymbol{\theta}}^\Gamma}}
\newcommand{\vgoal}{\tilde{v}}

\newcommand{\Ppigamma}{P_{\pi, \gamma}}
\newcommand{\rpigamma}{r_{\pi, \gamma}}
\newcommand{\polparams}{{\boldsymbol{\theta}^\pi}}

 \newcommand{\rgamma}{r_{\gamma}}
 \newcommand{\GamModel}{\Gamma}

\newcommand{\optionq}{\tilde{q}}

\usetikzlibrary{shapes.geometric, arrows}
\tikzstyle{process} = [rectangle, minimum width=3cm, minimum height=1cm, text width=10cm, text centered, draw=black]
\tikzstyle{arrow} = [thick,->,>=stealth]

\usepackage{lastpage}
\jmlrheading{25}{2024}{1-\pageref{LastPage}}{1/24; Revised -/-}{-/-}{21-0000}{Lo, Roice, Panahi, Jordan, White, Mihucz, Aminmansour and White}
\ShortHeadings{Goal-Space Planning with Subgoal Models}{Lo, Roice, Panahi, Jordan, White, Mihucz, Aminmansour and White}
\firstpageno{1}

\begin{document}

\title{Goal-Space Planning with Subgoal Models}

\author{\name Chunlok Lo\thanks{Equal contribution} \email chunlok@ualberta.ca 
\AND
\name Kevin Roice\footnotemark[1] \email roice@ualberta.ca 
\AND
\name Parham Mohammad Panahi\footnotemark[1]  \email parham1@ualberta.ca 
\AND
\name Scott Jordan \email sjordan@ualberta.ca 
\AND
\name Adam White$^\dagger$ \email amw8@ualberta.ca 
\AND \name G{\'a}bor Mihucz \email mihucz@ualberta.ca 
\AND \name Farzane Aminmansour \email aminmans@ualberta.ca 
\AND
\name Martha White$^\dagger$ \email whitem@ualberta.ca \\
\addr Department of Computing Science\\
   University of Alberta\\
    Alberta Machine Intelligence Institute (Amii)\\
    Canada CIFAR AI Chair$^\dagger$\\
   Edmonton, Canada
}

\editor{My editor}

\maketitle

\begin{abstract}
This paper investigates a new approach to model-based reinforcement learning using background planning: mixing (approximate) dynamic programming updates and model-free updates, similar to the Dyna architecture. Background planning with learned models is often worse than model-free alternatives, such as Double DQN, even though the former uses significantly more memory and computation. The fundamental problem is that learned models can be inaccurate and often generate invalid states, especially when iterated many steps. In this paper, we avoid this limitation by constraining background planning to a set of (abstract) subgoals and learning only local, subgoal-conditioned models. This goal-space planning (GSP) approach is more computationally efficient, naturally incorporates temporal abstraction for faster long-horizon planning and avoids learning the transition dynamics entirely. We show that our GSP algorithm can propagate value from an abstract space in a manner that helps a variety of base learners learn significantly faster in different domains.
\end{abstract}


\section{Introduction}

Planning with learned models in reinforcement learning (RL) is important for sample efficiency. 
Planning 
provides a mechanism for the agent to generate hypothetical experience, in the background during interaction, to improve value estimates. This hypothetical experience provides a stand-in for the real-world; the agent can generate many experiences (transitions) in it's head (via a model) and learn from those experiences.
Dyna \citep{sutton1991integrated} is a classic example of {\em background planning}. On each step, the agent generates several transitions according to its model, and updates with those transitions as if they were real experience. 

Background planning can be used to both adapt to the non-stationarity and exploit things that remain constant. In many interesting environments, like the real-world or multi-agent games, the agent will be under-parameterized and thus cannot learn or even represent a stationary optimal policy. The agent can overcome this limitation, however, by using a model to rapidly update its policy. Continually updating the model and replanning allows the agent to adapt to the current situation. In addition, many aspects of the environment remain stationary (e.g., fire hurts and objects fall). The model can capture these stationary facts about how the world works and planning can be used to reason about how the world works to produce a better policies.

The promise of background planning
is that we can to learn and adapt value estimates efficiently, but many open problems remain to make it more widely useful. These include that 1) long rollouts generated by one-step models can diverge or generate invalid states, 2) learning probabilities over outcome states can be complex, especially for high-dimensional tasks and 3) planning itself can be computationally expensive for large state spaces. 

One way to overcome these issues is to construct an abstract model of the environment and plan at a higher level of abstraction. In this paper, we construct abstract MDPs using both state abstraction as well as temporal abstraction. State abstraction is achieved by simply grouping states. Temporal abstraction is achieved using \emph{options}---a policy coupled with a termination condition and initiation set \citep{sutton1999options}. A temporally-abstract model based on options allows the agent to {\em jump} between abstract states potentially alleviating the need to generate long rollouts.

An abstract model can be used to directly compute a policy in the abstract MDP, but there are issues with this approach.
This idea was explored with an algorithm called Landmark-based Approximate Value Iteration (LAVI) \citep{mann2015approximate}. Though planning can be shown to be provably more efficient, the resulting policy is suboptimal, restricted to going between landmark states. 
This suboptimality issue forces a trade-off between increasing the size of the abstract MDP (to increase the policy's expressivity) and increasing the computational cost to update the value function. In this paper, we investigate abstract model-based planning methods that have a small computational cost, can quickly propagate changes in value over the entire state space, and do not limit the optimality of learned policy. 

An alternative strategy that we explore in this work is to use the policy computed from the abstract MDP to guide the learning process in solving the original MDP. More specifically, the purpose of the abstract MDP is to propagate value quickly over an abstract state space and then transfer that information to a value function estimate in the original MDP. This approach has two main benefits: 1) the abstract MDP can quickly propagate value with a small computational cost, and 2) the learned policy is not limited to the abstract value function's approximation. Overall, this approach increases the agent's ability to learn and adapt to changes in the environment quickly. 

Specifically, we introduce Goal-Space Planning (GSP), a new background planning formalism for the general online RL setting. 
The key novelty is designing the framework to leverage
\emph{subgoal-conditioned models}: temporally-extended models that condition on subgoals. 
These models output predictions of accumulated rewards and discounts for state-subgoal pairs, which can be estimated using standard value-function learning algorithms. 
The models are designed to be simple to learn, as they are only learned for states local to subgoals and they avoid generating entire next state vectors. We use background planning on transitions between 
subgoals, to quickly propagate (suboptimal) value estimates for subgoals. We then leverage these quickly computed subgoal values, without suffering from suboptimality, by incorporating them into any standard value-based algorithm via potential-based shaping. In fact, we layer GSP algorithm onto two different algorithms---Sarsa($\lambda$) and Double Deep Q-Network (DDQN)---and show that improves on both base learners. 
We prove that dynamic programming with our subgoal models is sound (Proposition \ref{prop_vi_gsp}) and highlight that using these subgoal values through potential-based shaping does not change the optimal policy. 
We carefully investigate the components of GSP, particularly showing that 1) it propagates value and learns an optimal policy faster than its base learner, 2) can perform well with somewhat suboptimal subgoal selection, but can harm performance if subgoals are very poorly selected, 3) is quite robust to inaccuracy in its models and 4) that alternative potential-based rewards and alternatives ways to incorporate subgoal values are not as effective as the particular approach used in GSP. We conclude with a discussion on the large literature of related work and a discussion on the benefits of GSP over other background planning approaches.

\section{Problem Formulation}
\label{sec:problem_and_bg_planning}

We consider the standard reinforcement learning setting, where an agent learns to make decisions through interaction with an environment, formulated as a Markov Decision Process (MDP) represented by the tuple $(\mathcal{S}, \mathcal{A}, \mathcal{R}, \mathcal{P})$, where $\mathcal{S}$ is the state space and $\mathcal{A}$ is the action space. The reward function $\mathcal{R}: \mathcal{S}\times \mathcal{A} \times \mathcal{S} \rightarrow \mathbb{R}$ and the transition probability $\mathcal{P}: \mathcal{S}\times \mathcal{A} \times \mathcal{S}\rightarrow [0,1]$ describe the expected reward and probability of transitioning to a state, for a given state and action. On each discrete timestep $t$ the agent selects an action $A_t$ in state $S_t$, the environment transitions to a new state $S_{t+1}$ and emits a scalar reward $R_{t+1}$.

The agent's objective is to find a policy $\pi: \States \times \Actions \rightarrow [0, 1]$ that maximizes expected \emph{return}, the future discounted reward  
$G_t \doteq R_{t+1} + \gamma_{t+1} G_{t+1}$ across all states. The state-based discount $\gamma_{t+1} \in [0,1]$ depends on $S_{t+1}$ \citep{sutton2011horde}, which allows us to specify termination. If $S_{t+1}$ is a terminal state, then $\gamma_{t+1} = 0$; else, $\gamma_{t+1} = \gamma_c$ for some constant $\gamma_c \in [0,1]$.
The policy can be learned using algorithms like Sarsa($\lambda$) \citep{sutton2018reinforcement}, which approximate the action-values: the expected return from a given state and action, $q(s,a) \doteq \mathbb{E}\left [ G_t | S_t =s, A_t=a \right ]$.

Most model-based reinforcement learning systems learn a state-to-state transition model. 
The transition dynamics model can be either an expectation model $\mathbb{E}[S_{t+1} | S_{t}, A_{t}]$ or a probabilistic model $P(S_{t+1} | S_{t}, A_{t})$.
If the state space or feature space is large, then the expected next state or distribution over it can be difficult to estimate, as has been repeatedly shown \citep{talvitie2017selfcorrecting}. Further, these errors can compound when iterating the model forward or backward \citep{jafferjee2020hallucinating,vanhasselt2019when}. It is common to use an expectation model, but unless the environment is deterministic or we are only learning the values rather than action-values, this model can result in invalid states and detrimental updates \citep{wan2019planning}. The goal in this work is to develop a model-based approach that avoids learning the state-to-state transition model, but still obtains the benefits of model-based learning for faster learning and adaptation.

\begin{figure}[t]
   \vspace{-0.05cm}
     \centering
   \includegraphics[width=0.6\textwidth]{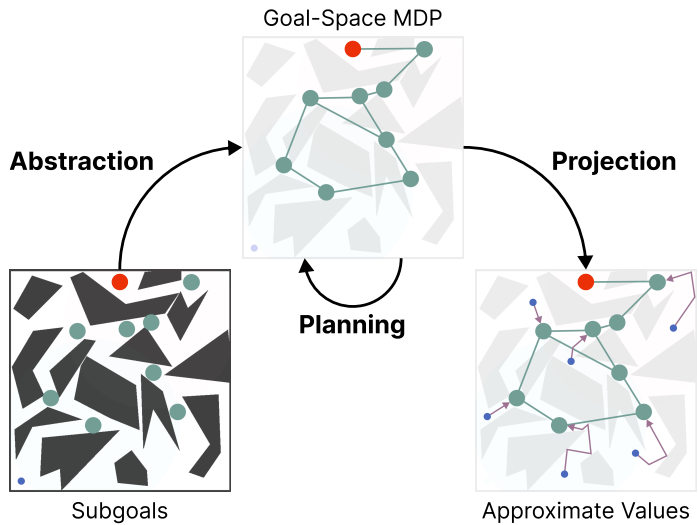}
   \caption{GSP in the PinBall domain. The agent begins with a set of subgoals (denoted in teal) and learns a set of subgoal-conditioned models. \textbf{(Abstraction)} Using these models, the agent forms an abstract MDP where the states are subgoals with options to reach each subgoal as actions. \textbf{(Planning)} The agent plans in this abstract MDP to quickly learn the values of these subgoals. \textbf{(Projection)} Using learned subgoal values, the agent obtains approximate values of states based on nearby subgoals and their values. These quickly updated approximate values are then used to speed up learning. }\label{fig:gsp_overview_diag}  
 \end{figure}
 
\section{Goal-Space Planning at a High-Level}
We consider three desiderata for when a model-based approach should be effective. (1) The model should be feasible to learn: we can get it to a sufficient level of accuracy that makes it beneficial to  plan with that model. 
(2) Planning should be computationally efficient, so that the agent's values can be quickly updated. (3) Finally, the model should be modular---composed of several local models or those that model a small part of the space---so that the model can quickly adapt to small changes in environment. These small changes might still result in large changes in the value; planning can quickly propagate these small changes, potentially changing the value function significantly. 

At a high level, the GSP algorithm focuses planning over a set of given abstract subgoals to provide quickly updated approximate values to speed up learning. To do so, the agent first learns a set of \emph{subgoal-conditioned models}, minimal models focused around planning utility. These models then form a temporally abstract goal-space semi-MDP, with subgoals as states, and options to reach each subgoal as actions. Finally, the agent can update its policy based on these subgoal values to speed up learning. 

Figure \ref{fig:gsp_overview_diag} provides a visual overview of this process. We visualize this is an environment called PinBall, which we also use in our experiments. PinBall is a continuous state domain where the agent must navigate a ball through a set of obstacles to reach the main goal, with a four-dimensional state space consisting of $(x,y, \dot{x}, \dot{y})$ positions and velocities. In this figure, the set of subgoals $\mathcal{G}$ are the teal dots, a finite space of 9 subgoals. The subgoals are abstract states, in that they correspond to many states: a subgoal is any $(x,y)$ location in a small ball, at any velocity. In this example, the subgoals are randomly distributed across the space. Subgoal discovery---identifying this set of subgoals $\mathcal{G}$---is an important part of this algorithm, as we show empirically in Section \ref{sec_subgoalselection}. For this paper, however, we focus on this planning formalism assuming these subgoals are given, already discovered by the agent.  

In the planning step (top central figure), we treat $\mathcal{G}$ as our finite set of states and do value iteration. The precise formula for this update is given later, in \eqref{eq:vi_gsp}, after we formally define the models that we learn for goal-space planning. In words, we compute the subgoal values $\vgoal: \mathcal{G} \rightarrow \mathbb{R}$, using $\vgg(g,g') = $ discounted return when trying to reach $g'$ from $g$ and $\Gammagg(g,g') = $ discounted probability of reaching $g'$ from $g$, 
\begin{equation*}
\vgoal(g) = \max\limits_{\substack{\text{relevant/nearby} \\ \text{subgoals } g'}} \vgg(g,g') + \Gammagg(g,g') \vgoal(g'). 
\end{equation*}

The projection step involves updating values for states, using the subgoal values. The most straightforward way to obtain the value for a state is to find the nearest subgoal $s$ and reason about
$\vsg(s,g) = $ discounted return when trying to reach $g$ from $s$ and $\Gammasg(s,g) = $ discounted probability of reaching $g$ from $s$, 
\begin{equation*}
\vsub(s) = \max\limits_{\substack{\text{relevant/nearby} \\ \text{subgoals } g}} \rgamma(s, g) + \GamModel(s,g) \vgoal(g).
\end{equation*}
Relevance here is defined as $s$ being within the initiation set of the option that reaches that subgoal. We learn an option policy to reach each subgoal, where the initiation set for the option is the set of states from which the option can be executed. The initiation set is a local region around the subgoal, which is why we say we have many local models. Again, we provide the formal definition later in \eqref{eq:vsub}. 

There are several ways we can use this value estimate: inside actor-critic or as a bootstrap target. For example, for a transition $(s,a,r,s')$, we could update action-values $q(s,a)$ using $r + \gamma \vsub(s')$. This naive approach, however, can result in significant bias, as we discuss in Section \ref{sec:gsp_policy}. Instead, we propose an approach to use $\vsub$ based on potential functions. 
These three key steps comprise GSP: abstraction, planning and projection. 

A key part of this algorithm is learning the subgoal models, $\vsg$ and $\Gammasg$. These models will be recognizable to many: they are universal value functions (UVFAs) \citep{schaul2015universal}. We can leverage advances in learning UVFAs to improve our model learning. These models are quite different from standard models in RL, in that most models in RL input a state (or abstract state) and action and output an expected next state (or expected next abstract state). Essentially, the model inputs the source and outputs the expected destination, or a distribution over the possible destinations. Here, the models take as inputs both the source 
and destination, and output only scalars (accumulated reward and discounted probability). The design of GSP is built around using these types of models, that avoids outputting predictions about entire state vectors.

\section{Motivating Experiments for GSP}
This section motivates the capabilities of the GSP framework through a series of demonstrative results. We investigate the utility of GSP in propagating value and speeding up learning.
%
We do so using learners in three domains: FourRooms, PinBall \citep{konidaris2009skill} and GridBall (a version of PinBall without velocities). Unless otherwise stated, all learning curves are averaged over 30 runs, with shaded regions representing one standard error.

\subsection{GSP on Propagating Value} \label{sec:exp_specification}

The central hypothesis of this work is that GSP can accelerate value propagation. By using information from local models in our updates, our belief is that GSP will have a larger change in value to more states, leading to policy changes over larger regions of the state space. 
In this section, we consider the effect of our background planning algorithm on value-based RL methods.
\begin{hypothesis}
GSP changes the value for more states with the same set of experience.
\label{hyp:val_prop_state}
\end{hypothesis}

\begin{wrapfigure}[19]{l}{0.3\textwidth}
  \vspace{-0.4cm}
  \begin{centering}
  \includegraphics[width=0.3\textwidth]{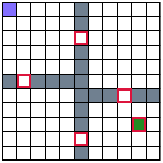}
  \caption{The FourRooms domain. The blue square is the initial state, green square the goal state, and red boxes the subgoals. A subgoal's initiation set contains the states in any room connected to that subgoal.}\label{fourrooms}  
  \end{centering}
\end{wrapfigure}
In order to verify whether GSP helps to quickly propagate value, we first test this hypothesis in a simple grid world environment: the FourRooms domain. The agent can choose from one of 4 actions in a discrete action space $\mathcal{A} = \{\texttt{up}, \texttt{down}, \texttt{left}, \texttt{right}\}$. All state transitions are deterministic. The grey squares in Figure \ref{fourrooms} indicate walls, and the state remains unchanged if the agent takes an action that leads into a wall. This is an episodic task, where the base learner has a fixed start state and must navigate to a fixed goal state where the episode terminates. Episodes can also terminate by timeout after 1000 timesteps.

In this domain, we test the effect of using GSP with pre-trained models on a Sarsa($\lambda$) base learner in the tabular setting (i.e. no function approximation for the value function). Full details on using GSP with this temporal difference (TD) learner can be found in Algorithm \ref{alg:MainPolicySarsaLambdaUpdate}. We set the four hallway states plus the goal state as subgoals, with their initiation sets being the two rooms they connect. Full details of option policy learning can be found in the appendix \ref{app_model_learning}.

\begin{figure}[htbp]
\vspace{0.44cm}
  \centering
   \includegraphics[width=\textwidth]{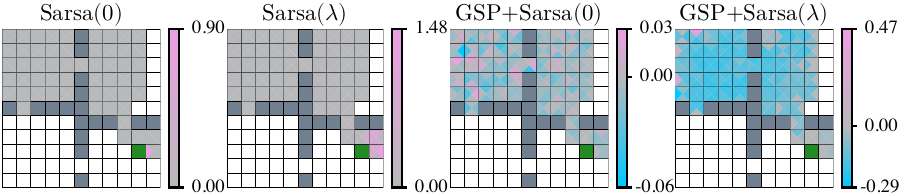}
  \caption{These four plots show the action values after a single episode of updates for Sarsa with and without GSP and eligibility traces, i.e., $\lambda = 0.9$. Each algorithm's update is simulated from the same data collected from a uniform random policy. Each state (square) is made up of four triangles representing each of the four available actions. White squares represent states not visited in the episode.}
  \label{fig:hm_fourrooms}
\end{figure}

Figure \ref{fig:hm_fourrooms} shows the base learner's action-value function after a single episode using four different algorithms: Sarsa($0$), Sarsa($\lambda$), Sarsa($0$)+GSP, and Sarsa($\lambda$)+GSP. In Figure \ref{fig:hm_fourrooms}, the Sarsa($0$) learner updates the value of the state-action pair that immediately preceded the +1 reward at the goal state. The plot for Sarsa($\lambda$) shows a decaying trail of updates made at the end of the episode, to assign credit to the state-action pairs that led to the +1 reward. The plots for the GSP variants show that all state-action pairs sampled receive instant feedback on the quality of their actions. The updates with GSP can be both positive or negative based on if the agent makes progress towards the goal state or not. This direction of update comes from the potential-based reward shaping rewards/penalizes transitions based on whether $\gamma_{t+1} \vsub(S_{t+1}) > \vsub(S_{t})$. It is clear that projecting subgoal values from the abstract MDP leads to action-value updates over more of the visited states, even without memory mechanisms such as eligibility traces. 

It is evident from these updates over a single episode that the resulting policy from GSP updates should be more likely to go to the goal. We would like to quantify how much faster this propagated value can help our base learner over multiple episodes of experience. More specifically, we want to test the following hypothesis.
\begin{hypothesis}
GSP enables a TD base-learner to learn faster. 
\label{hyp:val_prop_time}
\end{hypothesis}

We expect GSP to improve a base learner’s performance on a task within fewer environment interactions. We shall test whether the value propagation over the state-action space as seen in Figure \ref{fig:hm_fourrooms} makes this the case over the course of several episodes (i.e. we are now testing the effect of value propagation over time). Figure \ref{fig:fourrooms_curve} shows the performance of a Sarsa($\lambda$) base learner with and without GSP in the FourRooms domain with a reward of -1 per step. 
%
%
Full details on the hyperparameters used can be found in Appendix \ref{app_model_learning}. It is evident that the GSP-augmented Sarsa($\lambda$) learner is able to reach the optimal policy much faster. The GSP learner also starts at a much lower steps-to-goal.  
We \emph{believe} this first episode performance improvement is because the feedback from GSP teaches the agent which actions move towards or away from the goal during the first episode. 

\begin{figure}[htbp]
  \begin{centering}
    \scalebox{0.65}{\input{fourrooms_tab_stepstogoal_noLAVI.pgf}}
  \caption{This plot shows the average number of steps to goal smoothed over five episodes in the FourRooms domain. Shaded region represents 1 standard error across 100 runs.}
  \label{fig:fourrooms_curve}  
  \end{centering}
\end{figure}
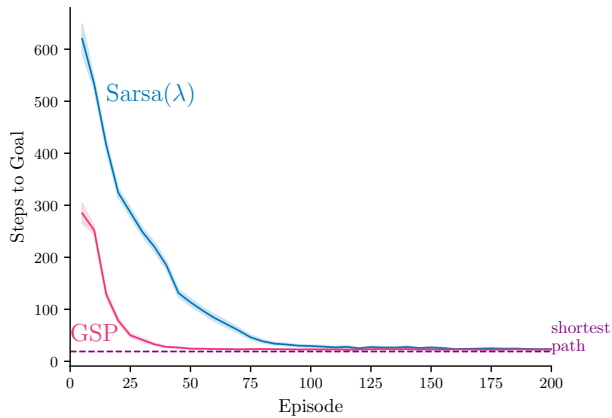

\subsection{GSP in Larger State Spaces}\label{sec:larger_state}
Many real world applications of RL involve large and/or continuous state spaces. Current planning techniques struggle with such state spaces. This motivates an investigation into how well Hypotheses \ref{hyp:val_prop_state} and \ref{hyp:val_prop_time} hold up when GSP is used in such environments (e.g. the PinBall domain). To better analyse GSP and its value propagation across state-space, we also created an intermediate environment between FourRooms and PinBall called GridBall. 

PinBall is a continuous state domain where the agent navigates a ball through a set of obstacles to reach the main goal. This domain uses a four-dimensional state representation of positions and velocities, $(x, y, \dot{x}, \dot{y}) \in [0,1] \times [0,1] \times [-2, 2] \times [-2, 2]$. The agent chooses from one of five actions at each timestep. $\mathcal{A} = \{\texttt{up}, \texttt{down}, \texttt{left}, \texttt{right}, \texttt{nothing}\}$, where the $\texttt{nothing}$ action adds no change to the ball's velocity, and the other actions each add an impulse force in one of the four cardinal directions. In all our experiments, the agent is initialised with zero velocity at a fixed start position at the beginning of every episode. All collisions are elastic and we use a drag coefficient of $0.995$. This is an episodic task with a fixed starting state and main goal. An episode ends when the agent reaches the main goal or after 1,000 time steps. It should be noted that, unlike in the FourRooms environment, there exists states which are not in the initiation set of any subgoal - a common occurence when deploying GSP in the state spaces of real-world applications.

\begin{wrapfigure}{r}{0.3\textwidth}
  \vspace{-0.44cm}
  \centering
  \includegraphics[width=0.27\textwidth]{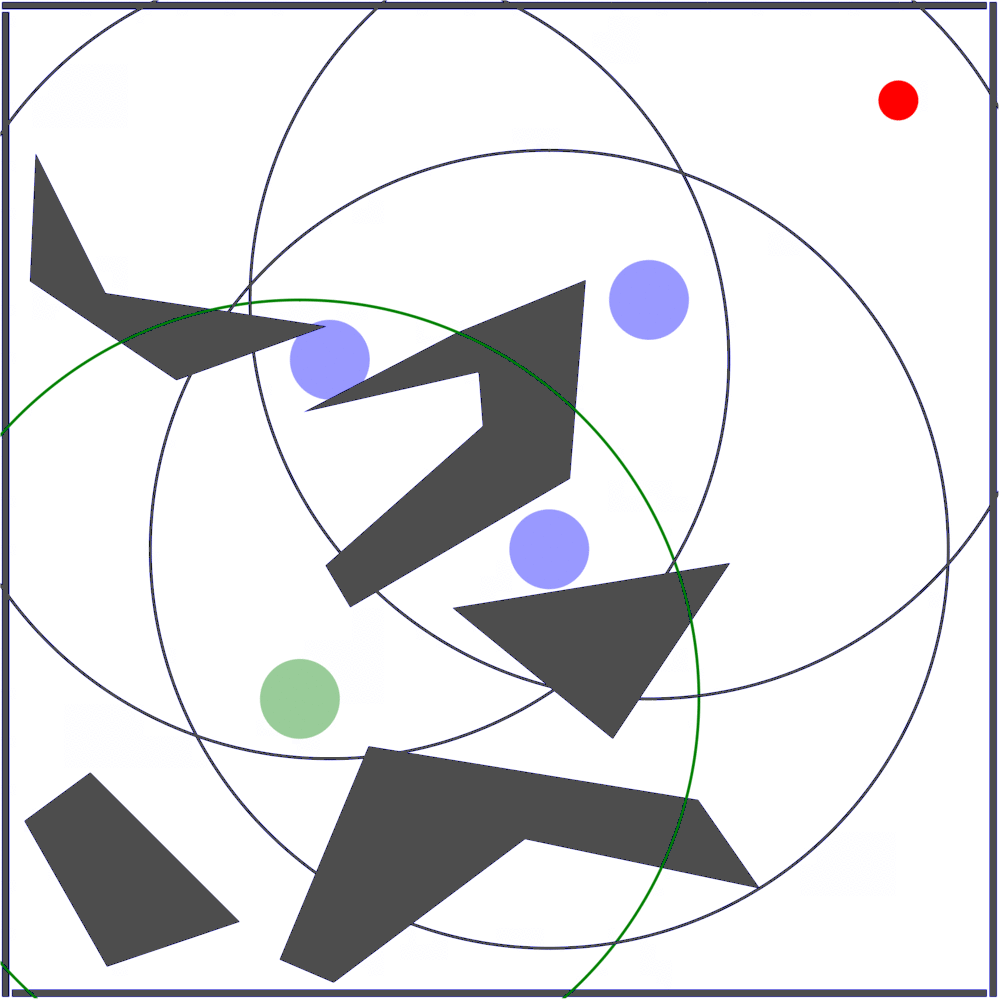}
  \caption{Obstacles and subgoals for GridBall and PinBall. The larger circles show the initiation set boundaries. Subgoals are defined in position space.}\label{fig:obstacle_layout}
  \vspace{-0.44cm}
\end{wrapfigure}

GridBall is like PinBall, but change to be more like a gridworld to facilitate visualization. The velocity components of the state are removed, meaning the state only consists of $(x,y)$ locations, and the action space is changed to displace the ball by a fixed amount in each cardinal dimension. We keep the same obstacle collision mechanics and calculations from PinBall. Since GridBall does not have any velocity components, we can plot heatmaps of value propagation without having to consider the velocity at which the agent arrived at a given position. 

For Hypothesis \ref{hyp:val_prop_state}, we repeat the experiments on GridBall with base learners that use tile-coded value features \citep{sutton2018reinforcement}, and linear value function approximation. Full details on the option policies and subgoal models used for this are outlined in shown in Appendices \ref{learn_opt_pol} and \ref{sec:model-learning}. Like in the FourRooms experiment, we set the reward to be 0 at all states and +1 once the agent reaches any state in the main goal to show value propagation. We collect a single episode of experience from the Sarsa($0$)+GSP learner and use its trajectory to perform a batch update on all learners. This controls for any variability in trajectories between learners, so we can isolate and study the change in value propagation. 

Figure \ref{hm_gridball} compares the state value function (averaged over the action value estimates) of Sarsa($0$), Sarsa($\lambda$), Sarsa($0$)+GSP and Sarsa($\lambda$)+GSP learners after a single episode of interaction with the environment. The results are similar to those on FourRooms. The Sarsa($0$) algorithm only updates the value of the tiles activated by the state preceding the goal. Sarsa($\lambda$) has a decaying trail of updates to the tiles activated preceeding the goal, and the GSP learners updates values at all states in the initiation set of a subgoal. 

\begin{figure}[tbp] 
\vspace{-0.44cm}
  \centering
  \begin{subfigure}{0.22\textwidth}
    \centering
    \caption{Sarsa(0)}
    \includegraphics[width=\textwidth]{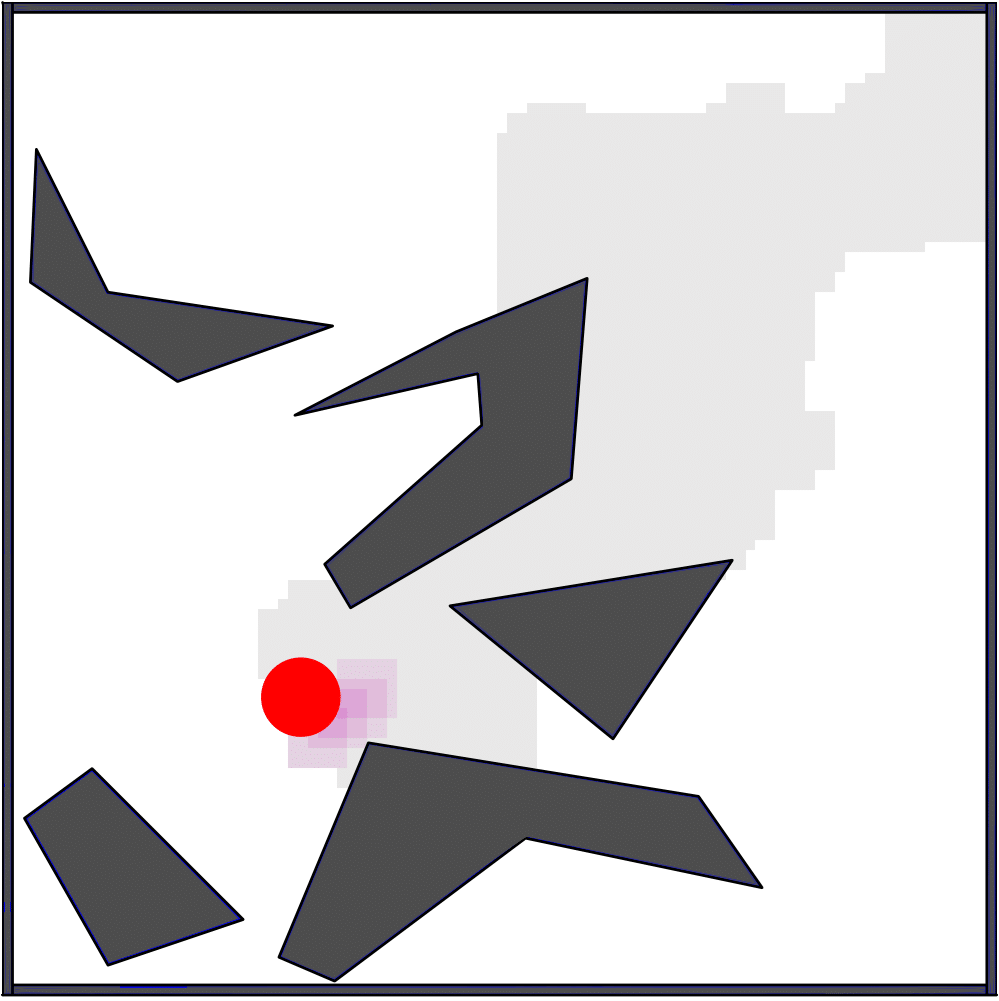}
    \label{sub_fig:hm_gridball_sarsa}
  \end{subfigure}
  \begin{subfigure}{0.22\textwidth}
    \centering
    \caption{Sarsa($\lambda$)}
    \includegraphics[width=\textwidth]{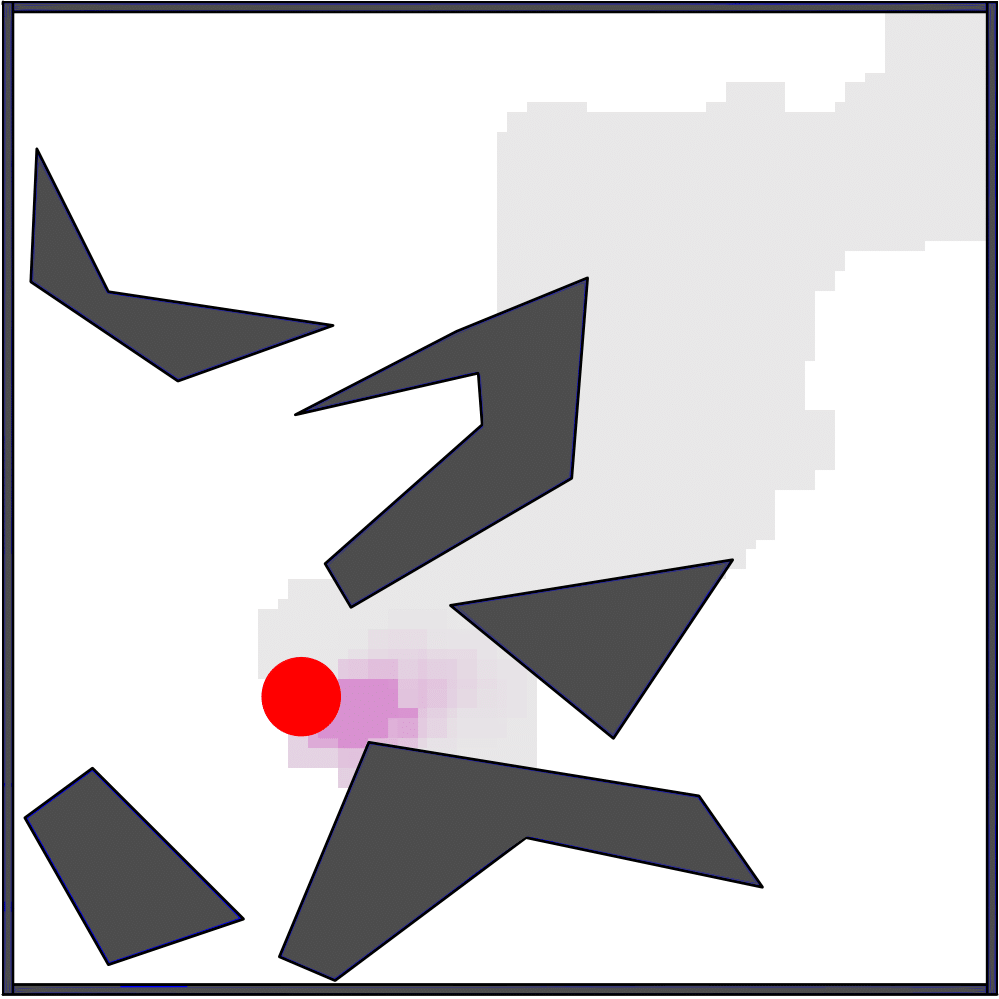}
    \label{sub_fig:hm_gridball_gsp}
  \end{subfigure}
  \begin{subfigure}{0.22\textwidth}
    \centering
    \caption{GSP+Sarsa(0)}
    \includegraphics[width=\textwidth]{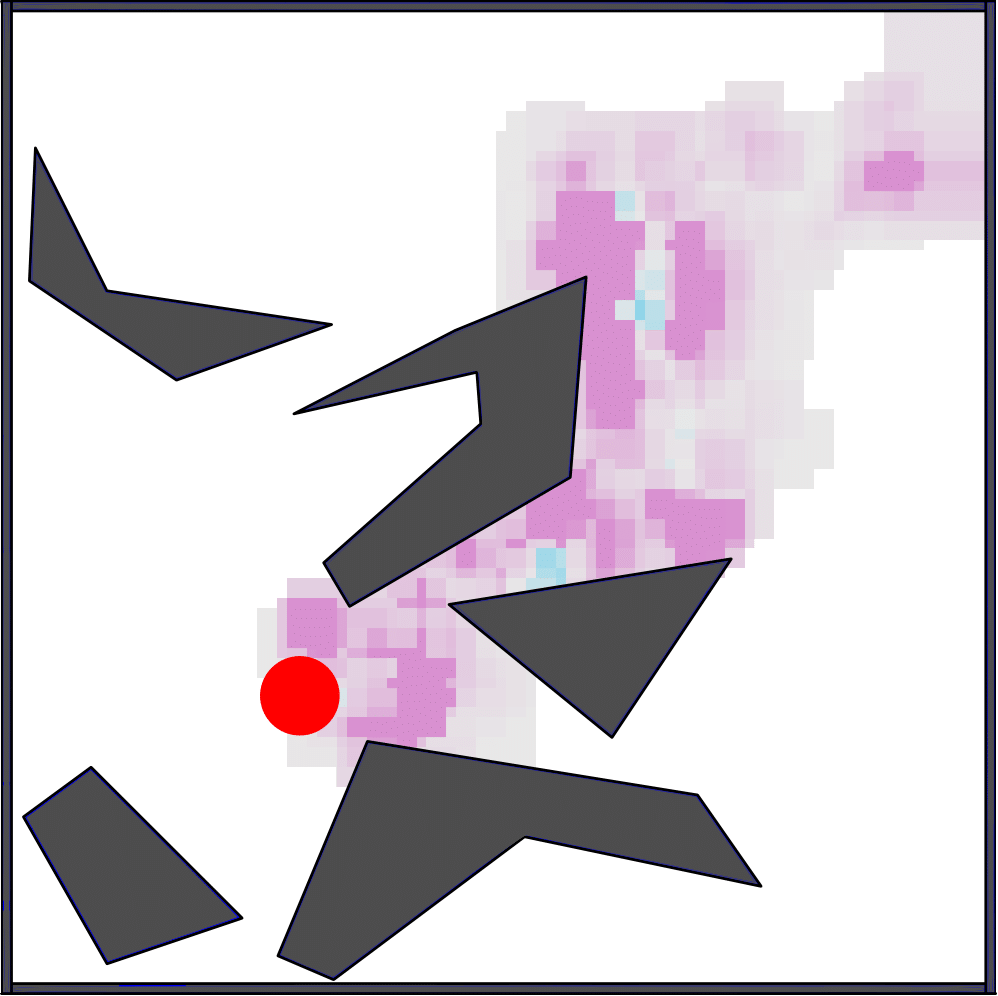}
    \label{sub_fig:hm_gridball_nogsp}
  \end{subfigure}
  \begin{subfigure}{0.22\textwidth}
    \centering
    \caption{GSP+Sarsa($\lambda$)}
    \includegraphics[width=\textwidth]{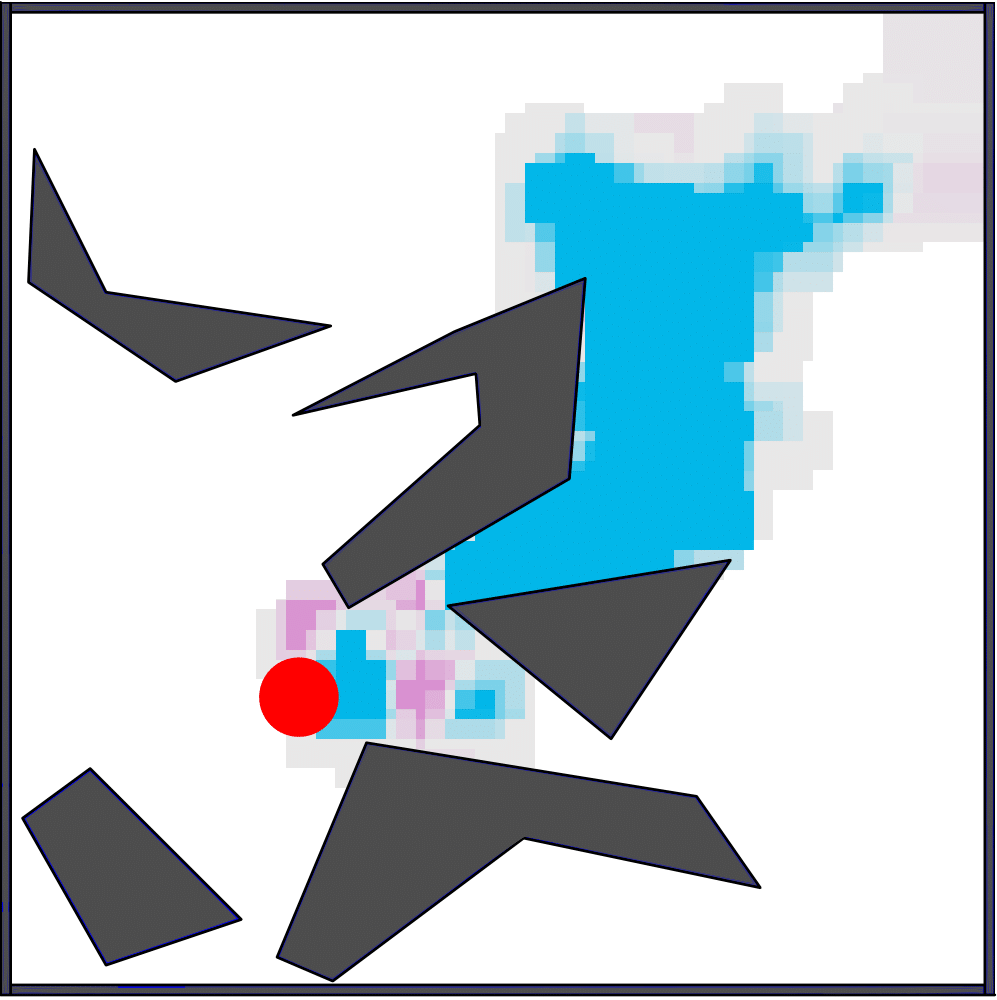}
    \label{sub_fig:hm_gridball_gsp_lambda}
  \end{subfigure}
  \begin{subfigure}{0.06\textwidth}
    \includegraphics[width=\textwidth]{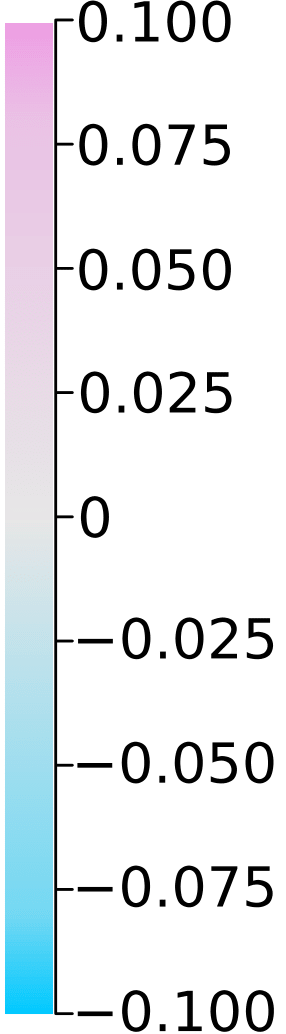}
    \vspace{0.0001cm}
  \end{subfigure}
  \vspace{-0.66cm}
  \caption{The tile-coded value function after one episode in GridBall. Like Figure \ref{fig:hm_fourrooms}, the gray regions show the visited states which were not updated. The red circle is the main goal.}
  \vspace{-0.88cm}
  \label{hm_gridball}
\end{figure}

To examine how GSP translates to faster learning (Hypothesis \ref{hyp:val_prop_time}), we measure the performance (steps to goal) over time for each algorithm in both GridBall and PinBall domains. Figure \ref{fig:gridball_pinball_curve} shows that GSP significantly improves the rate of learning in these larger domains too, with the base learner able to reach near its top performance within 75 and 100 epsiodes in GridBall and PinBall respectively.  All runs are able to find a similar length path to the goal. As the size of the state space increases, the benefit of using local models in the GSP updates still holds. 

Similar to the previous domains, the Sarsa($\lambda$) learner using GSP is able to reach a good policy much faster than the base learner without GSP. In both domains, the GSP and non-GSP Sarsa($\lambda$) learners plateau at the same average steps to goal. Even though the obstacles remain unchanged from GridBall, it takes roughly 50 episodes longer for even the GSP variant to reach a good policy in PinBall. This is likely due to the continuous 4-dimensional state space making the task harder.

\begin{figure}[tbp]
\begin{subfigure}{0.5\textwidth}
    \centering
    \caption{GridBall}
    \scalebox{0.6}{\input{gridball_stepstogoal_noLAVI.pgf}}
\end{subfigure}
\begin{subfigure}{0.5\textwidth}
    \centering
    \caption{PinBall}
    \scalebox{0.6}{\input{pinball_stepstogoal_noLAVI.pgf}}
\end{subfigure}
\caption{ Five episode moving average of return in the GridBall over 200 episodes (left) and PinBall over 500 episodes (right). We performed 30 runs, and showed 1 standard error in the shaded region. All learners used linear value function approximation on their tile coded features.}
  \label{fig:gridball_pinball_curve} 
\end{figure}
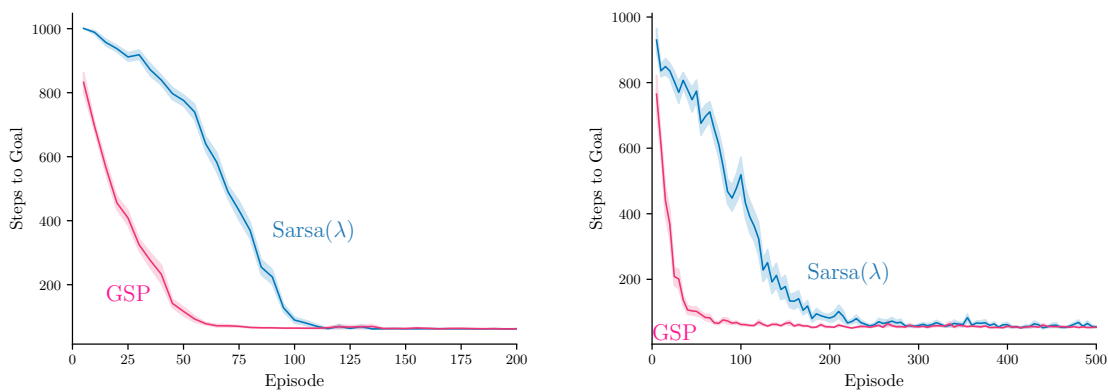

\section{Goal-Space Planning in Detail}

In this section we outline the technical definitions and details for the goal-space planning approach. We first discuss the definition of subgoals and then the corresponding subgoal-conditioned models. We then discuss how to use these models for planning, particularly how to do value iteration to compute the subgoal values and then how to use those values to influence values of states (the projection step). We conlude the section summarizing the overall goal-spacing planning framework, including how we can layer GSP into a standard algorithm called Double DQN.

\subsection{Defining Subgoals}
   
Similar to options \citep{sutton1999options}, we define a subgoal as having two components: a set of goal states, and a set of initiation states. These two sets are defined by the indicator functions $m$ and $d$ with $m$ specifying the states in the goal set, and $d$ specifying states in an initiation set. We say that a state $s$ is a member of subgoal $g$ if $\indsg(s,g)=1$ and we only reason about reaching a subgoal $g$ from states $s$ such that $\relsg(s,g)=1$. 

We consider a finite set of subgoals $\Goals$.
While subgoals could represent a single state, they can also describe more complex conditions that are common to a group of states.
For example, $g$ could correspond to a situation where both the front and side distance sensors of a robot report low readings---what a person would call being in a corner. 
If the first two elements of the state vector $s$ consist of the front and side distance sensor, $\indsg(s,g) = 1$ for any states where $s_1, s_2$ are less than some threshold $\epsilon$. 
As another example, in Figure \ref{fig:gsp_overview_diag}, we simply encode the nine subgoals---which correspond to regions with a small radius.
For a concrete example, we visualize subgoals for our experiments in Figures \ref{fourrooms} and \ref{fig:obstacle_layout}.
%
Essentially, our subgoals define a new state space in an abstract MDP, and these new abstract states (subgoals) can be represented in different ways, just like in regular MDPs. 
Finally, we only reason about reaching subgoals from a subset of states, called \emph{initiation sets}. This constraint is key for locality, to learn and reason about a subset of states for a subgoal. 
We assume the existence of a (learned) \emph{initiation function} $\relsg$ to indicate when the agent is sufficiently close in terms of reachability. 
We discuss some approaches to learn this initiation function in Appendix \ref{sec:model-learning}. But, here, we assume it is part of the discovery procedure for the subgoals and instead focus on how to use it. 

For the rest of this paper, we presume we are given subgoals and initiation sets. We develop algorithms to learn and use models, given those subgoals. We expect a complete agent to discover these subgoals on its own, including how to represent these subgoals to facilitate generalization and planning. In this work, though, we first focus on how the agent can leverage reasonably well-specified subgoals. 


\subsection{Defining Subgoal-Conditioned Models}

For planning and acting to operate in two different spaces, we define four models: two used in planning over subgoals (subgoal-to-subgoal) and two used to project these subgoal values back into the underlying state space (state-to-subgoal). Figure \ref{model_diag} visualizes these two spaces. 
 
The state-to-subgoal models are $\rgamma: \States \times \augGoals \rightarrow \RR$ and $\GamModel: \States \times \augGoals \rightarrow [0,1]$, where $\augGoals = \Goals \cup \{ \sterm\}$ if there is a terminal state (episodic problems) and otherwise $\augGoals = \Goals$.  
An option policy $\pi_g: \States \times \Actions \rightarrow [0,1]$ for subgoal $g$ starts from any $s$ in the initiation set, and terminates in $g$---in $\tilde{s}$ where $\indsg(\tilde{s},g) = 1$.
The reward-model $\rgamma(s,g)$ is the discounted rewards under option policy $\pi_g$:
\begin{equation*}
\rgamma(s,g) = \mathbb{E}_{\pi_g}[R_{t+1} + \gamma_g(S_{t+1}) \rgamma(S_{t+1}, g) | S_{t} = s],
\end{equation*}
where the discount is zero upon reaching subgoal $g$
{\small
\begin{equation*}
\gamma_g(S_{t+1})  \defeq \begin{cases}
0 & \text{if $\indsg(S_{t+1},g) = 1$, namely if subgoal $g$ is achieved by being in $S_{t+1}$}, \\
\gamma_{t+1} & \text{else.}
\end{cases} 
\end{equation*}
}
The discount-model $\GamModel(s,g)$ reflects the discounted number of steps until reaching subgoal $g$ starting from $s$, in expectation under option policy $\pi_g$
\begin{align*}
\GamModel(s,g) = \mathbb{E}_{\pi_g}[\indsg(S_{t+1}, g)\gamma_{t+1} + \gamma_{g}(S_{t+1}) \GamModel(S_{t+1}, g) | S_{t} = s].
\end{align*}
These state-to-subgoal will only be queried for $(s,g)$ where $\relsg(s,g) > 0$: they are local models. 

To define subgoal-to-subgoal models,\footnote{The first input is any $g \in \Goals$, the second is $g' \in \augGoals$, which includes $s_{\text{terminal}}$. We need to reason about reaching any subgoal or $s_{\text{terminal}}$. But $s_{\text{terminal}}$ is not a real state: we do not reason about starting from it to reach subgoals.} $\vgg: \Goals \times \augGoals \rightarrow \RR$ and $\Gammagg: \Goals \times \augGoals \rightarrow [0,1]$, we use the state-to-subgoal models. 
For each subgoal $g \in \Goals$, we aggregate $\rgamma(s,g')$ for all $s$ where $\indsg(s, g) = 1$.
\begin{equation}\label{sg_models}
{
\vgg(g,g') \defeq \frac{1}{z(g)}\sum\limits_{s: \indsg(s,g) = 1} \vsg(s,g') \quad \text{and} \quad
\Gammagg(g,g') \defeq \frac{1}{z(g)}\sum\limits_{s: \indsg(s,g) = 1}  \Gammasg(s,g')}
\end{equation}
for normalizer $z(g) \defeq \sum_{s: \indsg(s,g) = 1} \indsg(s,g)$. This definition assumes a uniform weighting over the states $s$ where $\indsg(s,g) = 1$. We could allow a non-uniform weighting, potentially based on visitation frequency in the environment. For this work, however, we assume that $\indsg(s,g) = 1$ for a smaller number of states $s$ with relatively similar $\vsg(s,g')$, making a uniform weighting reasonable.

These models are also local models, as we can similarly extract $\relgg(g,g')$ from $\relsg(s,g')$ and only reason about $g'$ nearby or relevant to $g$. We set $\relgg(g,g') = \max_{s \in \States: \indsg(s, g) > 0} \relsg(s,g')$, indicating that if there is a state $s$ that is in the initiation set for $g'$ and has membership in $g$, then $g'$ is also relevant to $g$. 

 \begin{wrapfigure}[9]{r}{0.5\textwidth}
 \vspace{-0.6cm}
	\begin{centering}
	\includegraphics[width=0.5\textwidth]{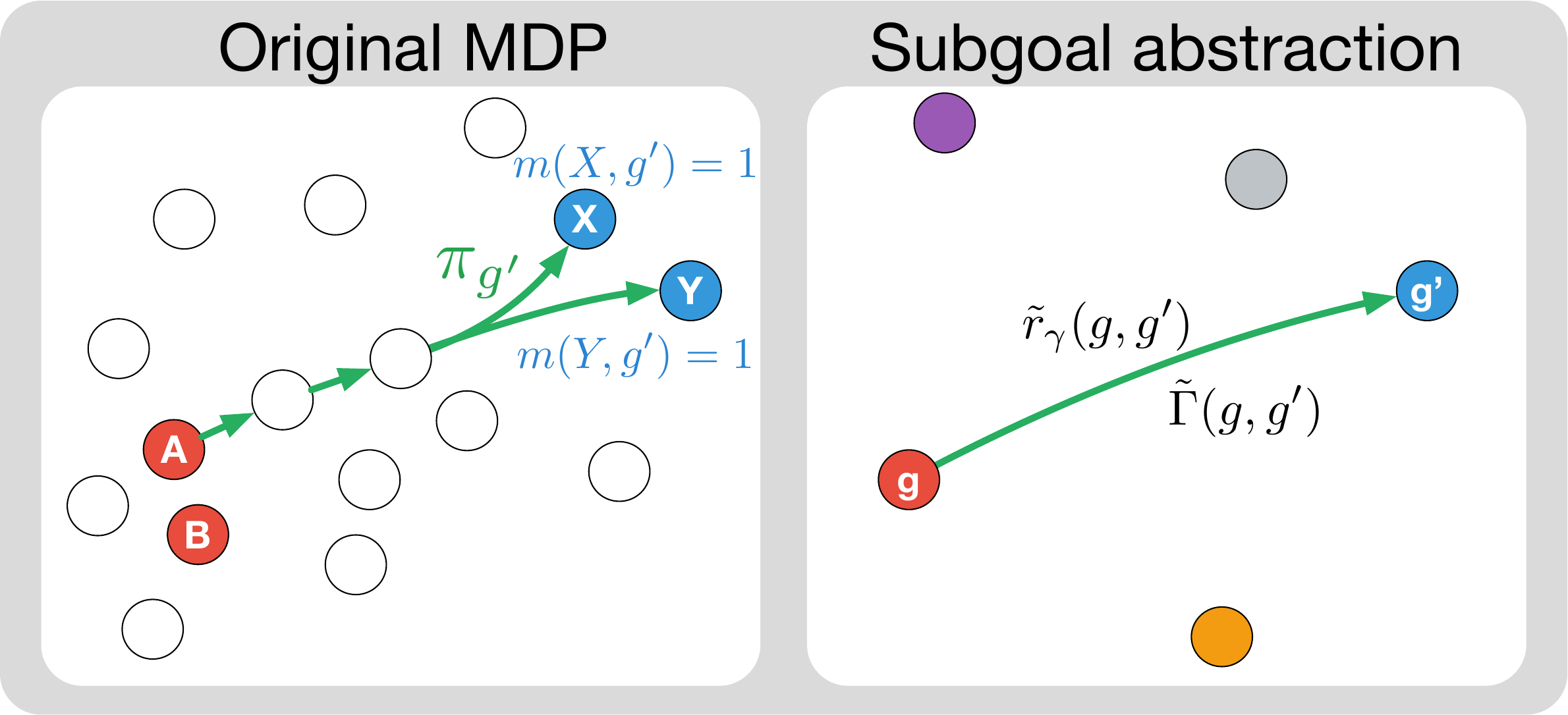}
  \vspace{-0.5cm}
	\caption{Original and abstract state spaces.}\label{model_diag}  
	\end{centering}
\end{wrapfigure}

Let us consider an example, in Figure \ref{model_diag}.
The red states are members of $g$ (namely $\indsg(A,g)=1$) and the blue members of $g'$ (namely $\indsg(X,g')=1$ and $\indsg(Y,g')=1$). For all $s$ in the diagram, $\relsg(s,g') > 0$ (all are in the initiation set): the policy $\pi_{g'}$ can be queried from any $s$ to get to $g'$. The green path in the left indicates the trajectory under $\pi_{g'}$ from $A$, stochastically reaching either $X$ or $Y$, with accumulated reward $\rgamma(A, g')$ and discount $\Gammasg(A,g')$ (averaged over reaching $X$ and $Y$). The subgoal-to-subgoal models, on the right, indicate $g'$ can be reached from $g$, with $\vgg(g,g')$ averaged over both $\rgamma(A, g')$ and $\rgamma(B, g')$ and  $\Gammagg(g, g')$ over $\Gammasg(A, g')$ and $\Gammasg(B, g')$, as described in \eqref{sg_models}.

\subsection{Goal-Space Planning with Subgoal-Conditioned Models}
We can now consider how to plan with these models. 
Planning involves learning $\vgoal(g)$: the value for different subgoals. This can be achieved using an update similar to value iteration, for all $g \in \Goals$
\begin{equation} 
\vgoal(g) = \max\limits_{g' \in \augGoals: \relgg(g,g') > 0} \vgg(g, g') + \Gammagg(g,g') \vgoal(g'). \label{eq:vi_gsp}
\end{equation}
The value of reaching $g'$ from $g$ is the discounted rewards along the way, $\vgg(g, g')$, plus the discounted value in $g'$. If  $\Gammagg(g,g')$ is very small, it is difficult to reach $g'$ from $g$---or takes many steps---and so the value in $g'$ is discounted by more. With a relatively small number of subgoals, we can sweep through them all to quickly compute $\vgoal(g)$. With a larger set of subgoals, we can instead do as many updates possible, in the background on each step, by stochastically sampling $g$. 

We can interpret this update as a standard value iteration update in a new MDP, where 1) the set of states is $\Goals$, 2) the actions from $g \in \Goals$ are state-dependent, corresponding to choosing which $g' \in \augGoals$ to go to in the set where $\relgg(g,g') > 0$ and 3) the rewards are $\vgg$ and the discounted transition probabilities are $\Gammagg$. 
The transition probabilities are typically separate from the discount, but it is equivalent to consider the discounted transition probabilities. 
Under this correspondence, it is straightforward to show that the above converges to the optimal values in this new Goal-Space MDP, shown in \textbf{Proposition \ref{prop_vi_gsp}} in Appendix \ref{app_control}.  

This goal-space planning approach does not suffer from typical issues with model-based RL. Firstly, even though the models are not directly iterated over multiple timesteps,  we still obtain temporal abstractions. This is because the subgoal models are temporally abstract by construction, and do not need to be called iteratively over many timesteps for long horizon planning - which could lead to compounding error. Second, we do not need to predict entire state vectors---or distributions over them---because we instead input the outcome $g'$ into the function approximator. This may feel like a false success as it potentially requires restricting ourselves to a smaller number of subgoals. If we want to use a larger number of subgoals, then we may need a function to generate these subgoal vectors anyway---bringing us back to the problem of generating vectors. However, this is likely easier as 1) the subgoals themselves can be much smaller and more abstract, making it more feasibly to procedurally generate them and 2) it may be more feasible to maintain a large set of subgoal vectors, or generate individual subgoal vectors, than producing relevant subgoal vectors from a given subgoal.

\subsection{Using Subgoal Values to Update the Policy}\label{sec:gsp_policy}
Now let us examine how to use $\vgoal(g)$ to update our main policy. 
The simplest way to decide how to behave from a state is to cycle through the subgoals, and pick the one with the highest value. 
In other words, we can set
\begin{equation} \label{eq:vsub}
\vsub(s) \defeq \begin{cases}
    \max\limits_{g \in \augGoals: \relsg(s,g) > 0} \rgamma(s, g) + \GamModel(s,g) \vgoal(g) & \text{if} \, \, \exists \, g \in \augGoals : \relsg(s,g) >0 , \ \text{(projection step)}\\
    \hfil \text{undefined} & \text{otherwise,}
\end{cases} 
\end{equation}
and take action $a$ that corresponds to the action given by $\pi_g$ for this maximizing $g$ as shown in Figure  \ref{fig:policy_diag}. Note that some states may not have any nearby subgoals, and $\vsub(s)$ is undefined for that state. This is not the only problem with this naive approach. 

 \begin{wrapfigure}{l}{0.3\textwidth}
	\centering
	\includegraphics[width=0.3\textwidth]{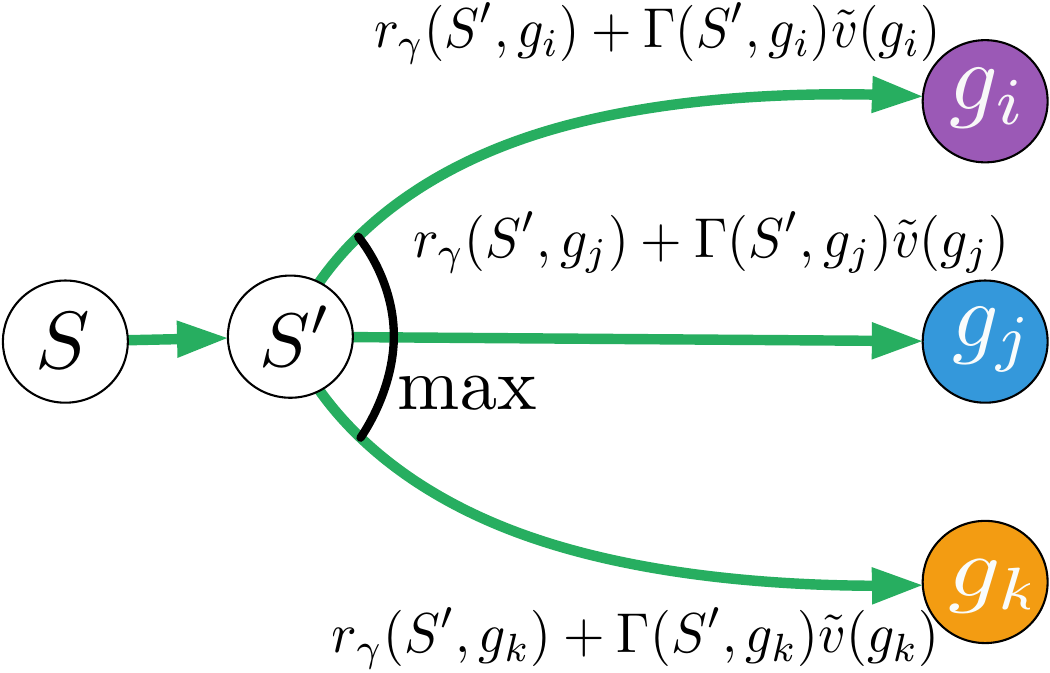}
	\caption{Computing $\vsub(S')$. This is used with $\vsub(S)$ to shape the reward signal for a transition $S\to S'$.}\label{fig:policy_diag}
  \vspace{-0.44cm}
\end{wrapfigure}
There are two other critical issues with this approach. Policies are restricted to go through subgoals, which might result in suboptimal policies. From a given state $s$, the set of relevant subgoals $g$ may not be on the optimal path. This limitation is expressly one we early mentioned we wished to avoid, and is a key limitation of Landmark Value Iteration (LAVI) developed for the setting where models are given \citep{mann2015approximate}. Second, the learned models themselves may have inaccuracies, or planning may not have been completed in the background, resulting in $\vgoal(g)$ that are not yet fully accurate. 

When incorporating $\vsub$ into the learning process, we want to ensure the optimal policy remains unchanged and that $\vsub$ guides the agent by helping evaluate the quality of its decisions. 
A simple way to satisfy these requirements is to use potential-based reward shaping \citep{ng1999reward}. Potential-based reward shaping defines a new MDP with a modified reward function where the agent receives the reward $\tilde R_{t+1} = R_{t+1} + \gamma \Phi(S_{t+1}) - \Phi(S_t)$, where $\Phi \colon \mathcal{S} \to \mathbb{R}$ is any state-dependent function. \citeauthor{ng1999reward} show that such a reward transformation preserves the optimal policy from the original MDP. 
We propose using $\Phi = \vsub$ to modify any temporal TD learning algorithm to be compatible with GSP. For example, in the Sarsa($\lambda$) algorithm, the update for the weights $\qparams$ of the function $q \colon \mathcal{S} \times \mathcal{A} \times \mathbb{R}^n \to \mathbb{R}$ would use the TD-error 
\begin{equation} \label{eq:oci_delta}
\delta \defeq \underbrace{R_{t+1} + \gamma_{t+1} \vsub(S_{t+1}) - \vsub(S_t)}_{\tilde R_{t+1}} \, + \, \gamma_{t+1} q(S_{t+1},A_{t+1}; \qparams) - q(S_t,A_t; \qparams).
\end{equation}

Potential-based shaping rewards the agent for selecting actions that result in transition that increase $\Phi$. 
Consider the case when $\Phi$ represents the \emph{negative} distance to a goal state. When $\Phi(S_{t+1}) > \Phi(S_t)$, then the agent has made progress towards getting to the goal, and it receives a positive addition to the reward. 
When $\Phi$ is an estimate of the value function, one can interpret the additive reward as rewarding the agent for taking actions that increase the value function estimate and penalizes actions that decrease the value.
In this way, using $\Phi=\vsub$, the agent can leverage immediate feedback on the quality of its actions using the information from the abstract value function about what an optimal policy might look like.

For intuition as to why potential-based reward shaping does not bias the optimal policy, notice that $\sum_{t=0}^\infty \gamma^{t} \left (\gamma \Phi(S_{t+1}) - \Phi(S_t) \right ) = -\Phi(S_0)$, which means the relative values of each action remain the same. 
The cancellations of these intermediate terms mean that algorithms like REINFORCE \citep{williams1992simple} or Proximal Policy Optimization \citep{schulman2017proximal} will see little benefit when combined with potential-based reward shaping (as they use the discount sum of all rewards to update the policy). 
For these algorithms, one could instead estimate a $q_{g^\star}$ and leverage trajectory-wise control variates \citep{pankov2018cv,cheng2019trajcv,huang2020drcv}.
We leave investigation of this approach to future work and focus on TD learning algorithms in this paper.

It is important to note that if $\vsub$ can help improve learning, it can also make learning harder if its guidance makes it less likely for an agent to sample optimal actions. This increase in difficulty is likely if the models used to construct the abstract MDP and $\vsub$ have substantial errors. In this case, the agent has to learn to overcome the bad ``advice" provided by $\vsub$. We investigate this further with non-stationary environments and inaccurate models in Sections \ref{sec:region_of_attraction} and \ref{sec:mod_acc} respectively.

\subsection{Putting it All Together: The Full Goal-Space Planning Algorithm}

The remaining piece is to learn the models and put it all together. Learning the models is straightforward, as we can leverage the large literature on general value functions \citep{sutton2011horde} and UVFAs \citep{schaul2015universal}. There are nuances involved in 1) restricting updating to relevant states according to $\relsg(s,g)$, 2) learning option policies that reach subgoals, but also maximize rewards along the way and 3) considering ways to jointly learn $d$ and $\GamModel$. We include these details in Appendix \ref{sec:model-learning} as well as detailed pseudocode in Appendix. \ref{app:pseudocode}. Here, we summarize the higher-level steps.

\begin{wrapfigure}{r}{0.47\textwidth}
 \vspace{-0.44cm}
	\centering
	\includegraphics[width=0.47\textwidth]{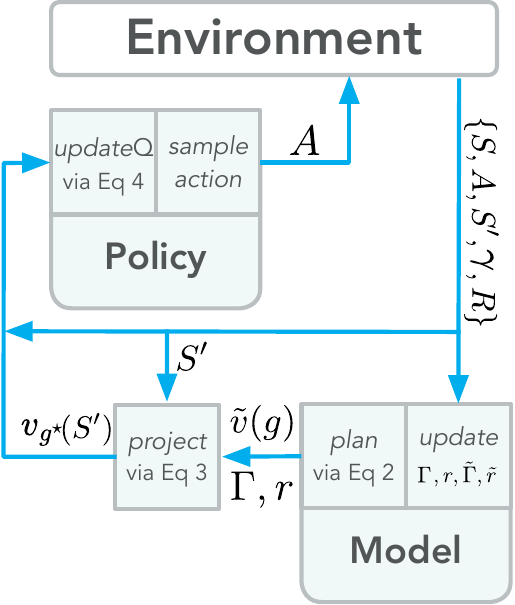}
	\caption{Goal-Space Planning.}\label{gsp_diagram}  
    \vspace{-2cm}
\end{wrapfigure}
The algorithm is visualized in Figure \ref{gsp_diagram}. The steps of agent-environment interaction include:
\begin{enumerate}
    \item take action $A_t$ in state $S_t$, to get $S_{t+1}, R_{t+1}$ and $\gamma_{t+1}$
    
    \item query the model for $\rgamma(S_{t+1},g)$, $\Gammasg(S_{t+1},g)$, $\vgoal(g)$ for all $g$ where $\relsg(S_{t+1},g) > 0$

    \item compute projection $\vsub(S_{t+1})$, using \eqref{eq:vsub}
    
    \item update the main policy with the transition and $\vsub(S_{t+1})$, using \eqref{eq:oci_delta}. 
\end{enumerate} 
All background computation is used for model learning using a replay buffer and for planning to obtain $\vgoal$, so that they can be queried at any time on step 2.

\newcommand{\qtarg}{\qparams_{\mathrm{targ}}}
\newcommand{\optionset}{\Pi}
\newcommand{\scqueue}{P}

To be more concrete,
Algorithm \ref{alg:DDQN_GSP} shows the GSP algorithm, layered on DDQN \citep{van2016deep}. DDQN is a variant of DQN---and so relies on replay---that additionally incorporates the idea behind Double Q-learning to avoid maximization bias in the Q-learning update \citep{hasselt2010double}. All new parts relevant to GSP are colored blue; without the blue components, it is a standard DDQN algorithm. The primary change is the addition of the potential to the action-value weights $\qparams$, with the other blue lines primarily around learning the model and doing planning. GSP should improve on replay because it simply augments replay with a  potential difference that more quickly guides the agent to take promising actions. 

\begin{algorithm}[t]
  \caption{\subroutine{{\color{blue}GSP} (built on DDQN)}}
  \label{alg:DDQN_GSP}
\begin{algorithmic}
\State Initialize base learner parameters $\qparams,\qtarg = \mathbf{w}_0$, {\color{blue} set of subgoals $\Goals$, relevance function $d$, model parameters $\mparams = (\rparams, \gamparams, \polparams), \tilde{\mparams} = (\vggparams, \gamggparams)$}
\State Sample initial state $s_0$ from the environment
  \For {$t \in 0, 1, 2, ...$}
      \State Take action $a_t$ using $q$ (e.g., $\epsilon$-greedy), observe $s_{t+1}, r_{t+1}, \gamma_{t+1}$
  \State Add experience $(s_t, a_t, s_{t+1}, r_{t+1}, \gamma_{t+1})$ to replay buffer $D$    
  \State {\color{blue}\subroutine{DDQNModelUpdate}$()$ (see Algorithm \ref{alg:DDQN_model})}
  \State {\color{blue}\subroutine{Planning}$()$ (see Algorithm \ref{alg:GoalSpacePlanning})}
  \For{$n$ mini-batches}
    \State Sample batch $B = \{ (s, a, r, s', \gamma )\}$ from $D$
    {\color{blue} \If{$d(s, \cdot), d(s', \cdot) > 0$}
      \State $\vsub(s) = \max_{g \in \augGoals: \relsg(s,g) > 0} \rgamma(s, g; \rparams) + \GamModel(s,g; \gamparams) \vgoal(g)$
      \State $\vsub(s') = \max_{g \in \augGoals: \relsg(s',g) > 0} \rgamma(s', g; \rparams) + \GamModel(s',g; \gamparams) \vgoal(g)$
      \State $\tilde{r} = r + \gamma \vsub(s') - \vsub(s) $
    \Else
      \State $\tilde{r} = r$
    \EndIf
    \State $Y(s,a,r,s',\gamma) = \tilde{r} + \gamma q(s',\argmax_{a'} q(s',a'; \qparams); \qtarg)$}
    \State $L = \frac{1}{|B|}\sum_{(s, a, r, s', \gamma) \in B} (Y(s,a,r,s',\gamma) - q(s,a; \qparams))^2$
    \State $\qparams \leftarrow \qparams - \alpha \nabla_\qparams L$ 
    \If{$n_{\text{updates}}\% \tau == 0$}
    \State $\qtarg \leftarrow \qparams $
    \EndIf
    \State $n_{\text{updates}}$ = $n_{\text{updates}}$ + 1
  \EndFor
  \EndFor
\end{algorithmic}
\end{algorithm}

\section{Experiments to understand GSP more deeply} \label{sec:experiment}

This section demonstrates the capabilities and limitations of the GSP framework. Having seen the role of GSP in propagting value and speeding up learning, we investigate how this utility of GSP is affected by:  1) deep non-linear function approximation, 2) subgoal placement, 3) sensitivity to models, 3) the $\vsub$ potential used.
%

\subsection{GSP with Deep Reinforcement Learning}\label{sec:non_linear}
The previous results shed light on the dynamics of value propagation with GSP when a learner is given a representation of it's environment (a look-up table or a tile coding). A natural next step is to look at the whether the reward and transition dynamics learnt in GSP can still propagate value (Hypothesis \ref{hyp:val_prop_time}) in the deep RL setting, where the learner must also learn a representation of its environment.
We test this by running a DDQN base learner \citep{van2016deep} in the PinBall domain, with GSP layered on DDQN as in Algorithm \ref{alg:DDQN_GSP}. 
The base learner's complete hyper-parameter specifications can be found in Appendix \ref{learn_opt_pol}. 

\begin{figure}[th]
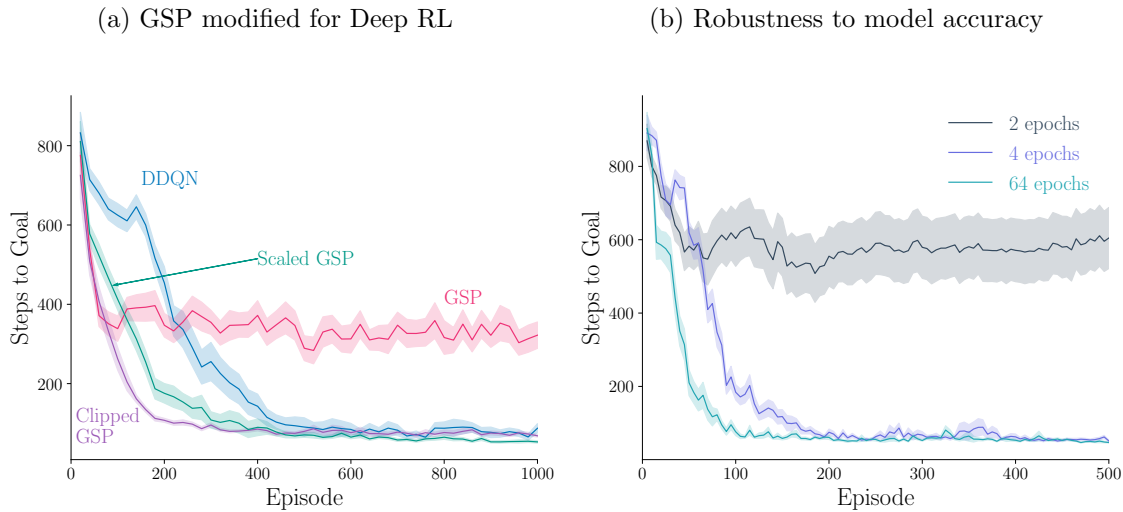

  \begin{subfigure}{0.49\textwidth}
    \centering
    \caption{GSP modified for Deep RL}
    \scalebox{0.45}{\input{ddqn.pgf}}
  \label{fig:dqn} 
\end{subfigure}
  \begin{subfigure}{0.49\textwidth}
    \centering
    \caption{Robustness to model accuracy}
    \scalebox{0.45}{\input{partial_fixed.pgf}}
  \label{fig:partial_mods} 
\end{subfigure}
  \caption{Investigating the behavior of GSP in the deep reinforcement learning setting in PinBall. (a) Following the format of Figure \ref{fig:gridball_pinball_curve}, we show the 20 episode moving average of steps to the main goal in PinBall. 
  (b) Five episode moving average of steps to goal in PinBall for GSP with models trained with differing numbers of epochs.}
  \label{fig:deep} 
\end{figure}


Unlike the previous experiments, using GSP out of the box resulted in the base learner converging to a sub-optimal policy. 
This is despite the fact that we used the same $\vsub$ as the previous PinBall experiments.
We investigated the distribution of shaping terms added to the environment reward and observed that they were occasionally an order of magnitude greater than the environment reward.
Though the linear and tabular methods handled these spikes in potential difference gracefully, these large displacements seemed to causes issues when using neural networks and a DDQN base learner. 

We tested two variants of GSP that better control the magnitudes of the raw potential differences ($\gamma \Phi(S_{t+1}) - \Phi(S_t)$). We adjusted for this by either clipping or down-scaling the potential difference added to the reward. The scaled reward multiplies the potential difference by 0.1. Clipped GSP clips the potential difference into the $[-1, 1]$ interval. With these basics magnitude controls, GSP again learns significantly faster than its base learner, as shown in Figure \ref{fig:dqn}. 

\subsection{Robustness to Accuracy of the Learned Models} \label{sec:mod_acc}

In this section, we investigate how robust GSP is to inaccuracy of its models. 
When examining the accuracy of the learned models, we found the the errors in $\rgamma$ and $\Gamma$ could be as high as 20\% in some parts of the state space (see Appendix \ref{app_model_learning} for more information). Despite this level of inaccuracy in some states, GSP still learned effectively, as seen in Sections \ref{sec:exp_specification} and \ref{sec:larger_state}. 

We conducted a targeted experiment controlling the level of accuracy to better understand this robustness and test the following hypothesis.
\begin{hypothesis}
     GSP can learn faster with more accurate models, but can still improve on the base learner even with partially learned models. 
     \label{hyp:model_accuracy}
\end{hypothesis}
We varied the number of epochs to
obtain models of varying accuracy. Our models were fully connected artificial neural networks, and we learn the models for each subgoal by performing mini-batch stochastic gradient descent on a dataset of trajectories that end in a member state of that subgoal $g$. Full implementation details for this mini-batch stochastic gradient descent can be found in Appendix \ref{app:sec_learning_models}.

As expected, Figure \ref{fig:partial_mods} shows that more epochs over the same dataset of transitions improves how quickly the base learner reaches the main goal. Within 4 epochs of model training, the learner is able to reach a good policy to the main goal. However, if the model is very inaccurate (2 epochs), the GSP update will bias the base learner to a sub-optimal policy. There is a trend of diminishing improvement when iterating over the same dataset of experience: doubling the number of epochs from two to four results in a policy that reaches the main goal 10$\times$ quicker, but a learner which used a further 16$\times$ the number of epochs attains a statistically identical episode length by episode 500. While more accurate models lead to faster learning, relatively few epochs are required to propagate enough value to help the learner reach a good policy.


\subsection{The Impact of Subgoal Selection}\label{sec_subgoalselection}

While the above experiments show that goal-space planning can speed up learning and propagate value faster, it is crucial to understand how value propagation depends on the selection of subgoals. 
Specifically, we want to identify 1) how the graphical structure of the subgoals impacts value propagation in $\vsub$ and 2) how quickly the base learner can change their policy.
To answer these questions we consider a setting where the agent is presented with new information that indicates it should change its behavior. We will then update the state-to-subgoal and subgoal-to-subgoal models online and measure how much $\vsub$ changes, along with how quickly the base learner can change its policy on different subgoal configurations.

For this task, the agent has to decide between taking one of two paths to a goal state. We initialize the agent to use an optimal policy so it takes the shorter of the two paths. Then we introduce a lava pool along the optimal path that gives the agent a large negative reward for entering it. This negative reward means the initial path is no longer optimal and that the agent needs to switch to the alternate path. 
The FourRooms environment uses $-1$ reward per step, and each state in the lava pool has a reward of $-20$. The agent, initialized with $q^\star$ for the original FourRooms environment, is run for 100 episode in the new FourRooms environment with the lava pool. We run GSP with Sarsa in this tabular setting for all subgoal configurations for 200 runs each. 

We test the following hypothesis.
\begin{hypothesis}
    The placement of subgoals along the initial and alternate optimal paths are essential for fast adaptation. 
\end{hypothesis}
To test this hypothesis, we will evaluate the following four subgoal configurations.
The first subgoal arrangement contains no subgoals near the goal state and the goal state is not connected the other subgoals. The second contains a subgoal on the initial optimal path, but no subgoal on the alternate path. The third is where there is subgoal on the alternate path but no subgoal on the initial optimal path.  The last is where there are subgoals on both paths. We illustrate these subgoal configurations in Figure \ref{fig:lavasubggoals}. Recall that a subgoal's initiation set is the states in the two adjacent rooms.

\begin{figure}[ht]
  \centering
    \includegraphics[width=\textwidth]{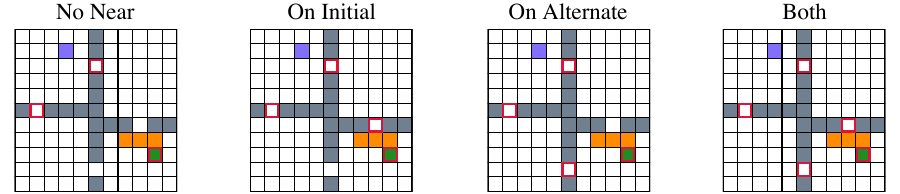}
    \caption{Different subgoal configurations in the FourRooms environment with a lava pool. The purple square is the learner's starting location, the gray squares the walls, the orange squares the location of the lava pool, and the green square the goal location. The only difference between these figures are the red boxes, which indicate the states that are subgoals for that configuration. 
    }
    \label{fig:lavasubggoals}
\end{figure}

For this experiment, the state-to-subgoal models and abstract MDP need to be updated online. However, since only the reward function is changing, we only need to update the reward models $\rgamma$ and $\vgg$. Furthermore, we can represent $\rgamma$ using successor features so that the agent only needs to estimate the reward function \citep{barreto2017successor}. Let $\boldsymbol{\psi}^{\pi_g}(s) \approx \EE_{\pi_g} \left [ \sum_{k=0}^\infty \prod_{k'=0}^k \gamma_{t+k'} \boldsymbol{\phi}(S_{t+k}) | S_t=s \right ]$, where $\boldsymbol{\phi}(S_{t}) \in \RR^n$ is a vector of features for state $S_t$ and actions are selected according to option policy $\pi_g$. Then $\rgamma(s,g) = \mathbf{w}^\top \boldsymbol{\psi}^{\pi_g}(s)$, where $\mathbf{w} \in \mathbb{R}^n$. 
The learner can then update $\rgamma$ by estimating the reward function with stochastic gradient descent, i.e., $\mathbf{w} \gets \mathbf{w} + \eta [R_t - \mathbf{w}^\top \boldsymbol{\phi}(S_t)] \boldsymbol{\phi}(S_t)$ for some scalar step size $\eta$.

\begin{figure}[tbp]
  \centering
    \includegraphics[width=\textwidth]{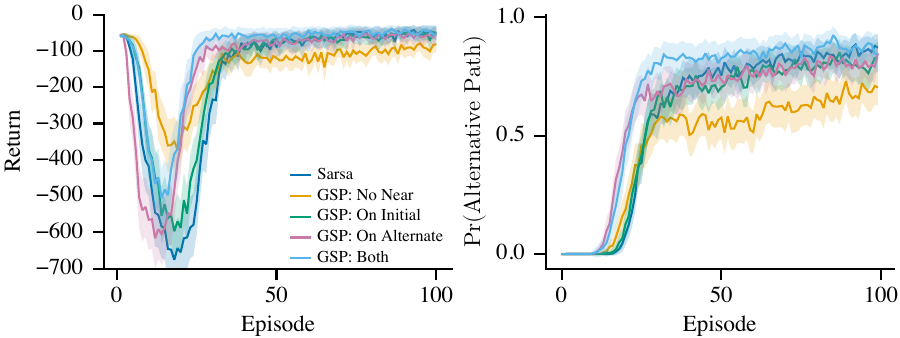}
    \caption{This figure shows the average return (left) and average probability the agent will take the alternative path (right) from each episode. Shaded regions represent (0.05,0.9) tolerance intervals \citep{patterson2020empirical} over 200 trials.} 
    \label{fig:lavapoolcurves}
\end{figure}

To understand how learning is impacted by the subgoal configuration we show the return and probability the agent takes the alternative path in Figure \ref{fig:lavapoolcurves}. The first thing that is apparent is that all configurations are able to change the policy so that the probability of taking the alternative path increases. The main differences come from how quickly, in expectation, each configuration is able to change the policy to have a high probability of taking the alternate path. The Both and On Alternate subgoal configurations have the quickest change in the policy on average, while the other methods are slower. The No Near configuration also seems to, on average, have the smallest increase in probability of taking the alternate path. These results suggest that for GSP to be most impactful, there needs to be a path through the subgoals that represents the desirable path. 

\begin{figure}[tb]
  \centering
    \includegraphics[width=\textwidth]{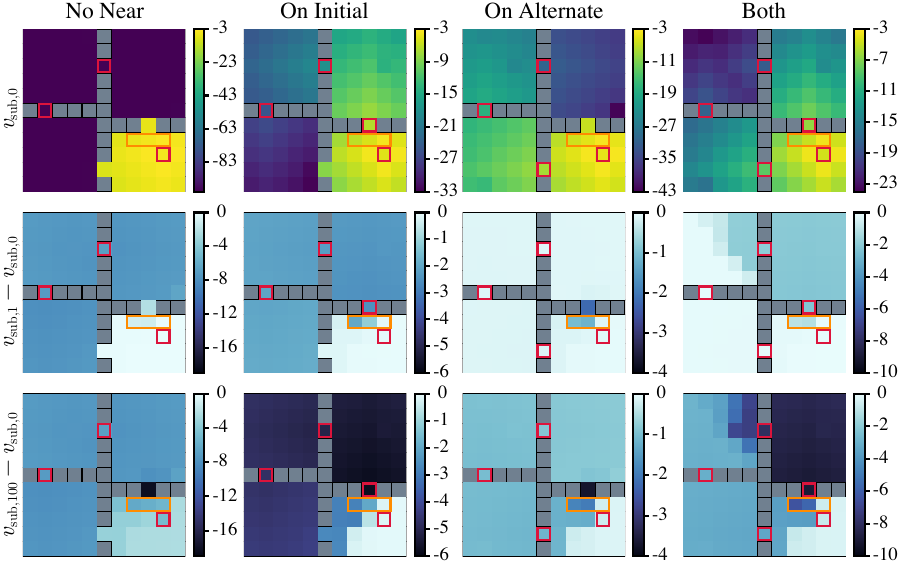}
    \caption{The top row of this figure shows the value of $\vsub$ for each state before the lava pool, for each subgoal configuration. The second and third rows show the change in $\vsub$ after the first and $100^{\text{th}}$ episode, after the lava pool is introduced.}
    \label{fig:lavapoolvsub}
\end{figure}

To better understand these results, we look more closely at $\vsub$ for each configuration. We measure how $\vsub$ changes over learning, i.e. $\vsubt{t} - \vsubt{0}$, where $\vsubt{i}$ is the value of $\vsub$ after episode $i$. 
We first examine the values of $\vsub$ for each subgoal configuration before the introduction of the lava pool (top row in Figure \ref{fig:lavapoolvsub}). For the No Near subgoals configuration, $\vsub$ has a disconnected graph, so all but the room with the goal state has a large negative value. 
For both On Initial and On Alternate configurations, $\vsub$ is the smallest in the room that is furthest from the goal state according to the abstract MDP. This is due to the structure of the abstract MDP only knowing about a single path to the goal state. 
In the Both subgoal configuration $\vsub$ closely represents the optimal value function in each state. 

We then look at the change in $\vsub$ after the lava pool is introduced, after one episode (middle row in Figure \ref{fig:lavapoolvsub}) and 100 episodes (bottom row in Figure \ref{fig:lavapoolvsub}).
We notice that the change in $\vsub$ follows the same patterns as the value representation. The value in the No Near subgoal configuration does not propagate information from the lava pool to rooms outside bottom left room. For the On Initial configuration, the value decreases quickly in the top right room, but also the other two rooms as well. After $100$ episodes the value is decreased in most states but the top right room sees the largest decrease. For the On Alternate configuration value is not quickly propagated after discovering the lava pool because there is no connected region from the path the agent took to the lava pool. However, overtime it does propagate small changes that comes from a small probability of hitting the lava pool on the alternate path. With the Both subgoal configuration, value is quickly decreased in the states that would take the initial path, but not the alternate path. This indicates the desirable path through subgoals changes in the abstract MDP. Over time the decrease in value is largely isolated to the top right room with the decreases in the other rooms coming from small chances of hitting the lava pool on the alternate path.

\textbf{Remark:} We also examined the utility of these subgoals for learning before the lava pool was introduced. Here we found that the On Alternate subgoal placement actually caused the agent to learn a suboptimal policy, because it biased it towards the alternate path initially. You can see a visualization of this $\vsub$ in Figure \ref{fig:lavapoolvsub} (top row, third column). The base learner does not use a smart exploration strategy to overcome this initial bias, and so settles on a suboptimal solution---namely, to take the slightly longer alternate path. See Appendix \ref{sec:region_of_attraction} for the full details and results for this experiment. Note that this suboptimality did not arise in the above experiment, because the lava pool made one path significantly worse than the other, pushing the agent.

\subsection{Comparison with Other Potentials} \label{sec:potentials}

Having shown several instances of $\vsub$ being used as a potential for reward shaping, we shall now investigate how much of the GSP performance improvements are due to $\vsub$ capturing useful information about the MDP, rather than just being a general consequence of using a good heuristic with potential-based reward shaping.
\begin{hypothesis}
    Using any potential function that captures the relative importance of a transition will increase the learning efficiency of the base learner, but $\vsub$ that is tailored to the MDP will allow for faster learning.
\end{hypothesis}
%
We test this by comparing $\vsub$ with two other potentials - an informative and an uninformative one - in the PinBall domain.
The first potential function is the negative $L_2$ distance in position space (scaled) to the main goal, $(x_g, y_g)$,
\begin{equation}    
    \Phi(S_t) = -100 \times \left \|       
        \begin{bmatrix}
            x_g \\ 
            y_g \end{bmatrix} - 
        \begin{bmatrix}
            x(S_t) \\ 
            y(S_t) 
        \end{bmatrix} 
        \right \|_2,
\end{equation}
where $x(S_t)$ and $y(S_t)$ are functions that return the $x$ and $y$ coordinates of the agent's state respectively. This potential function captures a measure of closeness to the goal state, but does not consider obstacles or the velocity component. So it should provide some learning benefit but should not be as helpful as $\vsub$. We scale this potential by a factor of $100$ to make it comparable in magnitude to $\vsub$. Reward shaping with the unscaled $L_2$ distance did not have any significant effect on the base learner.
The second potential is created by randomly assigning a potential value for each state, i.e., 
\begin{equation}
 \forall s \in \mathcal{S}, \ \Phi(s) \gets \mathcal{U}[-100,0].
 \end{equation}
This potential does not encode any useful information about the environment on average. It should not help learning and could even make it harder if it encourages the agent to take sub-optimal actions. 

\begin{figure}[htbp]
    \centering
    \scalebox{0.7}{\input{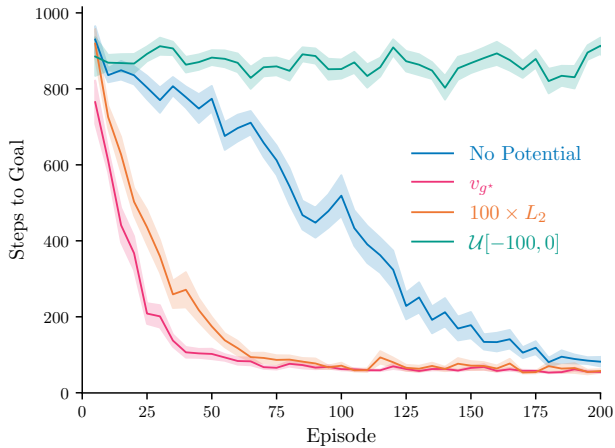}}
  \caption{Five episode moving average of steps to goal in PinBall with different potential functions for $\Phi(s)$. We follow the format of Figure \ref{fig:dqn}.}
  \label{fig:potentials}  
\end{figure}

We test the impacts of these potentials by comparing the performance of a Sarsa($\lambda$) base learner using each of the three potentials. We use the PinBall domain with the same subgoal configuration and settings as in Section \ref{sec:larger_state} and display the results in Figure \ref{fig:potentials}.
%
Using $\vsub$ for the potential reaches the main goal fastest, though
using $L_2$ also resulted in significant speed-ups over the base learner (No potential). The $L_2$ heuristic, however, is specific to navigation environments, and finding general purpose heuristics is difficult. Using a subgoal formulation for the potential is more easily extensible to other environments. The random potential harms performance, likely because it skews the reward and impacts exploration negatively. 

\subsection{Comparing to an Alternative Way of using $\vsub$}\label{sec:alt_way}

$\vsub$ is a core component of the new view on planning presented in this paper. As defined in \eqref{eq:vsub}, it reflects an approximate value of a state by using the value of a nearby subgoal. We used $\vsub$ through potential-based reward shaping, but other approaches are possible. For example, another approach is to solely bootstrap off of $\vsub$'s prediction, instead of the base learner's $q$ estimate,
\begin{equation} \label{eq:delta_lavi}
R_{t+1} + \gamma_{t+1}\vsub(S_{t+1}) - q(S_t,A_t; \qparams).
\end{equation}

The update with this TD error is reminiscent of an algorithm called Landmark Approximate Value Iteration (LAVI) \citep{mann2015approximate}. LAVI is designed for the setting where a model, or simulator, is given. Similar to GSP, the algorithm plans only over a set of landmarks (subgoals). They assume that they have options that terminate near the landmarks, and do value iteration with the simulator by executing options from only the landmarks. The greedy policy for a state uses the computed values for landmark states by selecting the option that takes the agent to the best landmark state, and using options to move only between landmark states from there. The planning is much more efficent, because the number of landmark states is relatively small, but the policies are suboptimal. 

We could similarly use $\vsub$, by running the option to bring the agent to the best nearby subgoal. However, a more direct comparison in our setting is to use the modified TD error update above. We call this update Approximate LAVI, to recognize the similarity to this elegant algorithm.
In all environments, the approximate LAVI learner either learns much slower or converges to a sub-optimal policy instead.
\begin{figure}[htbp]
\label{fig:alt_way}
\begin{subfigure}{0.32\textwidth}
    \centering
    \caption{FourRooms}
    \scalebox{0.4}{\input{fourrooms_tab_stepstogoal.pgf}}
\end{subfigure}
\begin{subfigure}{0.32\textwidth}
    \centering
    \caption{GridBall}
    \scalebox{0.4}{\input{gridball_stepstogoal.pgf}}
\end{subfigure}
\begin{subfigure}{0.32\textwidth}
    \centering
    \caption{PinBall}
    \scalebox{0.4}{\input{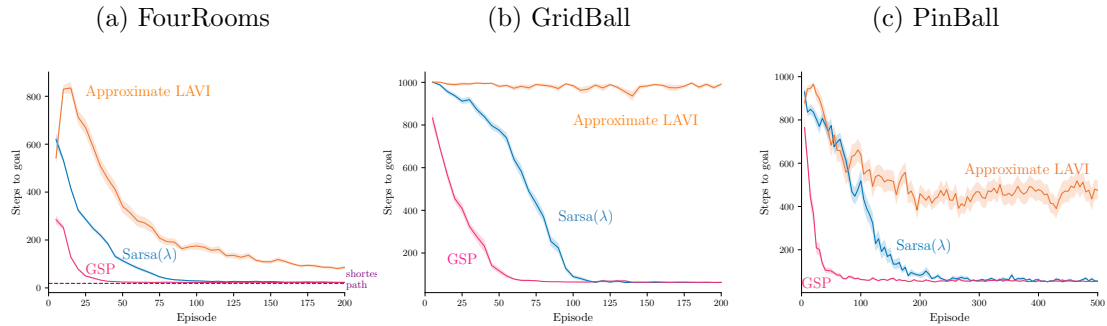}}
\end{subfigure}
\caption{ Five episode moving average of return in FourRooms, GridBall and PinBall. As with previous learning curves, we follow the format of 30 runs and shaded region representing 1 standard error.}
\end{figure}

In our preliminary experiments, we had investigated an update rule that partially bootstraps off $\vsub$. Namely, we used a TD error of $R_{t+1} + \gamma_{t+1}(\beta\vsub(S_{t+1}) + (1-\beta)q(S_{t+1}, A_{t+1})) - q(S_t, A_t)$, where $\beta \in [0,1]$. Potential based reward shaping with $\vsub$ was found to outperform this technique. We discuss this more in Appendix \ref{app:alt_way}.


\section{Relationships to Other Model-based Approaches}

Now that we have detailed the GSP algorithm and investigated its benefits and limitations, we contrast it to other approaches for background planning. In this section we first provide an overview of related work to better place GSP amongst the large literature of related ideas, beyond background planning. Then we more explicitly contrast GSP to Dyna and Dyna with options, which are two natural approaches to background planning. Finally, we provide a short discussion around one of the key properties of GSP, efficient planning, and why this is desirable for background planning approaches.

\subsection{Related Work}

A variety of approaches have been developed to handle issues with learning and iterating one-step models. Several papers have shown that using forward model simulations can produce simulated states that result in catastrophically misleading values \citep{jafferjee2020hallucinating,vanhasselt2019when,lambert2022investigating}. This problem has been tackled by using reverse models \citep{pan2018organizing,jafferjee2020hallucinating,vanhasselt2019when}; primarily using the model for decision-time planning \citep{vanhasselt2019when,silver2008samplebased,chelu2020forethought}; and improving training strategies to account for accumulated errors in rollouts \citep{talvitie2014model,venkatraman2015improving,talvitie2017selfcorrecting}. 
An emerging trend is to avoid approximating the true transition dynamics, and instead learn dynamics  tailored to predicting values on the next step correctly \citep{farahmand2017valueaware,farahmand2018iterative,ayoub2020modelbased}. This trend is also implicit in the variety of techniques that encode the planning procedure into neural network architectures that can then be trained end-to-end \citep{tamar2016value,silver2017predictron,oh2017value,weber2017imaginationaugmented,farquhar2018treeqn,schrittwieser2020mastering}. We similarly attempt to avoid issues with iterating models, but do so by considering a different type of model.

Current deep model-based RL techniques plan in a lower-dimensional abstract space where the relevant features from the original high-dimensional experience are preserved, often refered to as a \emph{latent space}. MuZero \citep{schrittwieser2020mastering}, for example, embeds the history of observations to then use predictive models of values, policies and one-step rewards. Using these three predictive models in the latent space guides MuZero’s Monte Carlo Tree Search without the need for a perfect simulator of the environment. Most recently, DreamerV3 demonstrated the capabilities of a discrete latent world model in a range of pixel-based environments \citep{hafner2023mastering}. There is growing evidence that it is easier to learn accurate models in a latent space.

Temporal abstraction has also been considered to make planning more efficient, through the use of hierarchical RL and/or options. MAXQ \cite{dietterich2000hierarchical} introduced the idea of learning hierarchical policies with multiple levels, breaking up the problem into multiple subgoals. A large literature followed, focused on efficient planning with hierarchical policies \citep{diuk2006hierarchical} and using a hierarchy of MDPs with state abstraction and macro-actions \citep{bakker2005hierarchical,konidaris2014constructing,konidaris2016constructing,gopalan2017planning}. See \citet{gopalan2017planning} for an excellent summary.

Rather than using a hierarchy and planning only in abstract MDPs, another strategy is simply to add options as additional (macro) actions in planning, still also including primitive actions. Similar ideas were explored before the introduction of options \citep{singh1992scaling, dayan1992feudal}. 
There has been some theoretical characterization of the utility of options for improving convergence rates of value iteration \citep{mann2014scaling} and sample efficiency \citep{brunskill2014pacinspired}, though also hardness results reflecting that the augmented MDP is not necessarily more efficient to solve \citep{zahavy2020planning} and hardness results around discovering options efficient for planning \citep{jinnai2019finding}. Empirically, 
incorporating options into planning has largely only been tested in tabular settings \citep{sutton1999options,singh2004intrinsically,wan2021averagereward}. Recent work has considered mechanism for identifying and learning option policies for planning under function approximation \citep{sutton2022rewardrespecting}, but as yet did not consider issues with learning the models. 

There has been some work using options for planning that is more similar to GSP, using only one-level of abstraction and restricting planning to the abstract MDP.
\citet{hauskrecht2013hierarchical} proposed to plan only in the abstract MDP with macro-actions (options) and abstract states corresponding to the boundaries of the regions spanned by the options, which is like restricting abstract states to subgoals. The most similar to our work is LAVI, which restricts value iteration to a small subset of landmark states \citep{mann2015approximate}.\footnote{A similar idea to landmark states has been considered in more classical AI approaches, under the term bi-level planning \citep{wolfe2010combined,hogg2010learning,chitnis2021learning}. These techniques are quite different from Dyna-style planning---updating values with (stochastic) dynamic programming updates---and so we do not consider them further here.} These methods also have similar flavors to using a hierarchy of MDPs, in that they focus planning in a smaller space and (mostly) avoid planning at the lowest level, obtaining significant computational speed-ups. The key distinction to GSP is that we are not in the traditional planning setting where a model is given; in our online setting, the agent needs to learn the model from interaction.

The use of landmark states has also been explored in \emph{goal-conditioned RL}, where the agent is given a desired goal state or states. This is a problem setting where the aim is to learn a policy $\pi(a| s, g)$ that can be conditioned on different possible goals. The agent learns for a given set of goals, with the assumption that at the start of each episode the goal state is explicitly given to the agent. After this training phase, the policy should generalize to previously unseen goals. Naturally, this idea has particularly been applied to navigation, having the agent learn to navigate to different states (goals) in the environment. The first work to exploit the idea of landmark states in GCRL
was for learning and using universal value function approximators (UVFAs) \citep{huang2019mapping}. The UVFA conditions action-values on both state-action pairs as well as landmark states. The agent can reach new goals by searching on a learned graph between landmark states, to identify which landmark to moves towards. A flurry of work followed, still in the goal-conditioned setting \citep{nasiriany2019planning,emmons2020sparse,zhang2020generating,zhang2021world,aubret2021distop,hoang2021successor,gieselmann2021planningaugmented,kim2021landmarkguided,dubey2021snap}. 

Some of this work focused on exploiting landmark states for planning in GCRL. \citet{huang2019mapping} used landmark states as interim subgoals, with a graph-based search to plan between these subgoals \citep{huang2019mapping}. The policy is set to reach the nearest goal (using action-values with cost-to-goal rewards of -1 per step) and learned distance functions between states and goals and between goals. These models are like our reward and discount models, but tailored to navigation and distances. \citet{nasiriany2019planning} built on this idea, introducing an algorithm called Latent Embeddings for Abstracted Planning (LEAP), that using gradient descent to search for a sequence of subgoals in a latent space. 

The idea of learning models that immediately apply to new subtasks using successor features is like GCRL, but does not explicitly use landmark states. 
The option keyboard involves encoding options (or policies) as vectors that describe the corresponding (pseudo) reward \citep{barreto2019option}. This work has been expanded more recently, using successor features \citep{barreto2020fast}. New policies can then be easily obtained for new reward functions, by linearly combining the (basis) vectors for the already learned options. However no planning is involved in that work, beyond a one-step decision-time choice amongst options.

\begin{algorithm}[t]
  \caption{\subroutine{Dyna {\color{red} with options}} (using the DDQN update)}
  \label{alg:dyna}
\begin{algorithmic}
\State Initialize base learner parameters $\qparams,\qtarg = \mathbf{w}_0$, model parameters $\mparams$, search-control queue $\scqueue$, {\color{red} set of options $\optionset$}
\State Sample initial state $s_0$ from the environment
  \For {$t \in 0, 1, 2, ...$}
      \State Take action $a_t$ using $q$ (e.g., $\epsilon$-greedy), observe $s_{t+1}, r_{t+1}, \gamma_{t+1}$
  \State Store $s_t$ in $\scqueue$ 
  \State \subroutine{ModelUpdate}$(s_t, a_t, s_{t+1}, r_{t+1},\gamma_{t+1})$ (see Algorithm \ref{alg:ModelUpdate})
  \For{$n_{\mathrm{main}}$ mini-batches}
    \State Sample batch $B = \{s\}$ from $\scqueue$
    \State For each $s \in B$, pick a random $\tilde{a}$ from $\mathcal{A} {\color{red} \cup \optionset}$ 
    \State {\color{red} // if $\tilde{a}$ is an option, $s'$ is an outcome state after many steps,} 
    \State {\color{red} // $r$ is a discounted sum of rewards under the option until termination}
    \State {\color{red} // and $\gamma$ is the discount raised to the number of steps that option executes} 
    \State Query model at each $(s,\tilde{a})$ to get outcome $s', r, \gamma$ and corresponding target
   \State $Y = r + \gamma q(s',\argmax_{a'} q(s',a'; \qparams); \qparams_{\mathrm{targ}})$
    \State $L = \frac{1}{|B|}\sum_{(s, a, Y) \in B} (Y - q(s,a; \qparams))^2$
    \State $\qparams \leftarrow \qparams - \alpha \nabla_\qparams L$ 
    \If{$n_{\text{updates}}\% \tau == 0$}
    \State $\qparams_{\mathrm{targ}} \leftarrow \qparams $
    \EndIf
    \State $n_{\text{updates}}$ = $n_{\text{updates}}$ + 1
  \EndFor
  \EndFor
\end{algorithmic}
\end{algorithm}

\subsection{Contrasting GSP to Dyna and Dyna with Options}

In this section we contrast GSP to Dyna and Dyna with options, which are two canonical approaches to do background planning. Dyna involves learning a transition model and updating with simulated experience, in the background. The original version of Dyna simply uses one step transitions from observed states, making it look quite similar to replay. Replay can actually be seen as a limited, non-parametric version of Dyna, and often Dyna and replay perform similarly \citep{pan2018organizing,vanhasselt2019when}, without more focused \emph{search control}, namely from which states we query the model. 
To truly obtain the benefits of the model with Dyna, it is key to consider which $(s,a)$ is the most useful to update from, which may even be a hypothetical $(s,a)$ never observed. Querying the model from such unseen $(s,a)$ leverages the generalization capabilities of the model much more than simply querying the model from an observed ($s,a)$. It is likely a clever search-control strategy could significantly improve Dyna, but it is also a hard problem. There remain very few proposed search-control approaches only a handful of works \citep{moore1993prioritized,wingate2005prioritization,pan2019hill}. 

If we go beyond one-step transitions, then we further deviate from replay and can benefit from having an explicit learned model. As mentioned above, Dyna with rollouts can suffer from model iteration error. An alternative approach to predict outcomes multiple steps into the future is to incorporate options into Dyna. This extension was first proposed for the tabular setting \citep{singh2004intrinsically}, with little follow-up work beyond a recent re-investigation still in the tabular setting \citep{sutton2022rewardrespecting}. The general idea behind Dyna with options is to treat options like macro-actions in the planning loop. Let us consider the one-step transition dynamics model, for a given $\pi$, where the model outputs $\tilde{s}', \tilde{r}, \tilde{\gamma}$ from $(s,\pi)$. For simplicity, assume we are learning expected transition dynamics. The model outputs the expected outcome state $\tilde{s}'$ after executing the option from $s$. In other words, $\tilde{s}'$ is the expected outcome state multiple steps into the future. The outputted reward $\tilde{r}$ from $(s,\pi)$ is the discounted cumulative sum of rewards of the option $\pi$ when starting from $s$, until termination. The outputted $\tilde{\gamma}$ is the discounted probability of terminating. For example, if the option always terminated after $n$ steps, then the model would output $\tilde{\gamma} = \gamma^n$. We show a possible variant of Dyna and Dyna with options, again using a similar update to DDQN, in Algorithm \ref{alg:dyna}, to make this extension more concrete. 

Dyna with options should allow for faster value propagation than Dyna alone. It effectively uses multi-step updates rather than simple single-step updates. However, it actually requires learning an even more complex model than Dyna, since it must learn the transition-dynamics for the options as well as the transition dynamics for the primitive actions. Moreover, it does still plan over all states and actions; again without smarter search-control, planning is likely to still be inefficient. In this sense, the variants of Dyna and Dyna with options presented here do not satisfy two key desiderata: feasible model learning and efficient planning. We discuss the importance of efficient planning in more depth in the next section. 

\textbf{Remark:} The well-versed reader may be confused why we consider Dyna algorithms on states, rather than on a latent state, also called agent state. It might seem that learning transition dynamics on a (compact) agent state might provide us the desired desiderata. Such a change is likely to make it more feasible to learn the model and make planning more efficient. Nonetheless, we are still stuck planning in a continuous latent space that is likely to have 32 dimensions, or more, based on typical design choices. It reduces, but does not remove, these issues.

  \begin{figure}[t]
   \centering
       \includegraphics[width=0.6\textwidth]{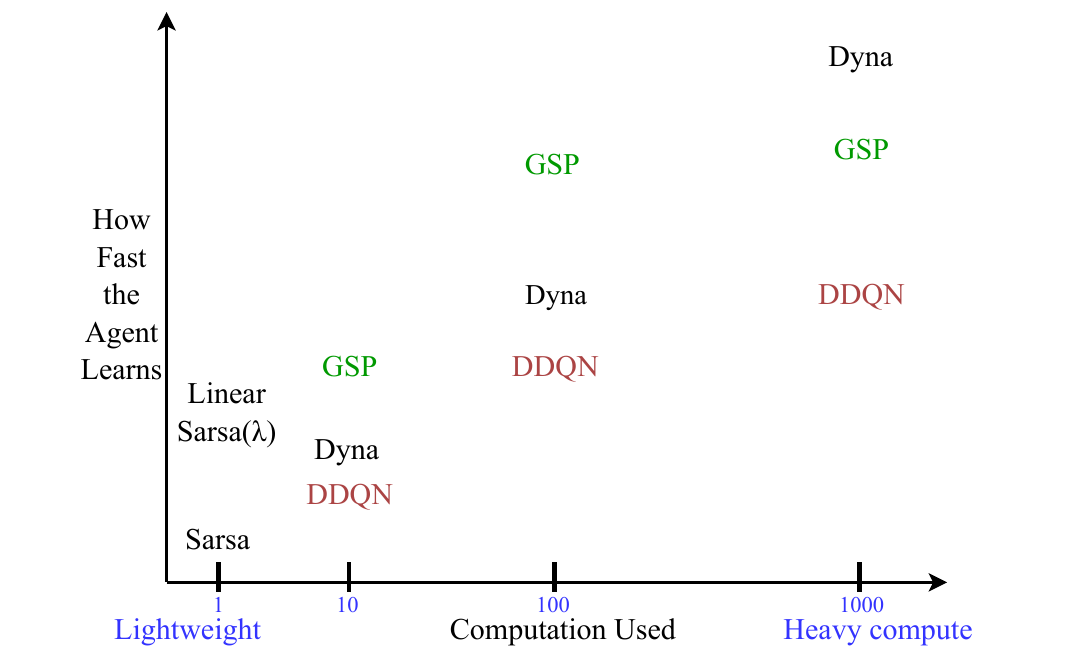}
            \caption{A visualization of the hypothetical trade-off between computation and how quickly the agent learns for different algorithms. This plot is focused on computation, rather than dealing with model errors, so Dyna means Dyna with a highly accurate model. Lightweight algorithms like Sarsa, that update only with the most recent sample, cannot leverage more computation to improve learning. DDQN, GSP and Dyna can leverage more computation by increasing the amount of replay and planning steps. Of course, this diagram is completely hypothetical, but reflects the thinking that guides this work as well as anticipated behavior of these algorithms. GSP is more effective than Dyna with less compute, but is likely to plateau at a slightly suboptimal point. With a lot of compute, Dyna with an accurate model is effectively doing dynamic programming and extracting a policy from the model. But with much less compute, it does not efficiently focus that compute to improve the policy. DDQN is limited by the limited data that it has in its buffer, and unlike Dyna, cannot reason about possible outcomes outside of this dataset.}\label{fig:trade-off} 
   \end{figure}
   
\subsection{The Importance of Efficient Planning}

GSP is designed to allow for efficient planning. We want changes in the environment to quickly propagate through the value function. This is achieved by focusing planning on a small set of subgoals. The local subgoal models can be updated efficiently, and value iteration to get the new subgoal values. Value iteration for a small set of states is very efficient, and the agent can perform many value iteration updates per step to keep these subgoal values accurate. Replay then propagates these subgoal values to the state values. 

Standard replay, Dyna and even Dyna with options does not have the same computational efficiency due to the lack of higher level planning. In practice, with a bounded agent, poor computational efficiency can result in poor sample efficiency. The learned model might even be perfectly accurate, and with unlimited computation per step, the agent could obtain the perfect value function. But with a computational budget---for example with a budget of ten planning steps per step---it may fail to transfer it's (immense) knowledge about the world into the policy. Eventually, over many steps (environment interactions), it will get an accurate value function. An algorithm, on the other hand, that can more quickly transfer knowledge from it's model to the value function will get closer to the true action-values in a smaller number of environment steps. We visualize this conceptual trade-off in Figure \ref{fig:trade-off}, for DDQN (namely replay), Dyna and GSP (layered on top of DDQN).

\section{Conclusion}
 
In this paper we introduced a new planning framework, called Goal-Space Planning (GSP).  
GSP provides a new approach to use background planning to improve value propagation, with minimalist, local models and computationally efficient planning. 
We showed that these subgoal-conditioned models can be accurately learned using standard value estimation algorithms, and can be used to quickly propagate value through state spaces of varying sizes. We find a consequent learning speed-up on base learners with different types of value function approximation. Subgoal selection was found to play a big role on the value function and policy the base learner reaches. In particular, we see that the GSP algorithm helps the base learner find a path through the state space, based on the high level path from found in the abstract MDP. We also verify that the performance improvement observed in GSP is result of $\vsub$ capturing the MDP dynamics, and not a general consequence of potential-based reward shaping.

This work introduces a new formalism, and many new technical questions along with it. We have tested GSP with pre-trained models and assumed a given set of subgoals. Our experiments learning the models online using successor representations indicate that GSP can get similar learning speed boosts. Using a recency buffer, however, accumulates transitions only along the optimal trajectory, sometimes causing the models to become inaccurate part-way through learning. An important next step is to incorporate smarter model learning strategies. The other critical open question is in subgoal discovery. For this work, we relied on hand-chosen subgoals, but in general the agent should discover its own subgoals. In general, though, option and subgoal discovery remain open questions. One utility of this work is that it could help narrow the scope of the discovery question, to that of finding abstract subgoals that help a learner plan more efficiently.

\bibliography{paper}

\newpage

\input{appendix}

\end{document}

%% file: fourrooms_tab_stepstogoal_noLAVI.pgf
\begingroup%
\makeatletter%
\begin{pgfpicture}%
\pgfpathrectangle{\pgfpointorigin}{\pgfqpoint{5.000000in}{3.750000in}}%
\pgfusepath{use as bounding box, clip}%
\begin{pgfscope}%
\pgfsetbuttcap%
\pgfsetmiterjoin%
\definecolor{currentfill}{rgb}{1.000000,1.000000,1.000000}%
\pgfsetfillcolor{currentfill}%
\pgfsetlinewidth{0.000000pt}%
\definecolor{currentstroke}{rgb}{1.000000,1.000000,1.000000}%
\pgfsetstrokecolor{currentstroke}%
\pgfsetdash{}{0pt}%
\pgfpathmoveto{\pgfqpoint{0.000000in}{0.000000in}}%
\pgfpathlineto{\pgfqpoint{5.000000in}{0.000000in}}%
\pgfpathlineto{\pgfqpoint{5.000000in}{3.750000in}}%
\pgfpathlineto{\pgfqpoint{0.000000in}{3.750000in}}%
\pgfpathlineto{\pgfqpoint{0.000000in}{0.000000in}}%
\pgfpathclose%
\pgfusepath{fill}%
\end{pgfscope}%
\begin{pgfscope}%
\pgfsetbuttcap%
\pgfsetmiterjoin%
\definecolor{currentfill}{rgb}{1.000000,1.000000,1.000000}%
\pgfsetfillcolor{currentfill}%
\pgfsetlinewidth{0.000000pt}%
\definecolor{currentstroke}{rgb}{0.000000,0.000000,0.000000}%
\pgfsetstrokecolor{currentstroke}%
\pgfsetstrokeopacity{0.000000}%
\pgfsetdash{}{0pt}%
\pgfpathmoveto{\pgfqpoint{0.625000in}{0.412500in}}%
\pgfpathlineto{\pgfqpoint{4.500000in}{0.412500in}}%
\pgfpathlineto{\pgfqpoint{4.500000in}{3.300000in}}%
\pgfpathlineto{\pgfqpoint{0.625000in}{3.300000in}}%
\pgfpathlineto{\pgfqpoint{0.625000in}{0.412500in}}%
\pgfpathclose%
\pgfusepath{fill}%
\end{pgfscope}%
\begin{pgfscope}%
\pgfpathrectangle{\pgfqpoint{0.625000in}{0.412500in}}{\pgfqpoint{3.875000in}{2.887500in}}%
\pgfusepath{clip}%
\pgfsetbuttcap%
\pgfsetroundjoin%
\definecolor{currentfill}{rgb}{0.000000,0.466667,0.733333}%
\pgfsetfillcolor{currentfill}%
\pgfsetfillopacity{0.200000}%
\pgfsetlinewidth{1.003750pt}%
\definecolor{currentstroke}{rgb}{0.000000,0.466667,0.733333}%
\pgfsetstrokecolor{currentstroke}%
\pgfsetstrokeopacity{0.200000}%
\pgfsetdash{}{0pt}%
\pgfsys@defobject{currentmarker}{\pgfqpoint{0.721875in}{0.545060in}}{\pgfqpoint{4.500000in}{3.168750in}}{%
\pgfpathmoveto{\pgfqpoint{0.721875in}{3.168750in}}%
\pgfpathlineto{\pgfqpoint{0.721875in}{2.931760in}}%
\pgfpathlineto{\pgfqpoint{0.818750in}{2.643502in}}%
\pgfpathlineto{\pgfqpoint{0.915625in}{2.154221in}}%
\pgfpathlineto{\pgfqpoint{1.012500in}{1.773051in}}%
\pgfpathlineto{\pgfqpoint{1.109375in}{1.615260in}}%
\pgfpathlineto{\pgfqpoint{1.206250in}{1.462656in}}%
\pgfpathlineto{\pgfqpoint{1.303125in}{1.338302in}}%
\pgfpathlineto{\pgfqpoint{1.400000in}{1.191454in}}%
\pgfpathlineto{\pgfqpoint{1.496875in}{0.971779in}}%
\pgfpathlineto{\pgfqpoint{1.593750in}{0.898081in}}%
\pgfpathlineto{\pgfqpoint{1.690625in}{0.835527in}}%
\pgfpathlineto{\pgfqpoint{1.787500in}{0.772803in}}%
\pgfpathlineto{\pgfqpoint{1.884375in}{0.724826in}}%
\pgfpathlineto{\pgfqpoint{1.981250in}{0.677590in}}%
\pgfpathlineto{\pgfqpoint{2.078125in}{0.627519in}}%
\pgfpathlineto{\pgfqpoint{2.175000in}{0.599630in}}%
\pgfpathlineto{\pgfqpoint{2.271875in}{0.583551in}}%
\pgfpathlineto{\pgfqpoint{2.368750in}{0.575936in}}%
\pgfpathlineto{\pgfqpoint{2.465625in}{0.568875in}}%
\pgfpathlineto{\pgfqpoint{2.562500in}{0.566230in}}%
\pgfpathlineto{\pgfqpoint{2.659375in}{0.561230in}}%
\pgfpathlineto{\pgfqpoint{2.756250in}{0.557712in}}%
\pgfpathlineto{\pgfqpoint{2.853125in}{0.558107in}}%
\pgfpathlineto{\pgfqpoint{2.950000in}{0.553500in}}%
\pgfpathlineto{\pgfqpoint{3.046875in}{0.558598in}}%
\pgfpathlineto{\pgfqpoint{3.143750in}{0.556527in}}%
\pgfpathlineto{\pgfqpoint{3.240625in}{0.555710in}}%
\pgfpathlineto{\pgfqpoint{3.337500in}{0.559981in}}%
\pgfpathlineto{\pgfqpoint{3.434375in}{0.553785in}}%
\pgfpathlineto{\pgfqpoint{3.531250in}{0.555661in}}%
\pgfpathlineto{\pgfqpoint{3.628125in}{0.551217in}}%
\pgfpathlineto{\pgfqpoint{3.725000in}{0.547838in}}%
\pgfpathlineto{\pgfqpoint{3.821875in}{0.548475in}}%
\pgfpathlineto{\pgfqpoint{3.918750in}{0.550104in}}%
\pgfpathlineto{\pgfqpoint{4.015625in}{0.548615in}}%
\pgfpathlineto{\pgfqpoint{4.112500in}{0.548593in}}%
\pgfpathlineto{\pgfqpoint{4.209375in}{0.548733in}}%
\pgfpathlineto{\pgfqpoint{4.306250in}{0.548094in}}%
\pgfpathlineto{\pgfqpoint{4.403125in}{0.545060in}}%
\pgfpathlineto{\pgfqpoint{4.500000in}{0.546394in}}%
\pgfpathlineto{\pgfqpoint{4.500000in}{0.548829in}}%
\pgfpathlineto{\pgfqpoint{4.500000in}{0.548829in}}%
\pgfpathlineto{\pgfqpoint{4.403125in}{0.550548in}}%
\pgfpathlineto{\pgfqpoint{4.306250in}{0.551572in}}%
\pgfpathlineto{\pgfqpoint{4.209375in}{0.557640in}}%
\pgfpathlineto{\pgfqpoint{4.112500in}{0.555735in}}%
\pgfpathlineto{\pgfqpoint{4.015625in}{0.560810in}}%
\pgfpathlineto{\pgfqpoint{3.918750in}{0.555498in}}%
\pgfpathlineto{\pgfqpoint{3.821875in}{0.553538in}}%
\pgfpathlineto{\pgfqpoint{3.725000in}{0.551225in}}%
\pgfpathlineto{\pgfqpoint{3.628125in}{0.562920in}}%
\pgfpathlineto{\pgfqpoint{3.531250in}{0.568469in}}%
\pgfpathlineto{\pgfqpoint{3.434375in}{0.561827in}}%
\pgfpathlineto{\pgfqpoint{3.337500in}{0.572667in}}%
\pgfpathlineto{\pgfqpoint{3.240625in}{0.566341in}}%
\pgfpathlineto{\pgfqpoint{3.143750in}{0.567654in}}%
\pgfpathlineto{\pgfqpoint{3.046875in}{0.571418in}}%
\pgfpathlineto{\pgfqpoint{2.950000in}{0.559077in}}%
\pgfpathlineto{\pgfqpoint{2.853125in}{0.575966in}}%
\pgfpathlineto{\pgfqpoint{2.756250in}{0.568984in}}%
\pgfpathlineto{\pgfqpoint{2.659375in}{0.577153in}}%
\pgfpathlineto{\pgfqpoint{2.562500in}{0.581878in}}%
\pgfpathlineto{\pgfqpoint{2.465625in}{0.584599in}}%
\pgfpathlineto{\pgfqpoint{2.368750in}{0.596117in}}%
\pgfpathlineto{\pgfqpoint{2.271875in}{0.602737in}}%
\pgfpathlineto{\pgfqpoint{2.175000in}{0.625056in}}%
\pgfpathlineto{\pgfqpoint{2.078125in}{0.660918in}}%
\pgfpathlineto{\pgfqpoint{1.981250in}{0.721765in}}%
\pgfpathlineto{\pgfqpoint{1.884375in}{0.775688in}}%
\pgfpathlineto{\pgfqpoint{1.787500in}{0.828100in}}%
\pgfpathlineto{\pgfqpoint{1.690625in}{0.885785in}}%
\pgfpathlineto{\pgfqpoint{1.593750in}{0.956014in}}%
\pgfpathlineto{\pgfqpoint{1.496875in}{1.031297in}}%
\pgfpathlineto{\pgfqpoint{1.400000in}{1.257742in}}%
\pgfpathlineto{\pgfqpoint{1.303125in}{1.409293in}}%
\pgfpathlineto{\pgfqpoint{1.206250in}{1.526141in}}%
\pgfpathlineto{\pgfqpoint{1.109375in}{1.688485in}}%
\pgfpathlineto{\pgfqpoint{1.012500in}{1.845121in}}%
\pgfpathlineto{\pgfqpoint{0.915625in}{2.234007in}}%
\pgfpathlineto{\pgfqpoint{0.818750in}{2.725636in}}%
\pgfpathlineto{\pgfqpoint{0.721875in}{3.168750in}}%
\pgfpathlineto{\pgfqpoint{0.721875in}{3.168750in}}%
\pgfpathclose%
\pgfusepath{stroke,fill}%
}%
\begin{pgfscope}%
\pgfsys@transformshift{0.000000in}{0.000000in}%
\pgfsys@useobject{currentmarker}{}%
\end{pgfscope}%
\end{pgfscope}%
\begin{pgfscope}%
\pgfpathrectangle{\pgfqpoint{0.625000in}{0.412500in}}{\pgfqpoint{3.875000in}{2.887500in}}%
\pgfusepath{clip}%
\pgfsetbuttcap%
\pgfsetroundjoin%
\definecolor{currentfill}{rgb}{0.933333,0.200000,0.466667}%
\pgfsetfillcolor{currentfill}%
\pgfsetfillopacity{0.200000}%
\pgfsetlinewidth{1.003750pt}%
\definecolor{currentstroke}{rgb}{0.933333,0.200000,0.466667}%
\pgfsetstrokecolor{currentstroke}%
\pgfsetstrokeopacity{0.200000}%
\pgfsetdash{}{0pt}%
\pgfsys@defobject{currentmarker}{\pgfqpoint{0.721875in}{0.543750in}}{\pgfqpoint{4.500000in}{1.723376in}}{%
\pgfpathmoveto{\pgfqpoint{0.721875in}{1.723376in}}%
\pgfpathlineto{\pgfqpoint{0.721875in}{1.565361in}}%
\pgfpathlineto{\pgfqpoint{0.818750in}{1.475088in}}%
\pgfpathlineto{\pgfqpoint{0.915625in}{0.964374in}}%
\pgfpathlineto{\pgfqpoint{1.012500in}{0.758133in}}%
\pgfpathlineto{\pgfqpoint{1.109375in}{0.646332in}}%
\pgfpathlineto{\pgfqpoint{1.206250in}{0.610651in}}%
\pgfpathlineto{\pgfqpoint{1.303125in}{0.579474in}}%
\pgfpathlineto{\pgfqpoint{1.400000in}{0.561713in}}%
\pgfpathlineto{\pgfqpoint{1.496875in}{0.557110in}}%
\pgfpathlineto{\pgfqpoint{1.593750in}{0.551094in}}%
\pgfpathlineto{\pgfqpoint{1.690625in}{0.550332in}}%
\pgfpathlineto{\pgfqpoint{1.787500in}{0.548455in}}%
\pgfpathlineto{\pgfqpoint{1.884375in}{0.548117in}}%
\pgfpathlineto{\pgfqpoint{1.981250in}{0.546007in}}%
\pgfpathlineto{\pgfqpoint{2.078125in}{0.548135in}}%
\pgfpathlineto{\pgfqpoint{2.175000in}{0.548784in}}%
\pgfpathlineto{\pgfqpoint{2.271875in}{0.547227in}}%
\pgfpathlineto{\pgfqpoint{2.368750in}{0.546699in}}%
\pgfpathlineto{\pgfqpoint{2.465625in}{0.545028in}}%
\pgfpathlineto{\pgfqpoint{2.562500in}{0.546070in}}%
\pgfpathlineto{\pgfqpoint{2.659375in}{0.544917in}}%
\pgfpathlineto{\pgfqpoint{2.756250in}{0.545606in}}%
\pgfpathlineto{\pgfqpoint{2.853125in}{0.543837in}}%
\pgfpathlineto{\pgfqpoint{2.950000in}{0.545081in}}%
\pgfpathlineto{\pgfqpoint{3.046875in}{0.546714in}}%
\pgfpathlineto{\pgfqpoint{3.143750in}{0.546722in}}%
\pgfpathlineto{\pgfqpoint{3.240625in}{0.546101in}}%
\pgfpathlineto{\pgfqpoint{3.337500in}{0.547463in}}%
\pgfpathlineto{\pgfqpoint{3.434375in}{0.545223in}}%
\pgfpathlineto{\pgfqpoint{3.531250in}{0.545773in}}%
\pgfpathlineto{\pgfqpoint{3.628125in}{0.544142in}}%
\pgfpathlineto{\pgfqpoint{3.725000in}{0.546122in}}%
\pgfpathlineto{\pgfqpoint{3.821875in}{0.546082in}}%
\pgfpathlineto{\pgfqpoint{3.918750in}{0.544341in}}%
\pgfpathlineto{\pgfqpoint{4.015625in}{0.543750in}}%
\pgfpathlineto{\pgfqpoint{4.112500in}{0.544961in}}%
\pgfpathlineto{\pgfqpoint{4.209375in}{0.545831in}}%
\pgfpathlineto{\pgfqpoint{4.306250in}{0.545374in}}%
\pgfpathlineto{\pgfqpoint{4.403125in}{0.544793in}}%
\pgfpathlineto{\pgfqpoint{4.500000in}{0.545033in}}%
\pgfpathlineto{\pgfqpoint{4.500000in}{0.549620in}}%
\pgfpathlineto{\pgfqpoint{4.500000in}{0.549620in}}%
\pgfpathlineto{\pgfqpoint{4.403125in}{0.546959in}}%
\pgfpathlineto{\pgfqpoint{4.306250in}{0.547234in}}%
\pgfpathlineto{\pgfqpoint{4.209375in}{0.550700in}}%
\pgfpathlineto{\pgfqpoint{4.112500in}{0.546741in}}%
\pgfpathlineto{\pgfqpoint{4.015625in}{0.545520in}}%
\pgfpathlineto{\pgfqpoint{3.918750in}{0.546153in}}%
\pgfpathlineto{\pgfqpoint{3.821875in}{0.548336in}}%
\pgfpathlineto{\pgfqpoint{3.725000in}{0.548463in}}%
\pgfpathlineto{\pgfqpoint{3.628125in}{0.545934in}}%
\pgfpathlineto{\pgfqpoint{3.531250in}{0.547672in}}%
\pgfpathlineto{\pgfqpoint{3.434375in}{0.547502in}}%
\pgfpathlineto{\pgfqpoint{3.337500in}{0.550157in}}%
\pgfpathlineto{\pgfqpoint{3.240625in}{0.548016in}}%
\pgfpathlineto{\pgfqpoint{3.143750in}{0.552995in}}%
\pgfpathlineto{\pgfqpoint{3.046875in}{0.548626in}}%
\pgfpathlineto{\pgfqpoint{2.950000in}{0.547073in}}%
\pgfpathlineto{\pgfqpoint{2.853125in}{0.545702in}}%
\pgfpathlineto{\pgfqpoint{2.756250in}{0.547671in}}%
\pgfpathlineto{\pgfqpoint{2.659375in}{0.546835in}}%
\pgfpathlineto{\pgfqpoint{2.562500in}{0.548097in}}%
\pgfpathlineto{\pgfqpoint{2.465625in}{0.547043in}}%
\pgfpathlineto{\pgfqpoint{2.368750in}{0.548876in}}%
\pgfpathlineto{\pgfqpoint{2.271875in}{0.549538in}}%
\pgfpathlineto{\pgfqpoint{2.175000in}{0.552241in}}%
\pgfpathlineto{\pgfqpoint{2.078125in}{0.550157in}}%
\pgfpathlineto{\pgfqpoint{1.981250in}{0.548294in}}%
\pgfpathlineto{\pgfqpoint{1.884375in}{0.550577in}}%
\pgfpathlineto{\pgfqpoint{1.787500in}{0.551446in}}%
\pgfpathlineto{\pgfqpoint{1.690625in}{0.554029in}}%
\pgfpathlineto{\pgfqpoint{1.593750in}{0.554626in}}%
\pgfpathlineto{\pgfqpoint{1.496875in}{0.566434in}}%
\pgfpathlineto{\pgfqpoint{1.400000in}{0.571405in}}%
\pgfpathlineto{\pgfqpoint{1.303125in}{0.596821in}}%
\pgfpathlineto{\pgfqpoint{1.206250in}{0.637292in}}%
\pgfpathlineto{\pgfqpoint{1.109375in}{0.674550in}}%
\pgfpathlineto{\pgfqpoint{1.012500in}{0.798805in}}%
\pgfpathlineto{\pgfqpoint{0.915625in}{1.016065in}}%
\pgfpathlineto{\pgfqpoint{0.818750in}{1.537821in}}%
\pgfpathlineto{\pgfqpoint{0.721875in}{1.723376in}}%
\pgfpathlineto{\pgfqpoint{0.721875in}{1.723376in}}%
\pgfpathclose%
\pgfusepath{stroke,fill}%
}%
\begin{pgfscope}%
\pgfsys@transformshift{0.000000in}{0.000000in}%
\pgfsys@useobject{currentmarker}{}%
\end{pgfscope}%
\end{pgfscope}%
\begin{pgfscope}%
\pgfsetbuttcap%
\pgfsetroundjoin%
\definecolor{currentfill}{rgb}{0.000000,0.000000,0.000000}%
\pgfsetfillcolor{currentfill}%
\pgfsetlinewidth{0.803000pt}%
\definecolor{currentstroke}{rgb}{0.000000,0.000000,0.000000}%
\pgfsetstrokecolor{currentstroke}%
\pgfsetdash{}{0pt}%
\pgfsys@defobject{currentmarker}{\pgfqpoint{0.000000in}{-0.048611in}}{\pgfqpoint{0.000000in}{0.000000in}}{%
\pgfpathmoveto{\pgfqpoint{0.000000in}{0.000000in}}%
\pgfpathlineto{\pgfqpoint{0.000000in}{-0.048611in}}%
\pgfusepath{stroke,fill}%
}%
\begin{pgfscope}%
\pgfsys@transformshift{0.625000in}{0.412500in}%
\pgfsys@useobject{currentmarker}{}%
\end{pgfscope}%
\end{pgfscope}%
\begin{pgfscope}%
\definecolor{textcolor}{rgb}{0.000000,0.000000,0.000000}%
\pgfsetstrokecolor{textcolor}%
\pgfsetfillcolor{textcolor}%
\pgftext[x=0.625000in,y=0.315278in,,top]{\color{textcolor}\rmfamily\fontsize{9.000000}{10.800000}\selectfont \(\displaystyle {0}\)}%
\end{pgfscope}%
\begin{pgfscope}%
\pgfsetbuttcap%
\pgfsetroundjoin%
\definecolor{currentfill}{rgb}{0.000000,0.000000,0.000000}%
\pgfsetfillcolor{currentfill}%
\pgfsetlinewidth{0.803000pt}%
\definecolor{currentstroke}{rgb}{0.000000,0.000000,0.000000}%
\pgfsetstrokecolor{currentstroke}%
\pgfsetdash{}{0pt}%
\pgfsys@defobject{currentmarker}{\pgfqpoint{0.000000in}{-0.048611in}}{\pgfqpoint{0.000000in}{0.000000in}}{%
\pgfpathmoveto{\pgfqpoint{0.000000in}{0.000000in}}%
\pgfpathlineto{\pgfqpoint{0.000000in}{-0.048611in}}%
\pgfusepath{stroke,fill}%
}%
\begin{pgfscope}%
\pgfsys@transformshift{1.109375in}{0.412500in}%
\pgfsys@useobject{currentmarker}{}%
\end{pgfscope}%
\end{pgfscope}%
\begin{pgfscope}%
\definecolor{textcolor}{rgb}{0.000000,0.000000,0.000000}%
\pgfsetstrokecolor{textcolor}%
\pgfsetfillcolor{textcolor}%
\pgftext[x=1.109375in,y=0.315278in,,top]{\color{textcolor}\rmfamily\fontsize{9.000000}{10.800000}\selectfont \(\displaystyle {25}\)}%
\end{pgfscope}%
\begin{pgfscope}%
\pgfsetbuttcap%
\pgfsetroundjoin%
\definecolor{currentfill}{rgb}{0.000000,0.000000,0.000000}%
\pgfsetfillcolor{currentfill}%
\pgfsetlinewidth{0.803000pt}%
\definecolor{currentstroke}{rgb}{0.000000,0.000000,0.000000}%
\pgfsetstrokecolor{currentstroke}%
\pgfsetdash{}{0pt}%
\pgfsys@defobject{currentmarker}{\pgfqpoint{0.000000in}{-0.048611in}}{\pgfqpoint{0.000000in}{0.000000in}}{%
\pgfpathmoveto{\pgfqpoint{0.000000in}{0.000000in}}%
\pgfpathlineto{\pgfqpoint{0.000000in}{-0.048611in}}%
\pgfusepath{stroke,fill}%
}%
\begin{pgfscope}%
\pgfsys@transformshift{1.593750in}{0.412500in}%
\pgfsys@useobject{currentmarker}{}%
\end{pgfscope}%
\end{pgfscope}%
\begin{pgfscope}%
\definecolor{textcolor}{rgb}{0.000000,0.000000,0.000000}%
\pgfsetstrokecolor{textcolor}%
\pgfsetfillcolor{textcolor}%
\pgftext[x=1.593750in,y=0.315278in,,top]{\color{textcolor}\rmfamily\fontsize{9.000000}{10.800000}\selectfont \(\displaystyle {50}\)}%
\end{pgfscope}%
\begin{pgfscope}%
\pgfsetbuttcap%
\pgfsetroundjoin%
\definecolor{currentfill}{rgb}{0.000000,0.000000,0.000000}%
\pgfsetfillcolor{currentfill}%
\pgfsetlinewidth{0.803000pt}%
\definecolor{currentstroke}{rgb}{0.000000,0.000000,0.000000}%
\pgfsetstrokecolor{currentstroke}%
\pgfsetdash{}{0pt}%
\pgfsys@defobject{currentmarker}{\pgfqpoint{0.000000in}{-0.048611in}}{\pgfqpoint{0.000000in}{0.000000in}}{%
\pgfpathmoveto{\pgfqpoint{0.000000in}{0.000000in}}%
\pgfpathlineto{\pgfqpoint{0.000000in}{-0.048611in}}%
\pgfusepath{stroke,fill}%
}%
\begin{pgfscope}%
\pgfsys@transformshift{2.078125in}{0.412500in}%
\pgfsys@useobject{currentmarker}{}%
\end{pgfscope}%
\end{pgfscope}%
\begin{pgfscope}%
\definecolor{textcolor}{rgb}{0.000000,0.000000,0.000000}%
\pgfsetstrokecolor{textcolor}%
\pgfsetfillcolor{textcolor}%
\pgftext[x=2.078125in,y=0.315278in,,top]{\color{textcolor}\rmfamily\fontsize{9.000000}{10.800000}\selectfont \(\displaystyle {75}\)}%
\end{pgfscope}%
\begin{pgfscope}%
\pgfsetbuttcap%
\pgfsetroundjoin%
\definecolor{currentfill}{rgb}{0.000000,0.000000,0.000000}%
\pgfsetfillcolor{currentfill}%
\pgfsetlinewidth{0.803000pt}%
\definecolor{currentstroke}{rgb}{0.000000,0.000000,0.000000}%
\pgfsetstrokecolor{currentstroke}%
\pgfsetdash{}{0pt}%
\pgfsys@defobject{currentmarker}{\pgfqpoint{0.000000in}{-0.048611in}}{\pgfqpoint{0.000000in}{0.000000in}}{%
\pgfpathmoveto{\pgfqpoint{0.000000in}{0.000000in}}%
\pgfpathlineto{\pgfqpoint{0.000000in}{-0.048611in}}%
\pgfusepath{stroke,fill}%
}%
\begin{pgfscope}%
\pgfsys@transformshift{2.562500in}{0.412500in}%
\pgfsys@useobject{currentmarker}{}%
\end{pgfscope}%
\end{pgfscope}%
\begin{pgfscope}%
\definecolor{textcolor}{rgb}{0.000000,0.000000,0.000000}%
\pgfsetstrokecolor{textcolor}%
\pgfsetfillcolor{textcolor}%
\pgftext[x=2.562500in,y=0.315278in,,top]{\color{textcolor}\rmfamily\fontsize{9.000000}{10.800000}\selectfont \(\displaystyle {100}\)}%
\end{pgfscope}%
\begin{pgfscope}%
\pgfsetbuttcap%
\pgfsetroundjoin%
\definecolor{currentfill}{rgb}{0.000000,0.000000,0.000000}%
\pgfsetfillcolor{currentfill}%
\pgfsetlinewidth{0.803000pt}%
\definecolor{currentstroke}{rgb}{0.000000,0.000000,0.000000}%
\pgfsetstrokecolor{currentstroke}%
\pgfsetdash{}{0pt}%
\pgfsys@defobject{currentmarker}{\pgfqpoint{0.000000in}{-0.048611in}}{\pgfqpoint{0.000000in}{0.000000in}}{%
\pgfpathmoveto{\pgfqpoint{0.000000in}{0.000000in}}%
\pgfpathlineto{\pgfqpoint{0.000000in}{-0.048611in}}%
\pgfusepath{stroke,fill}%
}%
\begin{pgfscope}%
\pgfsys@transformshift{3.046875in}{0.412500in}%
\pgfsys@useobject{currentmarker}{}%
\end{pgfscope}%
\end{pgfscope}%
\begin{pgfscope}%
\definecolor{textcolor}{rgb}{0.000000,0.000000,0.000000}%
\pgfsetstrokecolor{textcolor}%
\pgfsetfillcolor{textcolor}%
\pgftext[x=3.046875in,y=0.315278in,,top]{\color{textcolor}\rmfamily\fontsize{9.000000}{10.800000}\selectfont \(\displaystyle {125}\)}%
\end{pgfscope}%
\begin{pgfscope}%
\pgfsetbuttcap%
\pgfsetroundjoin%
\definecolor{currentfill}{rgb}{0.000000,0.000000,0.000000}%
\pgfsetfillcolor{currentfill}%
\pgfsetlinewidth{0.803000pt}%
\definecolor{currentstroke}{rgb}{0.000000,0.000000,0.000000}%
\pgfsetstrokecolor{currentstroke}%
\pgfsetdash{}{0pt}%
\pgfsys@defobject{currentmarker}{\pgfqpoint{0.000000in}{-0.048611in}}{\pgfqpoint{0.000000in}{0.000000in}}{%
\pgfpathmoveto{\pgfqpoint{0.000000in}{0.000000in}}%
\pgfpathlineto{\pgfqpoint{0.000000in}{-0.048611in}}%
\pgfusepath{stroke,fill}%
}%
\begin{pgfscope}%
\pgfsys@transformshift{3.531250in}{0.412500in}%
\pgfsys@useobject{currentmarker}{}%
\end{pgfscope}%
\end{pgfscope}%
\begin{pgfscope}%
\definecolor{textcolor}{rgb}{0.000000,0.000000,0.000000}%
\pgfsetstrokecolor{textcolor}%
\pgfsetfillcolor{textcolor}%
\pgftext[x=3.531250in,y=0.315278in,,top]{\color{textcolor}\rmfamily\fontsize{9.000000}{10.800000}\selectfont \(\displaystyle {150}\)}%
\end{pgfscope}%
\begin{pgfscope}%
\pgfsetbuttcap%
\pgfsetroundjoin%
\definecolor{currentfill}{rgb}{0.000000,0.000000,0.000000}%
\pgfsetfillcolor{currentfill}%
\pgfsetlinewidth{0.803000pt}%
\definecolor{currentstroke}{rgb}{0.000000,0.000000,0.000000}%
\pgfsetstrokecolor{currentstroke}%
\pgfsetdash{}{0pt}%
\pgfsys@defobject{currentmarker}{\pgfqpoint{0.000000in}{-0.048611in}}{\pgfqpoint{0.000000in}{0.000000in}}{%
\pgfpathmoveto{\pgfqpoint{0.000000in}{0.000000in}}%
\pgfpathlineto{\pgfqpoint{0.000000in}{-0.048611in}}%
\pgfusepath{stroke,fill}%
}%
\begin{pgfscope}%
\pgfsys@transformshift{4.015625in}{0.412500in}%
\pgfsys@useobject{currentmarker}{}%
\end{pgfscope}%
\end{pgfscope}%
\begin{pgfscope}%
\definecolor{textcolor}{rgb}{0.000000,0.000000,0.000000}%
\pgfsetstrokecolor{textcolor}%
\pgfsetfillcolor{textcolor}%
\pgftext[x=4.015625in,y=0.315278in,,top]{\color{textcolor}\rmfamily\fontsize{9.000000}{10.800000}\selectfont \(\displaystyle {175}\)}%
\end{pgfscope}%
\begin{pgfscope}%
\pgfsetbuttcap%
\pgfsetroundjoin%
\definecolor{currentfill}{rgb}{0.000000,0.000000,0.000000}%
\pgfsetfillcolor{currentfill}%
\pgfsetlinewidth{0.803000pt}%
\definecolor{currentstroke}{rgb}{0.000000,0.000000,0.000000}%
\pgfsetstrokecolor{currentstroke}%
\pgfsetdash{}{0pt}%
\pgfsys@defobject{currentmarker}{\pgfqpoint{0.000000in}{-0.048611in}}{\pgfqpoint{0.000000in}{0.000000in}}{%
\pgfpathmoveto{\pgfqpoint{0.000000in}{0.000000in}}%
\pgfpathlineto{\pgfqpoint{0.000000in}{-0.048611in}}%
\pgfusepath{stroke,fill}%
}%
\begin{pgfscope}%
\pgfsys@transformshift{4.500000in}{0.412500in}%
\pgfsys@useobject{currentmarker}{}%
\end{pgfscope}%
\end{pgfscope}%
\begin{pgfscope}%
\definecolor{textcolor}{rgb}{0.000000,0.000000,0.000000}%
\pgfsetstrokecolor{textcolor}%
\pgfsetfillcolor{textcolor}%
\pgftext[x=4.500000in,y=0.315278in,,top]{\color{textcolor}\rmfamily\fontsize{9.000000}{10.800000}\selectfont \(\displaystyle {200}\)}%
\end{pgfscope}%
\begin{pgfscope}%
\definecolor{textcolor}{rgb}{0.000000,0.000000,0.000000}%
\pgfsetstrokecolor{textcolor}%
\pgfsetfillcolor{textcolor}%
\pgftext[x=2.562500in,y=0.148611in,,top]{\color{textcolor}\rmfamily\fontsize{11.000000}{13.200000}\selectfont Episode}%
\end{pgfscope}%
\begin{pgfscope}%
\pgfsetbuttcap%
\pgfsetroundjoin%
\definecolor{currentfill}{rgb}{0.000000,0.000000,0.000000}%
\pgfsetfillcolor{currentfill}%
\pgfsetlinewidth{0.803000pt}%
\definecolor{currentstroke}{rgb}{0.000000,0.000000,0.000000}%
\pgfsetstrokecolor{currentstroke}%
\pgfsetdash{}{0pt}%
\pgfsys@defobject{currentmarker}{\pgfqpoint{-0.048611in}{0.000000in}}{\pgfqpoint{-0.000000in}{0.000000in}}{%
\pgfpathmoveto{\pgfqpoint{-0.000000in}{0.000000in}}%
\pgfpathlineto{\pgfqpoint{-0.048611in}{0.000000in}}%
\pgfusepath{stroke,fill}%
}%
\begin{pgfscope}%
\pgfsys@transformshift{0.625000in}{0.450719in}%
\pgfsys@useobject{currentmarker}{}%
\end{pgfscope}%
\end{pgfscope}%
\begin{pgfscope}%
\definecolor{textcolor}{rgb}{0.000000,0.000000,0.000000}%
\pgfsetstrokecolor{textcolor}%
\pgfsetfillcolor{textcolor}%
\pgftext[x=0.463542in, y=0.407317in, left, base]{\color{textcolor}\rmfamily\fontsize{9.000000}{10.800000}\selectfont \(\displaystyle {0}\)}%
\end{pgfscope}%
\begin{pgfscope}%
\pgfsetbuttcap%
\pgfsetroundjoin%
\definecolor{currentfill}{rgb}{0.000000,0.000000,0.000000}%
\pgfsetfillcolor{currentfill}%
\pgfsetlinewidth{0.803000pt}%
\definecolor{currentstroke}{rgb}{0.000000,0.000000,0.000000}%
\pgfsetstrokecolor{currentstroke}%
\pgfsetdash{}{0pt}%
\pgfsys@defobject{currentmarker}{\pgfqpoint{-0.048611in}{0.000000in}}{\pgfqpoint{-0.000000in}{0.000000in}}{%
\pgfpathmoveto{\pgfqpoint{-0.000000in}{0.000000in}}%
\pgfpathlineto{\pgfqpoint{-0.048611in}{0.000000in}}%
\pgfusepath{stroke,fill}%
}%
\begin{pgfscope}%
\pgfsys@transformshift{0.625000in}{0.869911in}%
\pgfsys@useobject{currentmarker}{}%
\end{pgfscope}%
\end{pgfscope}%
\begin{pgfscope}%
\definecolor{textcolor}{rgb}{0.000000,0.000000,0.000000}%
\pgfsetstrokecolor{textcolor}%
\pgfsetfillcolor{textcolor}%
\pgftext[x=0.335071in, y=0.826509in, left, base]{\color{textcolor}\rmfamily\fontsize{9.000000}{10.800000}\selectfont \(\displaystyle {100}\)}%
\end{pgfscope}%
\begin{pgfscope}%
\pgfsetbuttcap%
\pgfsetroundjoin%
\definecolor{currentfill}{rgb}{0.000000,0.000000,0.000000}%
\pgfsetfillcolor{currentfill}%
\pgfsetlinewidth{0.803000pt}%
\definecolor{currentstroke}{rgb}{0.000000,0.000000,0.000000}%
\pgfsetstrokecolor{currentstroke}%
\pgfsetdash{}{0pt}%
\pgfsys@defobject{currentmarker}{\pgfqpoint{-0.048611in}{0.000000in}}{\pgfqpoint{-0.000000in}{0.000000in}}{%
\pgfpathmoveto{\pgfqpoint{-0.000000in}{0.000000in}}%
\pgfpathlineto{\pgfqpoint{-0.048611in}{0.000000in}}%
\pgfusepath{stroke,fill}%
}%
\begin{pgfscope}%
\pgfsys@transformshift{0.625000in}{1.289104in}%
\pgfsys@useobject{currentmarker}{}%
\end{pgfscope}%
\end{pgfscope}%
\begin{pgfscope}%
\definecolor{textcolor}{rgb}{0.000000,0.000000,0.000000}%
\pgfsetstrokecolor{textcolor}%
\pgfsetfillcolor{textcolor}%
\pgftext[x=0.335071in, y=1.245701in, left, base]{\color{textcolor}\rmfamily\fontsize{9.000000}{10.800000}\selectfont \(\displaystyle {200}\)}%
\end{pgfscope}%
\begin{pgfscope}%
\pgfsetbuttcap%
\pgfsetroundjoin%
\definecolor{currentfill}{rgb}{0.000000,0.000000,0.000000}%
\pgfsetfillcolor{currentfill}%
\pgfsetlinewidth{0.803000pt}%
\definecolor{currentstroke}{rgb}{0.000000,0.000000,0.000000}%
\pgfsetstrokecolor{currentstroke}%
\pgfsetdash{}{0pt}%
\pgfsys@defobject{currentmarker}{\pgfqpoint{-0.048611in}{0.000000in}}{\pgfqpoint{-0.000000in}{0.000000in}}{%
\pgfpathmoveto{\pgfqpoint{-0.000000in}{0.000000in}}%
\pgfpathlineto{\pgfqpoint{-0.048611in}{0.000000in}}%
\pgfusepath{stroke,fill}%
}%
\begin{pgfscope}%
\pgfsys@transformshift{0.625000in}{1.708296in}%
\pgfsys@useobject{currentmarker}{}%
\end{pgfscope}%
\end{pgfscope}%
\begin{pgfscope}%
\definecolor{textcolor}{rgb}{0.000000,0.000000,0.000000}%
\pgfsetstrokecolor{textcolor}%
\pgfsetfillcolor{textcolor}%
\pgftext[x=0.335071in, y=1.664893in, left, base]{\color{textcolor}\rmfamily\fontsize{9.000000}{10.800000}\selectfont \(\displaystyle {300}\)}%
\end{pgfscope}%
\begin{pgfscope}%
\pgfsetbuttcap%
\pgfsetroundjoin%
\definecolor{currentfill}{rgb}{0.000000,0.000000,0.000000}%
\pgfsetfillcolor{currentfill}%
\pgfsetlinewidth{0.803000pt}%
\definecolor{currentstroke}{rgb}{0.000000,0.000000,0.000000}%
\pgfsetstrokecolor{currentstroke}%
\pgfsetdash{}{0pt}%
\pgfsys@defobject{currentmarker}{\pgfqpoint{-0.048611in}{0.000000in}}{\pgfqpoint{-0.000000in}{0.000000in}}{%
\pgfpathmoveto{\pgfqpoint{-0.000000in}{0.000000in}}%
\pgfpathlineto{\pgfqpoint{-0.048611in}{0.000000in}}%
\pgfusepath{stroke,fill}%
}%
\begin{pgfscope}%
\pgfsys@transformshift{0.625000in}{2.127488in}%
\pgfsys@useobject{currentmarker}{}%
\end{pgfscope}%
\end{pgfscope}%
\begin{pgfscope}%
\definecolor{textcolor}{rgb}{0.000000,0.000000,0.000000}%
\pgfsetstrokecolor{textcolor}%
\pgfsetfillcolor{textcolor}%
\pgftext[x=0.335071in, y=2.084085in, left, base]{\color{textcolor}\rmfamily\fontsize{9.000000}{10.800000}\selectfont \(\displaystyle {400}\)}%
\end{pgfscope}%
\begin{pgfscope}%
\pgfsetbuttcap%
\pgfsetroundjoin%
\definecolor{currentfill}{rgb}{0.000000,0.000000,0.000000}%
\pgfsetfillcolor{currentfill}%
\pgfsetlinewidth{0.803000pt}%
\definecolor{currentstroke}{rgb}{0.000000,0.000000,0.000000}%
\pgfsetstrokecolor{currentstroke}%
\pgfsetdash{}{0pt}%
\pgfsys@defobject{currentmarker}{\pgfqpoint{-0.048611in}{0.000000in}}{\pgfqpoint{-0.000000in}{0.000000in}}{%
\pgfpathmoveto{\pgfqpoint{-0.000000in}{0.000000in}}%
\pgfpathlineto{\pgfqpoint{-0.048611in}{0.000000in}}%
\pgfusepath{stroke,fill}%
}%
\begin{pgfscope}%
\pgfsys@transformshift{0.625000in}{2.546680in}%
\pgfsys@useobject{currentmarker}{}%
\end{pgfscope}%
\end{pgfscope}%
\begin{pgfscope}%
\definecolor{textcolor}{rgb}{0.000000,0.000000,0.000000}%
\pgfsetstrokecolor{textcolor}%
\pgfsetfillcolor{textcolor}%
\pgftext[x=0.335071in, y=2.503277in, left, base]{\color{textcolor}\rmfamily\fontsize{9.000000}{10.800000}\selectfont \(\displaystyle {500}\)}%
\end{pgfscope}%
\begin{pgfscope}%
\pgfsetbuttcap%
\pgfsetroundjoin%
\definecolor{currentfill}{rgb}{0.000000,0.000000,0.000000}%
\pgfsetfillcolor{currentfill}%
\pgfsetlinewidth{0.803000pt}%
\definecolor{currentstroke}{rgb}{0.000000,0.000000,0.000000}%
\pgfsetstrokecolor{currentstroke}%
\pgfsetdash{}{0pt}%
\pgfsys@defobject{currentmarker}{\pgfqpoint{-0.048611in}{0.000000in}}{\pgfqpoint{-0.000000in}{0.000000in}}{%
\pgfpathmoveto{\pgfqpoint{-0.000000in}{0.000000in}}%
\pgfpathlineto{\pgfqpoint{-0.048611in}{0.000000in}}%
\pgfusepath{stroke,fill}%
}%
\begin{pgfscope}%
\pgfsys@transformshift{0.625000in}{2.965872in}%
\pgfsys@useobject{currentmarker}{}%
\end{pgfscope}%
\end{pgfscope}%
\begin{pgfscope}%
\definecolor{textcolor}{rgb}{0.000000,0.000000,0.000000}%
\pgfsetstrokecolor{textcolor}%
\pgfsetfillcolor{textcolor}%
\pgftext[x=0.335071in, y=2.922469in, left, base]{\color{textcolor}\rmfamily\fontsize{9.000000}{10.800000}\selectfont \(\displaystyle {600}\)}%
\end{pgfscope}%
\begin{pgfscope}%
\definecolor{textcolor}{rgb}{0.000000,0.000000,0.000000}%
\pgfsetstrokecolor{textcolor}%
\pgfsetfillcolor{textcolor}%
\pgftext[x=0.279515in,y=1.856250in,,bottom,rotate=90.000000]{\color{textcolor}\rmfamily\fontsize{11.000000}{13.200000}\selectfont Steps to Goal}%
\end{pgfscope}%
\begin{pgfscope}%
\pgfpathrectangle{\pgfqpoint{0.625000in}{0.412500in}}{\pgfqpoint{3.875000in}{2.887500in}}%
\pgfusepath{clip}%
\pgfsetrectcap%
\pgfsetroundjoin%
\pgfsetlinewidth{1.003750pt}%
\definecolor{currentstroke}{rgb}{0.000000,0.466667,0.733333}%
\pgfsetstrokecolor{currentstroke}%
\pgfsetdash{}{0pt}%
\pgfpathmoveto{\pgfqpoint{0.721875in}{3.050255in}}%
\pgfpathlineto{\pgfqpoint{0.818750in}{2.684569in}}%
\pgfpathlineto{\pgfqpoint{0.915625in}{2.194114in}}%
\pgfpathlineto{\pgfqpoint{1.012500in}{1.809086in}}%
\pgfpathlineto{\pgfqpoint{1.109375in}{1.651872in}}%
\pgfpathlineto{\pgfqpoint{1.206250in}{1.494399in}}%
\pgfpathlineto{\pgfqpoint{1.303125in}{1.373797in}}%
\pgfpathlineto{\pgfqpoint{1.400000in}{1.224598in}}%
\pgfpathlineto{\pgfqpoint{1.496875in}{1.001538in}}%
\pgfpathlineto{\pgfqpoint{1.593750in}{0.927047in}}%
\pgfpathlineto{\pgfqpoint{1.690625in}{0.860656in}}%
\pgfpathlineto{\pgfqpoint{1.787500in}{0.800451in}}%
\pgfpathlineto{\pgfqpoint{1.884375in}{0.750257in}}%
\pgfpathlineto{\pgfqpoint{1.981250in}{0.699678in}}%
\pgfpathlineto{\pgfqpoint{2.078125in}{0.644218in}}%
\pgfpathlineto{\pgfqpoint{2.175000in}{0.612343in}}%
\pgfpathlineto{\pgfqpoint{2.271875in}{0.593144in}}%
\pgfpathlineto{\pgfqpoint{2.368750in}{0.586026in}}%
\pgfpathlineto{\pgfqpoint{2.465625in}{0.576737in}}%
\pgfpathlineto{\pgfqpoint{2.562500in}{0.574054in}}%
\pgfpathlineto{\pgfqpoint{2.659375in}{0.569191in}}%
\pgfpathlineto{\pgfqpoint{2.756250in}{0.563348in}}%
\pgfpathlineto{\pgfqpoint{2.853125in}{0.567037in}}%
\pgfpathlineto{\pgfqpoint{2.950000in}{0.556289in}}%
\pgfpathlineto{\pgfqpoint{3.046875in}{0.565008in}}%
\pgfpathlineto{\pgfqpoint{3.143750in}{0.562090in}}%
\pgfpathlineto{\pgfqpoint{3.240625in}{0.561026in}}%
\pgfpathlineto{\pgfqpoint{3.337500in}{0.566324in}}%
\pgfpathlineto{\pgfqpoint{3.434375in}{0.557806in}}%
\pgfpathlineto{\pgfqpoint{3.531250in}{0.562065in}}%
\pgfpathlineto{\pgfqpoint{3.628125in}{0.557068in}}%
\pgfpathlineto{\pgfqpoint{3.725000in}{0.549531in}}%
\pgfpathlineto{\pgfqpoint{3.821875in}{0.551007in}}%
\pgfpathlineto{\pgfqpoint{3.918750in}{0.552801in}}%
\pgfpathlineto{\pgfqpoint{4.015625in}{0.554713in}}%
\pgfpathlineto{\pgfqpoint{4.112500in}{0.552164in}}%
\pgfpathlineto{\pgfqpoint{4.209375in}{0.553187in}}%
\pgfpathlineto{\pgfqpoint{4.306250in}{0.549833in}}%
\pgfpathlineto{\pgfqpoint{4.403125in}{0.547804in}}%
\pgfpathlineto{\pgfqpoint{4.500000in}{0.547611in}}%
\pgfusepath{stroke}%
\end{pgfscope}%
\begin{pgfscope}%
\pgfpathrectangle{\pgfqpoint{0.625000in}{0.412500in}}{\pgfqpoint{3.875000in}{2.887500in}}%
\pgfusepath{clip}%
\pgfsetrectcap%
\pgfsetroundjoin%
\pgfsetlinewidth{1.003750pt}%
\definecolor{currentstroke}{rgb}{0.933333,0.200000,0.466667}%
\pgfsetstrokecolor{currentstroke}%
\pgfsetdash{}{0pt}%
\pgfpathmoveto{\pgfqpoint{0.721875in}{1.644369in}}%
\pgfpathlineto{\pgfqpoint{0.818750in}{1.506455in}}%
\pgfpathlineto{\pgfqpoint{0.915625in}{0.990220in}}%
\pgfpathlineto{\pgfqpoint{1.012500in}{0.778469in}}%
\pgfpathlineto{\pgfqpoint{1.109375in}{0.660441in}}%
\pgfpathlineto{\pgfqpoint{1.206250in}{0.623971in}}%
\pgfpathlineto{\pgfqpoint{1.303125in}{0.588147in}}%
\pgfpathlineto{\pgfqpoint{1.400000in}{0.566559in}}%
\pgfpathlineto{\pgfqpoint{1.496875in}{0.561772in}}%
\pgfpathlineto{\pgfqpoint{1.593750in}{0.552860in}}%
\pgfpathlineto{\pgfqpoint{1.690625in}{0.552181in}}%
\pgfpathlineto{\pgfqpoint{1.787500in}{0.549951in}}%
\pgfpathlineto{\pgfqpoint{1.884375in}{0.549347in}}%
\pgfpathlineto{\pgfqpoint{1.981250in}{0.547150in}}%
\pgfpathlineto{\pgfqpoint{2.078125in}{0.549146in}}%
\pgfpathlineto{\pgfqpoint{2.175000in}{0.550512in}}%
\pgfpathlineto{\pgfqpoint{2.271875in}{0.548383in}}%
\pgfpathlineto{\pgfqpoint{2.368750in}{0.547788in}}%
\pgfpathlineto{\pgfqpoint{2.465625in}{0.546035in}}%
\pgfpathlineto{\pgfqpoint{2.562500in}{0.547083in}}%
\pgfpathlineto{\pgfqpoint{2.659375in}{0.545876in}}%
\pgfpathlineto{\pgfqpoint{2.756250in}{0.546639in}}%
\pgfpathlineto{\pgfqpoint{2.853125in}{0.544769in}}%
\pgfpathlineto{\pgfqpoint{2.950000in}{0.546077in}}%
\pgfpathlineto{\pgfqpoint{3.046875in}{0.547670in}}%
\pgfpathlineto{\pgfqpoint{3.143750in}{0.549858in}}%
\pgfpathlineto{\pgfqpoint{3.240625in}{0.547058in}}%
\pgfpathlineto{\pgfqpoint{3.337500in}{0.548810in}}%
\pgfpathlineto{\pgfqpoint{3.434375in}{0.546362in}}%
\pgfpathlineto{\pgfqpoint{3.531250in}{0.546723in}}%
\pgfpathlineto{\pgfqpoint{3.628125in}{0.545038in}}%
\pgfpathlineto{\pgfqpoint{3.725000in}{0.547293in}}%
\pgfpathlineto{\pgfqpoint{3.821875in}{0.547209in}}%
\pgfpathlineto{\pgfqpoint{3.918750in}{0.545247in}}%
\pgfpathlineto{\pgfqpoint{4.015625in}{0.544635in}}%
\pgfpathlineto{\pgfqpoint{4.112500in}{0.545851in}}%
\pgfpathlineto{\pgfqpoint{4.209375in}{0.548265in}}%
\pgfpathlineto{\pgfqpoint{4.306250in}{0.546304in}}%
\pgfpathlineto{\pgfqpoint{4.403125in}{0.545876in}}%
\pgfpathlineto{\pgfqpoint{4.500000in}{0.547326in}}%
\pgfusepath{stroke}%
\end{pgfscope}%
\begin{pgfscope}%
\pgfpathrectangle{\pgfqpoint{0.625000in}{0.412500in}}{\pgfqpoint{3.875000in}{2.887500in}}%
\pgfusepath{clip}%
\pgfsetbuttcap%
\pgfsetroundjoin%
\pgfsetlinewidth{1.003750pt}%
\definecolor{currentstroke}{rgb}{0.501961,0.000000,0.501961}%
\pgfsetstrokecolor{currentstroke}%
\pgfsetdash{{3.700000pt}{1.600000pt}}{0.000000pt}%
\pgfpathmoveto{\pgfqpoint{0.625000in}{0.529912in}}%
\pgfpathlineto{\pgfqpoint{4.500000in}{0.529912in}}%
\pgfusepath{stroke}%
\end{pgfscope}%
\begin{pgfscope}%
\pgfsetrectcap%
\pgfsetmiterjoin%
\pgfsetlinewidth{0.803000pt}%
\definecolor{currentstroke}{rgb}{0.000000,0.000000,0.000000}%
\pgfsetstrokecolor{currentstroke}%
\pgfsetdash{}{0pt}%
\pgfpathmoveto{\pgfqpoint{0.625000in}{0.412500in}}%
\pgfpathlineto{\pgfqpoint{0.625000in}{3.300000in}}%
\pgfusepath{stroke}%
\end{pgfscope}%
\begin{pgfscope}%
\pgfsetrectcap%
\pgfsetmiterjoin%
\pgfsetlinewidth{0.803000pt}%
\definecolor{currentstroke}{rgb}{0.000000,0.000000,0.000000}%
\pgfsetstrokecolor{currentstroke}%
\pgfsetdash{}{0pt}%
\pgfpathmoveto{\pgfqpoint{0.625000in}{0.412500in}}%
\pgfpathlineto{\pgfqpoint{4.500000in}{0.412500in}}%
\pgfusepath{stroke}%
\end{pgfscope}%
\begin{pgfscope}%
\definecolor{textcolor}{rgb}{0.000000,0.466667,0.733333}%
\pgfsetstrokecolor{textcolor}%
\pgfsetfillcolor{textcolor}%
\pgftext[x=0.915625in,y=2.546680in,left,base]{\color{textcolor}\rmfamily\fontsize{14.000000}{16.800000}\selectfont Sarsa(\(\displaystyle \lambda\))}%
\end{pgfscope}%
\begin{pgfscope}%
\definecolor{textcolor}{rgb}{0.933333,0.200000,0.466667}%
\pgfsetstrokecolor{textcolor}%
\pgfsetfillcolor{textcolor}%
\pgftext[x=0.625000in,y=0.618396in,left,base]{\color{textcolor}\rmfamily\fontsize{14.000000}{16.800000}\selectfont GSP}%
\end{pgfscope}%
\begin{pgfscope}%
\definecolor{textcolor}{rgb}{0.501961,0.000000,0.501961}%
\pgfsetstrokecolor{textcolor}%
\pgfsetfillcolor{textcolor}%
\pgftext[x=4.500000in, y=0.677305in, left, base]{\color{textcolor}\rmfamily\fontsize{10.000000}{12.000000}\selectfont shortest }%
\end{pgfscope}%
\begin{pgfscope}%
\definecolor{textcolor}{rgb}{0.501961,0.000000,0.501961}%
\pgfsetstrokecolor{textcolor}%
\pgfsetfillcolor{textcolor}%
\pgftext[x=4.500000in, y=0.534558in, left, base]{\color{textcolor}\rmfamily\fontsize{10.000000}{12.000000}\selectfont  path}%
\end{pgfscope}%
\end{pgfpicture}%
\makeatother%
\endgroup%

%% file: gridball_stepstogoal_noLAVI.pgf
\begingroup%
\makeatletter%
\begin{pgfpicture}%
\pgfpathrectangle{\pgfpointorigin}{\pgfqpoint{5.000000in}{3.750000in}}%
\pgfusepath{use as bounding box, clip}%
\begin{pgfscope}%
\pgfsetbuttcap%
\pgfsetmiterjoin%
\definecolor{currentfill}{rgb}{1.000000,1.000000,1.000000}%
\pgfsetfillcolor{currentfill}%
\pgfsetlinewidth{0.000000pt}%
\definecolor{currentstroke}{rgb}{1.000000,1.000000,1.000000}%
\pgfsetstrokecolor{currentstroke}%
\pgfsetdash{}{0pt}%
\pgfpathmoveto{\pgfqpoint{0.000000in}{0.000000in}}%
\pgfpathlineto{\pgfqpoint{5.000000in}{0.000000in}}%
\pgfpathlineto{\pgfqpoint{5.000000in}{3.750000in}}%
\pgfpathlineto{\pgfqpoint{0.000000in}{3.750000in}}%
\pgfpathlineto{\pgfqpoint{0.000000in}{0.000000in}}%
\pgfpathclose%
\pgfusepath{fill}%
\end{pgfscope}%
\begin{pgfscope}%
\pgfsetbuttcap%
\pgfsetmiterjoin%
\definecolor{currentfill}{rgb}{1.000000,1.000000,1.000000}%
\pgfsetfillcolor{currentfill}%
\pgfsetlinewidth{0.000000pt}%
\definecolor{currentstroke}{rgb}{0.000000,0.000000,0.000000}%
\pgfsetstrokecolor{currentstroke}%
\pgfsetstrokeopacity{0.000000}%
\pgfsetdash{}{0pt}%
\pgfpathmoveto{\pgfqpoint{0.625000in}{0.412500in}}%
\pgfpathlineto{\pgfqpoint{4.500000in}{0.412500in}}%
\pgfpathlineto{\pgfqpoint{4.500000in}{3.300000in}}%
\pgfpathlineto{\pgfqpoint{0.625000in}{3.300000in}}%
\pgfpathlineto{\pgfqpoint{0.625000in}{0.412500in}}%
\pgfpathclose%
\pgfusepath{fill}%
\end{pgfscope}%
\begin{pgfscope}%
\pgfpathrectangle{\pgfqpoint{0.625000in}{0.412500in}}{\pgfqpoint{3.875000in}{2.887500in}}%
\pgfusepath{clip}%
\pgfsetbuttcap%
\pgfsetroundjoin%
\definecolor{currentfill}{rgb}{0.000000,0.466667,0.733333}%
\pgfsetfillcolor{currentfill}%
\pgfsetfillopacity{0.200000}%
\pgfsetlinewidth{1.003750pt}%
\definecolor{currentstroke}{rgb}{0.000000,0.466667,0.733333}%
\pgfsetstrokecolor{currentstroke}%
\pgfsetstrokeopacity{0.200000}%
\pgfsetdash{}{0pt}%
\pgfsys@defobject{currentmarker}{\pgfqpoint{0.721875in}{0.543861in}}{\pgfqpoint{4.500000in}{3.168750in}}{%
\pgfpathmoveto{\pgfqpoint{0.721875in}{3.168750in}}%
\pgfpathlineto{\pgfqpoint{0.721875in}{3.168750in}}%
\pgfpathlineto{\pgfqpoint{0.818750in}{3.117398in}}%
\pgfpathlineto{\pgfqpoint{0.915625in}{3.017577in}}%
\pgfpathlineto{\pgfqpoint{1.012500in}{2.962925in}}%
\pgfpathlineto{\pgfqpoint{1.109375in}{2.882077in}}%
\pgfpathlineto{\pgfqpoint{1.206250in}{2.897356in}}%
\pgfpathlineto{\pgfqpoint{1.303125in}{2.754384in}}%
\pgfpathlineto{\pgfqpoint{1.400000in}{2.674130in}}%
\pgfpathlineto{\pgfqpoint{1.496875in}{2.546734in}}%
\pgfpathlineto{\pgfqpoint{1.593750in}{2.491497in}}%
\pgfpathlineto{\pgfqpoint{1.690625in}{2.370217in}}%
\pgfpathlineto{\pgfqpoint{1.787500in}{2.093955in}}%
\pgfpathlineto{\pgfqpoint{1.884375in}{1.913986in}}%
\pgfpathlineto{\pgfqpoint{1.981250in}{1.685232in}}%
\pgfpathlineto{\pgfqpoint{2.078125in}{1.494586in}}%
\pgfpathlineto{\pgfqpoint{2.175000in}{1.332395in}}%
\pgfpathlineto{\pgfqpoint{2.271875in}{1.016701in}}%
\pgfpathlineto{\pgfqpoint{2.368750in}{0.933863in}}%
\pgfpathlineto{\pgfqpoint{2.465625in}{0.693893in}}%
\pgfpathlineto{\pgfqpoint{2.562500in}{0.595992in}}%
\pgfpathlineto{\pgfqpoint{2.659375in}{0.574417in}}%
\pgfpathlineto{\pgfqpoint{2.756250in}{0.549974in}}%
\pgfpathlineto{\pgfqpoint{2.853125in}{0.546048in}}%
\pgfpathlineto{\pgfqpoint{2.950000in}{0.548960in}}%
\pgfpathlineto{\pgfqpoint{3.046875in}{0.548295in}}%
\pgfpathlineto{\pgfqpoint{3.143750in}{0.546896in}}%
\pgfpathlineto{\pgfqpoint{3.240625in}{0.543861in}}%
\pgfpathlineto{\pgfqpoint{3.337500in}{0.545422in}}%
\pgfpathlineto{\pgfqpoint{3.434375in}{0.545756in}}%
\pgfpathlineto{\pgfqpoint{3.531250in}{0.545519in}}%
\pgfpathlineto{\pgfqpoint{3.628125in}{0.544964in}}%
\pgfpathlineto{\pgfqpoint{3.725000in}{0.546180in}}%
\pgfpathlineto{\pgfqpoint{3.821875in}{0.545521in}}%
\pgfpathlineto{\pgfqpoint{3.918750in}{0.545424in}}%
\pgfpathlineto{\pgfqpoint{4.015625in}{0.545070in}}%
\pgfpathlineto{\pgfqpoint{4.112500in}{0.545226in}}%
\pgfpathlineto{\pgfqpoint{4.209375in}{0.545090in}}%
\pgfpathlineto{\pgfqpoint{4.306250in}{0.545772in}}%
\pgfpathlineto{\pgfqpoint{4.403125in}{0.545130in}}%
\pgfpathlineto{\pgfqpoint{4.500000in}{0.544462in}}%
\pgfpathlineto{\pgfqpoint{4.500000in}{0.546574in}}%
\pgfpathlineto{\pgfqpoint{4.500000in}{0.546574in}}%
\pgfpathlineto{\pgfqpoint{4.403125in}{0.547061in}}%
\pgfpathlineto{\pgfqpoint{4.306250in}{0.547536in}}%
\pgfpathlineto{\pgfqpoint{4.209375in}{0.546803in}}%
\pgfpathlineto{\pgfqpoint{4.112500in}{0.547225in}}%
\pgfpathlineto{\pgfqpoint{4.015625in}{0.547195in}}%
\pgfpathlineto{\pgfqpoint{3.918750in}{0.547251in}}%
\pgfpathlineto{\pgfqpoint{3.821875in}{0.547899in}}%
\pgfpathlineto{\pgfqpoint{3.725000in}{0.548432in}}%
\pgfpathlineto{\pgfqpoint{3.628125in}{0.546892in}}%
\pgfpathlineto{\pgfqpoint{3.531250in}{0.547528in}}%
\pgfpathlineto{\pgfqpoint{3.434375in}{0.548855in}}%
\pgfpathlineto{\pgfqpoint{3.337500in}{0.547252in}}%
\pgfpathlineto{\pgfqpoint{3.240625in}{0.546392in}}%
\pgfpathlineto{\pgfqpoint{3.143750in}{0.582395in}}%
\pgfpathlineto{\pgfqpoint{3.046875in}{0.550973in}}%
\pgfpathlineto{\pgfqpoint{2.950000in}{0.584280in}}%
\pgfpathlineto{\pgfqpoint{2.853125in}{0.548452in}}%
\pgfpathlineto{\pgfqpoint{2.756250in}{0.585538in}}%
\pgfpathlineto{\pgfqpoint{2.659375in}{0.619057in}}%
\pgfpathlineto{\pgfqpoint{2.562500in}{0.648403in}}%
\pgfpathlineto{\pgfqpoint{2.465625in}{0.766031in}}%
\pgfpathlineto{\pgfqpoint{2.368750in}{1.063059in}}%
\pgfpathlineto{\pgfqpoint{2.271875in}{1.152094in}}%
\pgfpathlineto{\pgfqpoint{2.175000in}{1.476134in}}%
\pgfpathlineto{\pgfqpoint{2.078125in}{1.662084in}}%
\pgfpathlineto{\pgfqpoint{1.981250in}{1.799053in}}%
\pgfpathlineto{\pgfqpoint{1.884375in}{2.083569in}}%
\pgfpathlineto{\pgfqpoint{1.787500in}{2.226225in}}%
\pgfpathlineto{\pgfqpoint{1.690625in}{2.507821in}}%
\pgfpathlineto{\pgfqpoint{1.593750in}{2.588027in}}%
\pgfpathlineto{\pgfqpoint{1.496875in}{2.652810in}}%
\pgfpathlineto{\pgfqpoint{1.400000in}{2.761841in}}%
\pgfpathlineto{\pgfqpoint{1.303125in}{2.862474in}}%
\pgfpathlineto{\pgfqpoint{1.206250in}{2.979955in}}%
\pgfpathlineto{\pgfqpoint{1.109375in}{2.956531in}}%
\pgfpathlineto{\pgfqpoint{1.012500in}{3.020661in}}%
\pgfpathlineto{\pgfqpoint{0.915625in}{3.075078in}}%
\pgfpathlineto{\pgfqpoint{0.818750in}{3.150034in}}%
\pgfpathlineto{\pgfqpoint{0.721875in}{3.168750in}}%
\pgfpathlineto{\pgfqpoint{0.721875in}{3.168750in}}%
\pgfpathclose%
\pgfusepath{stroke,fill}%
}%
\begin{pgfscope}%
\pgfsys@transformshift{0.000000in}{0.000000in}%
\pgfsys@useobject{currentmarker}{}%
\end{pgfscope}%
\end{pgfscope}%
\begin{pgfscope}%
\pgfpathrectangle{\pgfqpoint{0.625000in}{0.412500in}}{\pgfqpoint{3.875000in}{2.887500in}}%
\pgfusepath{clip}%
\pgfsetbuttcap%
\pgfsetroundjoin%
\definecolor{currentfill}{rgb}{0.933333,0.200000,0.466667}%
\pgfsetfillcolor{currentfill}%
\pgfsetfillopacity{0.200000}%
\pgfsetlinewidth{1.003750pt}%
\definecolor{currentstroke}{rgb}{0.933333,0.200000,0.466667}%
\pgfsetstrokecolor{currentstroke}%
\pgfsetstrokeopacity{0.200000}%
\pgfsetdash{}{0pt}%
\pgfsys@defobject{currentmarker}{\pgfqpoint{0.721875in}{0.543750in}}{\pgfqpoint{4.500000in}{2.784578in}}{%
\pgfpathmoveto{\pgfqpoint{0.721875in}{2.784578in}}%
\pgfpathlineto{\pgfqpoint{0.721875in}{2.610865in}}%
\pgfpathlineto{\pgfqpoint{0.818750in}{2.269013in}}%
\pgfpathlineto{\pgfqpoint{0.915625in}{1.899305in}}%
\pgfpathlineto{\pgfqpoint{1.012500in}{1.586563in}}%
\pgfpathlineto{\pgfqpoint{1.109375in}{1.449157in}}%
\pgfpathlineto{\pgfqpoint{1.206250in}{1.228870in}}%
\pgfpathlineto{\pgfqpoint{1.303125in}{1.078291in}}%
\pgfpathlineto{\pgfqpoint{1.400000in}{0.942150in}}%
\pgfpathlineto{\pgfqpoint{1.496875in}{0.731590in}}%
\pgfpathlineto{\pgfqpoint{1.593750in}{0.659912in}}%
\pgfpathlineto{\pgfqpoint{1.690625in}{0.604917in}}%
\pgfpathlineto{\pgfqpoint{1.787500in}{0.579614in}}%
\pgfpathlineto{\pgfqpoint{1.884375in}{0.561138in}}%
\pgfpathlineto{\pgfqpoint{1.981250in}{0.561239in}}%
\pgfpathlineto{\pgfqpoint{2.078125in}{0.560976in}}%
\pgfpathlineto{\pgfqpoint{2.175000in}{0.553580in}}%
\pgfpathlineto{\pgfqpoint{2.271875in}{0.552052in}}%
\pgfpathlineto{\pgfqpoint{2.368750in}{0.551769in}}%
\pgfpathlineto{\pgfqpoint{2.465625in}{0.549953in}}%
\pgfpathlineto{\pgfqpoint{2.562500in}{0.549598in}}%
\pgfpathlineto{\pgfqpoint{2.659375in}{0.548384in}}%
\pgfpathlineto{\pgfqpoint{2.756250in}{0.548166in}}%
\pgfpathlineto{\pgfqpoint{2.853125in}{0.547457in}}%
\pgfpathlineto{\pgfqpoint{2.950000in}{0.547742in}}%
\pgfpathlineto{\pgfqpoint{3.046875in}{0.548651in}}%
\pgfpathlineto{\pgfqpoint{3.143750in}{0.550645in}}%
\pgfpathlineto{\pgfqpoint{3.240625in}{0.549489in}}%
\pgfpathlineto{\pgfqpoint{3.337500in}{0.546210in}}%
\pgfpathlineto{\pgfqpoint{3.434375in}{0.546865in}}%
\pgfpathlineto{\pgfqpoint{3.531250in}{0.547113in}}%
\pgfpathlineto{\pgfqpoint{3.628125in}{0.548413in}}%
\pgfpathlineto{\pgfqpoint{3.725000in}{0.546686in}}%
\pgfpathlineto{\pgfqpoint{3.821875in}{0.545180in}}%
\pgfpathlineto{\pgfqpoint{3.918750in}{0.547058in}}%
\pgfpathlineto{\pgfqpoint{4.015625in}{0.548232in}}%
\pgfpathlineto{\pgfqpoint{4.112500in}{0.546432in}}%
\pgfpathlineto{\pgfqpoint{4.209375in}{0.545396in}}%
\pgfpathlineto{\pgfqpoint{4.306250in}{0.546455in}}%
\pgfpathlineto{\pgfqpoint{4.403125in}{0.543750in}}%
\pgfpathlineto{\pgfqpoint{4.500000in}{0.545747in}}%
\pgfpathlineto{\pgfqpoint{4.500000in}{0.548232in}}%
\pgfpathlineto{\pgfqpoint{4.500000in}{0.548232in}}%
\pgfpathlineto{\pgfqpoint{4.403125in}{0.545684in}}%
\pgfpathlineto{\pgfqpoint{4.306250in}{0.549125in}}%
\pgfpathlineto{\pgfqpoint{4.209375in}{0.547726in}}%
\pgfpathlineto{\pgfqpoint{4.112500in}{0.549111in}}%
\pgfpathlineto{\pgfqpoint{4.015625in}{0.551409in}}%
\pgfpathlineto{\pgfqpoint{3.918750in}{0.551539in}}%
\pgfpathlineto{\pgfqpoint{3.821875in}{0.547457in}}%
\pgfpathlineto{\pgfqpoint{3.725000in}{0.553029in}}%
\pgfpathlineto{\pgfqpoint{3.628125in}{0.558640in}}%
\pgfpathlineto{\pgfqpoint{3.531250in}{0.549622in}}%
\pgfpathlineto{\pgfqpoint{3.434375in}{0.551398in}}%
\pgfpathlineto{\pgfqpoint{3.337500in}{0.551530in}}%
\pgfpathlineto{\pgfqpoint{3.240625in}{0.582895in}}%
\pgfpathlineto{\pgfqpoint{3.143750in}{0.576784in}}%
\pgfpathlineto{\pgfqpoint{3.046875in}{0.574979in}}%
\pgfpathlineto{\pgfqpoint{2.950000in}{0.580060in}}%
\pgfpathlineto{\pgfqpoint{2.853125in}{0.564104in}}%
\pgfpathlineto{\pgfqpoint{2.756250in}{0.553523in}}%
\pgfpathlineto{\pgfqpoint{2.659375in}{0.553752in}}%
\pgfpathlineto{\pgfqpoint{2.562500in}{0.554177in}}%
\pgfpathlineto{\pgfqpoint{2.465625in}{0.554269in}}%
\pgfpathlineto{\pgfqpoint{2.368750in}{0.557222in}}%
\pgfpathlineto{\pgfqpoint{2.271875in}{0.558205in}}%
\pgfpathlineto{\pgfqpoint{2.175000in}{0.561966in}}%
\pgfpathlineto{\pgfqpoint{2.078125in}{0.574424in}}%
\pgfpathlineto{\pgfqpoint{1.981250in}{0.582916in}}%
\pgfpathlineto{\pgfqpoint{1.884375in}{0.584022in}}%
\pgfpathlineto{\pgfqpoint{1.787500in}{0.601046in}}%
\pgfpathlineto{\pgfqpoint{1.690625in}{0.658662in}}%
\pgfpathlineto{\pgfqpoint{1.593750in}{0.736165in}}%
\pgfpathlineto{\pgfqpoint{1.496875in}{0.804585in}}%
\pgfpathlineto{\pgfqpoint{1.400000in}{1.102974in}}%
\pgfpathlineto{\pgfqpoint{1.303125in}{1.209854in}}%
\pgfpathlineto{\pgfqpoint{1.206250in}{1.330158in}}%
\pgfpathlineto{\pgfqpoint{1.109375in}{1.578590in}}%
\pgfpathlineto{\pgfqpoint{1.012500in}{1.704356in}}%
\pgfpathlineto{\pgfqpoint{0.915625in}{2.015629in}}%
\pgfpathlineto{\pgfqpoint{0.818750in}{2.351030in}}%
\pgfpathlineto{\pgfqpoint{0.721875in}{2.784578in}}%
\pgfpathlineto{\pgfqpoint{0.721875in}{2.784578in}}%
\pgfpathclose%
\pgfusepath{stroke,fill}%
}%
\begin{pgfscope}%
\pgfsys@transformshift{0.000000in}{0.000000in}%
\pgfsys@useobject{currentmarker}{}%
\end{pgfscope}%
\end{pgfscope}%
\begin{pgfscope}%
\pgfsetbuttcap%
\pgfsetroundjoin%
\definecolor{currentfill}{rgb}{0.000000,0.000000,0.000000}%
\pgfsetfillcolor{currentfill}%
\pgfsetlinewidth{0.803000pt}%
\definecolor{currentstroke}{rgb}{0.000000,0.000000,0.000000}%
\pgfsetstrokecolor{currentstroke}%
\pgfsetdash{}{0pt}%
\pgfsys@defobject{currentmarker}{\pgfqpoint{0.000000in}{-0.048611in}}{\pgfqpoint{0.000000in}{0.000000in}}{%
\pgfpathmoveto{\pgfqpoint{0.000000in}{0.000000in}}%
\pgfpathlineto{\pgfqpoint{0.000000in}{-0.048611in}}%
\pgfusepath{stroke,fill}%
}%
\begin{pgfscope}%
\pgfsys@transformshift{0.625000in}{0.412500in}%
\pgfsys@useobject{currentmarker}{}%
\end{pgfscope}%
\end{pgfscope}%
\begin{pgfscope}%
\definecolor{textcolor}{rgb}{0.000000,0.000000,0.000000}%
\pgfsetstrokecolor{textcolor}%
\pgfsetfillcolor{textcolor}%
\pgftext[x=0.625000in,y=0.315278in,,top]{\color{textcolor}\rmfamily\fontsize{9.000000}{10.800000}\selectfont \(\displaystyle {0}\)}%
\end{pgfscope}%
\begin{pgfscope}%
\pgfsetbuttcap%
\pgfsetroundjoin%
\definecolor{currentfill}{rgb}{0.000000,0.000000,0.000000}%
\pgfsetfillcolor{currentfill}%
\pgfsetlinewidth{0.803000pt}%
\definecolor{currentstroke}{rgb}{0.000000,0.000000,0.000000}%
\pgfsetstrokecolor{currentstroke}%
\pgfsetdash{}{0pt}%
\pgfsys@defobject{currentmarker}{\pgfqpoint{0.000000in}{-0.048611in}}{\pgfqpoint{0.000000in}{0.000000in}}{%
\pgfpathmoveto{\pgfqpoint{0.000000in}{0.000000in}}%
\pgfpathlineto{\pgfqpoint{0.000000in}{-0.048611in}}%
\pgfusepath{stroke,fill}%
}%
\begin{pgfscope}%
\pgfsys@transformshift{1.109375in}{0.412500in}%
\pgfsys@useobject{currentmarker}{}%
\end{pgfscope}%
\end{pgfscope}%
\begin{pgfscope}%
\definecolor{textcolor}{rgb}{0.000000,0.000000,0.000000}%
\pgfsetstrokecolor{textcolor}%
\pgfsetfillcolor{textcolor}%
\pgftext[x=1.109375in,y=0.315278in,,top]{\color{textcolor}\rmfamily\fontsize{9.000000}{10.800000}\selectfont \(\displaystyle {25}\)}%
\end{pgfscope}%
\begin{pgfscope}%
\pgfsetbuttcap%
\pgfsetroundjoin%
\definecolor{currentfill}{rgb}{0.000000,0.000000,0.000000}%
\pgfsetfillcolor{currentfill}%
\pgfsetlinewidth{0.803000pt}%
\definecolor{currentstroke}{rgb}{0.000000,0.000000,0.000000}%
\pgfsetstrokecolor{currentstroke}%
\pgfsetdash{}{0pt}%
\pgfsys@defobject{currentmarker}{\pgfqpoint{0.000000in}{-0.048611in}}{\pgfqpoint{0.000000in}{0.000000in}}{%
\pgfpathmoveto{\pgfqpoint{0.000000in}{0.000000in}}%
\pgfpathlineto{\pgfqpoint{0.000000in}{-0.048611in}}%
\pgfusepath{stroke,fill}%
}%
\begin{pgfscope}%
\pgfsys@transformshift{1.593750in}{0.412500in}%
\pgfsys@useobject{currentmarker}{}%
\end{pgfscope}%
\end{pgfscope}%
\begin{pgfscope}%
\definecolor{textcolor}{rgb}{0.000000,0.000000,0.000000}%
\pgfsetstrokecolor{textcolor}%
\pgfsetfillcolor{textcolor}%
\pgftext[x=1.593750in,y=0.315278in,,top]{\color{textcolor}\rmfamily\fontsize{9.000000}{10.800000}\selectfont \(\displaystyle {50}\)}%
\end{pgfscope}%
\begin{pgfscope}%
\pgfsetbuttcap%
\pgfsetroundjoin%
\definecolor{currentfill}{rgb}{0.000000,0.000000,0.000000}%
\pgfsetfillcolor{currentfill}%
\pgfsetlinewidth{0.803000pt}%
\definecolor{currentstroke}{rgb}{0.000000,0.000000,0.000000}%
\pgfsetstrokecolor{currentstroke}%
\pgfsetdash{}{0pt}%
\pgfsys@defobject{currentmarker}{\pgfqpoint{0.000000in}{-0.048611in}}{\pgfqpoint{0.000000in}{0.000000in}}{%
\pgfpathmoveto{\pgfqpoint{0.000000in}{0.000000in}}%
\pgfpathlineto{\pgfqpoint{0.000000in}{-0.048611in}}%
\pgfusepath{stroke,fill}%
}%
\begin{pgfscope}%
\pgfsys@transformshift{2.078125in}{0.412500in}%
\pgfsys@useobject{currentmarker}{}%
\end{pgfscope}%
\end{pgfscope}%
\begin{pgfscope}%
\definecolor{textcolor}{rgb}{0.000000,0.000000,0.000000}%
\pgfsetstrokecolor{textcolor}%
\pgfsetfillcolor{textcolor}%
\pgftext[x=2.078125in,y=0.315278in,,top]{\color{textcolor}\rmfamily\fontsize{9.000000}{10.800000}\selectfont \(\displaystyle {75}\)}%
\end{pgfscope}%
\begin{pgfscope}%
\pgfsetbuttcap%
\pgfsetroundjoin%
\definecolor{currentfill}{rgb}{0.000000,0.000000,0.000000}%
\pgfsetfillcolor{currentfill}%
\pgfsetlinewidth{0.803000pt}%
\definecolor{currentstroke}{rgb}{0.000000,0.000000,0.000000}%
\pgfsetstrokecolor{currentstroke}%
\pgfsetdash{}{0pt}%
\pgfsys@defobject{currentmarker}{\pgfqpoint{0.000000in}{-0.048611in}}{\pgfqpoint{0.000000in}{0.000000in}}{%
\pgfpathmoveto{\pgfqpoint{0.000000in}{0.000000in}}%
\pgfpathlineto{\pgfqpoint{0.000000in}{-0.048611in}}%
\pgfusepath{stroke,fill}%
}%
\begin{pgfscope}%
\pgfsys@transformshift{2.562500in}{0.412500in}%
\pgfsys@useobject{currentmarker}{}%
\end{pgfscope}%
\end{pgfscope}%
\begin{pgfscope}%
\definecolor{textcolor}{rgb}{0.000000,0.000000,0.000000}%
\pgfsetstrokecolor{textcolor}%
\pgfsetfillcolor{textcolor}%
\pgftext[x=2.562500in,y=0.315278in,,top]{\color{textcolor}\rmfamily\fontsize{9.000000}{10.800000}\selectfont \(\displaystyle {100}\)}%
\end{pgfscope}%
\begin{pgfscope}%
\pgfsetbuttcap%
\pgfsetroundjoin%
\definecolor{currentfill}{rgb}{0.000000,0.000000,0.000000}%
\pgfsetfillcolor{currentfill}%
\pgfsetlinewidth{0.803000pt}%
\definecolor{currentstroke}{rgb}{0.000000,0.000000,0.000000}%
\pgfsetstrokecolor{currentstroke}%
\pgfsetdash{}{0pt}%
\pgfsys@defobject{currentmarker}{\pgfqpoint{0.000000in}{-0.048611in}}{\pgfqpoint{0.000000in}{0.000000in}}{%
\pgfpathmoveto{\pgfqpoint{0.000000in}{0.000000in}}%
\pgfpathlineto{\pgfqpoint{0.000000in}{-0.048611in}}%
\pgfusepath{stroke,fill}%
}%
\begin{pgfscope}%
\pgfsys@transformshift{3.046875in}{0.412500in}%
\pgfsys@useobject{currentmarker}{}%
\end{pgfscope}%
\end{pgfscope}%
\begin{pgfscope}%
\definecolor{textcolor}{rgb}{0.000000,0.000000,0.000000}%
\pgfsetstrokecolor{textcolor}%
\pgfsetfillcolor{textcolor}%
\pgftext[x=3.046875in,y=0.315278in,,top]{\color{textcolor}\rmfamily\fontsize{9.000000}{10.800000}\selectfont \(\displaystyle {125}\)}%
\end{pgfscope}%
\begin{pgfscope}%
\pgfsetbuttcap%
\pgfsetroundjoin%
\definecolor{currentfill}{rgb}{0.000000,0.000000,0.000000}%
\pgfsetfillcolor{currentfill}%
\pgfsetlinewidth{0.803000pt}%
\definecolor{currentstroke}{rgb}{0.000000,0.000000,0.000000}%
\pgfsetstrokecolor{currentstroke}%
\pgfsetdash{}{0pt}%
\pgfsys@defobject{currentmarker}{\pgfqpoint{0.000000in}{-0.048611in}}{\pgfqpoint{0.000000in}{0.000000in}}{%
\pgfpathmoveto{\pgfqpoint{0.000000in}{0.000000in}}%
\pgfpathlineto{\pgfqpoint{0.000000in}{-0.048611in}}%
\pgfusepath{stroke,fill}%
}%
\begin{pgfscope}%
\pgfsys@transformshift{3.531250in}{0.412500in}%
\pgfsys@useobject{currentmarker}{}%
\end{pgfscope}%
\end{pgfscope}%
\begin{pgfscope}%
\definecolor{textcolor}{rgb}{0.000000,0.000000,0.000000}%
\pgfsetstrokecolor{textcolor}%
\pgfsetfillcolor{textcolor}%
\pgftext[x=3.531250in,y=0.315278in,,top]{\color{textcolor}\rmfamily\fontsize{9.000000}{10.800000}\selectfont \(\displaystyle {150}\)}%
\end{pgfscope}%
\begin{pgfscope}%
\pgfsetbuttcap%
\pgfsetroundjoin%
\definecolor{currentfill}{rgb}{0.000000,0.000000,0.000000}%
\pgfsetfillcolor{currentfill}%
\pgfsetlinewidth{0.803000pt}%
\definecolor{currentstroke}{rgb}{0.000000,0.000000,0.000000}%
\pgfsetstrokecolor{currentstroke}%
\pgfsetdash{}{0pt}%
\pgfsys@defobject{currentmarker}{\pgfqpoint{0.000000in}{-0.048611in}}{\pgfqpoint{0.000000in}{0.000000in}}{%
\pgfpathmoveto{\pgfqpoint{0.000000in}{0.000000in}}%
\pgfpathlineto{\pgfqpoint{0.000000in}{-0.048611in}}%
\pgfusepath{stroke,fill}%
}%
\begin{pgfscope}%
\pgfsys@transformshift{4.015625in}{0.412500in}%
\pgfsys@useobject{currentmarker}{}%
\end{pgfscope}%
\end{pgfscope}%
\begin{pgfscope}%
\definecolor{textcolor}{rgb}{0.000000,0.000000,0.000000}%
\pgfsetstrokecolor{textcolor}%
\pgfsetfillcolor{textcolor}%
\pgftext[x=4.015625in,y=0.315278in,,top]{\color{textcolor}\rmfamily\fontsize{9.000000}{10.800000}\selectfont \(\displaystyle {175}\)}%
\end{pgfscope}%
\begin{pgfscope}%
\pgfsetbuttcap%
\pgfsetroundjoin%
\definecolor{currentfill}{rgb}{0.000000,0.000000,0.000000}%
\pgfsetfillcolor{currentfill}%
\pgfsetlinewidth{0.803000pt}%
\definecolor{currentstroke}{rgb}{0.000000,0.000000,0.000000}%
\pgfsetstrokecolor{currentstroke}%
\pgfsetdash{}{0pt}%
\pgfsys@defobject{currentmarker}{\pgfqpoint{0.000000in}{-0.048611in}}{\pgfqpoint{0.000000in}{0.000000in}}{%
\pgfpathmoveto{\pgfqpoint{0.000000in}{0.000000in}}%
\pgfpathlineto{\pgfqpoint{0.000000in}{-0.048611in}}%
\pgfusepath{stroke,fill}%
}%
\begin{pgfscope}%
\pgfsys@transformshift{4.500000in}{0.412500in}%
\pgfsys@useobject{currentmarker}{}%
\end{pgfscope}%
\end{pgfscope}%
\begin{pgfscope}%
\definecolor{textcolor}{rgb}{0.000000,0.000000,0.000000}%
\pgfsetstrokecolor{textcolor}%
\pgfsetfillcolor{textcolor}%
\pgftext[x=4.500000in,y=0.315278in,,top]{\color{textcolor}\rmfamily\fontsize{9.000000}{10.800000}\selectfont \(\displaystyle {200}\)}%
\end{pgfscope}%
\begin{pgfscope}%
\definecolor{textcolor}{rgb}{0.000000,0.000000,0.000000}%
\pgfsetstrokecolor{textcolor}%
\pgfsetfillcolor{textcolor}%
\pgftext[x=2.562500in,y=0.148611in,,top]{\color{textcolor}\rmfamily\fontsize{11.000000}{13.200000}\selectfont Episode}%
\end{pgfscope}%
\begin{pgfscope}%
\pgfsetbuttcap%
\pgfsetroundjoin%
\definecolor{currentfill}{rgb}{0.000000,0.000000,0.000000}%
\pgfsetfillcolor{currentfill}%
\pgfsetlinewidth{0.803000pt}%
\definecolor{currentstroke}{rgb}{0.000000,0.000000,0.000000}%
\pgfsetstrokecolor{currentstroke}%
\pgfsetdash{}{0pt}%
\pgfsys@defobject{currentmarker}{\pgfqpoint{-0.048611in}{0.000000in}}{\pgfqpoint{-0.000000in}{0.000000in}}{%
\pgfpathmoveto{\pgfqpoint{-0.000000in}{0.000000in}}%
\pgfpathlineto{\pgfqpoint{-0.048611in}{0.000000in}}%
\pgfusepath{stroke,fill}%
}%
\begin{pgfscope}%
\pgfsys@transformshift{0.625000in}{0.930945in}%
\pgfsys@useobject{currentmarker}{}%
\end{pgfscope}%
\end{pgfscope}%
\begin{pgfscope}%
\definecolor{textcolor}{rgb}{0.000000,0.000000,0.000000}%
\pgfsetstrokecolor{textcolor}%
\pgfsetfillcolor{textcolor}%
\pgftext[x=0.335071in, y=0.887543in, left, base]{\color{textcolor}\rmfamily\fontsize{9.000000}{10.800000}\selectfont \(\displaystyle {200}\)}%
\end{pgfscope}%
\begin{pgfscope}%
\pgfsetbuttcap%
\pgfsetroundjoin%
\definecolor{currentfill}{rgb}{0.000000,0.000000,0.000000}%
\pgfsetfillcolor{currentfill}%
\pgfsetlinewidth{0.803000pt}%
\definecolor{currentstroke}{rgb}{0.000000,0.000000,0.000000}%
\pgfsetstrokecolor{currentstroke}%
\pgfsetdash{}{0pt}%
\pgfsys@defobject{currentmarker}{\pgfqpoint{-0.048611in}{0.000000in}}{\pgfqpoint{-0.000000in}{0.000000in}}{%
\pgfpathmoveto{\pgfqpoint{-0.000000in}{0.000000in}}%
\pgfpathlineto{\pgfqpoint{-0.048611in}{0.000000in}}%
\pgfusepath{stroke,fill}%
}%
\begin{pgfscope}%
\pgfsys@transformshift{0.625000in}{1.489698in}%
\pgfsys@useobject{currentmarker}{}%
\end{pgfscope}%
\end{pgfscope}%
\begin{pgfscope}%
\definecolor{textcolor}{rgb}{0.000000,0.000000,0.000000}%
\pgfsetstrokecolor{textcolor}%
\pgfsetfillcolor{textcolor}%
\pgftext[x=0.335071in, y=1.446295in, left, base]{\color{textcolor}\rmfamily\fontsize{9.000000}{10.800000}\selectfont \(\displaystyle {400}\)}%
\end{pgfscope}%
\begin{pgfscope}%
\pgfsetbuttcap%
\pgfsetroundjoin%
\definecolor{currentfill}{rgb}{0.000000,0.000000,0.000000}%
\pgfsetfillcolor{currentfill}%
\pgfsetlinewidth{0.803000pt}%
\definecolor{currentstroke}{rgb}{0.000000,0.000000,0.000000}%
\pgfsetstrokecolor{currentstroke}%
\pgfsetdash{}{0pt}%
\pgfsys@defobject{currentmarker}{\pgfqpoint{-0.048611in}{0.000000in}}{\pgfqpoint{-0.000000in}{0.000000in}}{%
\pgfpathmoveto{\pgfqpoint{-0.000000in}{0.000000in}}%
\pgfpathlineto{\pgfqpoint{-0.048611in}{0.000000in}}%
\pgfusepath{stroke,fill}%
}%
\begin{pgfscope}%
\pgfsys@transformshift{0.625000in}{2.048451in}%
\pgfsys@useobject{currentmarker}{}%
\end{pgfscope}%
\end{pgfscope}%
\begin{pgfscope}%
\definecolor{textcolor}{rgb}{0.000000,0.000000,0.000000}%
\pgfsetstrokecolor{textcolor}%
\pgfsetfillcolor{textcolor}%
\pgftext[x=0.335071in, y=2.005048in, left, base]{\color{textcolor}\rmfamily\fontsize{9.000000}{10.800000}\selectfont \(\displaystyle {600}\)}%
\end{pgfscope}%
\begin{pgfscope}%
\pgfsetbuttcap%
\pgfsetroundjoin%
\definecolor{currentfill}{rgb}{0.000000,0.000000,0.000000}%
\pgfsetfillcolor{currentfill}%
\pgfsetlinewidth{0.803000pt}%
\definecolor{currentstroke}{rgb}{0.000000,0.000000,0.000000}%
\pgfsetstrokecolor{currentstroke}%
\pgfsetdash{}{0pt}%
\pgfsys@defobject{currentmarker}{\pgfqpoint{-0.048611in}{0.000000in}}{\pgfqpoint{-0.000000in}{0.000000in}}{%
\pgfpathmoveto{\pgfqpoint{-0.000000in}{0.000000in}}%
\pgfpathlineto{\pgfqpoint{-0.048611in}{0.000000in}}%
\pgfusepath{stroke,fill}%
}%
\begin{pgfscope}%
\pgfsys@transformshift{0.625000in}{2.607204in}%
\pgfsys@useobject{currentmarker}{}%
\end{pgfscope}%
\end{pgfscope}%
\begin{pgfscope}%
\definecolor{textcolor}{rgb}{0.000000,0.000000,0.000000}%
\pgfsetstrokecolor{textcolor}%
\pgfsetfillcolor{textcolor}%
\pgftext[x=0.335071in, y=2.563801in, left, base]{\color{textcolor}\rmfamily\fontsize{9.000000}{10.800000}\selectfont \(\displaystyle {800}\)}%
\end{pgfscope}%
\begin{pgfscope}%
\pgfsetbuttcap%
\pgfsetroundjoin%
\definecolor{currentfill}{rgb}{0.000000,0.000000,0.000000}%
\pgfsetfillcolor{currentfill}%
\pgfsetlinewidth{0.803000pt}%
\definecolor{currentstroke}{rgb}{0.000000,0.000000,0.000000}%
\pgfsetstrokecolor{currentstroke}%
\pgfsetdash{}{0pt}%
\pgfsys@defobject{currentmarker}{\pgfqpoint{-0.048611in}{0.000000in}}{\pgfqpoint{-0.000000in}{0.000000in}}{%
\pgfpathmoveto{\pgfqpoint{-0.000000in}{0.000000in}}%
\pgfpathlineto{\pgfqpoint{-0.048611in}{0.000000in}}%
\pgfusepath{stroke,fill}%
}%
\begin{pgfscope}%
\pgfsys@transformshift{0.625000in}{3.165956in}%
\pgfsys@useobject{currentmarker}{}%
\end{pgfscope}%
\end{pgfscope}%
\begin{pgfscope}%
\definecolor{textcolor}{rgb}{0.000000,0.000000,0.000000}%
\pgfsetstrokecolor{textcolor}%
\pgfsetfillcolor{textcolor}%
\pgftext[x=0.270835in, y=3.122553in, left, base]{\color{textcolor}\rmfamily\fontsize{9.000000}{10.800000}\selectfont \(\displaystyle {1000}\)}%
\end{pgfscope}%
\begin{pgfscope}%
\definecolor{textcolor}{rgb}{0.000000,0.000000,0.000000}%
\pgfsetstrokecolor{textcolor}%
\pgfsetfillcolor{textcolor}%
\pgftext[x=0.215279in,y=1.856250in,,bottom,rotate=90.000000]{\color{textcolor}\rmfamily\fontsize{11.000000}{13.200000}\selectfont Steps to Goal}%
\end{pgfscope}%
\begin{pgfscope}%
\pgfpathrectangle{\pgfqpoint{0.625000in}{0.412500in}}{\pgfqpoint{3.875000in}{2.887500in}}%
\pgfusepath{clip}%
\pgfsetrectcap%
\pgfsetroundjoin%
\pgfsetlinewidth{1.003750pt}%
\definecolor{currentstroke}{rgb}{0.000000,0.466667,0.733333}%
\pgfsetstrokecolor{currentstroke}%
\pgfsetdash{}{0pt}%
\pgfpathmoveto{\pgfqpoint{0.721875in}{3.168750in}}%
\pgfpathlineto{\pgfqpoint{0.818750in}{3.133716in}}%
\pgfpathlineto{\pgfqpoint{0.915625in}{3.046327in}}%
\pgfpathlineto{\pgfqpoint{1.012500in}{2.991793in}}%
\pgfpathlineto{\pgfqpoint{1.109375in}{2.919304in}}%
\pgfpathlineto{\pgfqpoint{1.206250in}{2.938656in}}%
\pgfpathlineto{\pgfqpoint{1.303125in}{2.808429in}}%
\pgfpathlineto{\pgfqpoint{1.400000in}{2.717986in}}%
\pgfpathlineto{\pgfqpoint{1.496875in}{2.599772in}}%
\pgfpathlineto{\pgfqpoint{1.593750in}{2.539762in}}%
\pgfpathlineto{\pgfqpoint{1.690625in}{2.439019in}}%
\pgfpathlineto{\pgfqpoint{1.787500in}{2.160090in}}%
\pgfpathlineto{\pgfqpoint{1.884375in}{1.998778in}}%
\pgfpathlineto{\pgfqpoint{1.981250in}{1.742143in}}%
\pgfpathlineto{\pgfqpoint{2.078125in}{1.578335in}}%
\pgfpathlineto{\pgfqpoint{2.175000in}{1.404265in}}%
\pgfpathlineto{\pgfqpoint{2.271875in}{1.084398in}}%
\pgfpathlineto{\pgfqpoint{2.368750in}{0.998461in}}%
\pgfpathlineto{\pgfqpoint{2.465625in}{0.729962in}}%
\pgfpathlineto{\pgfqpoint{2.562500in}{0.622197in}}%
\pgfpathlineto{\pgfqpoint{2.659375in}{0.596737in}}%
\pgfpathlineto{\pgfqpoint{2.756250in}{0.567756in}}%
\pgfpathlineto{\pgfqpoint{2.853125in}{0.547250in}}%
\pgfpathlineto{\pgfqpoint{2.950000in}{0.566620in}}%
\pgfpathlineto{\pgfqpoint{3.046875in}{0.549634in}}%
\pgfpathlineto{\pgfqpoint{3.143750in}{0.564646in}}%
\pgfpathlineto{\pgfqpoint{3.240625in}{0.545127in}}%
\pgfpathlineto{\pgfqpoint{3.337500in}{0.546337in}}%
\pgfpathlineto{\pgfqpoint{3.434375in}{0.547306in}}%
\pgfpathlineto{\pgfqpoint{3.531250in}{0.546524in}}%
\pgfpathlineto{\pgfqpoint{3.628125in}{0.545928in}}%
\pgfpathlineto{\pgfqpoint{3.725000in}{0.547306in}}%
\pgfpathlineto{\pgfqpoint{3.821875in}{0.546710in}}%
\pgfpathlineto{\pgfqpoint{3.918750in}{0.546337in}}%
\pgfpathlineto{\pgfqpoint{4.015625in}{0.546132in}}%
\pgfpathlineto{\pgfqpoint{4.112500in}{0.546226in}}%
\pgfpathlineto{\pgfqpoint{4.209375in}{0.545946in}}%
\pgfpathlineto{\pgfqpoint{4.306250in}{0.546654in}}%
\pgfpathlineto{\pgfqpoint{4.403125in}{0.546095in}}%
\pgfpathlineto{\pgfqpoint{4.500000in}{0.545518in}}%
\pgfusepath{stroke}%
\end{pgfscope}%
\begin{pgfscope}%
\pgfpathrectangle{\pgfqpoint{0.625000in}{0.412500in}}{\pgfqpoint{3.875000in}{2.887500in}}%
\pgfusepath{clip}%
\pgfsetrectcap%
\pgfsetroundjoin%
\pgfsetlinewidth{1.003750pt}%
\definecolor{currentstroke}{rgb}{0.933333,0.200000,0.466667}%
\pgfsetstrokecolor{currentstroke}%
\pgfsetdash{}{0pt}%
\pgfpathmoveto{\pgfqpoint{0.721875in}{2.697721in}}%
\pgfpathlineto{\pgfqpoint{0.818750in}{2.310022in}}%
\pgfpathlineto{\pgfqpoint{0.915625in}{1.957467in}}%
\pgfpathlineto{\pgfqpoint{1.012500in}{1.645460in}}%
\pgfpathlineto{\pgfqpoint{1.109375in}{1.513873in}}%
\pgfpathlineto{\pgfqpoint{1.206250in}{1.279514in}}%
\pgfpathlineto{\pgfqpoint{1.303125in}{1.144072in}}%
\pgfpathlineto{\pgfqpoint{1.400000in}{1.022562in}}%
\pgfpathlineto{\pgfqpoint{1.496875in}{0.768088in}}%
\pgfpathlineto{\pgfqpoint{1.593750in}{0.698039in}}%
\pgfpathlineto{\pgfqpoint{1.690625in}{0.631789in}}%
\pgfpathlineto{\pgfqpoint{1.787500in}{0.590330in}}%
\pgfpathlineto{\pgfqpoint{1.884375in}{0.572580in}}%
\pgfpathlineto{\pgfqpoint{1.981250in}{0.572077in}}%
\pgfpathlineto{\pgfqpoint{2.078125in}{0.567700in}}%
\pgfpathlineto{\pgfqpoint{2.175000in}{0.557773in}}%
\pgfpathlineto{\pgfqpoint{2.271875in}{0.555128in}}%
\pgfpathlineto{\pgfqpoint{2.368750in}{0.554495in}}%
\pgfpathlineto{\pgfqpoint{2.465625in}{0.552111in}}%
\pgfpathlineto{\pgfqpoint{2.562500in}{0.551888in}}%
\pgfpathlineto{\pgfqpoint{2.659375in}{0.551068in}}%
\pgfpathlineto{\pgfqpoint{2.756250in}{0.550845in}}%
\pgfpathlineto{\pgfqpoint{2.853125in}{0.555780in}}%
\pgfpathlineto{\pgfqpoint{2.950000in}{0.563901in}}%
\pgfpathlineto{\pgfqpoint{3.046875in}{0.561815in}}%
\pgfpathlineto{\pgfqpoint{3.143750in}{0.563715in}}%
\pgfpathlineto{\pgfqpoint{3.240625in}{0.566192in}}%
\pgfpathlineto{\pgfqpoint{3.337500in}{0.548870in}}%
\pgfpathlineto{\pgfqpoint{3.434375in}{0.549131in}}%
\pgfpathlineto{\pgfqpoint{3.531250in}{0.548367in}}%
\pgfpathlineto{\pgfqpoint{3.628125in}{0.553527in}}%
\pgfpathlineto{\pgfqpoint{3.725000in}{0.549857in}}%
\pgfpathlineto{\pgfqpoint{3.821875in}{0.546319in}}%
\pgfpathlineto{\pgfqpoint{3.918750in}{0.549299in}}%
\pgfpathlineto{\pgfqpoint{4.015625in}{0.549820in}}%
\pgfpathlineto{\pgfqpoint{4.112500in}{0.547771in}}%
\pgfpathlineto{\pgfqpoint{4.209375in}{0.546561in}}%
\pgfpathlineto{\pgfqpoint{4.306250in}{0.547790in}}%
\pgfpathlineto{\pgfqpoint{4.403125in}{0.544717in}}%
\pgfpathlineto{\pgfqpoint{4.500000in}{0.546989in}}%
\pgfusepath{stroke}%
\end{pgfscope}%
\begin{pgfscope}%
\pgfsetrectcap%
\pgfsetmiterjoin%
\pgfsetlinewidth{0.803000pt}%
\definecolor{currentstroke}{rgb}{0.000000,0.000000,0.000000}%
\pgfsetstrokecolor{currentstroke}%
\pgfsetdash{}{0pt}%
\pgfpathmoveto{\pgfqpoint{0.625000in}{0.412500in}}%
\pgfpathlineto{\pgfqpoint{0.625000in}{3.300000in}}%
\pgfusepath{stroke}%
\end{pgfscope}%
\begin{pgfscope}%
\pgfsetrectcap%
\pgfsetmiterjoin%
\pgfsetlinewidth{0.803000pt}%
\definecolor{currentstroke}{rgb}{0.000000,0.000000,0.000000}%
\pgfsetstrokecolor{currentstroke}%
\pgfsetdash{}{0pt}%
\pgfpathmoveto{\pgfqpoint{0.625000in}{0.412500in}}%
\pgfpathlineto{\pgfqpoint{4.500000in}{0.412500in}}%
\pgfusepath{stroke}%
\end{pgfscope}%
\begin{pgfscope}%
\definecolor{textcolor}{rgb}{0.121569,0.466667,0.705882}%
\pgfsetstrokecolor{textcolor}%
\pgfsetfillcolor{textcolor}%
\pgftext[x=2.368750in,y=1.350010in,left,base]{\color{textcolor}\rmfamily\fontsize{14.000000}{16.800000}\selectfont Sarsa(\(\displaystyle \lambda\))}%
\end{pgfscope}%
\begin{pgfscope}%
\definecolor{textcolor}{rgb}{1.000000,0.078431,0.576471}%
\pgfsetstrokecolor{textcolor}%
\pgfsetfillcolor{textcolor}%
\pgftext[x=0.915625in,y=0.791257in,left,base]{\color{textcolor}\rmfamily\fontsize{14.000000}{16.800000}\selectfont GSP}%
\end{pgfscope}%
\end{pgfpicture}%
\makeatother%
\endgroup%

%% file: pinball_stepstogoal_noLAVI.pgf
\begingroup%
\makeatletter%
\begin{pgfpicture}%
\pgfpathrectangle{\pgfpointorigin}{\pgfqpoint{5.000000in}{3.750000in}}%
\pgfusepath{use as bounding box, clip}%
\begin{pgfscope}%
\pgfsetbuttcap%
\pgfsetmiterjoin%
\definecolor{currentfill}{rgb}{1.000000,1.000000,1.000000}%
\pgfsetfillcolor{currentfill}%
\pgfsetlinewidth{0.000000pt}%
\definecolor{currentstroke}{rgb}{1.000000,1.000000,1.000000}%
\pgfsetstrokecolor{currentstroke}%
\pgfsetdash{}{0pt}%
\pgfpathmoveto{\pgfqpoint{0.000000in}{0.000000in}}%
\pgfpathlineto{\pgfqpoint{5.000000in}{0.000000in}}%
\pgfpathlineto{\pgfqpoint{5.000000in}{3.750000in}}%
\pgfpathlineto{\pgfqpoint{0.000000in}{3.750000in}}%
\pgfpathlineto{\pgfqpoint{0.000000in}{0.000000in}}%
\pgfpathclose%
\pgfusepath{fill}%
\end{pgfscope}%
\begin{pgfscope}%
\pgfsetbuttcap%
\pgfsetmiterjoin%
\definecolor{currentfill}{rgb}{1.000000,1.000000,1.000000}%
\pgfsetfillcolor{currentfill}%
\pgfsetlinewidth{0.000000pt}%
\definecolor{currentstroke}{rgb}{0.000000,0.000000,0.000000}%
\pgfsetstrokecolor{currentstroke}%
\pgfsetstrokeopacity{0.000000}%
\pgfsetdash{}{0pt}%
\pgfpathmoveto{\pgfqpoint{0.625000in}{0.412500in}}%
\pgfpathlineto{\pgfqpoint{4.500000in}{0.412500in}}%
\pgfpathlineto{\pgfqpoint{4.500000in}{3.300000in}}%
\pgfpathlineto{\pgfqpoint{0.625000in}{3.300000in}}%
\pgfpathlineto{\pgfqpoint{0.625000in}{0.412500in}}%
\pgfpathclose%
\pgfusepath{fill}%
\end{pgfscope}%
\begin{pgfscope}%
\pgfpathrectangle{\pgfqpoint{0.625000in}{0.412500in}}{\pgfqpoint{3.875000in}{2.887500in}}%
\pgfusepath{clip}%
\pgfsetbuttcap%
\pgfsetroundjoin%
\definecolor{currentfill}{rgb}{0.000000,0.466667,0.733333}%
\pgfsetfillcolor{currentfill}%
\pgfsetfillopacity{0.200000}%
\pgfsetlinewidth{1.003750pt}%
\definecolor{currentstroke}{rgb}{0.000000,0.466667,0.733333}%
\pgfsetstrokecolor{currentstroke}%
\pgfsetstrokeopacity{0.200000}%
\pgfsetdash{}{0pt}%
\pgfsys@defobject{currentmarker}{\pgfqpoint{0.663750in}{0.543750in}}{\pgfqpoint{4.500000in}{3.168750in}}{%
\pgfpathmoveto{\pgfqpoint{0.663750in}{3.168750in}}%
\pgfpathlineto{\pgfqpoint{0.663750in}{2.965316in}}%
\pgfpathlineto{\pgfqpoint{0.702500in}{2.744817in}}%
\pgfpathlineto{\pgfqpoint{0.741250in}{2.765274in}}%
\pgfpathlineto{\pgfqpoint{0.780000in}{2.728839in}}%
\pgfpathlineto{\pgfqpoint{0.818750in}{2.617114in}}%
\pgfpathlineto{\pgfqpoint{0.857500in}{2.514665in}}%
\pgfpathlineto{\pgfqpoint{0.896250in}{2.645745in}}%
\pgfpathlineto{\pgfqpoint{0.935000in}{2.552421in}}%
\pgfpathlineto{\pgfqpoint{0.973750in}{2.437717in}}%
\pgfpathlineto{\pgfqpoint{1.012500in}{2.524098in}}%
\pgfpathlineto{\pgfqpoint{1.051250in}{2.235246in}}%
\pgfpathlineto{\pgfqpoint{1.090000in}{2.298862in}}%
\pgfpathlineto{\pgfqpoint{1.128750in}{2.351914in}}%
\pgfpathlineto{\pgfqpoint{1.167500in}{2.182145in}}%
\pgfpathlineto{\pgfqpoint{1.206250in}{2.046228in}}%
\pgfpathlineto{\pgfqpoint{1.245000in}{1.824021in}}%
\pgfpathlineto{\pgfqpoint{1.283750in}{1.635513in}}%
\pgfpathlineto{\pgfqpoint{1.322500in}{1.573914in}}%
\pgfpathlineto{\pgfqpoint{1.361250in}{1.648332in}}%
\pgfpathlineto{\pgfqpoint{1.400000in}{1.736234in}}%
\pgfpathlineto{\pgfqpoint{1.438750in}{1.511969in}}%
\pgfpathlineto{\pgfqpoint{1.477500in}{1.383231in}}%
\pgfpathlineto{\pgfqpoint{1.516250in}{1.309323in}}%
\pgfpathlineto{\pgfqpoint{1.555000in}{1.187700in}}%
\pgfpathlineto{\pgfqpoint{1.593750in}{0.958525in}}%
\pgfpathlineto{\pgfqpoint{1.632500in}{1.009683in}}%
\pgfpathlineto{\pgfqpoint{1.671250in}{0.853393in}}%
\pgfpathlineto{\pgfqpoint{1.710000in}{0.905608in}}%
\pgfpathlineto{\pgfqpoint{1.748750in}{0.799733in}}%
\pgfpathlineto{\pgfqpoint{1.787500in}{0.816210in}}%
\pgfpathlineto{\pgfqpoint{1.826250in}{0.731367in}}%
\pgfpathlineto{\pgfqpoint{1.865000in}{0.716953in}}%
\pgfpathlineto{\pgfqpoint{1.903750in}{0.738656in}}%
\pgfpathlineto{\pgfqpoint{1.942500in}{0.665945in}}%
\pgfpathlineto{\pgfqpoint{1.981250in}{0.696173in}}%
\pgfpathlineto{\pgfqpoint{2.020000in}{0.602898in}}%
\pgfpathlineto{\pgfqpoint{2.058750in}{0.629135in}}%
\pgfpathlineto{\pgfqpoint{2.097500in}{0.626563in}}%
\pgfpathlineto{\pgfqpoint{2.136250in}{0.622086in}}%
\pgfpathlineto{\pgfqpoint{2.175000in}{0.603761in}}%
\pgfpathlineto{\pgfqpoint{2.213750in}{0.608805in}}%
\pgfpathlineto{\pgfqpoint{2.252500in}{0.640491in}}%
\pgfpathlineto{\pgfqpoint{2.291250in}{0.603161in}}%
\pgfpathlineto{\pgfqpoint{2.330000in}{0.580927in}}%
\pgfpathlineto{\pgfqpoint{2.368750in}{0.590705in}}%
\pgfpathlineto{\pgfqpoint{2.407500in}{0.604120in}}%
\pgfpathlineto{\pgfqpoint{2.446250in}{0.575312in}}%
\pgfpathlineto{\pgfqpoint{2.485000in}{0.573278in}}%
\pgfpathlineto{\pgfqpoint{2.523750in}{0.560306in}}%
\pgfpathlineto{\pgfqpoint{2.562500in}{0.566187in}}%
\pgfpathlineto{\pgfqpoint{2.601250in}{0.577655in}}%
\pgfpathlineto{\pgfqpoint{2.640000in}{0.560570in}}%
\pgfpathlineto{\pgfqpoint{2.678750in}{0.575718in}}%
\pgfpathlineto{\pgfqpoint{2.717500in}{0.580992in}}%
\pgfpathlineto{\pgfqpoint{2.756250in}{0.575332in}}%
\pgfpathlineto{\pgfqpoint{2.795000in}{0.575723in}}%
\pgfpathlineto{\pgfqpoint{2.833750in}{0.563080in}}%
\pgfpathlineto{\pgfqpoint{2.872500in}{0.563750in}}%
\pgfpathlineto{\pgfqpoint{2.911250in}{0.555743in}}%
\pgfpathlineto{\pgfqpoint{2.950000in}{0.562774in}}%
\pgfpathlineto{\pgfqpoint{2.988750in}{0.569531in}}%
\pgfpathlineto{\pgfqpoint{3.027500in}{0.556126in}}%
\pgfpathlineto{\pgfqpoint{3.066250in}{0.566102in}}%
\pgfpathlineto{\pgfqpoint{3.105000in}{0.577430in}}%
\pgfpathlineto{\pgfqpoint{3.143750in}{0.570110in}}%
\pgfpathlineto{\pgfqpoint{3.182500in}{0.557211in}}%
\pgfpathlineto{\pgfqpoint{3.221250in}{0.564686in}}%
\pgfpathlineto{\pgfqpoint{3.260000in}{0.563400in}}%
\pgfpathlineto{\pgfqpoint{3.298750in}{0.571254in}}%
\pgfpathlineto{\pgfqpoint{3.337500in}{0.564083in}}%
\pgfpathlineto{\pgfqpoint{3.376250in}{0.608862in}}%
\pgfpathlineto{\pgfqpoint{3.415000in}{0.568628in}}%
\pgfpathlineto{\pgfqpoint{3.453750in}{0.578211in}}%
\pgfpathlineto{\pgfqpoint{3.492500in}{0.573246in}}%
\pgfpathlineto{\pgfqpoint{3.531250in}{0.572339in}}%
\pgfpathlineto{\pgfqpoint{3.570000in}{0.567695in}}%
\pgfpathlineto{\pgfqpoint{3.608750in}{0.570458in}}%
\pgfpathlineto{\pgfqpoint{3.647500in}{0.552790in}}%
\pgfpathlineto{\pgfqpoint{3.686250in}{0.566996in}}%
\pgfpathlineto{\pgfqpoint{3.725000in}{0.560770in}}%
\pgfpathlineto{\pgfqpoint{3.763750in}{0.551028in}}%
\pgfpathlineto{\pgfqpoint{3.802500in}{0.556327in}}%
\pgfpathlineto{\pgfqpoint{3.841250in}{0.548770in}}%
\pgfpathlineto{\pgfqpoint{3.880000in}{0.552148in}}%
\pgfpathlineto{\pgfqpoint{3.918750in}{0.550734in}}%
\pgfpathlineto{\pgfqpoint{3.957500in}{0.564684in}}%
\pgfpathlineto{\pgfqpoint{3.996250in}{0.565543in}}%
\pgfpathlineto{\pgfqpoint{4.035000in}{0.543750in}}%
\pgfpathlineto{\pgfqpoint{4.073750in}{0.554309in}}%
\pgfpathlineto{\pgfqpoint{4.112500in}{0.556710in}}%
\pgfpathlineto{\pgfqpoint{4.151250in}{0.556730in}}%
\pgfpathlineto{\pgfqpoint{4.190000in}{0.551182in}}%
\pgfpathlineto{\pgfqpoint{4.228750in}{0.550697in}}%
\pgfpathlineto{\pgfqpoint{4.267500in}{0.558882in}}%
\pgfpathlineto{\pgfqpoint{4.306250in}{0.570554in}}%
\pgfpathlineto{\pgfqpoint{4.345000in}{0.579330in}}%
\pgfpathlineto{\pgfqpoint{4.383750in}{0.559363in}}%
\pgfpathlineto{\pgfqpoint{4.422500in}{0.568431in}}%
\pgfpathlineto{\pgfqpoint{4.461250in}{0.555417in}}%
\pgfpathlineto{\pgfqpoint{4.500000in}{0.553230in}}%
\pgfpathlineto{\pgfqpoint{4.500000in}{0.571603in}}%
\pgfpathlineto{\pgfqpoint{4.500000in}{0.571603in}}%
\pgfpathlineto{\pgfqpoint{4.461250in}{0.578729in}}%
\pgfpathlineto{\pgfqpoint{4.422500in}{0.617773in}}%
\pgfpathlineto{\pgfqpoint{4.383750in}{0.575927in}}%
\pgfpathlineto{\pgfqpoint{4.345000in}{0.624125in}}%
\pgfpathlineto{\pgfqpoint{4.306250in}{0.600727in}}%
\pgfpathlineto{\pgfqpoint{4.267500in}{0.580568in}}%
\pgfpathlineto{\pgfqpoint{4.228750in}{0.565434in}}%
\pgfpathlineto{\pgfqpoint{4.190000in}{0.566438in}}%
\pgfpathlineto{\pgfqpoint{4.151250in}{0.598368in}}%
\pgfpathlineto{\pgfqpoint{4.112500in}{0.594648in}}%
\pgfpathlineto{\pgfqpoint{4.073750in}{0.572356in}}%
\pgfpathlineto{\pgfqpoint{4.035000in}{0.566885in}}%
\pgfpathlineto{\pgfqpoint{3.996250in}{0.597685in}}%
\pgfpathlineto{\pgfqpoint{3.957500in}{0.614802in}}%
\pgfpathlineto{\pgfqpoint{3.918750in}{0.579366in}}%
\pgfpathlineto{\pgfqpoint{3.880000in}{0.580357in}}%
\pgfpathlineto{\pgfqpoint{3.841250in}{0.561980in}}%
\pgfpathlineto{\pgfqpoint{3.802500in}{0.572017in}}%
\pgfpathlineto{\pgfqpoint{3.763750in}{0.571019in}}%
\pgfpathlineto{\pgfqpoint{3.725000in}{0.602877in}}%
\pgfpathlineto{\pgfqpoint{3.686250in}{0.594972in}}%
\pgfpathlineto{\pgfqpoint{3.647500in}{0.568914in}}%
\pgfpathlineto{\pgfqpoint{3.608750in}{0.607616in}}%
\pgfpathlineto{\pgfqpoint{3.570000in}{0.585152in}}%
\pgfpathlineto{\pgfqpoint{3.531250in}{0.625811in}}%
\pgfpathlineto{\pgfqpoint{3.492500in}{0.615210in}}%
\pgfpathlineto{\pgfqpoint{3.453750in}{0.615320in}}%
\pgfpathlineto{\pgfqpoint{3.415000in}{0.587997in}}%
\pgfpathlineto{\pgfqpoint{3.376250in}{0.677718in}}%
\pgfpathlineto{\pgfqpoint{3.337500in}{0.602809in}}%
\pgfpathlineto{\pgfqpoint{3.298750in}{0.602660in}}%
\pgfpathlineto{\pgfqpoint{3.260000in}{0.623185in}}%
\pgfpathlineto{\pgfqpoint{3.221250in}{0.588352in}}%
\pgfpathlineto{\pgfqpoint{3.182500in}{0.576859in}}%
\pgfpathlineto{\pgfqpoint{3.143750in}{0.600561in}}%
\pgfpathlineto{\pgfqpoint{3.105000in}{0.621216in}}%
\pgfpathlineto{\pgfqpoint{3.066250in}{0.598309in}}%
\pgfpathlineto{\pgfqpoint{3.027500in}{0.589470in}}%
\pgfpathlineto{\pgfqpoint{2.988750in}{0.600261in}}%
\pgfpathlineto{\pgfqpoint{2.950000in}{0.587287in}}%
\pgfpathlineto{\pgfqpoint{2.911250in}{0.578517in}}%
\pgfpathlineto{\pgfqpoint{2.872500in}{0.582456in}}%
\pgfpathlineto{\pgfqpoint{2.833750in}{0.581714in}}%
\pgfpathlineto{\pgfqpoint{2.795000in}{0.630136in}}%
\pgfpathlineto{\pgfqpoint{2.756250in}{0.620642in}}%
\pgfpathlineto{\pgfqpoint{2.717500in}{0.645286in}}%
\pgfpathlineto{\pgfqpoint{2.678750in}{0.625179in}}%
\pgfpathlineto{\pgfqpoint{2.640000in}{0.625176in}}%
\pgfpathlineto{\pgfqpoint{2.601250in}{0.638204in}}%
\pgfpathlineto{\pgfqpoint{2.562500in}{0.594789in}}%
\pgfpathlineto{\pgfqpoint{2.523750in}{0.583190in}}%
\pgfpathlineto{\pgfqpoint{2.485000in}{0.598690in}}%
\pgfpathlineto{\pgfqpoint{2.446250in}{0.625853in}}%
\pgfpathlineto{\pgfqpoint{2.407500in}{0.666926in}}%
\pgfpathlineto{\pgfqpoint{2.368750in}{0.629848in}}%
\pgfpathlineto{\pgfqpoint{2.330000in}{0.612871in}}%
\pgfpathlineto{\pgfqpoint{2.291250in}{0.710136in}}%
\pgfpathlineto{\pgfqpoint{2.252500in}{0.753716in}}%
\pgfpathlineto{\pgfqpoint{2.213750in}{0.700560in}}%
\pgfpathlineto{\pgfqpoint{2.175000in}{0.675720in}}%
\pgfpathlineto{\pgfqpoint{2.136250in}{0.673539in}}%
\pgfpathlineto{\pgfqpoint{2.097500in}{0.693794in}}%
\pgfpathlineto{\pgfqpoint{2.058750in}{0.725227in}}%
\pgfpathlineto{\pgfqpoint{2.020000in}{0.669179in}}%
\pgfpathlineto{\pgfqpoint{1.981250in}{0.794175in}}%
\pgfpathlineto{\pgfqpoint{1.942500in}{0.749178in}}%
\pgfpathlineto{\pgfqpoint{1.903750in}{0.879395in}}%
\pgfpathlineto{\pgfqpoint{1.865000in}{0.857436in}}%
\pgfpathlineto{\pgfqpoint{1.826250in}{0.845999in}}%
\pgfpathlineto{\pgfqpoint{1.787500in}{1.014081in}}%
\pgfpathlineto{\pgfqpoint{1.748750in}{0.978958in}}%
\pgfpathlineto{\pgfqpoint{1.710000in}{1.119176in}}%
\pgfpathlineto{\pgfqpoint{1.671250in}{1.056512in}}%
\pgfpathlineto{\pgfqpoint{1.632500in}{1.237990in}}%
\pgfpathlineto{\pgfqpoint{1.593750in}{1.160758in}}%
\pgfpathlineto{\pgfqpoint{1.555000in}{1.476668in}}%
\pgfpathlineto{\pgfqpoint{1.516250in}{1.573965in}}%
\pgfpathlineto{\pgfqpoint{1.477500in}{1.664094in}}%
\pgfpathlineto{\pgfqpoint{1.438750in}{1.780037in}}%
\pgfpathlineto{\pgfqpoint{1.400000in}{2.045402in}}%
\pgfpathlineto{\pgfqpoint{1.361250in}{1.900530in}}%
\pgfpathlineto{\pgfqpoint{1.322500in}{1.800377in}}%
\pgfpathlineto{\pgfqpoint{1.283750in}{1.852475in}}%
\pgfpathlineto{\pgfqpoint{1.245000in}{2.097722in}}%
\pgfpathlineto{\pgfqpoint{1.206250in}{2.268623in}}%
\pgfpathlineto{\pgfqpoint{1.167500in}{2.396471in}}%
\pgfpathlineto{\pgfqpoint{1.128750in}{2.528823in}}%
\pgfpathlineto{\pgfqpoint{1.090000in}{2.500810in}}%
\pgfpathlineto{\pgfqpoint{1.051250in}{2.445273in}}%
\pgfpathlineto{\pgfqpoint{1.012500in}{2.717917in}}%
\pgfpathlineto{\pgfqpoint{0.973750in}{2.655680in}}%
\pgfpathlineto{\pgfqpoint{0.935000in}{2.713256in}}%
\pgfpathlineto{\pgfqpoint{0.896250in}{2.784619in}}%
\pgfpathlineto{\pgfqpoint{0.857500in}{2.706245in}}%
\pgfpathlineto{\pgfqpoint{0.818750in}{2.794319in}}%
\pgfpathlineto{\pgfqpoint{0.780000in}{2.866211in}}%
\pgfpathlineto{\pgfqpoint{0.741250in}{2.905840in}}%
\pgfpathlineto{\pgfqpoint{0.702500in}{2.852179in}}%
\pgfpathlineto{\pgfqpoint{0.663750in}{3.168750in}}%
\pgfpathlineto{\pgfqpoint{0.663750in}{3.168750in}}%
\pgfpathclose%
\pgfusepath{stroke,fill}%
}%
\begin{pgfscope}%
\pgfsys@transformshift{0.000000in}{0.000000in}%
\pgfsys@useobject{currentmarker}{}%
\end{pgfscope}%
\end{pgfscope}%
\begin{pgfscope}%
\pgfpathrectangle{\pgfqpoint{0.625000in}{0.412500in}}{\pgfqpoint{3.875000in}{2.887500in}}%
\pgfusepath{clip}%
\pgfsetbuttcap%
\pgfsetroundjoin%
\definecolor{currentfill}{rgb}{0.933333,0.200000,0.466667}%
\pgfsetfillcolor{currentfill}%
\pgfsetfillopacity{0.200000}%
\pgfsetlinewidth{1.003750pt}%
\definecolor{currentstroke}{rgb}{0.933333,0.200000,0.466667}%
\pgfsetstrokecolor{currentstroke}%
\pgfsetstrokeopacity{0.200000}%
\pgfsetdash{}{0pt}%
\pgfsys@defobject{currentmarker}{\pgfqpoint{0.663750in}{0.548935in}}{\pgfqpoint{4.500000in}{2.758752in}}{%
\pgfpathmoveto{\pgfqpoint{0.663750in}{2.758752in}}%
\pgfpathlineto{\pgfqpoint{0.663750in}{2.431663in}}%
\pgfpathlineto{\pgfqpoint{0.702500in}{2.067488in}}%
\pgfpathlineto{\pgfqpoint{0.741250in}{1.538223in}}%
\pgfpathlineto{\pgfqpoint{0.780000in}{1.329328in}}%
\pgfpathlineto{\pgfqpoint{0.818750in}{0.924145in}}%
\pgfpathlineto{\pgfqpoint{0.857500in}{0.896986in}}%
\pgfpathlineto{\pgfqpoint{0.896250in}{0.752239in}}%
\pgfpathlineto{\pgfqpoint{0.935000in}{0.670296in}}%
\pgfpathlineto{\pgfqpoint{0.973750in}{0.668354in}}%
\pgfpathlineto{\pgfqpoint{1.012500in}{0.658866in}}%
\pgfpathlineto{\pgfqpoint{1.051250in}{0.646366in}}%
\pgfpathlineto{\pgfqpoint{1.090000in}{0.617146in}}%
\pgfpathlineto{\pgfqpoint{1.128750in}{0.612529in}}%
\pgfpathlineto{\pgfqpoint{1.167500in}{0.585692in}}%
\pgfpathlineto{\pgfqpoint{1.206250in}{0.580312in}}%
\pgfpathlineto{\pgfqpoint{1.245000in}{0.602114in}}%
\pgfpathlineto{\pgfqpoint{1.283750in}{0.588666in}}%
\pgfpathlineto{\pgfqpoint{1.322500in}{0.582819in}}%
\pgfpathlineto{\pgfqpoint{1.361250in}{0.587568in}}%
\pgfpathlineto{\pgfqpoint{1.400000in}{0.574068in}}%
\pgfpathlineto{\pgfqpoint{1.438750in}{0.572185in}}%
\pgfpathlineto{\pgfqpoint{1.477500in}{0.569533in}}%
\pgfpathlineto{\pgfqpoint{1.516250in}{0.568593in}}%
\pgfpathlineto{\pgfqpoint{1.555000in}{0.584402in}}%
\pgfpathlineto{\pgfqpoint{1.593750in}{0.572748in}}%
\pgfpathlineto{\pgfqpoint{1.632500in}{0.564014in}}%
\pgfpathlineto{\pgfqpoint{1.671250in}{0.573809in}}%
\pgfpathlineto{\pgfqpoint{1.710000in}{0.575044in}}%
\pgfpathlineto{\pgfqpoint{1.748750in}{0.566818in}}%
\pgfpathlineto{\pgfqpoint{1.787500in}{0.579549in}}%
\pgfpathlineto{\pgfqpoint{1.826250in}{0.579379in}}%
\pgfpathlineto{\pgfqpoint{1.865000in}{0.564449in}}%
\pgfpathlineto{\pgfqpoint{1.903750in}{0.570710in}}%
\pgfpathlineto{\pgfqpoint{1.942500in}{0.563288in}}%
\pgfpathlineto{\pgfqpoint{1.981250in}{0.564457in}}%
\pgfpathlineto{\pgfqpoint{2.020000in}{0.552323in}}%
\pgfpathlineto{\pgfqpoint{2.058750in}{0.555559in}}%
\pgfpathlineto{\pgfqpoint{2.097500in}{0.569820in}}%
\pgfpathlineto{\pgfqpoint{2.136250in}{0.557120in}}%
\pgfpathlineto{\pgfqpoint{2.175000in}{0.557381in}}%
\pgfpathlineto{\pgfqpoint{2.213750in}{0.555672in}}%
\pgfpathlineto{\pgfqpoint{2.252500in}{0.566189in}}%
\pgfpathlineto{\pgfqpoint{2.291250in}{0.563227in}}%
\pgfpathlineto{\pgfqpoint{2.330000in}{0.554249in}}%
\pgfpathlineto{\pgfqpoint{2.368750in}{0.548935in}}%
\pgfpathlineto{\pgfqpoint{2.407500in}{0.556382in}}%
\pgfpathlineto{\pgfqpoint{2.446250in}{0.557166in}}%
\pgfpathlineto{\pgfqpoint{2.485000in}{0.560189in}}%
\pgfpathlineto{\pgfqpoint{2.523750in}{0.554877in}}%
\pgfpathlineto{\pgfqpoint{2.562500in}{0.560978in}}%
\pgfpathlineto{\pgfqpoint{2.601250in}{0.566619in}}%
\pgfpathlineto{\pgfqpoint{2.640000in}{0.555800in}}%
\pgfpathlineto{\pgfqpoint{2.678750in}{0.565469in}}%
\pgfpathlineto{\pgfqpoint{2.717500in}{0.564563in}}%
\pgfpathlineto{\pgfqpoint{2.756250in}{0.556791in}}%
\pgfpathlineto{\pgfqpoint{2.795000in}{0.557405in}}%
\pgfpathlineto{\pgfqpoint{2.833750in}{0.556134in}}%
\pgfpathlineto{\pgfqpoint{2.872500in}{0.559637in}}%
\pgfpathlineto{\pgfqpoint{2.911250in}{0.561088in}}%
\pgfpathlineto{\pgfqpoint{2.950000in}{0.557391in}}%
\pgfpathlineto{\pgfqpoint{2.988750in}{0.561104in}}%
\pgfpathlineto{\pgfqpoint{3.027500in}{0.554818in}}%
\pgfpathlineto{\pgfqpoint{3.066250in}{0.560018in}}%
\pgfpathlineto{\pgfqpoint{3.105000in}{0.564067in}}%
\pgfpathlineto{\pgfqpoint{3.143750in}{0.563846in}}%
\pgfpathlineto{\pgfqpoint{3.182500in}{0.562113in}}%
\pgfpathlineto{\pgfqpoint{3.221250in}{0.566434in}}%
\pgfpathlineto{\pgfqpoint{3.260000in}{0.555739in}}%
\pgfpathlineto{\pgfqpoint{3.298750in}{0.559028in}}%
\pgfpathlineto{\pgfqpoint{3.337500in}{0.563057in}}%
\pgfpathlineto{\pgfqpoint{3.376250in}{0.558880in}}%
\pgfpathlineto{\pgfqpoint{3.415000in}{0.558743in}}%
\pgfpathlineto{\pgfqpoint{3.453750in}{0.560035in}}%
\pgfpathlineto{\pgfqpoint{3.492500in}{0.556995in}}%
\pgfpathlineto{\pgfqpoint{3.531250in}{0.563907in}}%
\pgfpathlineto{\pgfqpoint{3.570000in}{0.552635in}}%
\pgfpathlineto{\pgfqpoint{3.608750in}{0.553086in}}%
\pgfpathlineto{\pgfqpoint{3.647500in}{0.554294in}}%
\pgfpathlineto{\pgfqpoint{3.686250in}{0.550773in}}%
\pgfpathlineto{\pgfqpoint{3.725000in}{0.560223in}}%
\pgfpathlineto{\pgfqpoint{3.763750in}{0.555156in}}%
\pgfpathlineto{\pgfqpoint{3.802500in}{0.555548in}}%
\pgfpathlineto{\pgfqpoint{3.841250in}{0.551519in}}%
\pgfpathlineto{\pgfqpoint{3.880000in}{0.555066in}}%
\pgfpathlineto{\pgfqpoint{3.918750in}{0.557629in}}%
\pgfpathlineto{\pgfqpoint{3.957500in}{0.556946in}}%
\pgfpathlineto{\pgfqpoint{3.996250in}{0.556019in}}%
\pgfpathlineto{\pgfqpoint{4.035000in}{0.566448in}}%
\pgfpathlineto{\pgfqpoint{4.073750in}{0.560857in}}%
\pgfpathlineto{\pgfqpoint{4.112500in}{0.560867in}}%
\pgfpathlineto{\pgfqpoint{4.151250in}{0.560276in}}%
\pgfpathlineto{\pgfqpoint{4.190000in}{0.559951in}}%
\pgfpathlineto{\pgfqpoint{4.228750in}{0.557705in}}%
\pgfpathlineto{\pgfqpoint{4.267500in}{0.555153in}}%
\pgfpathlineto{\pgfqpoint{4.306250in}{0.559383in}}%
\pgfpathlineto{\pgfqpoint{4.345000in}{0.553710in}}%
\pgfpathlineto{\pgfqpoint{4.383750in}{0.557917in}}%
\pgfpathlineto{\pgfqpoint{4.422500in}{0.552628in}}%
\pgfpathlineto{\pgfqpoint{4.461250in}{0.554797in}}%
\pgfpathlineto{\pgfqpoint{4.500000in}{0.557296in}}%
\pgfpathlineto{\pgfqpoint{4.500000in}{0.566468in}}%
\pgfpathlineto{\pgfqpoint{4.500000in}{0.566468in}}%
\pgfpathlineto{\pgfqpoint{4.461250in}{0.561105in}}%
\pgfpathlineto{\pgfqpoint{4.422500in}{0.561061in}}%
\pgfpathlineto{\pgfqpoint{4.383750in}{0.568290in}}%
\pgfpathlineto{\pgfqpoint{4.345000in}{0.560742in}}%
\pgfpathlineto{\pgfqpoint{4.306250in}{0.570221in}}%
\pgfpathlineto{\pgfqpoint{4.267500in}{0.564604in}}%
\pgfpathlineto{\pgfqpoint{4.228750in}{0.569265in}}%
\pgfpathlineto{\pgfqpoint{4.190000in}{0.570225in}}%
\pgfpathlineto{\pgfqpoint{4.151250in}{0.568984in}}%
\pgfpathlineto{\pgfqpoint{4.112500in}{0.571332in}}%
\pgfpathlineto{\pgfqpoint{4.073750in}{0.573098in}}%
\pgfpathlineto{\pgfqpoint{4.035000in}{0.591819in}}%
\pgfpathlineto{\pgfqpoint{3.996250in}{0.565417in}}%
\pgfpathlineto{\pgfqpoint{3.957500in}{0.573307in}}%
\pgfpathlineto{\pgfqpoint{3.918750in}{0.574074in}}%
\pgfpathlineto{\pgfqpoint{3.880000in}{0.565760in}}%
\pgfpathlineto{\pgfqpoint{3.841250in}{0.559994in}}%
\pgfpathlineto{\pgfqpoint{3.802500in}{0.568865in}}%
\pgfpathlineto{\pgfqpoint{3.763750in}{0.570555in}}%
\pgfpathlineto{\pgfqpoint{3.725000in}{0.576976in}}%
\pgfpathlineto{\pgfqpoint{3.686250in}{0.559060in}}%
\pgfpathlineto{\pgfqpoint{3.647500in}{0.562906in}}%
\pgfpathlineto{\pgfqpoint{3.608750in}{0.562511in}}%
\pgfpathlineto{\pgfqpoint{3.570000in}{0.560786in}}%
\pgfpathlineto{\pgfqpoint{3.531250in}{0.577605in}}%
\pgfpathlineto{\pgfqpoint{3.492500in}{0.567533in}}%
\pgfpathlineto{\pgfqpoint{3.453750in}{0.572050in}}%
\pgfpathlineto{\pgfqpoint{3.415000in}{0.581929in}}%
\pgfpathlineto{\pgfqpoint{3.376250in}{0.570952in}}%
\pgfpathlineto{\pgfqpoint{3.337500in}{0.581584in}}%
\pgfpathlineto{\pgfqpoint{3.298750in}{0.596376in}}%
\pgfpathlineto{\pgfqpoint{3.260000in}{0.568102in}}%
\pgfpathlineto{\pgfqpoint{3.221250in}{0.611374in}}%
\pgfpathlineto{\pgfqpoint{3.182500in}{0.580238in}}%
\pgfpathlineto{\pgfqpoint{3.143750in}{0.579803in}}%
\pgfpathlineto{\pgfqpoint{3.105000in}{0.580613in}}%
\pgfpathlineto{\pgfqpoint{3.066250in}{0.580768in}}%
\pgfpathlineto{\pgfqpoint{3.027500in}{0.563451in}}%
\pgfpathlineto{\pgfqpoint{2.988750in}{0.574415in}}%
\pgfpathlineto{\pgfqpoint{2.950000in}{0.568663in}}%
\pgfpathlineto{\pgfqpoint{2.911250in}{0.571111in}}%
\pgfpathlineto{\pgfqpoint{2.872500in}{0.595653in}}%
\pgfpathlineto{\pgfqpoint{2.833750in}{0.566638in}}%
\pgfpathlineto{\pgfqpoint{2.795000in}{0.570710in}}%
\pgfpathlineto{\pgfqpoint{2.756250in}{0.570332in}}%
\pgfpathlineto{\pgfqpoint{2.717500in}{0.581146in}}%
\pgfpathlineto{\pgfqpoint{2.678750in}{0.616957in}}%
\pgfpathlineto{\pgfqpoint{2.640000in}{0.564110in}}%
\pgfpathlineto{\pgfqpoint{2.601250in}{0.585350in}}%
\pgfpathlineto{\pgfqpoint{2.562500in}{0.574847in}}%
\pgfpathlineto{\pgfqpoint{2.523750in}{0.565338in}}%
\pgfpathlineto{\pgfqpoint{2.485000in}{0.571514in}}%
\pgfpathlineto{\pgfqpoint{2.446250in}{0.575377in}}%
\pgfpathlineto{\pgfqpoint{2.407500in}{0.565665in}}%
\pgfpathlineto{\pgfqpoint{2.368750in}{0.556395in}}%
\pgfpathlineto{\pgfqpoint{2.330000in}{0.565050in}}%
\pgfpathlineto{\pgfqpoint{2.291250in}{0.575079in}}%
\pgfpathlineto{\pgfqpoint{2.252500in}{0.589979in}}%
\pgfpathlineto{\pgfqpoint{2.213750in}{0.564199in}}%
\pgfpathlineto{\pgfqpoint{2.175000in}{0.567872in}}%
\pgfpathlineto{\pgfqpoint{2.136250in}{0.572331in}}%
\pgfpathlineto{\pgfqpoint{2.097500in}{0.594820in}}%
\pgfpathlineto{\pgfqpoint{2.058750in}{0.566297in}}%
\pgfpathlineto{\pgfqpoint{2.020000in}{0.563159in}}%
\pgfpathlineto{\pgfqpoint{1.981250in}{0.576597in}}%
\pgfpathlineto{\pgfqpoint{1.942500in}{0.579025in}}%
\pgfpathlineto{\pgfqpoint{1.903750in}{0.593930in}}%
\pgfpathlineto{\pgfqpoint{1.865000in}{0.576261in}}%
\pgfpathlineto{\pgfqpoint{1.826250in}{0.620908in}}%
\pgfpathlineto{\pgfqpoint{1.787500in}{0.606158in}}%
\pgfpathlineto{\pgfqpoint{1.748750in}{0.578701in}}%
\pgfpathlineto{\pgfqpoint{1.710000in}{0.595245in}}%
\pgfpathlineto{\pgfqpoint{1.671250in}{0.594953in}}%
\pgfpathlineto{\pgfqpoint{1.632500in}{0.575398in}}%
\pgfpathlineto{\pgfqpoint{1.593750in}{0.596319in}}%
\pgfpathlineto{\pgfqpoint{1.555000in}{0.625274in}}%
\pgfpathlineto{\pgfqpoint{1.516250in}{0.581429in}}%
\pgfpathlineto{\pgfqpoint{1.477500in}{0.581596in}}%
\pgfpathlineto{\pgfqpoint{1.438750in}{0.591768in}}%
\pgfpathlineto{\pgfqpoint{1.400000in}{0.593435in}}%
\pgfpathlineto{\pgfqpoint{1.361250in}{0.614093in}}%
\pgfpathlineto{\pgfqpoint{1.322500in}{0.607774in}}%
\pgfpathlineto{\pgfqpoint{1.283750in}{0.640093in}}%
\pgfpathlineto{\pgfqpoint{1.245000in}{0.645575in}}%
\pgfpathlineto{\pgfqpoint{1.206250in}{0.607304in}}%
\pgfpathlineto{\pgfqpoint{1.167500in}{0.611274in}}%
\pgfpathlineto{\pgfqpoint{1.128750in}{0.672028in}}%
\pgfpathlineto{\pgfqpoint{1.090000in}{0.672639in}}%
\pgfpathlineto{\pgfqpoint{1.051250in}{0.696966in}}%
\pgfpathlineto{\pgfqpoint{1.012500in}{0.737441in}}%
\pgfpathlineto{\pgfqpoint{0.973750in}{0.736197in}}%
\pgfpathlineto{\pgfqpoint{0.935000in}{0.750513in}}%
\pgfpathlineto{\pgfqpoint{0.896250in}{0.845431in}}%
\pgfpathlineto{\pgfqpoint{0.857500in}{1.065626in}}%
\pgfpathlineto{\pgfqpoint{0.818750in}{1.080526in}}%
\pgfpathlineto{\pgfqpoint{0.780000in}{1.584569in}}%
\pgfpathlineto{\pgfqpoint{0.741250in}{1.797788in}}%
\pgfpathlineto{\pgfqpoint{0.702500in}{2.242707in}}%
\pgfpathlineto{\pgfqpoint{0.663750in}{2.758752in}}%
\pgfpathlineto{\pgfqpoint{0.663750in}{2.758752in}}%
\pgfpathclose%
\pgfusepath{stroke,fill}%
}%
\begin{pgfscope}%
\pgfsys@transformshift{0.000000in}{0.000000in}%
\pgfsys@useobject{currentmarker}{}%
\end{pgfscope}%
\end{pgfscope}%
\begin{pgfscope}%
\pgfsetbuttcap%
\pgfsetroundjoin%
\definecolor{currentfill}{rgb}{0.000000,0.000000,0.000000}%
\pgfsetfillcolor{currentfill}%
\pgfsetlinewidth{0.803000pt}%
\definecolor{currentstroke}{rgb}{0.000000,0.000000,0.000000}%
\pgfsetstrokecolor{currentstroke}%
\pgfsetdash{}{0pt}%
\pgfsys@defobject{currentmarker}{\pgfqpoint{0.000000in}{-0.048611in}}{\pgfqpoint{0.000000in}{0.000000in}}{%
\pgfpathmoveto{\pgfqpoint{0.000000in}{0.000000in}}%
\pgfpathlineto{\pgfqpoint{0.000000in}{-0.048611in}}%
\pgfusepath{stroke,fill}%
}%
\begin{pgfscope}%
\pgfsys@transformshift{0.625000in}{0.412500in}%
\pgfsys@useobject{currentmarker}{}%
\end{pgfscope}%
\end{pgfscope}%
\begin{pgfscope}%
\definecolor{textcolor}{rgb}{0.000000,0.000000,0.000000}%
\pgfsetstrokecolor{textcolor}%
\pgfsetfillcolor{textcolor}%
\pgftext[x=0.625000in,y=0.315278in,,top]{\color{textcolor}\rmfamily\fontsize{9.000000}{10.800000}\selectfont \(\displaystyle {0}\)}%
\end{pgfscope}%
\begin{pgfscope}%
\pgfsetbuttcap%
\pgfsetroundjoin%
\definecolor{currentfill}{rgb}{0.000000,0.000000,0.000000}%
\pgfsetfillcolor{currentfill}%
\pgfsetlinewidth{0.803000pt}%
\definecolor{currentstroke}{rgb}{0.000000,0.000000,0.000000}%
\pgfsetstrokecolor{currentstroke}%
\pgfsetdash{}{0pt}%
\pgfsys@defobject{currentmarker}{\pgfqpoint{0.000000in}{-0.048611in}}{\pgfqpoint{0.000000in}{0.000000in}}{%
\pgfpathmoveto{\pgfqpoint{0.000000in}{0.000000in}}%
\pgfpathlineto{\pgfqpoint{0.000000in}{-0.048611in}}%
\pgfusepath{stroke,fill}%
}%
\begin{pgfscope}%
\pgfsys@transformshift{1.400000in}{0.412500in}%
\pgfsys@useobject{currentmarker}{}%
\end{pgfscope}%
\end{pgfscope}%
\begin{pgfscope}%
\definecolor{textcolor}{rgb}{0.000000,0.000000,0.000000}%
\pgfsetstrokecolor{textcolor}%
\pgfsetfillcolor{textcolor}%
\pgftext[x=1.400000in,y=0.315278in,,top]{\color{textcolor}\rmfamily\fontsize{9.000000}{10.800000}\selectfont \(\displaystyle {100}\)}%
\end{pgfscope}%
\begin{pgfscope}%
\pgfsetbuttcap%
\pgfsetroundjoin%
\definecolor{currentfill}{rgb}{0.000000,0.000000,0.000000}%
\pgfsetfillcolor{currentfill}%
\pgfsetlinewidth{0.803000pt}%
\definecolor{currentstroke}{rgb}{0.000000,0.000000,0.000000}%
\pgfsetstrokecolor{currentstroke}%
\pgfsetdash{}{0pt}%
\pgfsys@defobject{currentmarker}{\pgfqpoint{0.000000in}{-0.048611in}}{\pgfqpoint{0.000000in}{0.000000in}}{%
\pgfpathmoveto{\pgfqpoint{0.000000in}{0.000000in}}%
\pgfpathlineto{\pgfqpoint{0.000000in}{-0.048611in}}%
\pgfusepath{stroke,fill}%
}%
\begin{pgfscope}%
\pgfsys@transformshift{2.175000in}{0.412500in}%
\pgfsys@useobject{currentmarker}{}%
\end{pgfscope}%
\end{pgfscope}%
\begin{pgfscope}%
\definecolor{textcolor}{rgb}{0.000000,0.000000,0.000000}%
\pgfsetstrokecolor{textcolor}%
\pgfsetfillcolor{textcolor}%
\pgftext[x=2.175000in,y=0.315278in,,top]{\color{textcolor}\rmfamily\fontsize{9.000000}{10.800000}\selectfont \(\displaystyle {200}\)}%
\end{pgfscope}%
\begin{pgfscope}%
\pgfsetbuttcap%
\pgfsetroundjoin%
\definecolor{currentfill}{rgb}{0.000000,0.000000,0.000000}%
\pgfsetfillcolor{currentfill}%
\pgfsetlinewidth{0.803000pt}%
\definecolor{currentstroke}{rgb}{0.000000,0.000000,0.000000}%
\pgfsetstrokecolor{currentstroke}%
\pgfsetdash{}{0pt}%
\pgfsys@defobject{currentmarker}{\pgfqpoint{0.000000in}{-0.048611in}}{\pgfqpoint{0.000000in}{0.000000in}}{%
\pgfpathmoveto{\pgfqpoint{0.000000in}{0.000000in}}%
\pgfpathlineto{\pgfqpoint{0.000000in}{-0.048611in}}%
\pgfusepath{stroke,fill}%
}%
\begin{pgfscope}%
\pgfsys@transformshift{2.950000in}{0.412500in}%
\pgfsys@useobject{currentmarker}{}%
\end{pgfscope}%
\end{pgfscope}%
\begin{pgfscope}%
\definecolor{textcolor}{rgb}{0.000000,0.000000,0.000000}%
\pgfsetstrokecolor{textcolor}%
\pgfsetfillcolor{textcolor}%
\pgftext[x=2.950000in,y=0.315278in,,top]{\color{textcolor}\rmfamily\fontsize{9.000000}{10.800000}\selectfont \(\displaystyle {300}\)}%
\end{pgfscope}%
\begin{pgfscope}%
\pgfsetbuttcap%
\pgfsetroundjoin%
\definecolor{currentfill}{rgb}{0.000000,0.000000,0.000000}%
\pgfsetfillcolor{currentfill}%
\pgfsetlinewidth{0.803000pt}%
\definecolor{currentstroke}{rgb}{0.000000,0.000000,0.000000}%
\pgfsetstrokecolor{currentstroke}%
\pgfsetdash{}{0pt}%
\pgfsys@defobject{currentmarker}{\pgfqpoint{0.000000in}{-0.048611in}}{\pgfqpoint{0.000000in}{0.000000in}}{%
\pgfpathmoveto{\pgfqpoint{0.000000in}{0.000000in}}%
\pgfpathlineto{\pgfqpoint{0.000000in}{-0.048611in}}%
\pgfusepath{stroke,fill}%
}%
\begin{pgfscope}%
\pgfsys@transformshift{3.725000in}{0.412500in}%
\pgfsys@useobject{currentmarker}{}%
\end{pgfscope}%
\end{pgfscope}%
\begin{pgfscope}%
\definecolor{textcolor}{rgb}{0.000000,0.000000,0.000000}%
\pgfsetstrokecolor{textcolor}%
\pgfsetfillcolor{textcolor}%
\pgftext[x=3.725000in,y=0.315278in,,top]{\color{textcolor}\rmfamily\fontsize{9.000000}{10.800000}\selectfont \(\displaystyle {400}\)}%
\end{pgfscope}%
\begin{pgfscope}%
\pgfsetbuttcap%
\pgfsetroundjoin%
\definecolor{currentfill}{rgb}{0.000000,0.000000,0.000000}%
\pgfsetfillcolor{currentfill}%
\pgfsetlinewidth{0.803000pt}%
\definecolor{currentstroke}{rgb}{0.000000,0.000000,0.000000}%
\pgfsetstrokecolor{currentstroke}%
\pgfsetdash{}{0pt}%
\pgfsys@defobject{currentmarker}{\pgfqpoint{0.000000in}{-0.048611in}}{\pgfqpoint{0.000000in}{0.000000in}}{%
\pgfpathmoveto{\pgfqpoint{0.000000in}{0.000000in}}%
\pgfpathlineto{\pgfqpoint{0.000000in}{-0.048611in}}%
\pgfusepath{stroke,fill}%
}%
\begin{pgfscope}%
\pgfsys@transformshift{4.500000in}{0.412500in}%
\pgfsys@useobject{currentmarker}{}%
\end{pgfscope}%
\end{pgfscope}%
\begin{pgfscope}%
\definecolor{textcolor}{rgb}{0.000000,0.000000,0.000000}%
\pgfsetstrokecolor{textcolor}%
\pgfsetfillcolor{textcolor}%
\pgftext[x=4.500000in,y=0.315278in,,top]{\color{textcolor}\rmfamily\fontsize{9.000000}{10.800000}\selectfont \(\displaystyle {500}\)}%
\end{pgfscope}%
\begin{pgfscope}%
\definecolor{textcolor}{rgb}{0.000000,0.000000,0.000000}%
\pgfsetstrokecolor{textcolor}%
\pgfsetfillcolor{textcolor}%
\pgftext[x=2.562500in,y=0.148611in,,top]{\color{textcolor}\rmfamily\fontsize{11.000000}{13.200000}\selectfont Episode}%
\end{pgfscope}%
\begin{pgfscope}%
\pgfsetbuttcap%
\pgfsetroundjoin%
\definecolor{currentfill}{rgb}{0.000000,0.000000,0.000000}%
\pgfsetfillcolor{currentfill}%
\pgfsetlinewidth{0.803000pt}%
\definecolor{currentstroke}{rgb}{0.000000,0.000000,0.000000}%
\pgfsetstrokecolor{currentstroke}%
\pgfsetdash{}{0pt}%
\pgfsys@defobject{currentmarker}{\pgfqpoint{-0.048611in}{0.000000in}}{\pgfqpoint{-0.000000in}{0.000000in}}{%
\pgfpathmoveto{\pgfqpoint{-0.000000in}{0.000000in}}%
\pgfpathlineto{\pgfqpoint{-0.048611in}{0.000000in}}%
\pgfusepath{stroke,fill}%
}%
\begin{pgfscope}%
\pgfsys@transformshift{0.625000in}{0.977928in}%
\pgfsys@useobject{currentmarker}{}%
\end{pgfscope}%
\end{pgfscope}%
\begin{pgfscope}%
\definecolor{textcolor}{rgb}{0.000000,0.000000,0.000000}%
\pgfsetstrokecolor{textcolor}%
\pgfsetfillcolor{textcolor}%
\pgftext[x=0.335071in, y=0.934526in, left, base]{\color{textcolor}\rmfamily\fontsize{9.000000}{10.800000}\selectfont \(\displaystyle {200}\)}%
\end{pgfscope}%
\begin{pgfscope}%
\pgfsetbuttcap%
\pgfsetroundjoin%
\definecolor{currentfill}{rgb}{0.000000,0.000000,0.000000}%
\pgfsetfillcolor{currentfill}%
\pgfsetlinewidth{0.803000pt}%
\definecolor{currentstroke}{rgb}{0.000000,0.000000,0.000000}%
\pgfsetstrokecolor{currentstroke}%
\pgfsetdash{}{0pt}%
\pgfsys@defobject{currentmarker}{\pgfqpoint{-0.048611in}{0.000000in}}{\pgfqpoint{-0.000000in}{0.000000in}}{%
\pgfpathmoveto{\pgfqpoint{-0.000000in}{0.000000in}}%
\pgfpathlineto{\pgfqpoint{-0.048611in}{0.000000in}}%
\pgfusepath{stroke,fill}%
}%
\begin{pgfscope}%
\pgfsys@transformshift{0.625000in}{1.550417in}%
\pgfsys@useobject{currentmarker}{}%
\end{pgfscope}%
\end{pgfscope}%
\begin{pgfscope}%
\definecolor{textcolor}{rgb}{0.000000,0.000000,0.000000}%
\pgfsetstrokecolor{textcolor}%
\pgfsetfillcolor{textcolor}%
\pgftext[x=0.335071in, y=1.507014in, left, base]{\color{textcolor}\rmfamily\fontsize{9.000000}{10.800000}\selectfont \(\displaystyle {400}\)}%
\end{pgfscope}%
\begin{pgfscope}%
\pgfsetbuttcap%
\pgfsetroundjoin%
\definecolor{currentfill}{rgb}{0.000000,0.000000,0.000000}%
\pgfsetfillcolor{currentfill}%
\pgfsetlinewidth{0.803000pt}%
\definecolor{currentstroke}{rgb}{0.000000,0.000000,0.000000}%
\pgfsetstrokecolor{currentstroke}%
\pgfsetdash{}{0pt}%
\pgfsys@defobject{currentmarker}{\pgfqpoint{-0.048611in}{0.000000in}}{\pgfqpoint{-0.000000in}{0.000000in}}{%
\pgfpathmoveto{\pgfqpoint{-0.000000in}{0.000000in}}%
\pgfpathlineto{\pgfqpoint{-0.048611in}{0.000000in}}%
\pgfusepath{stroke,fill}%
}%
\begin{pgfscope}%
\pgfsys@transformshift{0.625000in}{2.122905in}%
\pgfsys@useobject{currentmarker}{}%
\end{pgfscope}%
\end{pgfscope}%
\begin{pgfscope}%
\definecolor{textcolor}{rgb}{0.000000,0.000000,0.000000}%
\pgfsetstrokecolor{textcolor}%
\pgfsetfillcolor{textcolor}%
\pgftext[x=0.335071in, y=2.079502in, left, base]{\color{textcolor}\rmfamily\fontsize{9.000000}{10.800000}\selectfont \(\displaystyle {600}\)}%
\end{pgfscope}%
\begin{pgfscope}%
\pgfsetbuttcap%
\pgfsetroundjoin%
\definecolor{currentfill}{rgb}{0.000000,0.000000,0.000000}%
\pgfsetfillcolor{currentfill}%
\pgfsetlinewidth{0.803000pt}%
\definecolor{currentstroke}{rgb}{0.000000,0.000000,0.000000}%
\pgfsetstrokecolor{currentstroke}%
\pgfsetdash{}{0pt}%
\pgfsys@defobject{currentmarker}{\pgfqpoint{-0.048611in}{0.000000in}}{\pgfqpoint{-0.000000in}{0.000000in}}{%
\pgfpathmoveto{\pgfqpoint{-0.000000in}{0.000000in}}%
\pgfpathlineto{\pgfqpoint{-0.048611in}{0.000000in}}%
\pgfusepath{stroke,fill}%
}%
\begin{pgfscope}%
\pgfsys@transformshift{0.625000in}{2.695393in}%
\pgfsys@useobject{currentmarker}{}%
\end{pgfscope}%
\end{pgfscope}%
\begin{pgfscope}%
\definecolor{textcolor}{rgb}{0.000000,0.000000,0.000000}%
\pgfsetstrokecolor{textcolor}%
\pgfsetfillcolor{textcolor}%
\pgftext[x=0.335071in, y=2.651990in, left, base]{\color{textcolor}\rmfamily\fontsize{9.000000}{10.800000}\selectfont \(\displaystyle {800}\)}%
\end{pgfscope}%
\begin{pgfscope}%
\pgfsetbuttcap%
\pgfsetroundjoin%
\definecolor{currentfill}{rgb}{0.000000,0.000000,0.000000}%
\pgfsetfillcolor{currentfill}%
\pgfsetlinewidth{0.803000pt}%
\definecolor{currentstroke}{rgb}{0.000000,0.000000,0.000000}%
\pgfsetstrokecolor{currentstroke}%
\pgfsetdash{}{0pt}%
\pgfsys@defobject{currentmarker}{\pgfqpoint{-0.048611in}{0.000000in}}{\pgfqpoint{-0.000000in}{0.000000in}}{%
\pgfpathmoveto{\pgfqpoint{-0.000000in}{0.000000in}}%
\pgfpathlineto{\pgfqpoint{-0.048611in}{0.000000in}}%
\pgfusepath{stroke,fill}%
}%
\begin{pgfscope}%
\pgfsys@transformshift{0.625000in}{3.267881in}%
\pgfsys@useobject{currentmarker}{}%
\end{pgfscope}%
\end{pgfscope}%
\begin{pgfscope}%
\definecolor{textcolor}{rgb}{0.000000,0.000000,0.000000}%
\pgfsetstrokecolor{textcolor}%
\pgfsetfillcolor{textcolor}%
\pgftext[x=0.270835in, y=3.224478in, left, base]{\color{textcolor}\rmfamily\fontsize{9.000000}{10.800000}\selectfont \(\displaystyle {1000}\)}%
\end{pgfscope}%
\begin{pgfscope}%
\definecolor{textcolor}{rgb}{0.000000,0.000000,0.000000}%
\pgfsetstrokecolor{textcolor}%
\pgfsetfillcolor{textcolor}%
\pgftext[x=0.215279in,y=1.856250in,,bottom,rotate=90.000000]{\color{textcolor}\rmfamily\fontsize{11.000000}{13.200000}\selectfont Steps to Goal}%
\end{pgfscope}%
\begin{pgfscope}%
\pgfpathrectangle{\pgfqpoint{0.625000in}{0.412500in}}{\pgfqpoint{3.875000in}{2.887500in}}%
\pgfusepath{clip}%
\pgfsetrectcap%
\pgfsetroundjoin%
\pgfsetlinewidth{1.003750pt}%
\definecolor{currentstroke}{rgb}{0.000000,0.466667,0.733333}%
\pgfsetstrokecolor{currentstroke}%
\pgfsetdash{}{0pt}%
\pgfpathmoveto{\pgfqpoint{0.663750in}{3.067033in}}%
\pgfpathlineto{\pgfqpoint{0.702500in}{2.798498in}}%
\pgfpathlineto{\pgfqpoint{0.741250in}{2.835557in}}%
\pgfpathlineto{\pgfqpoint{0.780000in}{2.797525in}}%
\pgfpathlineto{\pgfqpoint{0.818750in}{2.705717in}}%
\pgfpathlineto{\pgfqpoint{0.857500in}{2.610455in}}%
\pgfpathlineto{\pgfqpoint{0.896250in}{2.715182in}}%
\pgfpathlineto{\pgfqpoint{0.935000in}{2.632839in}}%
\pgfpathlineto{\pgfqpoint{0.973750in}{2.546698in}}%
\pgfpathlineto{\pgfqpoint{1.012500in}{2.621007in}}%
\pgfpathlineto{\pgfqpoint{1.051250in}{2.340259in}}%
\pgfpathlineto{\pgfqpoint{1.090000in}{2.399836in}}%
\pgfpathlineto{\pgfqpoint{1.128750in}{2.440368in}}%
\pgfpathlineto{\pgfqpoint{1.167500in}{2.289308in}}%
\pgfpathlineto{\pgfqpoint{1.206250in}{2.157426in}}%
\pgfpathlineto{\pgfqpoint{1.245000in}{1.960871in}}%
\pgfpathlineto{\pgfqpoint{1.283750in}{1.743994in}}%
\pgfpathlineto{\pgfqpoint{1.322500in}{1.687146in}}%
\pgfpathlineto{\pgfqpoint{1.361250in}{1.774431in}}%
\pgfpathlineto{\pgfqpoint{1.400000in}{1.890818in}}%
\pgfpathlineto{\pgfqpoint{1.438750in}{1.646003in}}%
\pgfpathlineto{\pgfqpoint{1.477500in}{1.523662in}}%
\pgfpathlineto{\pgfqpoint{1.516250in}{1.441644in}}%
\pgfpathlineto{\pgfqpoint{1.555000in}{1.332184in}}%
\pgfpathlineto{\pgfqpoint{1.593750in}{1.059642in}}%
\pgfpathlineto{\pgfqpoint{1.632500in}{1.123837in}}%
\pgfpathlineto{\pgfqpoint{1.671250in}{0.954953in}}%
\pgfpathlineto{\pgfqpoint{1.710000in}{1.012392in}}%
\pgfpathlineto{\pgfqpoint{1.748750in}{0.889345in}}%
\pgfpathlineto{\pgfqpoint{1.787500in}{0.915146in}}%
\pgfpathlineto{\pgfqpoint{1.826250in}{0.788683in}}%
\pgfpathlineto{\pgfqpoint{1.865000in}{0.787194in}}%
\pgfpathlineto{\pgfqpoint{1.903750in}{0.809025in}}%
\pgfpathlineto{\pgfqpoint{1.942500in}{0.707561in}}%
\pgfpathlineto{\pgfqpoint{1.981250in}{0.745174in}}%
\pgfpathlineto{\pgfqpoint{2.020000in}{0.636039in}}%
\pgfpathlineto{\pgfqpoint{2.058750in}{0.677181in}}%
\pgfpathlineto{\pgfqpoint{2.097500in}{0.660178in}}%
\pgfpathlineto{\pgfqpoint{2.136250in}{0.647813in}}%
\pgfpathlineto{\pgfqpoint{2.175000in}{0.639741in}}%
\pgfpathlineto{\pgfqpoint{2.213750in}{0.654683in}}%
\pgfpathlineto{\pgfqpoint{2.252500in}{0.697104in}}%
\pgfpathlineto{\pgfqpoint{2.291250in}{0.656648in}}%
\pgfpathlineto{\pgfqpoint{2.330000in}{0.596899in}}%
\pgfpathlineto{\pgfqpoint{2.368750in}{0.610277in}}%
\pgfpathlineto{\pgfqpoint{2.407500in}{0.635523in}}%
\pgfpathlineto{\pgfqpoint{2.446250in}{0.600582in}}%
\pgfpathlineto{\pgfqpoint{2.485000in}{0.585984in}}%
\pgfpathlineto{\pgfqpoint{2.523750in}{0.571748in}}%
\pgfpathlineto{\pgfqpoint{2.562500in}{0.580488in}}%
\pgfpathlineto{\pgfqpoint{2.601250in}{0.607929in}}%
\pgfpathlineto{\pgfqpoint{2.640000in}{0.592873in}}%
\pgfpathlineto{\pgfqpoint{2.678750in}{0.600449in}}%
\pgfpathlineto{\pgfqpoint{2.717500in}{0.613139in}}%
\pgfpathlineto{\pgfqpoint{2.756250in}{0.597987in}}%
\pgfpathlineto{\pgfqpoint{2.795000in}{0.602930in}}%
\pgfpathlineto{\pgfqpoint{2.833750in}{0.572397in}}%
\pgfpathlineto{\pgfqpoint{2.872500in}{0.573103in}}%
\pgfpathlineto{\pgfqpoint{2.911250in}{0.567130in}}%
\pgfpathlineto{\pgfqpoint{2.950000in}{0.575030in}}%
\pgfpathlineto{\pgfqpoint{2.988750in}{0.584896in}}%
\pgfpathlineto{\pgfqpoint{3.027500in}{0.572798in}}%
\pgfpathlineto{\pgfqpoint{3.066250in}{0.582206in}}%
\pgfpathlineto{\pgfqpoint{3.105000in}{0.599323in}}%
\pgfpathlineto{\pgfqpoint{3.143750in}{0.585335in}}%
\pgfpathlineto{\pgfqpoint{3.182500in}{0.567035in}}%
\pgfpathlineto{\pgfqpoint{3.221250in}{0.576519in}}%
\pgfpathlineto{\pgfqpoint{3.260000in}{0.593293in}}%
\pgfpathlineto{\pgfqpoint{3.298750in}{0.586957in}}%
\pgfpathlineto{\pgfqpoint{3.337500in}{0.583446in}}%
\pgfpathlineto{\pgfqpoint{3.376250in}{0.643290in}}%
\pgfpathlineto{\pgfqpoint{3.415000in}{0.578313in}}%
\pgfpathlineto{\pgfqpoint{3.453750in}{0.596766in}}%
\pgfpathlineto{\pgfqpoint{3.492500in}{0.594228in}}%
\pgfpathlineto{\pgfqpoint{3.531250in}{0.599075in}}%
\pgfpathlineto{\pgfqpoint{3.570000in}{0.576423in}}%
\pgfpathlineto{\pgfqpoint{3.608750in}{0.589037in}}%
\pgfpathlineto{\pgfqpoint{3.647500in}{0.560852in}}%
\pgfpathlineto{\pgfqpoint{3.686250in}{0.580984in}}%
\pgfpathlineto{\pgfqpoint{3.725000in}{0.581824in}}%
\pgfpathlineto{\pgfqpoint{3.763750in}{0.561023in}}%
\pgfpathlineto{\pgfqpoint{3.802500in}{0.564172in}}%
\pgfpathlineto{\pgfqpoint{3.841250in}{0.555375in}}%
\pgfpathlineto{\pgfqpoint{3.880000in}{0.566252in}}%
\pgfpathlineto{\pgfqpoint{3.918750in}{0.565050in}}%
\pgfpathlineto{\pgfqpoint{3.957500in}{0.589743in}}%
\pgfpathlineto{\pgfqpoint{3.996250in}{0.581614in}}%
\pgfpathlineto{\pgfqpoint{4.035000in}{0.555318in}}%
\pgfpathlineto{\pgfqpoint{4.073750in}{0.563333in}}%
\pgfpathlineto{\pgfqpoint{4.112500in}{0.575679in}}%
\pgfpathlineto{\pgfqpoint{4.151250in}{0.577549in}}%
\pgfpathlineto{\pgfqpoint{4.190000in}{0.558810in}}%
\pgfpathlineto{\pgfqpoint{4.228750in}{0.558066in}}%
\pgfpathlineto{\pgfqpoint{4.267500in}{0.569725in}}%
\pgfpathlineto{\pgfqpoint{4.306250in}{0.585640in}}%
\pgfpathlineto{\pgfqpoint{4.345000in}{0.601727in}}%
\pgfpathlineto{\pgfqpoint{4.383750in}{0.567645in}}%
\pgfpathlineto{\pgfqpoint{4.422500in}{0.593102in}}%
\pgfpathlineto{\pgfqpoint{4.461250in}{0.567073in}}%
\pgfpathlineto{\pgfqpoint{4.500000in}{0.562417in}}%
\pgfusepath{stroke}%
\end{pgfscope}%
\begin{pgfscope}%
\pgfpathrectangle{\pgfqpoint{0.625000in}{0.412500in}}{\pgfqpoint{3.875000in}{2.887500in}}%
\pgfusepath{clip}%
\pgfsetrectcap%
\pgfsetroundjoin%
\pgfsetlinewidth{1.003750pt}%
\definecolor{currentstroke}{rgb}{0.933333,0.200000,0.466667}%
\pgfsetstrokecolor{currentstroke}%
\pgfsetdash{}{0pt}%
\pgfpathmoveto{\pgfqpoint{0.663750in}{2.595207in}}%
\pgfpathlineto{\pgfqpoint{0.702500in}{2.155098in}}%
\pgfpathlineto{\pgfqpoint{0.741250in}{1.668006in}}%
\pgfpathlineto{\pgfqpoint{0.780000in}{1.456948in}}%
\pgfpathlineto{\pgfqpoint{0.818750in}{1.002335in}}%
\pgfpathlineto{\pgfqpoint{0.857500in}{0.981306in}}%
\pgfpathlineto{\pgfqpoint{0.896250in}{0.798835in}}%
\pgfpathlineto{\pgfqpoint{0.935000in}{0.710405in}}%
\pgfpathlineto{\pgfqpoint{0.973750in}{0.702275in}}%
\pgfpathlineto{\pgfqpoint{1.012500in}{0.698153in}}%
\pgfpathlineto{\pgfqpoint{1.051250in}{0.671666in}}%
\pgfpathlineto{\pgfqpoint{1.090000in}{0.644893in}}%
\pgfpathlineto{\pgfqpoint{1.128750in}{0.642279in}}%
\pgfpathlineto{\pgfqpoint{1.167500in}{0.598483in}}%
\pgfpathlineto{\pgfqpoint{1.206250in}{0.593808in}}%
\pgfpathlineto{\pgfqpoint{1.245000in}{0.623845in}}%
\pgfpathlineto{\pgfqpoint{1.283750in}{0.614379in}}%
\pgfpathlineto{\pgfqpoint{1.322500in}{0.595296in}}%
\pgfpathlineto{\pgfqpoint{1.361250in}{0.600830in}}%
\pgfpathlineto{\pgfqpoint{1.400000in}{0.583751in}}%
\pgfpathlineto{\pgfqpoint{1.438750in}{0.581977in}}%
\pgfpathlineto{\pgfqpoint{1.477500in}{0.575565in}}%
\pgfpathlineto{\pgfqpoint{1.516250in}{0.575011in}}%
\pgfpathlineto{\pgfqpoint{1.555000in}{0.604838in}}%
\pgfpathlineto{\pgfqpoint{1.593750in}{0.584534in}}%
\pgfpathlineto{\pgfqpoint{1.632500in}{0.569706in}}%
\pgfpathlineto{\pgfqpoint{1.671250in}{0.584381in}}%
\pgfpathlineto{\pgfqpoint{1.710000in}{0.585144in}}%
\pgfpathlineto{\pgfqpoint{1.748750in}{0.572759in}}%
\pgfpathlineto{\pgfqpoint{1.787500in}{0.592854in}}%
\pgfpathlineto{\pgfqpoint{1.826250in}{0.600143in}}%
\pgfpathlineto{\pgfqpoint{1.865000in}{0.570355in}}%
\pgfpathlineto{\pgfqpoint{1.903750in}{0.582320in}}%
\pgfpathlineto{\pgfqpoint{1.942500in}{0.571157in}}%
\pgfpathlineto{\pgfqpoint{1.981250in}{0.570527in}}%
\pgfpathlineto{\pgfqpoint{2.020000in}{0.557741in}}%
\pgfpathlineto{\pgfqpoint{2.058750in}{0.560928in}}%
\pgfpathlineto{\pgfqpoint{2.097500in}{0.582320in}}%
\pgfpathlineto{\pgfqpoint{2.136250in}{0.564726in}}%
\pgfpathlineto{\pgfqpoint{2.175000in}{0.562626in}}%
\pgfpathlineto{\pgfqpoint{2.213750in}{0.559936in}}%
\pgfpathlineto{\pgfqpoint{2.252500in}{0.578084in}}%
\pgfpathlineto{\pgfqpoint{2.291250in}{0.569153in}}%
\pgfpathlineto{\pgfqpoint{2.330000in}{0.559650in}}%
\pgfpathlineto{\pgfqpoint{2.368750in}{0.552665in}}%
\pgfpathlineto{\pgfqpoint{2.407500in}{0.561023in}}%
\pgfpathlineto{\pgfqpoint{2.446250in}{0.566271in}}%
\pgfpathlineto{\pgfqpoint{2.485000in}{0.565851in}}%
\pgfpathlineto{\pgfqpoint{2.523750in}{0.560107in}}%
\pgfpathlineto{\pgfqpoint{2.562500in}{0.567912in}}%
\pgfpathlineto{\pgfqpoint{2.601250in}{0.575985in}}%
\pgfpathlineto{\pgfqpoint{2.640000in}{0.559955in}}%
\pgfpathlineto{\pgfqpoint{2.678750in}{0.591213in}}%
\pgfpathlineto{\pgfqpoint{2.717500in}{0.572855in}}%
\pgfpathlineto{\pgfqpoint{2.756250in}{0.563562in}}%
\pgfpathlineto{\pgfqpoint{2.795000in}{0.564058in}}%
\pgfpathlineto{\pgfqpoint{2.833750in}{0.561386in}}%
\pgfpathlineto{\pgfqpoint{2.872500in}{0.577645in}}%
\pgfpathlineto{\pgfqpoint{2.911250in}{0.566100in}}%
\pgfpathlineto{\pgfqpoint{2.950000in}{0.563027in}}%
\pgfpathlineto{\pgfqpoint{2.988750in}{0.567760in}}%
\pgfpathlineto{\pgfqpoint{3.027500in}{0.559134in}}%
\pgfpathlineto{\pgfqpoint{3.066250in}{0.570393in}}%
\pgfpathlineto{\pgfqpoint{3.105000in}{0.572340in}}%
\pgfpathlineto{\pgfqpoint{3.143750in}{0.571824in}}%
\pgfpathlineto{\pgfqpoint{3.182500in}{0.571176in}}%
\pgfpathlineto{\pgfqpoint{3.221250in}{0.588904in}}%
\pgfpathlineto{\pgfqpoint{3.260000in}{0.561920in}}%
\pgfpathlineto{\pgfqpoint{3.298750in}{0.577702in}}%
\pgfpathlineto{\pgfqpoint{3.337500in}{0.572321in}}%
\pgfpathlineto{\pgfqpoint{3.376250in}{0.564916in}}%
\pgfpathlineto{\pgfqpoint{3.415000in}{0.570336in}}%
\pgfpathlineto{\pgfqpoint{3.453750in}{0.566042in}}%
\pgfpathlineto{\pgfqpoint{3.492500in}{0.562264in}}%
\pgfpathlineto{\pgfqpoint{3.531250in}{0.570756in}}%
\pgfpathlineto{\pgfqpoint{3.570000in}{0.556711in}}%
\pgfpathlineto{\pgfqpoint{3.608750in}{0.557798in}}%
\pgfpathlineto{\pgfqpoint{3.647500in}{0.558600in}}%
\pgfpathlineto{\pgfqpoint{3.686250in}{0.554917in}}%
\pgfpathlineto{\pgfqpoint{3.725000in}{0.568599in}}%
\pgfpathlineto{\pgfqpoint{3.763750in}{0.562855in}}%
\pgfpathlineto{\pgfqpoint{3.802500in}{0.562207in}}%
\pgfpathlineto{\pgfqpoint{3.841250in}{0.555757in}}%
\pgfpathlineto{\pgfqpoint{3.880000in}{0.560413in}}%
\pgfpathlineto{\pgfqpoint{3.918750in}{0.565851in}}%
\pgfpathlineto{\pgfqpoint{3.957500in}{0.565126in}}%
\pgfpathlineto{\pgfqpoint{3.996250in}{0.560718in}}%
\pgfpathlineto{\pgfqpoint{4.035000in}{0.579133in}}%
\pgfpathlineto{\pgfqpoint{4.073750in}{0.566977in}}%
\pgfpathlineto{\pgfqpoint{4.112500in}{0.566100in}}%
\pgfpathlineto{\pgfqpoint{4.151250in}{0.564630in}}%
\pgfpathlineto{\pgfqpoint{4.190000in}{0.565088in}}%
\pgfpathlineto{\pgfqpoint{4.228750in}{0.563485in}}%
\pgfpathlineto{\pgfqpoint{4.267500in}{0.559879in}}%
\pgfpathlineto{\pgfqpoint{4.306250in}{0.564802in}}%
\pgfpathlineto{\pgfqpoint{4.345000in}{0.557226in}}%
\pgfpathlineto{\pgfqpoint{4.383750in}{0.563104in}}%
\pgfpathlineto{\pgfqpoint{4.422500in}{0.556844in}}%
\pgfpathlineto{\pgfqpoint{4.461250in}{0.557951in}}%
\pgfpathlineto{\pgfqpoint{4.500000in}{0.561882in}}%
\pgfusepath{stroke}%
\end{pgfscope}%
\begin{pgfscope}%
\pgfsetrectcap%
\pgfsetmiterjoin%
\pgfsetlinewidth{0.803000pt}%
\definecolor{currentstroke}{rgb}{0.000000,0.000000,0.000000}%
\pgfsetstrokecolor{currentstroke}%
\pgfsetdash{}{0pt}%
\pgfpathmoveto{\pgfqpoint{0.625000in}{0.412500in}}%
\pgfpathlineto{\pgfqpoint{0.625000in}{3.300000in}}%
\pgfusepath{stroke}%
\end{pgfscope}%
\begin{pgfscope}%
\pgfsetrectcap%
\pgfsetmiterjoin%
\pgfsetlinewidth{0.803000pt}%
\definecolor{currentstroke}{rgb}{0.000000,0.000000,0.000000}%
\pgfsetstrokecolor{currentstroke}%
\pgfsetdash{}{0pt}%
\pgfpathmoveto{\pgfqpoint{0.625000in}{0.412500in}}%
\pgfpathlineto{\pgfqpoint{4.500000in}{0.412500in}}%
\pgfusepath{stroke}%
\end{pgfscope}%
\begin{pgfscope}%
\definecolor{textcolor}{rgb}{0.121569,0.466667,0.705882}%
\pgfsetstrokecolor{textcolor}%
\pgfsetfillcolor{textcolor}%
\pgftext[x=1.981250in,y=0.977928in,left,base]{\color{textcolor}\rmfamily\fontsize{14.000000}{16.800000}\selectfont Sarsa(\(\displaystyle \lambda\))}%
\end{pgfscope}%
\begin{pgfscope}%
\definecolor{textcolor}{rgb}{1.000000,0.078431,0.576471}%
\pgfsetstrokecolor{textcolor}%
\pgfsetfillcolor{textcolor}%
\pgftext[x=0.625000in,y=0.448377in,left,base]{\color{textcolor}\rmfamily\fontsize{14.000000}{16.800000}\selectfont GSP}%
\end{pgfscope}%
\end{pgfpicture}%
\makeatother%
\endgroup%

%% file: fourrooms_tab_stepstogoal.pgf
\begingroup%
\makeatletter%
\begin{pgfpicture}%
\pgfpathrectangle{\pgfpointorigin}{\pgfqpoint{5.000000in}{3.750000in}}%
\pgfusepath{use as bounding box, clip}%
\begin{pgfscope}%
\pgfsetbuttcap%
\pgfsetmiterjoin%
\definecolor{currentfill}{rgb}{1.000000,1.000000,1.000000}%
\pgfsetfillcolor{currentfill}%
\pgfsetlinewidth{0.000000pt}%
\definecolor{currentstroke}{rgb}{1.000000,1.000000,1.000000}%
\pgfsetstrokecolor{currentstroke}%
\pgfsetdash{}{0pt}%
\pgfpathmoveto{\pgfqpoint{0.000000in}{0.000000in}}%
\pgfpathlineto{\pgfqpoint{5.000000in}{0.000000in}}%
\pgfpathlineto{\pgfqpoint{5.000000in}{3.750000in}}%
\pgfpathlineto{\pgfqpoint{0.000000in}{3.750000in}}%
\pgfpathlineto{\pgfqpoint{0.000000in}{0.000000in}}%
\pgfpathclose%
\pgfusepath{fill}%
\end{pgfscope}%
\begin{pgfscope}%
\pgfsetbuttcap%
\pgfsetmiterjoin%
\definecolor{currentfill}{rgb}{1.000000,1.000000,1.000000}%
\pgfsetfillcolor{currentfill}%
\pgfsetlinewidth{0.000000pt}%
\definecolor{currentstroke}{rgb}{0.000000,0.000000,0.000000}%
\pgfsetstrokecolor{currentstroke}%
\pgfsetstrokeopacity{0.000000}%
\pgfsetdash{}{0pt}%
\pgfpathmoveto{\pgfqpoint{0.625000in}{0.412500in}}%
\pgfpathlineto{\pgfqpoint{4.500000in}{0.412500in}}%
\pgfpathlineto{\pgfqpoint{4.500000in}{3.300000in}}%
\pgfpathlineto{\pgfqpoint{0.625000in}{3.300000in}}%
\pgfpathlineto{\pgfqpoint{0.625000in}{0.412500in}}%
\pgfpathclose%
\pgfusepath{fill}%
\end{pgfscope}%
\begin{pgfscope}%
\pgfpathrectangle{\pgfqpoint{0.625000in}{0.412500in}}{\pgfqpoint{3.875000in}{2.887500in}}%
\pgfusepath{clip}%
\pgfsetbuttcap%
\pgfsetroundjoin%
\definecolor{currentfill}{rgb}{0.000000,0.466667,0.733333}%
\pgfsetfillcolor{currentfill}%
\pgfsetfillopacity{0.200000}%
\pgfsetlinewidth{0.000000pt}%
\definecolor{currentstroke}{rgb}{0.000000,0.000000,0.000000}%
\pgfsetstrokecolor{currentstroke}%
\pgfsetdash{}{0pt}%
\pgfpathmoveto{\pgfqpoint{0.721875in}{2.330568in}}%
\pgfpathlineto{\pgfqpoint{0.721875in}{2.507895in}}%
\pgfpathlineto{\pgfqpoint{0.818750in}{2.176336in}}%
\pgfpathlineto{\pgfqpoint{0.915625in}{1.808477in}}%
\pgfpathlineto{\pgfqpoint{1.012500in}{1.517495in}}%
\pgfpathlineto{\pgfqpoint{1.109375in}{1.400293in}}%
\pgfpathlineto{\pgfqpoint{1.206250in}{1.278820in}}%
\pgfpathlineto{\pgfqpoint{1.303125in}{1.191389in}}%
\pgfpathlineto{\pgfqpoint{1.400000in}{1.077992in}}%
\pgfpathlineto{\pgfqpoint{1.496875in}{0.908555in}}%
\pgfpathlineto{\pgfqpoint{1.593750in}{0.852225in}}%
\pgfpathlineto{\pgfqpoint{1.690625in}{0.799676in}}%
\pgfpathlineto{\pgfqpoint{1.787500in}{0.756514in}}%
\pgfpathlineto{\pgfqpoint{1.884375in}{0.717297in}}%
\pgfpathlineto{\pgfqpoint{1.981250in}{0.676949in}}%
\pgfpathlineto{\pgfqpoint{2.078125in}{0.631420in}}%
\pgfpathlineto{\pgfqpoint{2.175000in}{0.604587in}}%
\pgfpathlineto{\pgfqpoint{2.271875in}{0.587887in}}%
\pgfpathlineto{\pgfqpoint{2.368750in}{0.582933in}}%
\pgfpathlineto{\pgfqpoint{2.465625in}{0.574315in}}%
\pgfpathlineto{\pgfqpoint{2.562500in}{0.572279in}}%
\pgfpathlineto{\pgfqpoint{2.659375in}{0.568744in}}%
\pgfpathlineto{\pgfqpoint{2.756250in}{0.562631in}}%
\pgfpathlineto{\pgfqpoint{2.853125in}{0.567856in}}%
\pgfpathlineto{\pgfqpoint{2.950000in}{0.555219in}}%
\pgfpathlineto{\pgfqpoint{3.046875in}{0.564453in}}%
\pgfpathlineto{\pgfqpoint{3.143750in}{0.561636in}}%
\pgfpathlineto{\pgfqpoint{3.240625in}{0.560654in}}%
\pgfpathlineto{\pgfqpoint{3.337500in}{0.565387in}}%
\pgfpathlineto{\pgfqpoint{3.434375in}{0.557276in}}%
\pgfpathlineto{\pgfqpoint{3.531250in}{0.562246in}}%
\pgfpathlineto{\pgfqpoint{3.628125in}{0.558094in}}%
\pgfpathlineto{\pgfqpoint{3.725000in}{0.549343in}}%
\pgfpathlineto{\pgfqpoint{3.821875in}{0.551074in}}%
\pgfpathlineto{\pgfqpoint{3.918750in}{0.552541in}}%
\pgfpathlineto{\pgfqpoint{4.015625in}{0.556515in}}%
\pgfpathlineto{\pgfqpoint{4.112500in}{0.552718in}}%
\pgfpathlineto{\pgfqpoint{4.209375in}{0.554143in}}%
\pgfpathlineto{\pgfqpoint{4.306250in}{0.549603in}}%
\pgfpathlineto{\pgfqpoint{4.403125in}{0.548837in}}%
\pgfpathlineto{\pgfqpoint{4.500000in}{0.547551in}}%
\pgfpathlineto{\pgfqpoint{4.500000in}{0.545728in}}%
\pgfpathlineto{\pgfqpoint{4.500000in}{0.545728in}}%
\pgfpathlineto{\pgfqpoint{4.403125in}{0.544730in}}%
\pgfpathlineto{\pgfqpoint{4.306250in}{0.547001in}}%
\pgfpathlineto{\pgfqpoint{4.209375in}{0.547479in}}%
\pgfpathlineto{\pgfqpoint{4.112500in}{0.547374in}}%
\pgfpathlineto{\pgfqpoint{4.015625in}{0.547390in}}%
\pgfpathlineto{\pgfqpoint{3.918750in}{0.548504in}}%
\pgfpathlineto{\pgfqpoint{3.821875in}{0.547286in}}%
\pgfpathlineto{\pgfqpoint{3.725000in}{0.546809in}}%
\pgfpathlineto{\pgfqpoint{3.628125in}{0.549337in}}%
\pgfpathlineto{\pgfqpoint{3.531250in}{0.552662in}}%
\pgfpathlineto{\pgfqpoint{3.434375in}{0.551259in}}%
\pgfpathlineto{\pgfqpoint{3.337500in}{0.555895in}}%
\pgfpathlineto{\pgfqpoint{3.240625in}{0.552699in}}%
\pgfpathlineto{\pgfqpoint{3.143750in}{0.553310in}}%
\pgfpathlineto{\pgfqpoint{3.046875in}{0.554860in}}%
\pgfpathlineto{\pgfqpoint{2.950000in}{0.551045in}}%
\pgfpathlineto{\pgfqpoint{2.853125in}{0.554493in}}%
\pgfpathlineto{\pgfqpoint{2.756250in}{0.554197in}}%
\pgfpathlineto{\pgfqpoint{2.659375in}{0.556829in}}%
\pgfpathlineto{\pgfqpoint{2.562500in}{0.560571in}}%
\pgfpathlineto{\pgfqpoint{2.465625in}{0.562549in}}%
\pgfpathlineto{\pgfqpoint{2.368750in}{0.567833in}}%
\pgfpathlineto{\pgfqpoint{2.271875in}{0.573531in}}%
\pgfpathlineto{\pgfqpoint{2.175000in}{0.585562in}}%
\pgfpathlineto{\pgfqpoint{2.078125in}{0.606430in}}%
\pgfpathlineto{\pgfqpoint{1.981250in}{0.643895in}}%
\pgfpathlineto{\pgfqpoint{1.884375in}{0.679239in}}%
\pgfpathlineto{\pgfqpoint{1.787500in}{0.715138in}}%
\pgfpathlineto{\pgfqpoint{1.690625in}{0.762070in}}%
\pgfpathlineto{\pgfqpoint{1.593750in}{0.808876in}}%
\pgfpathlineto{\pgfqpoint{1.496875in}{0.864021in}}%
\pgfpathlineto{\pgfqpoint{1.400000in}{1.028392in}}%
\pgfpathlineto{\pgfqpoint{1.303125in}{1.138270in}}%
\pgfpathlineto{\pgfqpoint{1.206250in}{1.231317in}}%
\pgfpathlineto{\pgfqpoint{1.109375in}{1.345503in}}%
\pgfpathlineto{\pgfqpoint{1.012500in}{1.463569in}}%
\pgfpathlineto{\pgfqpoint{0.915625in}{1.748778in}}%
\pgfpathlineto{\pgfqpoint{0.818750in}{2.114880in}}%
\pgfpathlineto{\pgfqpoint{0.721875in}{2.330568in}}%
\pgfpathlineto{\pgfqpoint{0.721875in}{2.330568in}}%
\pgfpathclose%
\pgfusepath{fill}%
\end{pgfscope}%
\begin{pgfscope}%
\pgfpathrectangle{\pgfqpoint{0.625000in}{0.412500in}}{\pgfqpoint{3.875000in}{2.887500in}}%
\pgfusepath{clip}%
\pgfsetbuttcap%
\pgfsetroundjoin%
\definecolor{currentfill}{rgb}{0.933333,0.466667,0.200000}%
\pgfsetfillcolor{currentfill}%
\pgfsetfillopacity{0.200000}%
\pgfsetlinewidth{0.000000pt}%
\definecolor{currentstroke}{rgb}{0.000000,0.000000,0.000000}%
\pgfsetstrokecolor{currentstroke}%
\pgfsetdash{}{0pt}%
\pgfpathmoveto{\pgfqpoint{0.721875in}{2.079561in}}%
\pgfpathlineto{\pgfqpoint{0.721875in}{2.268842in}}%
\pgfpathlineto{\pgfqpoint{0.818750in}{3.128246in}}%
\pgfpathlineto{\pgfqpoint{0.915625in}{3.168750in}}%
\pgfpathlineto{\pgfqpoint{1.012500in}{2.798076in}}%
\pgfpathlineto{\pgfqpoint{1.109375in}{2.676533in}}%
\pgfpathlineto{\pgfqpoint{1.206250in}{2.428152in}}%
\pgfpathlineto{\pgfqpoint{1.303125in}{2.168331in}}%
\pgfpathlineto{\pgfqpoint{1.400000in}{2.001688in}}%
\pgfpathlineto{\pgfqpoint{1.496875in}{1.859006in}}%
\pgfpathlineto{\pgfqpoint{1.593750in}{1.625191in}}%
\pgfpathlineto{\pgfqpoint{1.690625in}{1.543539in}}%
\pgfpathlineto{\pgfqpoint{1.787500in}{1.429016in}}%
\pgfpathlineto{\pgfqpoint{1.884375in}{1.401145in}}%
\pgfpathlineto{\pgfqpoint{1.981250in}{1.330194in}}%
\pgfpathlineto{\pgfqpoint{2.078125in}{1.177973in}}%
\pgfpathlineto{\pgfqpoint{2.175000in}{1.128381in}}%
\pgfpathlineto{\pgfqpoint{2.271875in}{1.123158in}}%
\pgfpathlineto{\pgfqpoint{2.368750in}{1.021481in}}%
\pgfpathlineto{\pgfqpoint{2.465625in}{1.049542in}}%
\pgfpathlineto{\pgfqpoint{2.562500in}{1.072124in}}%
\pgfpathlineto{\pgfqpoint{2.659375in}{1.052507in}}%
\pgfpathlineto{\pgfqpoint{2.756250in}{1.005135in}}%
\pgfpathlineto{\pgfqpoint{2.853125in}{1.023752in}}%
\pgfpathlineto{\pgfqpoint{2.950000in}{0.928824in}}%
\pgfpathlineto{\pgfqpoint{3.046875in}{0.938800in}}%
\pgfpathlineto{\pgfqpoint{3.143750in}{0.908830in}}%
\pgfpathlineto{\pgfqpoint{3.240625in}{0.937091in}}%
\pgfpathlineto{\pgfqpoint{3.337500in}{0.859361in}}%
\pgfpathlineto{\pgfqpoint{3.434375in}{0.837904in}}%
\pgfpathlineto{\pgfqpoint{3.531250in}{0.836889in}}%
\pgfpathlineto{\pgfqpoint{3.628125in}{0.876175in}}%
\pgfpathlineto{\pgfqpoint{3.725000in}{0.840195in}}%
\pgfpathlineto{\pgfqpoint{3.821875in}{0.805126in}}%
\pgfpathlineto{\pgfqpoint{3.918750in}{0.779224in}}%
\pgfpathlineto{\pgfqpoint{4.015625in}{0.788397in}}%
\pgfpathlineto{\pgfqpoint{4.112500in}{0.782869in}}%
\pgfpathlineto{\pgfqpoint{4.209375in}{0.762315in}}%
\pgfpathlineto{\pgfqpoint{4.306250in}{0.760693in}}%
\pgfpathlineto{\pgfqpoint{4.403125in}{0.744984in}}%
\pgfpathlineto{\pgfqpoint{4.500000in}{0.765842in}}%
\pgfpathlineto{\pgfqpoint{4.500000in}{0.718594in}}%
\pgfpathlineto{\pgfqpoint{4.500000in}{0.718594in}}%
\pgfpathlineto{\pgfqpoint{4.403125in}{0.710972in}}%
\pgfpathlineto{\pgfqpoint{4.306250in}{0.726917in}}%
\pgfpathlineto{\pgfqpoint{4.209375in}{0.729460in}}%
\pgfpathlineto{\pgfqpoint{4.112500in}{0.742732in}}%
\pgfpathlineto{\pgfqpoint{4.015625in}{0.750754in}}%
\pgfpathlineto{\pgfqpoint{3.918750in}{0.746565in}}%
\pgfpathlineto{\pgfqpoint{3.821875in}{0.768188in}}%
\pgfpathlineto{\pgfqpoint{3.725000in}{0.797570in}}%
\pgfpathlineto{\pgfqpoint{3.628125in}{0.816505in}}%
\pgfpathlineto{\pgfqpoint{3.531250in}{0.793361in}}%
\pgfpathlineto{\pgfqpoint{3.434375in}{0.794591in}}%
\pgfpathlineto{\pgfqpoint{3.337500in}{0.813420in}}%
\pgfpathlineto{\pgfqpoint{3.240625in}{0.865244in}}%
\pgfpathlineto{\pgfqpoint{3.143750in}{0.842417in}}%
\pgfpathlineto{\pgfqpoint{3.046875in}{0.861428in}}%
\pgfpathlineto{\pgfqpoint{2.950000in}{0.862458in}}%
\pgfpathlineto{\pgfqpoint{2.853125in}{0.929792in}}%
\pgfpathlineto{\pgfqpoint{2.756250in}{0.928460in}}%
\pgfpathlineto{\pgfqpoint{2.659375in}{0.961334in}}%
\pgfpathlineto{\pgfqpoint{2.562500in}{0.973711in}}%
\pgfpathlineto{\pgfqpoint{2.465625in}{0.974725in}}%
\pgfpathlineto{\pgfqpoint{2.368750in}{0.954835in}}%
\pgfpathlineto{\pgfqpoint{2.271875in}{1.022784in}}%
\pgfpathlineto{\pgfqpoint{2.175000in}{1.030534in}}%
\pgfpathlineto{\pgfqpoint{2.078125in}{1.073998in}}%
\pgfpathlineto{\pgfqpoint{1.981250in}{1.192114in}}%
\pgfpathlineto{\pgfqpoint{1.884375in}{1.246701in}}%
\pgfpathlineto{\pgfqpoint{1.787500in}{1.278425in}}%
\pgfpathlineto{\pgfqpoint{1.690625in}{1.374844in}}%
\pgfpathlineto{\pgfqpoint{1.593750in}{1.457800in}}%
\pgfpathlineto{\pgfqpoint{1.496875in}{1.664136in}}%
\pgfpathlineto{\pgfqpoint{1.400000in}{1.803546in}}%
\pgfpathlineto{\pgfqpoint{1.303125in}{1.969406in}}%
\pgfpathlineto{\pgfqpoint{1.206250in}{2.219807in}}%
\pgfpathlineto{\pgfqpoint{1.109375in}{2.473706in}}%
\pgfpathlineto{\pgfqpoint{1.012500in}{2.617281in}}%
\pgfpathlineto{\pgfqpoint{0.915625in}{3.016463in}}%
\pgfpathlineto{\pgfqpoint{0.818750in}{3.025789in}}%
\pgfpathlineto{\pgfqpoint{0.721875in}{2.079561in}}%
\pgfpathlineto{\pgfqpoint{0.721875in}{2.079561in}}%
\pgfpathclose%
\pgfusepath{fill}%
\end{pgfscope}%
\begin{pgfscope}%
\pgfpathrectangle{\pgfqpoint{0.625000in}{0.412500in}}{\pgfqpoint{3.875000in}{2.887500in}}%
\pgfusepath{clip}%
\pgfsetbuttcap%
\pgfsetroundjoin%
\definecolor{currentfill}{rgb}{0.933333,0.200000,0.466667}%
\pgfsetfillcolor{currentfill}%
\pgfsetfillopacity{0.200000}%
\pgfsetlinewidth{0.000000pt}%
\definecolor{currentstroke}{rgb}{0.000000,0.000000,0.000000}%
\pgfsetstrokecolor{currentstroke}%
\pgfsetdash{}{0pt}%
\pgfpathmoveto{\pgfqpoint{0.721875in}{1.308166in}}%
\pgfpathlineto{\pgfqpoint{0.721875in}{1.426400in}}%
\pgfpathlineto{\pgfqpoint{0.818750in}{1.287559in}}%
\pgfpathlineto{\pgfqpoint{0.915625in}{0.897158in}}%
\pgfpathlineto{\pgfqpoint{1.012500in}{0.734593in}}%
\pgfpathlineto{\pgfqpoint{1.109375in}{0.641621in}}%
\pgfpathlineto{\pgfqpoint{1.206250in}{0.613742in}}%
\pgfpathlineto{\pgfqpoint{1.303125in}{0.583460in}}%
\pgfpathlineto{\pgfqpoint{1.400000in}{0.564443in}}%
\pgfpathlineto{\pgfqpoint{1.496875in}{0.560723in}}%
\pgfpathlineto{\pgfqpoint{1.593750in}{0.551888in}}%
\pgfpathlineto{\pgfqpoint{1.690625in}{0.551441in}}%
\pgfpathlineto{\pgfqpoint{1.787500in}{0.549508in}}%
\pgfpathlineto{\pgfqpoint{1.884375in}{0.548858in}}%
\pgfpathlineto{\pgfqpoint{1.981250in}{0.547150in}}%
\pgfpathlineto{\pgfqpoint{2.078125in}{0.548544in}}%
\pgfpathlineto{\pgfqpoint{2.175000in}{0.550103in}}%
\pgfpathlineto{\pgfqpoint{2.271875in}{0.548081in}}%
\pgfpathlineto{\pgfqpoint{2.368750in}{0.547586in}}%
\pgfpathlineto{\pgfqpoint{2.465625in}{0.546214in}}%
\pgfpathlineto{\pgfqpoint{2.562500in}{0.547002in}}%
\pgfpathlineto{\pgfqpoint{2.659375in}{0.546058in}}%
\pgfpathlineto{\pgfqpoint{2.756250in}{0.546684in}}%
\pgfpathlineto{\pgfqpoint{2.853125in}{0.545211in}}%
\pgfpathlineto{\pgfqpoint{2.950000in}{0.546237in}}%
\pgfpathlineto{\pgfqpoint{3.046875in}{0.547398in}}%
\pgfpathlineto{\pgfqpoint{3.143750in}{0.550667in}}%
\pgfpathlineto{\pgfqpoint{3.240625in}{0.546942in}}%
\pgfpathlineto{\pgfqpoint{3.337500in}{0.548544in}}%
\pgfpathlineto{\pgfqpoint{3.434375in}{0.546557in}}%
\pgfpathlineto{\pgfqpoint{3.531250in}{0.546685in}}%
\pgfpathlineto{\pgfqpoint{3.628125in}{0.545384in}}%
\pgfpathlineto{\pgfqpoint{3.725000in}{0.547277in}}%
\pgfpathlineto{\pgfqpoint{3.821875in}{0.547181in}}%
\pgfpathlineto{\pgfqpoint{3.918750in}{0.545548in}}%
\pgfpathlineto{\pgfqpoint{4.015625in}{0.545075in}}%
\pgfpathlineto{\pgfqpoint{4.112500in}{0.545988in}}%
\pgfpathlineto{\pgfqpoint{4.209375in}{0.548950in}}%
\pgfpathlineto{\pgfqpoint{4.306250in}{0.546357in}}%
\pgfpathlineto{\pgfqpoint{4.403125in}{0.546151in}}%
\pgfpathlineto{\pgfqpoint{4.500000in}{0.548142in}}%
\pgfpathlineto{\pgfqpoint{4.500000in}{0.544710in}}%
\pgfpathlineto{\pgfqpoint{4.500000in}{0.544710in}}%
\pgfpathlineto{\pgfqpoint{4.403125in}{0.544531in}}%
\pgfpathlineto{\pgfqpoint{4.306250in}{0.544965in}}%
\pgfpathlineto{\pgfqpoint{4.209375in}{0.545307in}}%
\pgfpathlineto{\pgfqpoint{4.112500in}{0.544656in}}%
\pgfpathlineto{\pgfqpoint{4.015625in}{0.543750in}}%
\pgfpathlineto{\pgfqpoint{3.918750in}{0.544192in}}%
\pgfpathlineto{\pgfqpoint{3.821875in}{0.545495in}}%
\pgfpathlineto{\pgfqpoint{3.725000in}{0.545525in}}%
\pgfpathlineto{\pgfqpoint{3.628125in}{0.544043in}}%
\pgfpathlineto{\pgfqpoint{3.531250in}{0.545264in}}%
\pgfpathlineto{\pgfqpoint{3.434375in}{0.544852in}}%
\pgfpathlineto{\pgfqpoint{3.337500in}{0.546528in}}%
\pgfpathlineto{\pgfqpoint{3.240625in}{0.545509in}}%
\pgfpathlineto{\pgfqpoint{3.143750in}{0.545974in}}%
\pgfpathlineto{\pgfqpoint{3.046875in}{0.545968in}}%
\pgfpathlineto{\pgfqpoint{2.950000in}{0.544746in}}%
\pgfpathlineto{\pgfqpoint{2.853125in}{0.543815in}}%
\pgfpathlineto{\pgfqpoint{2.756250in}{0.545139in}}%
\pgfpathlineto{\pgfqpoint{2.659375in}{0.544623in}}%
\pgfpathlineto{\pgfqpoint{2.562500in}{0.545486in}}%
\pgfpathlineto{\pgfqpoint{2.465625in}{0.544706in}}%
\pgfpathlineto{\pgfqpoint{2.368750in}{0.545956in}}%
\pgfpathlineto{\pgfqpoint{2.271875in}{0.546352in}}%
\pgfpathlineto{\pgfqpoint{2.175000in}{0.547516in}}%
\pgfpathlineto{\pgfqpoint{2.078125in}{0.547031in}}%
\pgfpathlineto{\pgfqpoint{1.981250in}{0.545439in}}%
\pgfpathlineto{\pgfqpoint{1.884375in}{0.547018in}}%
\pgfpathlineto{\pgfqpoint{1.787500in}{0.547271in}}%
\pgfpathlineto{\pgfqpoint{1.690625in}{0.548675in}}%
\pgfpathlineto{\pgfqpoint{1.593750in}{0.549245in}}%
\pgfpathlineto{\pgfqpoint{1.496875in}{0.553747in}}%
\pgfpathlineto{\pgfqpoint{1.400000in}{0.557191in}}%
\pgfpathlineto{\pgfqpoint{1.303125in}{0.570480in}}%
\pgfpathlineto{\pgfqpoint{1.206250in}{0.593809in}}%
\pgfpathlineto{\pgfqpoint{1.109375in}{0.620507in}}%
\pgfpathlineto{\pgfqpoint{1.012500in}{0.704161in}}%
\pgfpathlineto{\pgfqpoint{0.915625in}{0.858480in}}%
\pgfpathlineto{\pgfqpoint{0.818750in}{1.240620in}}%
\pgfpathlineto{\pgfqpoint{0.721875in}{1.308166in}}%
\pgfpathlineto{\pgfqpoint{0.721875in}{1.308166in}}%
\pgfpathclose%
\pgfusepath{fill}%
\end{pgfscope}%
\begin{pgfscope}%
\pgfsetbuttcap%
\pgfsetroundjoin%
\definecolor{currentfill}{rgb}{0.000000,0.000000,0.000000}%
\pgfsetfillcolor{currentfill}%
\pgfsetlinewidth{0.803000pt}%
\definecolor{currentstroke}{rgb}{0.000000,0.000000,0.000000}%
\pgfsetstrokecolor{currentstroke}%
\pgfsetdash{}{0pt}%
\pgfsys@defobject{currentmarker}{\pgfqpoint{0.000000in}{-0.048611in}}{\pgfqpoint{0.000000in}{0.000000in}}{%
\pgfpathmoveto{\pgfqpoint{0.000000in}{0.000000in}}%
\pgfpathlineto{\pgfqpoint{0.000000in}{-0.048611in}}%
\pgfusepath{stroke,fill}%
}%
\begin{pgfscope}%
\pgfsys@transformshift{0.625000in}{0.412500in}%
\pgfsys@useobject{currentmarker}{}%
\end{pgfscope}%
\end{pgfscope}%
\begin{pgfscope}%
\definecolor{textcolor}{rgb}{0.000000,0.000000,0.000000}%
\pgfsetstrokecolor{textcolor}%
\pgfsetfillcolor{textcolor}%
\pgftext[x=0.625000in,y=0.315278in,,top]{\color{textcolor}\rmfamily\fontsize{9.000000}{10.800000}\selectfont \(\displaystyle {0}\)}%
\end{pgfscope}%
\begin{pgfscope}%
\pgfsetbuttcap%
\pgfsetroundjoin%
\definecolor{currentfill}{rgb}{0.000000,0.000000,0.000000}%
\pgfsetfillcolor{currentfill}%
\pgfsetlinewidth{0.803000pt}%
\definecolor{currentstroke}{rgb}{0.000000,0.000000,0.000000}%
\pgfsetstrokecolor{currentstroke}%
\pgfsetdash{}{0pt}%
\pgfsys@defobject{currentmarker}{\pgfqpoint{0.000000in}{-0.048611in}}{\pgfqpoint{0.000000in}{0.000000in}}{%
\pgfpathmoveto{\pgfqpoint{0.000000in}{0.000000in}}%
\pgfpathlineto{\pgfqpoint{0.000000in}{-0.048611in}}%
\pgfusepath{stroke,fill}%
}%
\begin{pgfscope}%
\pgfsys@transformshift{1.109375in}{0.412500in}%
\pgfsys@useobject{currentmarker}{}%
\end{pgfscope}%
\end{pgfscope}%
\begin{pgfscope}%
\definecolor{textcolor}{rgb}{0.000000,0.000000,0.000000}%
\pgfsetstrokecolor{textcolor}%
\pgfsetfillcolor{textcolor}%
\pgftext[x=1.109375in,y=0.315278in,,top]{\color{textcolor}\rmfamily\fontsize{9.000000}{10.800000}\selectfont \(\displaystyle {25}\)}%
\end{pgfscope}%
\begin{pgfscope}%
\pgfsetbuttcap%
\pgfsetroundjoin%
\definecolor{currentfill}{rgb}{0.000000,0.000000,0.000000}%
\pgfsetfillcolor{currentfill}%
\pgfsetlinewidth{0.803000pt}%
\definecolor{currentstroke}{rgb}{0.000000,0.000000,0.000000}%
\pgfsetstrokecolor{currentstroke}%
\pgfsetdash{}{0pt}%
\pgfsys@defobject{currentmarker}{\pgfqpoint{0.000000in}{-0.048611in}}{\pgfqpoint{0.000000in}{0.000000in}}{%
\pgfpathmoveto{\pgfqpoint{0.000000in}{0.000000in}}%
\pgfpathlineto{\pgfqpoint{0.000000in}{-0.048611in}}%
\pgfusepath{stroke,fill}%
}%
\begin{pgfscope}%
\pgfsys@transformshift{1.593750in}{0.412500in}%
\pgfsys@useobject{currentmarker}{}%
\end{pgfscope}%
\end{pgfscope}%
\begin{pgfscope}%
\definecolor{textcolor}{rgb}{0.000000,0.000000,0.000000}%
\pgfsetstrokecolor{textcolor}%
\pgfsetfillcolor{textcolor}%
\pgftext[x=1.593750in,y=0.315278in,,top]{\color{textcolor}\rmfamily\fontsize{9.000000}{10.800000}\selectfont \(\displaystyle {50}\)}%
\end{pgfscope}%
\begin{pgfscope}%
\pgfsetbuttcap%
\pgfsetroundjoin%
\definecolor{currentfill}{rgb}{0.000000,0.000000,0.000000}%
\pgfsetfillcolor{currentfill}%
\pgfsetlinewidth{0.803000pt}%
\definecolor{currentstroke}{rgb}{0.000000,0.000000,0.000000}%
\pgfsetstrokecolor{currentstroke}%
\pgfsetdash{}{0pt}%
\pgfsys@defobject{currentmarker}{\pgfqpoint{0.000000in}{-0.048611in}}{\pgfqpoint{0.000000in}{0.000000in}}{%
\pgfpathmoveto{\pgfqpoint{0.000000in}{0.000000in}}%
\pgfpathlineto{\pgfqpoint{0.000000in}{-0.048611in}}%
\pgfusepath{stroke,fill}%
}%
\begin{pgfscope}%
\pgfsys@transformshift{2.078125in}{0.412500in}%
\pgfsys@useobject{currentmarker}{}%
\end{pgfscope}%
\end{pgfscope}%
\begin{pgfscope}%
\definecolor{textcolor}{rgb}{0.000000,0.000000,0.000000}%
\pgfsetstrokecolor{textcolor}%
\pgfsetfillcolor{textcolor}%
\pgftext[x=2.078125in,y=0.315278in,,top]{\color{textcolor}\rmfamily\fontsize{9.000000}{10.800000}\selectfont \(\displaystyle {75}\)}%
\end{pgfscope}%
\begin{pgfscope}%
\pgfsetbuttcap%
\pgfsetroundjoin%
\definecolor{currentfill}{rgb}{0.000000,0.000000,0.000000}%
\pgfsetfillcolor{currentfill}%
\pgfsetlinewidth{0.803000pt}%
\definecolor{currentstroke}{rgb}{0.000000,0.000000,0.000000}%
\pgfsetstrokecolor{currentstroke}%
\pgfsetdash{}{0pt}%
\pgfsys@defobject{currentmarker}{\pgfqpoint{0.000000in}{-0.048611in}}{\pgfqpoint{0.000000in}{0.000000in}}{%
\pgfpathmoveto{\pgfqpoint{0.000000in}{0.000000in}}%
\pgfpathlineto{\pgfqpoint{0.000000in}{-0.048611in}}%
\pgfusepath{stroke,fill}%
}%
\begin{pgfscope}%
\pgfsys@transformshift{2.562500in}{0.412500in}%
\pgfsys@useobject{currentmarker}{}%
\end{pgfscope}%
\end{pgfscope}%
\begin{pgfscope}%
\definecolor{textcolor}{rgb}{0.000000,0.000000,0.000000}%
\pgfsetstrokecolor{textcolor}%
\pgfsetfillcolor{textcolor}%
\pgftext[x=2.562500in,y=0.315278in,,top]{\color{textcolor}\rmfamily\fontsize{9.000000}{10.800000}\selectfont \(\displaystyle {100}\)}%
\end{pgfscope}%
\begin{pgfscope}%
\pgfsetbuttcap%
\pgfsetroundjoin%
\definecolor{currentfill}{rgb}{0.000000,0.000000,0.000000}%
\pgfsetfillcolor{currentfill}%
\pgfsetlinewidth{0.803000pt}%
\definecolor{currentstroke}{rgb}{0.000000,0.000000,0.000000}%
\pgfsetstrokecolor{currentstroke}%
\pgfsetdash{}{0pt}%
\pgfsys@defobject{currentmarker}{\pgfqpoint{0.000000in}{-0.048611in}}{\pgfqpoint{0.000000in}{0.000000in}}{%
\pgfpathmoveto{\pgfqpoint{0.000000in}{0.000000in}}%
\pgfpathlineto{\pgfqpoint{0.000000in}{-0.048611in}}%
\pgfusepath{stroke,fill}%
}%
\begin{pgfscope}%
\pgfsys@transformshift{3.046875in}{0.412500in}%
\pgfsys@useobject{currentmarker}{}%
\end{pgfscope}%
\end{pgfscope}%
\begin{pgfscope}%
\definecolor{textcolor}{rgb}{0.000000,0.000000,0.000000}%
\pgfsetstrokecolor{textcolor}%
\pgfsetfillcolor{textcolor}%
\pgftext[x=3.046875in,y=0.315278in,,top]{\color{textcolor}\rmfamily\fontsize{9.000000}{10.800000}\selectfont \(\displaystyle {125}\)}%
\end{pgfscope}%
\begin{pgfscope}%
\pgfsetbuttcap%
\pgfsetroundjoin%
\definecolor{currentfill}{rgb}{0.000000,0.000000,0.000000}%
\pgfsetfillcolor{currentfill}%
\pgfsetlinewidth{0.803000pt}%
\definecolor{currentstroke}{rgb}{0.000000,0.000000,0.000000}%
\pgfsetstrokecolor{currentstroke}%
\pgfsetdash{}{0pt}%
\pgfsys@defobject{currentmarker}{\pgfqpoint{0.000000in}{-0.048611in}}{\pgfqpoint{0.000000in}{0.000000in}}{%
\pgfpathmoveto{\pgfqpoint{0.000000in}{0.000000in}}%
\pgfpathlineto{\pgfqpoint{0.000000in}{-0.048611in}}%
\pgfusepath{stroke,fill}%
}%
\begin{pgfscope}%
\pgfsys@transformshift{3.531250in}{0.412500in}%
\pgfsys@useobject{currentmarker}{}%
\end{pgfscope}%
\end{pgfscope}%
\begin{pgfscope}%
\definecolor{textcolor}{rgb}{0.000000,0.000000,0.000000}%
\pgfsetstrokecolor{textcolor}%
\pgfsetfillcolor{textcolor}%
\pgftext[x=3.531250in,y=0.315278in,,top]{\color{textcolor}\rmfamily\fontsize{9.000000}{10.800000}\selectfont \(\displaystyle {150}\)}%
\end{pgfscope}%
\begin{pgfscope}%
\pgfsetbuttcap%
\pgfsetroundjoin%
\definecolor{currentfill}{rgb}{0.000000,0.000000,0.000000}%
\pgfsetfillcolor{currentfill}%
\pgfsetlinewidth{0.803000pt}%
\definecolor{currentstroke}{rgb}{0.000000,0.000000,0.000000}%
\pgfsetstrokecolor{currentstroke}%
\pgfsetdash{}{0pt}%
\pgfsys@defobject{currentmarker}{\pgfqpoint{0.000000in}{-0.048611in}}{\pgfqpoint{0.000000in}{0.000000in}}{%
\pgfpathmoveto{\pgfqpoint{0.000000in}{0.000000in}}%
\pgfpathlineto{\pgfqpoint{0.000000in}{-0.048611in}}%
\pgfusepath{stroke,fill}%
}%
\begin{pgfscope}%
\pgfsys@transformshift{4.015625in}{0.412500in}%
\pgfsys@useobject{currentmarker}{}%
\end{pgfscope}%
\end{pgfscope}%
\begin{pgfscope}%
\definecolor{textcolor}{rgb}{0.000000,0.000000,0.000000}%
\pgfsetstrokecolor{textcolor}%
\pgfsetfillcolor{textcolor}%
\pgftext[x=4.015625in,y=0.315278in,,top]{\color{textcolor}\rmfamily\fontsize{9.000000}{10.800000}\selectfont \(\displaystyle {175}\)}%
\end{pgfscope}%
\begin{pgfscope}%
\pgfsetbuttcap%
\pgfsetroundjoin%
\definecolor{currentfill}{rgb}{0.000000,0.000000,0.000000}%
\pgfsetfillcolor{currentfill}%
\pgfsetlinewidth{0.803000pt}%
\definecolor{currentstroke}{rgb}{0.000000,0.000000,0.000000}%
\pgfsetstrokecolor{currentstroke}%
\pgfsetdash{}{0pt}%
\pgfsys@defobject{currentmarker}{\pgfqpoint{0.000000in}{-0.048611in}}{\pgfqpoint{0.000000in}{0.000000in}}{%
\pgfpathmoveto{\pgfqpoint{0.000000in}{0.000000in}}%
\pgfpathlineto{\pgfqpoint{0.000000in}{-0.048611in}}%
\pgfusepath{stroke,fill}%
}%
\begin{pgfscope}%
\pgfsys@transformshift{4.500000in}{0.412500in}%
\pgfsys@useobject{currentmarker}{}%
\end{pgfscope}%
\end{pgfscope}%
\begin{pgfscope}%
\definecolor{textcolor}{rgb}{0.000000,0.000000,0.000000}%
\pgfsetstrokecolor{textcolor}%
\pgfsetfillcolor{textcolor}%
\pgftext[x=4.500000in,y=0.315278in,,top]{\color{textcolor}\rmfamily\fontsize{9.000000}{10.800000}\selectfont \(\displaystyle {200}\)}%
\end{pgfscope}%
\begin{pgfscope}%
\definecolor{textcolor}{rgb}{0.000000,0.000000,0.000000}%
\pgfsetstrokecolor{textcolor}%
\pgfsetfillcolor{textcolor}%
\pgftext[x=2.562500in,y=0.148611in,,top]{\color{textcolor}\rmfamily\fontsize{11.000000}{13.200000}\selectfont Episode}%
\end{pgfscope}%
\begin{pgfscope}%
\pgfsetbuttcap%
\pgfsetroundjoin%
\definecolor{currentfill}{rgb}{0.000000,0.000000,0.000000}%
\pgfsetfillcolor{currentfill}%
\pgfsetlinewidth{0.803000pt}%
\definecolor{currentstroke}{rgb}{0.000000,0.000000,0.000000}%
\pgfsetstrokecolor{currentstroke}%
\pgfsetdash{}{0pt}%
\pgfsys@defobject{currentmarker}{\pgfqpoint{-0.048611in}{0.000000in}}{\pgfqpoint{-0.000000in}{0.000000in}}{%
\pgfpathmoveto{\pgfqpoint{-0.000000in}{0.000000in}}%
\pgfpathlineto{\pgfqpoint{-0.048611in}{0.000000in}}%
\pgfusepath{stroke,fill}%
}%
\begin{pgfscope}%
\pgfsys@transformshift{0.625000in}{0.474140in}%
\pgfsys@useobject{currentmarker}{}%
\end{pgfscope}%
\end{pgfscope}%
\begin{pgfscope}%
\definecolor{textcolor}{rgb}{0.000000,0.000000,0.000000}%
\pgfsetstrokecolor{textcolor}%
\pgfsetfillcolor{textcolor}%
\pgftext[x=0.463542in, y=0.430737in, left, base]{\color{textcolor}\rmfamily\fontsize{9.000000}{10.800000}\selectfont \(\displaystyle {0}\)}%
\end{pgfscope}%
\begin{pgfscope}%
\pgfsetbuttcap%
\pgfsetroundjoin%
\definecolor{currentfill}{rgb}{0.000000,0.000000,0.000000}%
\pgfsetfillcolor{currentfill}%
\pgfsetlinewidth{0.803000pt}%
\definecolor{currentstroke}{rgb}{0.000000,0.000000,0.000000}%
\pgfsetstrokecolor{currentstroke}%
\pgfsetdash{}{0pt}%
\pgfsys@defobject{currentmarker}{\pgfqpoint{-0.048611in}{0.000000in}}{\pgfqpoint{-0.000000in}{0.000000in}}{%
\pgfpathmoveto{\pgfqpoint{-0.000000in}{0.000000in}}%
\pgfpathlineto{\pgfqpoint{-0.048611in}{0.000000in}}%
\pgfusepath{stroke,fill}%
}%
\begin{pgfscope}%
\pgfsys@transformshift{0.625000in}{1.101458in}%
\pgfsys@useobject{currentmarker}{}%
\end{pgfscope}%
\end{pgfscope}%
\begin{pgfscope}%
\definecolor{textcolor}{rgb}{0.000000,0.000000,0.000000}%
\pgfsetstrokecolor{textcolor}%
\pgfsetfillcolor{textcolor}%
\pgftext[x=0.335071in, y=1.058055in, left, base]{\color{textcolor}\rmfamily\fontsize{9.000000}{10.800000}\selectfont \(\displaystyle {200}\)}%
\end{pgfscope}%
\begin{pgfscope}%
\pgfsetbuttcap%
\pgfsetroundjoin%
\definecolor{currentfill}{rgb}{0.000000,0.000000,0.000000}%
\pgfsetfillcolor{currentfill}%
\pgfsetlinewidth{0.803000pt}%
\definecolor{currentstroke}{rgb}{0.000000,0.000000,0.000000}%
\pgfsetstrokecolor{currentstroke}%
\pgfsetdash{}{0pt}%
\pgfsys@defobject{currentmarker}{\pgfqpoint{-0.048611in}{0.000000in}}{\pgfqpoint{-0.000000in}{0.000000in}}{%
\pgfpathmoveto{\pgfqpoint{-0.000000in}{0.000000in}}%
\pgfpathlineto{\pgfqpoint{-0.048611in}{0.000000in}}%
\pgfusepath{stroke,fill}%
}%
\begin{pgfscope}%
\pgfsys@transformshift{0.625000in}{1.728775in}%
\pgfsys@useobject{currentmarker}{}%
\end{pgfscope}%
\end{pgfscope}%
\begin{pgfscope}%
\definecolor{textcolor}{rgb}{0.000000,0.000000,0.000000}%
\pgfsetstrokecolor{textcolor}%
\pgfsetfillcolor{textcolor}%
\pgftext[x=0.335071in, y=1.685372in, left, base]{\color{textcolor}\rmfamily\fontsize{9.000000}{10.800000}\selectfont \(\displaystyle {400}\)}%
\end{pgfscope}%
\begin{pgfscope}%
\pgfsetbuttcap%
\pgfsetroundjoin%
\definecolor{currentfill}{rgb}{0.000000,0.000000,0.000000}%
\pgfsetfillcolor{currentfill}%
\pgfsetlinewidth{0.803000pt}%
\definecolor{currentstroke}{rgb}{0.000000,0.000000,0.000000}%
\pgfsetstrokecolor{currentstroke}%
\pgfsetdash{}{0pt}%
\pgfsys@defobject{currentmarker}{\pgfqpoint{-0.048611in}{0.000000in}}{\pgfqpoint{-0.000000in}{0.000000in}}{%
\pgfpathmoveto{\pgfqpoint{-0.000000in}{0.000000in}}%
\pgfpathlineto{\pgfqpoint{-0.048611in}{0.000000in}}%
\pgfusepath{stroke,fill}%
}%
\begin{pgfscope}%
\pgfsys@transformshift{0.625000in}{2.356092in}%
\pgfsys@useobject{currentmarker}{}%
\end{pgfscope}%
\end{pgfscope}%
\begin{pgfscope}%
\definecolor{textcolor}{rgb}{0.000000,0.000000,0.000000}%
\pgfsetstrokecolor{textcolor}%
\pgfsetfillcolor{textcolor}%
\pgftext[x=0.335071in, y=2.312689in, left, base]{\color{textcolor}\rmfamily\fontsize{9.000000}{10.800000}\selectfont \(\displaystyle {600}\)}%
\end{pgfscope}%
\begin{pgfscope}%
\pgfsetbuttcap%
\pgfsetroundjoin%
\definecolor{currentfill}{rgb}{0.000000,0.000000,0.000000}%
\pgfsetfillcolor{currentfill}%
\pgfsetlinewidth{0.803000pt}%
\definecolor{currentstroke}{rgb}{0.000000,0.000000,0.000000}%
\pgfsetstrokecolor{currentstroke}%
\pgfsetdash{}{0pt}%
\pgfsys@defobject{currentmarker}{\pgfqpoint{-0.048611in}{0.000000in}}{\pgfqpoint{-0.000000in}{0.000000in}}{%
\pgfpathmoveto{\pgfqpoint{-0.000000in}{0.000000in}}%
\pgfpathlineto{\pgfqpoint{-0.048611in}{0.000000in}}%
\pgfusepath{stroke,fill}%
}%
\begin{pgfscope}%
\pgfsys@transformshift{0.625000in}{2.983409in}%
\pgfsys@useobject{currentmarker}{}%
\end{pgfscope}%
\end{pgfscope}%
\begin{pgfscope}%
\definecolor{textcolor}{rgb}{0.000000,0.000000,0.000000}%
\pgfsetstrokecolor{textcolor}%
\pgfsetfillcolor{textcolor}%
\pgftext[x=0.335071in, y=2.940006in, left, base]{\color{textcolor}\rmfamily\fontsize{9.000000}{10.800000}\selectfont \(\displaystyle {800}\)}%
\end{pgfscope}%
\begin{pgfscope}%
\definecolor{textcolor}{rgb}{0.000000,0.000000,0.000000}%
\pgfsetstrokecolor{textcolor}%
\pgfsetfillcolor{textcolor}%
\pgftext[x=0.279515in,y=1.856250in,,bottom,rotate=90.000000]{\color{textcolor}\rmfamily\fontsize{11.000000}{13.200000}\selectfont Steps to goal}%
\end{pgfscope}%
\begin{pgfscope}%
\pgfpathrectangle{\pgfqpoint{0.625000in}{0.412500in}}{\pgfqpoint{3.875000in}{2.887500in}}%
\pgfusepath{clip}%
\pgfsetrectcap%
\pgfsetroundjoin%
\pgfsetlinewidth{1.003750pt}%
\definecolor{currentstroke}{rgb}{0.000000,0.466667,0.733333}%
\pgfsetstrokecolor{currentstroke}%
\pgfsetdash{}{0pt}%
\pgfpathmoveto{\pgfqpoint{0.721875in}{2.419231in}}%
\pgfpathlineto{\pgfqpoint{0.818750in}{2.145608in}}%
\pgfpathlineto{\pgfqpoint{0.915625in}{1.778628in}}%
\pgfpathlineto{\pgfqpoint{1.012500in}{1.490532in}}%
\pgfpathlineto{\pgfqpoint{1.109375in}{1.372898in}}%
\pgfpathlineto{\pgfqpoint{1.206250in}{1.255069in}}%
\pgfpathlineto{\pgfqpoint{1.303125in}{1.164829in}}%
\pgfpathlineto{\pgfqpoint{1.400000in}{1.053192in}}%
\pgfpathlineto{\pgfqpoint{1.496875in}{0.886288in}}%
\pgfpathlineto{\pgfqpoint{1.593750in}{0.830551in}}%
\pgfpathlineto{\pgfqpoint{1.690625in}{0.780873in}}%
\pgfpathlineto{\pgfqpoint{1.787500in}{0.735826in}}%
\pgfpathlineto{\pgfqpoint{1.884375in}{0.698268in}}%
\pgfpathlineto{\pgfqpoint{1.981250in}{0.660422in}}%
\pgfpathlineto{\pgfqpoint{2.078125in}{0.618925in}}%
\pgfpathlineto{\pgfqpoint{2.175000in}{0.595074in}}%
\pgfpathlineto{\pgfqpoint{2.271875in}{0.580709in}}%
\pgfpathlineto{\pgfqpoint{2.368750in}{0.575383in}}%
\pgfpathlineto{\pgfqpoint{2.465625in}{0.568432in}}%
\pgfpathlineto{\pgfqpoint{2.562500in}{0.566425in}}%
\pgfpathlineto{\pgfqpoint{2.659375in}{0.562786in}}%
\pgfpathlineto{\pgfqpoint{2.756250in}{0.558414in}}%
\pgfpathlineto{\pgfqpoint{2.853125in}{0.561174in}}%
\pgfpathlineto{\pgfqpoint{2.950000in}{0.553132in}}%
\pgfpathlineto{\pgfqpoint{3.046875in}{0.559656in}}%
\pgfpathlineto{\pgfqpoint{3.143750in}{0.557473in}}%
\pgfpathlineto{\pgfqpoint{3.240625in}{0.556676in}}%
\pgfpathlineto{\pgfqpoint{3.337500in}{0.560641in}}%
\pgfpathlineto{\pgfqpoint{3.434375in}{0.554267in}}%
\pgfpathlineto{\pgfqpoint{3.531250in}{0.557454in}}%
\pgfpathlineto{\pgfqpoint{3.628125in}{0.553715in}}%
\pgfpathlineto{\pgfqpoint{3.725000in}{0.548076in}}%
\pgfpathlineto{\pgfqpoint{3.821875in}{0.549180in}}%
\pgfpathlineto{\pgfqpoint{3.918750in}{0.550522in}}%
\pgfpathlineto{\pgfqpoint{4.015625in}{0.551953in}}%
\pgfpathlineto{\pgfqpoint{4.112500in}{0.550046in}}%
\pgfpathlineto{\pgfqpoint{4.209375in}{0.550811in}}%
\pgfpathlineto{\pgfqpoint{4.306250in}{0.548302in}}%
\pgfpathlineto{\pgfqpoint{4.403125in}{0.546784in}}%
\pgfpathlineto{\pgfqpoint{4.500000in}{0.546639in}}%
\pgfusepath{stroke}%
\end{pgfscope}%
\begin{pgfscope}%
\pgfpathrectangle{\pgfqpoint{0.625000in}{0.412500in}}{\pgfqpoint{3.875000in}{2.887500in}}%
\pgfusepath{clip}%
\pgfsetrectcap%
\pgfsetroundjoin%
\pgfsetlinewidth{1.003750pt}%
\definecolor{currentstroke}{rgb}{0.933333,0.466667,0.200000}%
\pgfsetstrokecolor{currentstroke}%
\pgfsetdash{}{0pt}%
\pgfpathmoveto{\pgfqpoint{0.721875in}{2.174201in}}%
\pgfpathlineto{\pgfqpoint{0.818750in}{3.077018in}}%
\pgfpathlineto{\pgfqpoint{0.915625in}{3.092606in}}%
\pgfpathlineto{\pgfqpoint{1.012500in}{2.707678in}}%
\pgfpathlineto{\pgfqpoint{1.109375in}{2.575120in}}%
\pgfpathlineto{\pgfqpoint{1.206250in}{2.323980in}}%
\pgfpathlineto{\pgfqpoint{1.303125in}{2.068869in}}%
\pgfpathlineto{\pgfqpoint{1.400000in}{1.902617in}}%
\pgfpathlineto{\pgfqpoint{1.496875in}{1.761571in}}%
\pgfpathlineto{\pgfqpoint{1.593750in}{1.541495in}}%
\pgfpathlineto{\pgfqpoint{1.690625in}{1.459191in}}%
\pgfpathlineto{\pgfqpoint{1.787500in}{1.353721in}}%
\pgfpathlineto{\pgfqpoint{1.884375in}{1.323923in}}%
\pgfpathlineto{\pgfqpoint{1.981250in}{1.261154in}}%
\pgfpathlineto{\pgfqpoint{2.078125in}{1.125986in}}%
\pgfpathlineto{\pgfqpoint{2.175000in}{1.079457in}}%
\pgfpathlineto{\pgfqpoint{2.271875in}{1.072971in}}%
\pgfpathlineto{\pgfqpoint{2.368750in}{0.988158in}}%
\pgfpathlineto{\pgfqpoint{2.465625in}{1.012134in}}%
\pgfpathlineto{\pgfqpoint{2.562500in}{1.022917in}}%
\pgfpathlineto{\pgfqpoint{2.659375in}{1.006921in}}%
\pgfpathlineto{\pgfqpoint{2.756250in}{0.966798in}}%
\pgfpathlineto{\pgfqpoint{2.853125in}{0.976772in}}%
\pgfpathlineto{\pgfqpoint{2.950000in}{0.895641in}}%
\pgfpathlineto{\pgfqpoint{3.046875in}{0.900114in}}%
\pgfpathlineto{\pgfqpoint{3.143750in}{0.875623in}}%
\pgfpathlineto{\pgfqpoint{3.240625in}{0.901168in}}%
\pgfpathlineto{\pgfqpoint{3.337500in}{0.836391in}}%
\pgfpathlineto{\pgfqpoint{3.434375in}{0.816248in}}%
\pgfpathlineto{\pgfqpoint{3.531250in}{0.815125in}}%
\pgfpathlineto{\pgfqpoint{3.628125in}{0.846340in}}%
\pgfpathlineto{\pgfqpoint{3.725000in}{0.818882in}}%
\pgfpathlineto{\pgfqpoint{3.821875in}{0.786657in}}%
\pgfpathlineto{\pgfqpoint{3.918750in}{0.762894in}}%
\pgfpathlineto{\pgfqpoint{4.015625in}{0.769575in}}%
\pgfpathlineto{\pgfqpoint{4.112500in}{0.762800in}}%
\pgfpathlineto{\pgfqpoint{4.209375in}{0.745888in}}%
\pgfpathlineto{\pgfqpoint{4.306250in}{0.743805in}}%
\pgfpathlineto{\pgfqpoint{4.403125in}{0.727978in}}%
\pgfpathlineto{\pgfqpoint{4.500000in}{0.742218in}}%
\pgfusepath{stroke}%
\end{pgfscope}%
\begin{pgfscope}%
\pgfpathrectangle{\pgfqpoint{0.625000in}{0.412500in}}{\pgfqpoint{3.875000in}{2.887500in}}%
\pgfusepath{clip}%
\pgfsetrectcap%
\pgfsetroundjoin%
\pgfsetlinewidth{1.003750pt}%
\definecolor{currentstroke}{rgb}{0.933333,0.200000,0.466667}%
\pgfsetstrokecolor{currentstroke}%
\pgfsetdash{}{0pt}%
\pgfpathmoveto{\pgfqpoint{0.721875in}{1.367283in}}%
\pgfpathlineto{\pgfqpoint{0.818750in}{1.264090in}}%
\pgfpathlineto{\pgfqpoint{0.915625in}{0.877819in}}%
\pgfpathlineto{\pgfqpoint{1.012500in}{0.719377in}}%
\pgfpathlineto{\pgfqpoint{1.109375in}{0.631064in}}%
\pgfpathlineto{\pgfqpoint{1.206250in}{0.603775in}}%
\pgfpathlineto{\pgfqpoint{1.303125in}{0.576970in}}%
\pgfpathlineto{\pgfqpoint{1.400000in}{0.560817in}}%
\pgfpathlineto{\pgfqpoint{1.496875in}{0.557235in}}%
\pgfpathlineto{\pgfqpoint{1.593750in}{0.550566in}}%
\pgfpathlineto{\pgfqpoint{1.690625in}{0.550058in}}%
\pgfpathlineto{\pgfqpoint{1.787500in}{0.548390in}}%
\pgfpathlineto{\pgfqpoint{1.884375in}{0.547938in}}%
\pgfpathlineto{\pgfqpoint{1.981250in}{0.546294in}}%
\pgfpathlineto{\pgfqpoint{2.078125in}{0.547787in}}%
\pgfpathlineto{\pgfqpoint{2.175000in}{0.548810in}}%
\pgfpathlineto{\pgfqpoint{2.271875in}{0.547216in}}%
\pgfpathlineto{\pgfqpoint{2.368750in}{0.546771in}}%
\pgfpathlineto{\pgfqpoint{2.465625in}{0.545460in}}%
\pgfpathlineto{\pgfqpoint{2.562500in}{0.546244in}}%
\pgfpathlineto{\pgfqpoint{2.659375in}{0.545341in}}%
\pgfpathlineto{\pgfqpoint{2.756250in}{0.545912in}}%
\pgfpathlineto{\pgfqpoint{2.853125in}{0.544513in}}%
\pgfpathlineto{\pgfqpoint{2.950000in}{0.545491in}}%
\pgfpathlineto{\pgfqpoint{3.046875in}{0.546683in}}%
\pgfpathlineto{\pgfqpoint{3.143750in}{0.548321in}}%
\pgfpathlineto{\pgfqpoint{3.240625in}{0.546225in}}%
\pgfpathlineto{\pgfqpoint{3.337500in}{0.547536in}}%
\pgfpathlineto{\pgfqpoint{3.434375in}{0.545705in}}%
\pgfpathlineto{\pgfqpoint{3.531250in}{0.545974in}}%
\pgfpathlineto{\pgfqpoint{3.628125in}{0.544713in}}%
\pgfpathlineto{\pgfqpoint{3.725000in}{0.546401in}}%
\pgfpathlineto{\pgfqpoint{3.821875in}{0.546338in}}%
\pgfpathlineto{\pgfqpoint{3.918750in}{0.544870in}}%
\pgfpathlineto{\pgfqpoint{4.015625in}{0.544412in}}%
\pgfpathlineto{\pgfqpoint{4.112500in}{0.545322in}}%
\pgfpathlineto{\pgfqpoint{4.209375in}{0.547129in}}%
\pgfpathlineto{\pgfqpoint{4.306250in}{0.545661in}}%
\pgfpathlineto{\pgfqpoint{4.403125in}{0.545341in}}%
\pgfpathlineto{\pgfqpoint{4.500000in}{0.546426in}}%
\pgfusepath{stroke}%
\end{pgfscope}%
\begin{pgfscope}%
\pgfpathrectangle{\pgfqpoint{0.625000in}{0.412500in}}{\pgfqpoint{3.875000in}{2.887500in}}%
\pgfusepath{clip}%
\pgfsetbuttcap%
\pgfsetroundjoin%
\pgfsetlinewidth{1.003750pt}%
\definecolor{currentstroke}{rgb}{0.501961,0.000000,0.501961}%
\pgfsetstrokecolor{currentstroke}%
\pgfsetdash{{3.700000pt}{1.600000pt}}{0.000000pt}%
\pgfpathmoveto{\pgfqpoint{0.625000in}{0.533396in}}%
\pgfpathlineto{\pgfqpoint{4.500000in}{0.533396in}}%
\pgfusepath{stroke}%
\end{pgfscope}%
\begin{pgfscope}%
\pgfsetrectcap%
\pgfsetmiterjoin%
\pgfsetlinewidth{0.803000pt}%
\definecolor{currentstroke}{rgb}{0.000000,0.000000,0.000000}%
\pgfsetstrokecolor{currentstroke}%
\pgfsetdash{}{0pt}%
\pgfpathmoveto{\pgfqpoint{0.625000in}{0.412500in}}%
\pgfpathlineto{\pgfqpoint{0.625000in}{3.300000in}}%
\pgfusepath{stroke}%
\end{pgfscope}%
\begin{pgfscope}%
\pgfsetrectcap%
\pgfsetmiterjoin%
\pgfsetlinewidth{0.803000pt}%
\definecolor{currentstroke}{rgb}{0.000000,0.000000,0.000000}%
\pgfsetstrokecolor{currentstroke}%
\pgfsetdash{}{0pt}%
\pgfpathmoveto{\pgfqpoint{0.625000in}{0.412500in}}%
\pgfpathlineto{\pgfqpoint{4.500000in}{0.412500in}}%
\pgfusepath{stroke}%
\end{pgfscope}%
\begin{pgfscope}%
\definecolor{textcolor}{rgb}{0.121569,0.466667,0.705882}%
\pgfsetstrokecolor{textcolor}%
\pgfsetfillcolor{textcolor}%
\pgftext[x=1.593750in,y=0.850531in,left,base]{\color{textcolor}\rmfamily\fontsize{14.000000}{16.800000}\selectfont Sarsa(\(\displaystyle \lambda\))}%
\end{pgfscope}%
\begin{pgfscope}%
\definecolor{textcolor}{rgb}{0.933333,0.200000,0.466667}%
\pgfsetstrokecolor{textcolor}%
\pgfsetfillcolor{textcolor}%
\pgftext[x=1.109375in,y=0.662335in,left,base]{\color{textcolor}\rmfamily\fontsize{14.000000}{16.800000}\selectfont GSP}%
\end{pgfscope}%
\begin{pgfscope}%
\definecolor{textcolor}{rgb}{1.000000,0.498039,0.054902}%
\pgfsetstrokecolor{textcolor}%
\pgfsetfillcolor{textcolor}%
\pgftext[x=1.109375in,y=2.983409in,left,base]{\color{textcolor}\rmfamily\fontsize{14.000000}{16.800000}\selectfont Approximate LAVI\(\displaystyle \)}%
\end{pgfscope}%
\begin{pgfscope}%
\definecolor{textcolor}{rgb}{0.501961,0.000000,0.501961}%
\pgfsetstrokecolor{textcolor}%
\pgfsetfillcolor{textcolor}%
\pgftext[x=4.519375in, y=0.616887in, left, base]{\color{textcolor}\rmfamily\fontsize{10.000000}{12.000000}\selectfont shortest }%
\end{pgfscope}%
\begin{pgfscope}%
\definecolor{textcolor}{rgb}{0.501961,0.000000,0.501961}%
\pgfsetstrokecolor{textcolor}%
\pgfsetfillcolor{textcolor}%
\pgftext[x=4.519375in, y=0.474140in, left, base]{\color{textcolor}\rmfamily\fontsize{10.000000}{12.000000}\selectfont  path}%
\end{pgfscope}%
\end{pgfpicture}%
\makeatother%
\endgroup%

%% file: gridball_stepstogoal.pgf
\begingroup%
\makeatletter%
\begin{pgfpicture}%
\pgfpathrectangle{\pgfpointorigin}{\pgfqpoint{5.000000in}{3.750000in}}%
\pgfusepath{use as bounding box, clip}%
\begin{pgfscope}%
\pgfsetbuttcap%
\pgfsetmiterjoin%
\definecolor{currentfill}{rgb}{1.000000,1.000000,1.000000}%
\pgfsetfillcolor{currentfill}%
\pgfsetlinewidth{0.000000pt}%
\definecolor{currentstroke}{rgb}{1.000000,1.000000,1.000000}%
\pgfsetstrokecolor{currentstroke}%
\pgfsetdash{}{0pt}%
\pgfpathmoveto{\pgfqpoint{0.000000in}{0.000000in}}%
\pgfpathlineto{\pgfqpoint{5.000000in}{0.000000in}}%
\pgfpathlineto{\pgfqpoint{5.000000in}{3.750000in}}%
\pgfpathlineto{\pgfqpoint{0.000000in}{3.750000in}}%
\pgfpathlineto{\pgfqpoint{0.000000in}{0.000000in}}%
\pgfpathclose%
\pgfusepath{fill}%
\end{pgfscope}%
\begin{pgfscope}%
\pgfsetbuttcap%
\pgfsetmiterjoin%
\definecolor{currentfill}{rgb}{1.000000,1.000000,1.000000}%
\pgfsetfillcolor{currentfill}%
\pgfsetlinewidth{0.000000pt}%
\definecolor{currentstroke}{rgb}{0.000000,0.000000,0.000000}%
\pgfsetstrokecolor{currentstroke}%
\pgfsetstrokeopacity{0.000000}%
\pgfsetdash{}{0pt}%
\pgfpathmoveto{\pgfqpoint{0.625000in}{0.412500in}}%
\pgfpathlineto{\pgfqpoint{4.500000in}{0.412500in}}%
\pgfpathlineto{\pgfqpoint{4.500000in}{3.300000in}}%
\pgfpathlineto{\pgfqpoint{0.625000in}{3.300000in}}%
\pgfpathlineto{\pgfqpoint{0.625000in}{0.412500in}}%
\pgfpathclose%
\pgfusepath{fill}%
\end{pgfscope}%
\begin{pgfscope}%
\pgfpathrectangle{\pgfqpoint{0.625000in}{0.412500in}}{\pgfqpoint{3.875000in}{2.887500in}}%
\pgfusepath{clip}%
\pgfsetbuttcap%
\pgfsetroundjoin%
\definecolor{currentfill}{rgb}{0.000000,0.466667,0.733333}%
\pgfsetfillcolor{currentfill}%
\pgfsetfillopacity{0.200000}%
\pgfsetlinewidth{0.000000pt}%
\definecolor{currentstroke}{rgb}{0.000000,0.000000,0.000000}%
\pgfsetstrokecolor{currentstroke}%
\pgfsetdash{}{0pt}%
\pgfpathmoveto{\pgfqpoint{0.721875in}{3.168750in}}%
\pgfpathlineto{\pgfqpoint{0.721875in}{3.168750in}}%
\pgfpathlineto{\pgfqpoint{0.818750in}{3.117398in}}%
\pgfpathlineto{\pgfqpoint{0.915625in}{3.017577in}}%
\pgfpathlineto{\pgfqpoint{1.012500in}{2.962925in}}%
\pgfpathlineto{\pgfqpoint{1.109375in}{2.882077in}}%
\pgfpathlineto{\pgfqpoint{1.206250in}{2.897356in}}%
\pgfpathlineto{\pgfqpoint{1.303125in}{2.754384in}}%
\pgfpathlineto{\pgfqpoint{1.400000in}{2.674130in}}%
\pgfpathlineto{\pgfqpoint{1.496875in}{2.546734in}}%
\pgfpathlineto{\pgfqpoint{1.593750in}{2.491497in}}%
\pgfpathlineto{\pgfqpoint{1.690625in}{2.370217in}}%
\pgfpathlineto{\pgfqpoint{1.787500in}{2.093955in}}%
\pgfpathlineto{\pgfqpoint{1.884375in}{1.913986in}}%
\pgfpathlineto{\pgfqpoint{1.981250in}{1.685232in}}%
\pgfpathlineto{\pgfqpoint{2.078125in}{1.494586in}}%
\pgfpathlineto{\pgfqpoint{2.175000in}{1.332395in}}%
\pgfpathlineto{\pgfqpoint{2.271875in}{1.016701in}}%
\pgfpathlineto{\pgfqpoint{2.368750in}{0.933863in}}%
\pgfpathlineto{\pgfqpoint{2.465625in}{0.693893in}}%
\pgfpathlineto{\pgfqpoint{2.562500in}{0.595992in}}%
\pgfpathlineto{\pgfqpoint{2.659375in}{0.574417in}}%
\pgfpathlineto{\pgfqpoint{2.756250in}{0.549974in}}%
\pgfpathlineto{\pgfqpoint{2.853125in}{0.546048in}}%
\pgfpathlineto{\pgfqpoint{2.950000in}{0.548960in}}%
\pgfpathlineto{\pgfqpoint{3.046875in}{0.548295in}}%
\pgfpathlineto{\pgfqpoint{3.143750in}{0.546896in}}%
\pgfpathlineto{\pgfqpoint{3.240625in}{0.543861in}}%
\pgfpathlineto{\pgfqpoint{3.337500in}{0.545422in}}%
\pgfpathlineto{\pgfqpoint{3.434375in}{0.545756in}}%
\pgfpathlineto{\pgfqpoint{3.531250in}{0.545519in}}%
\pgfpathlineto{\pgfqpoint{3.628125in}{0.544964in}}%
\pgfpathlineto{\pgfqpoint{3.725000in}{0.546180in}}%
\pgfpathlineto{\pgfqpoint{3.821875in}{0.545521in}}%
\pgfpathlineto{\pgfqpoint{3.918750in}{0.545424in}}%
\pgfpathlineto{\pgfqpoint{4.015625in}{0.545070in}}%
\pgfpathlineto{\pgfqpoint{4.112500in}{0.545226in}}%
\pgfpathlineto{\pgfqpoint{4.209375in}{0.545090in}}%
\pgfpathlineto{\pgfqpoint{4.306250in}{0.545772in}}%
\pgfpathlineto{\pgfqpoint{4.403125in}{0.545130in}}%
\pgfpathlineto{\pgfqpoint{4.500000in}{0.544462in}}%
\pgfpathlineto{\pgfqpoint{4.500000in}{0.546574in}}%
\pgfpathlineto{\pgfqpoint{4.500000in}{0.546574in}}%
\pgfpathlineto{\pgfqpoint{4.403125in}{0.547061in}}%
\pgfpathlineto{\pgfqpoint{4.306250in}{0.547536in}}%
\pgfpathlineto{\pgfqpoint{4.209375in}{0.546803in}}%
\pgfpathlineto{\pgfqpoint{4.112500in}{0.547225in}}%
\pgfpathlineto{\pgfqpoint{4.015625in}{0.547195in}}%
\pgfpathlineto{\pgfqpoint{3.918750in}{0.547251in}}%
\pgfpathlineto{\pgfqpoint{3.821875in}{0.547899in}}%
\pgfpathlineto{\pgfqpoint{3.725000in}{0.548432in}}%
\pgfpathlineto{\pgfqpoint{3.628125in}{0.546892in}}%
\pgfpathlineto{\pgfqpoint{3.531250in}{0.547528in}}%
\pgfpathlineto{\pgfqpoint{3.434375in}{0.548855in}}%
\pgfpathlineto{\pgfqpoint{3.337500in}{0.547252in}}%
\pgfpathlineto{\pgfqpoint{3.240625in}{0.546392in}}%
\pgfpathlineto{\pgfqpoint{3.143750in}{0.582395in}}%
\pgfpathlineto{\pgfqpoint{3.046875in}{0.550973in}}%
\pgfpathlineto{\pgfqpoint{2.950000in}{0.584280in}}%
\pgfpathlineto{\pgfqpoint{2.853125in}{0.548452in}}%
\pgfpathlineto{\pgfqpoint{2.756250in}{0.585538in}}%
\pgfpathlineto{\pgfqpoint{2.659375in}{0.619057in}}%
\pgfpathlineto{\pgfqpoint{2.562500in}{0.648403in}}%
\pgfpathlineto{\pgfqpoint{2.465625in}{0.766031in}}%
\pgfpathlineto{\pgfqpoint{2.368750in}{1.063059in}}%
\pgfpathlineto{\pgfqpoint{2.271875in}{1.152094in}}%
\pgfpathlineto{\pgfqpoint{2.175000in}{1.476134in}}%
\pgfpathlineto{\pgfqpoint{2.078125in}{1.662084in}}%
\pgfpathlineto{\pgfqpoint{1.981250in}{1.799053in}}%
\pgfpathlineto{\pgfqpoint{1.884375in}{2.083569in}}%
\pgfpathlineto{\pgfqpoint{1.787500in}{2.226225in}}%
\pgfpathlineto{\pgfqpoint{1.690625in}{2.507821in}}%
\pgfpathlineto{\pgfqpoint{1.593750in}{2.588027in}}%
\pgfpathlineto{\pgfqpoint{1.496875in}{2.652810in}}%
\pgfpathlineto{\pgfqpoint{1.400000in}{2.761841in}}%
\pgfpathlineto{\pgfqpoint{1.303125in}{2.862474in}}%
\pgfpathlineto{\pgfqpoint{1.206250in}{2.979955in}}%
\pgfpathlineto{\pgfqpoint{1.109375in}{2.956531in}}%
\pgfpathlineto{\pgfqpoint{1.012500in}{3.020661in}}%
\pgfpathlineto{\pgfqpoint{0.915625in}{3.075078in}}%
\pgfpathlineto{\pgfqpoint{0.818750in}{3.150034in}}%
\pgfpathlineto{\pgfqpoint{0.721875in}{3.168750in}}%
\pgfpathlineto{\pgfqpoint{0.721875in}{3.168750in}}%
\pgfpathclose%
\pgfusepath{fill}%
\end{pgfscope}%
\begin{pgfscope}%
\pgfpathrectangle{\pgfqpoint{0.625000in}{0.412500in}}{\pgfqpoint{3.875000in}{2.887500in}}%
\pgfusepath{clip}%
\pgfsetbuttcap%
\pgfsetroundjoin%
\definecolor{currentfill}{rgb}{0.933333,0.200000,0.466667}%
\pgfsetfillcolor{currentfill}%
\pgfsetfillopacity{0.200000}%
\pgfsetlinewidth{0.000000pt}%
\definecolor{currentstroke}{rgb}{0.000000,0.000000,0.000000}%
\pgfsetstrokecolor{currentstroke}%
\pgfsetdash{}{0pt}%
\pgfpathmoveto{\pgfqpoint{0.721875in}{2.784578in}}%
\pgfpathlineto{\pgfqpoint{0.721875in}{2.610865in}}%
\pgfpathlineto{\pgfqpoint{0.818750in}{2.269013in}}%
\pgfpathlineto{\pgfqpoint{0.915625in}{1.899305in}}%
\pgfpathlineto{\pgfqpoint{1.012500in}{1.586563in}}%
\pgfpathlineto{\pgfqpoint{1.109375in}{1.449157in}}%
\pgfpathlineto{\pgfqpoint{1.206250in}{1.228870in}}%
\pgfpathlineto{\pgfqpoint{1.303125in}{1.078291in}}%
\pgfpathlineto{\pgfqpoint{1.400000in}{0.942150in}}%
\pgfpathlineto{\pgfqpoint{1.496875in}{0.731590in}}%
\pgfpathlineto{\pgfqpoint{1.593750in}{0.659912in}}%
\pgfpathlineto{\pgfqpoint{1.690625in}{0.604917in}}%
\pgfpathlineto{\pgfqpoint{1.787500in}{0.579614in}}%
\pgfpathlineto{\pgfqpoint{1.884375in}{0.561138in}}%
\pgfpathlineto{\pgfqpoint{1.981250in}{0.561239in}}%
\pgfpathlineto{\pgfqpoint{2.078125in}{0.560976in}}%
\pgfpathlineto{\pgfqpoint{2.175000in}{0.553580in}}%
\pgfpathlineto{\pgfqpoint{2.271875in}{0.552052in}}%
\pgfpathlineto{\pgfqpoint{2.368750in}{0.551769in}}%
\pgfpathlineto{\pgfqpoint{2.465625in}{0.549953in}}%
\pgfpathlineto{\pgfqpoint{2.562500in}{0.549598in}}%
\pgfpathlineto{\pgfqpoint{2.659375in}{0.548384in}}%
\pgfpathlineto{\pgfqpoint{2.756250in}{0.548166in}}%
\pgfpathlineto{\pgfqpoint{2.853125in}{0.547457in}}%
\pgfpathlineto{\pgfqpoint{2.950000in}{0.547742in}}%
\pgfpathlineto{\pgfqpoint{3.046875in}{0.548651in}}%
\pgfpathlineto{\pgfqpoint{3.143750in}{0.550645in}}%
\pgfpathlineto{\pgfqpoint{3.240625in}{0.549489in}}%
\pgfpathlineto{\pgfqpoint{3.337500in}{0.546210in}}%
\pgfpathlineto{\pgfqpoint{3.434375in}{0.546865in}}%
\pgfpathlineto{\pgfqpoint{3.531250in}{0.547113in}}%
\pgfpathlineto{\pgfqpoint{3.628125in}{0.548413in}}%
\pgfpathlineto{\pgfqpoint{3.725000in}{0.546686in}}%
\pgfpathlineto{\pgfqpoint{3.821875in}{0.545180in}}%
\pgfpathlineto{\pgfqpoint{3.918750in}{0.547058in}}%
\pgfpathlineto{\pgfqpoint{4.015625in}{0.548232in}}%
\pgfpathlineto{\pgfqpoint{4.112500in}{0.546432in}}%
\pgfpathlineto{\pgfqpoint{4.209375in}{0.545396in}}%
\pgfpathlineto{\pgfqpoint{4.306250in}{0.546455in}}%
\pgfpathlineto{\pgfqpoint{4.403125in}{0.543750in}}%
\pgfpathlineto{\pgfqpoint{4.500000in}{0.545747in}}%
\pgfpathlineto{\pgfqpoint{4.500000in}{0.548232in}}%
\pgfpathlineto{\pgfqpoint{4.500000in}{0.548232in}}%
\pgfpathlineto{\pgfqpoint{4.403125in}{0.545684in}}%
\pgfpathlineto{\pgfqpoint{4.306250in}{0.549125in}}%
\pgfpathlineto{\pgfqpoint{4.209375in}{0.547726in}}%
\pgfpathlineto{\pgfqpoint{4.112500in}{0.549111in}}%
\pgfpathlineto{\pgfqpoint{4.015625in}{0.551409in}}%
\pgfpathlineto{\pgfqpoint{3.918750in}{0.551539in}}%
\pgfpathlineto{\pgfqpoint{3.821875in}{0.547457in}}%
\pgfpathlineto{\pgfqpoint{3.725000in}{0.553029in}}%
\pgfpathlineto{\pgfqpoint{3.628125in}{0.558640in}}%
\pgfpathlineto{\pgfqpoint{3.531250in}{0.549622in}}%
\pgfpathlineto{\pgfqpoint{3.434375in}{0.551398in}}%
\pgfpathlineto{\pgfqpoint{3.337500in}{0.551530in}}%
\pgfpathlineto{\pgfqpoint{3.240625in}{0.582895in}}%
\pgfpathlineto{\pgfqpoint{3.143750in}{0.576784in}}%
\pgfpathlineto{\pgfqpoint{3.046875in}{0.574979in}}%
\pgfpathlineto{\pgfqpoint{2.950000in}{0.580060in}}%
\pgfpathlineto{\pgfqpoint{2.853125in}{0.564104in}}%
\pgfpathlineto{\pgfqpoint{2.756250in}{0.553523in}}%
\pgfpathlineto{\pgfqpoint{2.659375in}{0.553752in}}%
\pgfpathlineto{\pgfqpoint{2.562500in}{0.554177in}}%
\pgfpathlineto{\pgfqpoint{2.465625in}{0.554269in}}%
\pgfpathlineto{\pgfqpoint{2.368750in}{0.557222in}}%
\pgfpathlineto{\pgfqpoint{2.271875in}{0.558205in}}%
\pgfpathlineto{\pgfqpoint{2.175000in}{0.561966in}}%
\pgfpathlineto{\pgfqpoint{2.078125in}{0.574424in}}%
\pgfpathlineto{\pgfqpoint{1.981250in}{0.582916in}}%
\pgfpathlineto{\pgfqpoint{1.884375in}{0.584022in}}%
\pgfpathlineto{\pgfqpoint{1.787500in}{0.601046in}}%
\pgfpathlineto{\pgfqpoint{1.690625in}{0.658662in}}%
\pgfpathlineto{\pgfqpoint{1.593750in}{0.736165in}}%
\pgfpathlineto{\pgfqpoint{1.496875in}{0.804585in}}%
\pgfpathlineto{\pgfqpoint{1.400000in}{1.102974in}}%
\pgfpathlineto{\pgfqpoint{1.303125in}{1.209854in}}%
\pgfpathlineto{\pgfqpoint{1.206250in}{1.330158in}}%
\pgfpathlineto{\pgfqpoint{1.109375in}{1.578590in}}%
\pgfpathlineto{\pgfqpoint{1.012500in}{1.704356in}}%
\pgfpathlineto{\pgfqpoint{0.915625in}{2.015629in}}%
\pgfpathlineto{\pgfqpoint{0.818750in}{2.351030in}}%
\pgfpathlineto{\pgfqpoint{0.721875in}{2.784578in}}%
\pgfpathlineto{\pgfqpoint{0.721875in}{2.784578in}}%
\pgfpathclose%
\pgfusepath{fill}%
\end{pgfscope}%
\begin{pgfscope}%
\pgfpathrectangle{\pgfqpoint{0.625000in}{0.412500in}}{\pgfqpoint{3.875000in}{2.887500in}}%
\pgfusepath{clip}%
\pgfsetbuttcap%
\pgfsetroundjoin%
\definecolor{currentfill}{rgb}{0.933333,0.466667,0.200000}%
\pgfsetfillcolor{currentfill}%
\pgfsetfillopacity{0.200000}%
\pgfsetlinewidth{0.000000pt}%
\definecolor{currentstroke}{rgb}{0.000000,0.000000,0.000000}%
\pgfsetstrokecolor{currentstroke}%
\pgfsetdash{}{0pt}%
\pgfpathmoveto{\pgfqpoint{0.721875in}{3.168750in}}%
\pgfpathlineto{\pgfqpoint{0.721875in}{3.168750in}}%
\pgfpathlineto{\pgfqpoint{0.818750in}{3.168750in}}%
\pgfpathlineto{\pgfqpoint{0.915625in}{3.127550in}}%
\pgfpathlineto{\pgfqpoint{1.012500in}{3.115176in}}%
\pgfpathlineto{\pgfqpoint{1.109375in}{3.125540in}}%
\pgfpathlineto{\pgfqpoint{1.206250in}{3.127181in}}%
\pgfpathlineto{\pgfqpoint{1.303125in}{3.146919in}}%
\pgfpathlineto{\pgfqpoint{1.400000in}{3.136674in}}%
\pgfpathlineto{\pgfqpoint{1.496875in}{3.145188in}}%
\pgfpathlineto{\pgfqpoint{1.593750in}{3.090923in}}%
\pgfpathlineto{\pgfqpoint{1.690625in}{3.108153in}}%
\pgfpathlineto{\pgfqpoint{1.787500in}{3.055916in}}%
\pgfpathlineto{\pgfqpoint{1.884375in}{3.090143in}}%
\pgfpathlineto{\pgfqpoint{1.981250in}{3.057127in}}%
\pgfpathlineto{\pgfqpoint{2.078125in}{3.118551in}}%
\pgfpathlineto{\pgfqpoint{2.175000in}{3.104705in}}%
\pgfpathlineto{\pgfqpoint{2.271875in}{3.089127in}}%
\pgfpathlineto{\pgfqpoint{2.368750in}{3.047964in}}%
\pgfpathlineto{\pgfqpoint{2.465625in}{3.133645in}}%
\pgfpathlineto{\pgfqpoint{2.562500in}{3.098851in}}%
\pgfpathlineto{\pgfqpoint{2.659375in}{3.024546in}}%
\pgfpathlineto{\pgfqpoint{2.756250in}{3.033138in}}%
\pgfpathlineto{\pgfqpoint{2.853125in}{3.109557in}}%
\pgfpathlineto{\pgfqpoint{2.950000in}{3.046994in}}%
\pgfpathlineto{\pgfqpoint{3.046875in}{3.118465in}}%
\pgfpathlineto{\pgfqpoint{3.143750in}{3.071544in}}%
\pgfpathlineto{\pgfqpoint{3.240625in}{3.006644in}}%
\pgfpathlineto{\pgfqpoint{3.337500in}{2.923735in}}%
\pgfpathlineto{\pgfqpoint{3.434375in}{3.082084in}}%
\pgfpathlineto{\pgfqpoint{3.531250in}{3.097512in}}%
\pgfpathlineto{\pgfqpoint{3.628125in}{3.120197in}}%
\pgfpathlineto{\pgfqpoint{3.725000in}{3.110259in}}%
\pgfpathlineto{\pgfqpoint{3.821875in}{3.053473in}}%
\pgfpathlineto{\pgfqpoint{3.918750in}{3.053961in}}%
\pgfpathlineto{\pgfqpoint{4.015625in}{3.090477in}}%
\pgfpathlineto{\pgfqpoint{4.112500in}{3.146549in}}%
\pgfpathlineto{\pgfqpoint{4.209375in}{3.074772in}}%
\pgfpathlineto{\pgfqpoint{4.306250in}{3.121381in}}%
\pgfpathlineto{\pgfqpoint{4.403125in}{3.067265in}}%
\pgfpathlineto{\pgfqpoint{4.500000in}{3.129421in}}%
\pgfpathlineto{\pgfqpoint{4.500000in}{3.154998in}}%
\pgfpathlineto{\pgfqpoint{4.500000in}{3.154998in}}%
\pgfpathlineto{\pgfqpoint{4.403125in}{3.126710in}}%
\pgfpathlineto{\pgfqpoint{4.306250in}{3.159685in}}%
\pgfpathlineto{\pgfqpoint{4.209375in}{3.129782in}}%
\pgfpathlineto{\pgfqpoint{4.112500in}{3.168750in}}%
\pgfpathlineto{\pgfqpoint{4.015625in}{3.139258in}}%
\pgfpathlineto{\pgfqpoint{3.918750in}{3.133756in}}%
\pgfpathlineto{\pgfqpoint{3.821875in}{3.139571in}}%
\pgfpathlineto{\pgfqpoint{3.725000in}{3.146632in}}%
\pgfpathlineto{\pgfqpoint{3.628125in}{3.153717in}}%
\pgfpathlineto{\pgfqpoint{3.531250in}{3.138593in}}%
\pgfpathlineto{\pgfqpoint{3.434375in}{3.135582in}}%
\pgfpathlineto{\pgfqpoint{3.337500in}{3.049794in}}%
\pgfpathlineto{\pgfqpoint{3.240625in}{3.080832in}}%
\pgfpathlineto{\pgfqpoint{3.143750in}{3.128727in}}%
\pgfpathlineto{\pgfqpoint{3.046875in}{3.153847in}}%
\pgfpathlineto{\pgfqpoint{2.950000in}{3.137855in}}%
\pgfpathlineto{\pgfqpoint{2.853125in}{3.150016in}}%
\pgfpathlineto{\pgfqpoint{2.756250in}{3.120495in}}%
\pgfpathlineto{\pgfqpoint{2.659375in}{3.098430in}}%
\pgfpathlineto{\pgfqpoint{2.562500in}{3.144257in}}%
\pgfpathlineto{\pgfqpoint{2.465625in}{3.158187in}}%
\pgfpathlineto{\pgfqpoint{2.368750in}{3.122208in}}%
\pgfpathlineto{\pgfqpoint{2.271875in}{3.141279in}}%
\pgfpathlineto{\pgfqpoint{2.175000in}{3.150807in}}%
\pgfpathlineto{\pgfqpoint{2.078125in}{3.155959in}}%
\pgfpathlineto{\pgfqpoint{1.981250in}{3.131484in}}%
\pgfpathlineto{\pgfqpoint{1.884375in}{3.139070in}}%
\pgfpathlineto{\pgfqpoint{1.787500in}{3.122712in}}%
\pgfpathlineto{\pgfqpoint{1.690625in}{3.142368in}}%
\pgfpathlineto{\pgfqpoint{1.593750in}{3.136279in}}%
\pgfpathlineto{\pgfqpoint{1.496875in}{3.162251in}}%
\pgfpathlineto{\pgfqpoint{1.400000in}{3.164023in}}%
\pgfpathlineto{\pgfqpoint{1.303125in}{3.167188in}}%
\pgfpathlineto{\pgfqpoint{1.206250in}{3.162490in}}%
\pgfpathlineto{\pgfqpoint{1.109375in}{3.168750in}}%
\pgfpathlineto{\pgfqpoint{1.012500in}{3.156465in}}%
\pgfpathlineto{\pgfqpoint{0.915625in}{3.161525in}}%
\pgfpathlineto{\pgfqpoint{0.818750in}{3.168750in}}%
\pgfpathlineto{\pgfqpoint{0.721875in}{3.168750in}}%
\pgfpathlineto{\pgfqpoint{0.721875in}{3.168750in}}%
\pgfpathclose%
\pgfusepath{fill}%
\end{pgfscope}%
\begin{pgfscope}%
\pgfsetbuttcap%
\pgfsetroundjoin%
\definecolor{currentfill}{rgb}{0.000000,0.000000,0.000000}%
\pgfsetfillcolor{currentfill}%
\pgfsetlinewidth{0.803000pt}%
\definecolor{currentstroke}{rgb}{0.000000,0.000000,0.000000}%
\pgfsetstrokecolor{currentstroke}%
\pgfsetdash{}{0pt}%
\pgfsys@defobject{currentmarker}{\pgfqpoint{0.000000in}{-0.048611in}}{\pgfqpoint{0.000000in}{0.000000in}}{%
\pgfpathmoveto{\pgfqpoint{0.000000in}{0.000000in}}%
\pgfpathlineto{\pgfqpoint{0.000000in}{-0.048611in}}%
\pgfusepath{stroke,fill}%
}%
\begin{pgfscope}%
\pgfsys@transformshift{0.625000in}{0.412500in}%
\pgfsys@useobject{currentmarker}{}%
\end{pgfscope}%
\end{pgfscope}%
\begin{pgfscope}%
\definecolor{textcolor}{rgb}{0.000000,0.000000,0.000000}%
\pgfsetstrokecolor{textcolor}%
\pgfsetfillcolor{textcolor}%
\pgftext[x=0.625000in,y=0.315278in,,top]{\color{textcolor}\rmfamily\fontsize{9.000000}{10.800000}\selectfont \(\displaystyle {0}\)}%
\end{pgfscope}%
\begin{pgfscope}%
\pgfsetbuttcap%
\pgfsetroundjoin%
\definecolor{currentfill}{rgb}{0.000000,0.000000,0.000000}%
\pgfsetfillcolor{currentfill}%
\pgfsetlinewidth{0.803000pt}%
\definecolor{currentstroke}{rgb}{0.000000,0.000000,0.000000}%
\pgfsetstrokecolor{currentstroke}%
\pgfsetdash{}{0pt}%
\pgfsys@defobject{currentmarker}{\pgfqpoint{0.000000in}{-0.048611in}}{\pgfqpoint{0.000000in}{0.000000in}}{%
\pgfpathmoveto{\pgfqpoint{0.000000in}{0.000000in}}%
\pgfpathlineto{\pgfqpoint{0.000000in}{-0.048611in}}%
\pgfusepath{stroke,fill}%
}%
\begin{pgfscope}%
\pgfsys@transformshift{1.109375in}{0.412500in}%
\pgfsys@useobject{currentmarker}{}%
\end{pgfscope}%
\end{pgfscope}%
\begin{pgfscope}%
\definecolor{textcolor}{rgb}{0.000000,0.000000,0.000000}%
\pgfsetstrokecolor{textcolor}%
\pgfsetfillcolor{textcolor}%
\pgftext[x=1.109375in,y=0.315278in,,top]{\color{textcolor}\rmfamily\fontsize{9.000000}{10.800000}\selectfont \(\displaystyle {25}\)}%
\end{pgfscope}%
\begin{pgfscope}%
\pgfsetbuttcap%
\pgfsetroundjoin%
\definecolor{currentfill}{rgb}{0.000000,0.000000,0.000000}%
\pgfsetfillcolor{currentfill}%
\pgfsetlinewidth{0.803000pt}%
\definecolor{currentstroke}{rgb}{0.000000,0.000000,0.000000}%
\pgfsetstrokecolor{currentstroke}%
\pgfsetdash{}{0pt}%
\pgfsys@defobject{currentmarker}{\pgfqpoint{0.000000in}{-0.048611in}}{\pgfqpoint{0.000000in}{0.000000in}}{%
\pgfpathmoveto{\pgfqpoint{0.000000in}{0.000000in}}%
\pgfpathlineto{\pgfqpoint{0.000000in}{-0.048611in}}%
\pgfusepath{stroke,fill}%
}%
\begin{pgfscope}%
\pgfsys@transformshift{1.593750in}{0.412500in}%
\pgfsys@useobject{currentmarker}{}%
\end{pgfscope}%
\end{pgfscope}%
\begin{pgfscope}%
\definecolor{textcolor}{rgb}{0.000000,0.000000,0.000000}%
\pgfsetstrokecolor{textcolor}%
\pgfsetfillcolor{textcolor}%
\pgftext[x=1.593750in,y=0.315278in,,top]{\color{textcolor}\rmfamily\fontsize{9.000000}{10.800000}\selectfont \(\displaystyle {50}\)}%
\end{pgfscope}%
\begin{pgfscope}%
\pgfsetbuttcap%
\pgfsetroundjoin%
\definecolor{currentfill}{rgb}{0.000000,0.000000,0.000000}%
\pgfsetfillcolor{currentfill}%
\pgfsetlinewidth{0.803000pt}%
\definecolor{currentstroke}{rgb}{0.000000,0.000000,0.000000}%
\pgfsetstrokecolor{currentstroke}%
\pgfsetdash{}{0pt}%
\pgfsys@defobject{currentmarker}{\pgfqpoint{0.000000in}{-0.048611in}}{\pgfqpoint{0.000000in}{0.000000in}}{%
\pgfpathmoveto{\pgfqpoint{0.000000in}{0.000000in}}%
\pgfpathlineto{\pgfqpoint{0.000000in}{-0.048611in}}%
\pgfusepath{stroke,fill}%
}%
\begin{pgfscope}%
\pgfsys@transformshift{2.078125in}{0.412500in}%
\pgfsys@useobject{currentmarker}{}%
\end{pgfscope}%
\end{pgfscope}%
\begin{pgfscope}%
\definecolor{textcolor}{rgb}{0.000000,0.000000,0.000000}%
\pgfsetstrokecolor{textcolor}%
\pgfsetfillcolor{textcolor}%
\pgftext[x=2.078125in,y=0.315278in,,top]{\color{textcolor}\rmfamily\fontsize{9.000000}{10.800000}\selectfont \(\displaystyle {75}\)}%
\end{pgfscope}%
\begin{pgfscope}%
\pgfsetbuttcap%
\pgfsetroundjoin%
\definecolor{currentfill}{rgb}{0.000000,0.000000,0.000000}%
\pgfsetfillcolor{currentfill}%
\pgfsetlinewidth{0.803000pt}%
\definecolor{currentstroke}{rgb}{0.000000,0.000000,0.000000}%
\pgfsetstrokecolor{currentstroke}%
\pgfsetdash{}{0pt}%
\pgfsys@defobject{currentmarker}{\pgfqpoint{0.000000in}{-0.048611in}}{\pgfqpoint{0.000000in}{0.000000in}}{%
\pgfpathmoveto{\pgfqpoint{0.000000in}{0.000000in}}%
\pgfpathlineto{\pgfqpoint{0.000000in}{-0.048611in}}%
\pgfusepath{stroke,fill}%
}%
\begin{pgfscope}%
\pgfsys@transformshift{2.562500in}{0.412500in}%
\pgfsys@useobject{currentmarker}{}%
\end{pgfscope}%
\end{pgfscope}%
\begin{pgfscope}%
\definecolor{textcolor}{rgb}{0.000000,0.000000,0.000000}%
\pgfsetstrokecolor{textcolor}%
\pgfsetfillcolor{textcolor}%
\pgftext[x=2.562500in,y=0.315278in,,top]{\color{textcolor}\rmfamily\fontsize{9.000000}{10.800000}\selectfont \(\displaystyle {100}\)}%
\end{pgfscope}%
\begin{pgfscope}%
\pgfsetbuttcap%
\pgfsetroundjoin%
\definecolor{currentfill}{rgb}{0.000000,0.000000,0.000000}%
\pgfsetfillcolor{currentfill}%
\pgfsetlinewidth{0.803000pt}%
\definecolor{currentstroke}{rgb}{0.000000,0.000000,0.000000}%
\pgfsetstrokecolor{currentstroke}%
\pgfsetdash{}{0pt}%
\pgfsys@defobject{currentmarker}{\pgfqpoint{0.000000in}{-0.048611in}}{\pgfqpoint{0.000000in}{0.000000in}}{%
\pgfpathmoveto{\pgfqpoint{0.000000in}{0.000000in}}%
\pgfpathlineto{\pgfqpoint{0.000000in}{-0.048611in}}%
\pgfusepath{stroke,fill}%
}%
\begin{pgfscope}%
\pgfsys@transformshift{3.046875in}{0.412500in}%
\pgfsys@useobject{currentmarker}{}%
\end{pgfscope}%
\end{pgfscope}%
\begin{pgfscope}%
\definecolor{textcolor}{rgb}{0.000000,0.000000,0.000000}%
\pgfsetstrokecolor{textcolor}%
\pgfsetfillcolor{textcolor}%
\pgftext[x=3.046875in,y=0.315278in,,top]{\color{textcolor}\rmfamily\fontsize{9.000000}{10.800000}\selectfont \(\displaystyle {125}\)}%
\end{pgfscope}%
\begin{pgfscope}%
\pgfsetbuttcap%
\pgfsetroundjoin%
\definecolor{currentfill}{rgb}{0.000000,0.000000,0.000000}%
\pgfsetfillcolor{currentfill}%
\pgfsetlinewidth{0.803000pt}%
\definecolor{currentstroke}{rgb}{0.000000,0.000000,0.000000}%
\pgfsetstrokecolor{currentstroke}%
\pgfsetdash{}{0pt}%
\pgfsys@defobject{currentmarker}{\pgfqpoint{0.000000in}{-0.048611in}}{\pgfqpoint{0.000000in}{0.000000in}}{%
\pgfpathmoveto{\pgfqpoint{0.000000in}{0.000000in}}%
\pgfpathlineto{\pgfqpoint{0.000000in}{-0.048611in}}%
\pgfusepath{stroke,fill}%
}%
\begin{pgfscope}%
\pgfsys@transformshift{3.531250in}{0.412500in}%
\pgfsys@useobject{currentmarker}{}%
\end{pgfscope}%
\end{pgfscope}%
\begin{pgfscope}%
\definecolor{textcolor}{rgb}{0.000000,0.000000,0.000000}%
\pgfsetstrokecolor{textcolor}%
\pgfsetfillcolor{textcolor}%
\pgftext[x=3.531250in,y=0.315278in,,top]{\color{textcolor}\rmfamily\fontsize{9.000000}{10.800000}\selectfont \(\displaystyle {150}\)}%
\end{pgfscope}%
\begin{pgfscope}%
\pgfsetbuttcap%
\pgfsetroundjoin%
\definecolor{currentfill}{rgb}{0.000000,0.000000,0.000000}%
\pgfsetfillcolor{currentfill}%
\pgfsetlinewidth{0.803000pt}%
\definecolor{currentstroke}{rgb}{0.000000,0.000000,0.000000}%
\pgfsetstrokecolor{currentstroke}%
\pgfsetdash{}{0pt}%
\pgfsys@defobject{currentmarker}{\pgfqpoint{0.000000in}{-0.048611in}}{\pgfqpoint{0.000000in}{0.000000in}}{%
\pgfpathmoveto{\pgfqpoint{0.000000in}{0.000000in}}%
\pgfpathlineto{\pgfqpoint{0.000000in}{-0.048611in}}%
\pgfusepath{stroke,fill}%
}%
\begin{pgfscope}%
\pgfsys@transformshift{4.015625in}{0.412500in}%
\pgfsys@useobject{currentmarker}{}%
\end{pgfscope}%
\end{pgfscope}%
\begin{pgfscope}%
\definecolor{textcolor}{rgb}{0.000000,0.000000,0.000000}%
\pgfsetstrokecolor{textcolor}%
\pgfsetfillcolor{textcolor}%
\pgftext[x=4.015625in,y=0.315278in,,top]{\color{textcolor}\rmfamily\fontsize{9.000000}{10.800000}\selectfont \(\displaystyle {175}\)}%
\end{pgfscope}%
\begin{pgfscope}%
\pgfsetbuttcap%
\pgfsetroundjoin%
\definecolor{currentfill}{rgb}{0.000000,0.000000,0.000000}%
\pgfsetfillcolor{currentfill}%
\pgfsetlinewidth{0.803000pt}%
\definecolor{currentstroke}{rgb}{0.000000,0.000000,0.000000}%
\pgfsetstrokecolor{currentstroke}%
\pgfsetdash{}{0pt}%
\pgfsys@defobject{currentmarker}{\pgfqpoint{0.000000in}{-0.048611in}}{\pgfqpoint{0.000000in}{0.000000in}}{%
\pgfpathmoveto{\pgfqpoint{0.000000in}{0.000000in}}%
\pgfpathlineto{\pgfqpoint{0.000000in}{-0.048611in}}%
\pgfusepath{stroke,fill}%
}%
\begin{pgfscope}%
\pgfsys@transformshift{4.500000in}{0.412500in}%
\pgfsys@useobject{currentmarker}{}%
\end{pgfscope}%
\end{pgfscope}%
\begin{pgfscope}%
\definecolor{textcolor}{rgb}{0.000000,0.000000,0.000000}%
\pgfsetstrokecolor{textcolor}%
\pgfsetfillcolor{textcolor}%
\pgftext[x=4.500000in,y=0.315278in,,top]{\color{textcolor}\rmfamily\fontsize{9.000000}{10.800000}\selectfont \(\displaystyle {200}\)}%
\end{pgfscope}%
\begin{pgfscope}%
\definecolor{textcolor}{rgb}{0.000000,0.000000,0.000000}%
\pgfsetstrokecolor{textcolor}%
\pgfsetfillcolor{textcolor}%
\pgftext[x=2.562500in,y=0.148611in,,top]{\color{textcolor}\rmfamily\fontsize{11.000000}{13.200000}\selectfont Episode}%
\end{pgfscope}%
\begin{pgfscope}%
\pgfsetbuttcap%
\pgfsetroundjoin%
\definecolor{currentfill}{rgb}{0.000000,0.000000,0.000000}%
\pgfsetfillcolor{currentfill}%
\pgfsetlinewidth{0.803000pt}%
\definecolor{currentstroke}{rgb}{0.000000,0.000000,0.000000}%
\pgfsetstrokecolor{currentstroke}%
\pgfsetdash{}{0pt}%
\pgfsys@defobject{currentmarker}{\pgfqpoint{-0.048611in}{0.000000in}}{\pgfqpoint{-0.000000in}{0.000000in}}{%
\pgfpathmoveto{\pgfqpoint{-0.000000in}{0.000000in}}%
\pgfpathlineto{\pgfqpoint{-0.048611in}{0.000000in}}%
\pgfusepath{stroke,fill}%
}%
\begin{pgfscope}%
\pgfsys@transformshift{0.625000in}{0.930945in}%
\pgfsys@useobject{currentmarker}{}%
\end{pgfscope}%
\end{pgfscope}%
\begin{pgfscope}%
\definecolor{textcolor}{rgb}{0.000000,0.000000,0.000000}%
\pgfsetstrokecolor{textcolor}%
\pgfsetfillcolor{textcolor}%
\pgftext[x=0.335071in, y=0.887543in, left, base]{\color{textcolor}\rmfamily\fontsize{9.000000}{10.800000}\selectfont \(\displaystyle {200}\)}%
\end{pgfscope}%
\begin{pgfscope}%
\pgfsetbuttcap%
\pgfsetroundjoin%
\definecolor{currentfill}{rgb}{0.000000,0.000000,0.000000}%
\pgfsetfillcolor{currentfill}%
\pgfsetlinewidth{0.803000pt}%
\definecolor{currentstroke}{rgb}{0.000000,0.000000,0.000000}%
\pgfsetstrokecolor{currentstroke}%
\pgfsetdash{}{0pt}%
\pgfsys@defobject{currentmarker}{\pgfqpoint{-0.048611in}{0.000000in}}{\pgfqpoint{-0.000000in}{0.000000in}}{%
\pgfpathmoveto{\pgfqpoint{-0.000000in}{0.000000in}}%
\pgfpathlineto{\pgfqpoint{-0.048611in}{0.000000in}}%
\pgfusepath{stroke,fill}%
}%
\begin{pgfscope}%
\pgfsys@transformshift{0.625000in}{1.489698in}%
\pgfsys@useobject{currentmarker}{}%
\end{pgfscope}%
\end{pgfscope}%
\begin{pgfscope}%
\definecolor{textcolor}{rgb}{0.000000,0.000000,0.000000}%
\pgfsetstrokecolor{textcolor}%
\pgfsetfillcolor{textcolor}%
\pgftext[x=0.335071in, y=1.446295in, left, base]{\color{textcolor}\rmfamily\fontsize{9.000000}{10.800000}\selectfont \(\displaystyle {400}\)}%
\end{pgfscope}%
\begin{pgfscope}%
\pgfsetbuttcap%
\pgfsetroundjoin%
\definecolor{currentfill}{rgb}{0.000000,0.000000,0.000000}%
\pgfsetfillcolor{currentfill}%
\pgfsetlinewidth{0.803000pt}%
\definecolor{currentstroke}{rgb}{0.000000,0.000000,0.000000}%
\pgfsetstrokecolor{currentstroke}%
\pgfsetdash{}{0pt}%
\pgfsys@defobject{currentmarker}{\pgfqpoint{-0.048611in}{0.000000in}}{\pgfqpoint{-0.000000in}{0.000000in}}{%
\pgfpathmoveto{\pgfqpoint{-0.000000in}{0.000000in}}%
\pgfpathlineto{\pgfqpoint{-0.048611in}{0.000000in}}%
\pgfusepath{stroke,fill}%
}%
\begin{pgfscope}%
\pgfsys@transformshift{0.625000in}{2.048451in}%
\pgfsys@useobject{currentmarker}{}%
\end{pgfscope}%
\end{pgfscope}%
\begin{pgfscope}%
\definecolor{textcolor}{rgb}{0.000000,0.000000,0.000000}%
\pgfsetstrokecolor{textcolor}%
\pgfsetfillcolor{textcolor}%
\pgftext[x=0.335071in, y=2.005048in, left, base]{\color{textcolor}\rmfamily\fontsize{9.000000}{10.800000}\selectfont \(\displaystyle {600}\)}%
\end{pgfscope}%
\begin{pgfscope}%
\pgfsetbuttcap%
\pgfsetroundjoin%
\definecolor{currentfill}{rgb}{0.000000,0.000000,0.000000}%
\pgfsetfillcolor{currentfill}%
\pgfsetlinewidth{0.803000pt}%
\definecolor{currentstroke}{rgb}{0.000000,0.000000,0.000000}%
\pgfsetstrokecolor{currentstroke}%
\pgfsetdash{}{0pt}%
\pgfsys@defobject{currentmarker}{\pgfqpoint{-0.048611in}{0.000000in}}{\pgfqpoint{-0.000000in}{0.000000in}}{%
\pgfpathmoveto{\pgfqpoint{-0.000000in}{0.000000in}}%
\pgfpathlineto{\pgfqpoint{-0.048611in}{0.000000in}}%
\pgfusepath{stroke,fill}%
}%
\begin{pgfscope}%
\pgfsys@transformshift{0.625000in}{2.607204in}%
\pgfsys@useobject{currentmarker}{}%
\end{pgfscope}%
\end{pgfscope}%
\begin{pgfscope}%
\definecolor{textcolor}{rgb}{0.000000,0.000000,0.000000}%
\pgfsetstrokecolor{textcolor}%
\pgfsetfillcolor{textcolor}%
\pgftext[x=0.335071in, y=2.563801in, left, base]{\color{textcolor}\rmfamily\fontsize{9.000000}{10.800000}\selectfont \(\displaystyle {800}\)}%
\end{pgfscope}%
\begin{pgfscope}%
\pgfsetbuttcap%
\pgfsetroundjoin%
\definecolor{currentfill}{rgb}{0.000000,0.000000,0.000000}%
\pgfsetfillcolor{currentfill}%
\pgfsetlinewidth{0.803000pt}%
\definecolor{currentstroke}{rgb}{0.000000,0.000000,0.000000}%
\pgfsetstrokecolor{currentstroke}%
\pgfsetdash{}{0pt}%
\pgfsys@defobject{currentmarker}{\pgfqpoint{-0.048611in}{0.000000in}}{\pgfqpoint{-0.000000in}{0.000000in}}{%
\pgfpathmoveto{\pgfqpoint{-0.000000in}{0.000000in}}%
\pgfpathlineto{\pgfqpoint{-0.048611in}{0.000000in}}%
\pgfusepath{stroke,fill}%
}%
\begin{pgfscope}%
\pgfsys@transformshift{0.625000in}{3.165956in}%
\pgfsys@useobject{currentmarker}{}%
\end{pgfscope}%
\end{pgfscope}%
\begin{pgfscope}%
\definecolor{textcolor}{rgb}{0.000000,0.000000,0.000000}%
\pgfsetstrokecolor{textcolor}%
\pgfsetfillcolor{textcolor}%
\pgftext[x=0.270835in, y=3.122553in, left, base]{\color{textcolor}\rmfamily\fontsize{9.000000}{10.800000}\selectfont \(\displaystyle {1000}\)}%
\end{pgfscope}%
\begin{pgfscope}%
\definecolor{textcolor}{rgb}{0.000000,0.000000,0.000000}%
\pgfsetstrokecolor{textcolor}%
\pgfsetfillcolor{textcolor}%
\pgftext[x=0.215279in,y=1.856250in,,bottom,rotate=90.000000]{\color{textcolor}\rmfamily\fontsize{11.000000}{13.200000}\selectfont Steps to goal}%
\end{pgfscope}%
\begin{pgfscope}%
\pgfpathrectangle{\pgfqpoint{0.625000in}{0.412500in}}{\pgfqpoint{3.875000in}{2.887500in}}%
\pgfusepath{clip}%
\pgfsetrectcap%
\pgfsetroundjoin%
\pgfsetlinewidth{1.003750pt}%
\definecolor{currentstroke}{rgb}{0.000000,0.466667,0.733333}%
\pgfsetstrokecolor{currentstroke}%
\pgfsetdash{}{0pt}%
\pgfpathmoveto{\pgfqpoint{0.721875in}{3.168750in}}%
\pgfpathlineto{\pgfqpoint{0.818750in}{3.133716in}}%
\pgfpathlineto{\pgfqpoint{0.915625in}{3.046327in}}%
\pgfpathlineto{\pgfqpoint{1.012500in}{2.991793in}}%
\pgfpathlineto{\pgfqpoint{1.109375in}{2.919304in}}%
\pgfpathlineto{\pgfqpoint{1.206250in}{2.938656in}}%
\pgfpathlineto{\pgfqpoint{1.303125in}{2.808429in}}%
\pgfpathlineto{\pgfqpoint{1.400000in}{2.717986in}}%
\pgfpathlineto{\pgfqpoint{1.496875in}{2.599772in}}%
\pgfpathlineto{\pgfqpoint{1.593750in}{2.539762in}}%
\pgfpathlineto{\pgfqpoint{1.690625in}{2.439019in}}%
\pgfpathlineto{\pgfqpoint{1.787500in}{2.160090in}}%
\pgfpathlineto{\pgfqpoint{1.884375in}{1.998778in}}%
\pgfpathlineto{\pgfqpoint{1.981250in}{1.742143in}}%
\pgfpathlineto{\pgfqpoint{2.078125in}{1.578335in}}%
\pgfpathlineto{\pgfqpoint{2.175000in}{1.404265in}}%
\pgfpathlineto{\pgfqpoint{2.271875in}{1.084398in}}%
\pgfpathlineto{\pgfqpoint{2.368750in}{0.998461in}}%
\pgfpathlineto{\pgfqpoint{2.465625in}{0.729962in}}%
\pgfpathlineto{\pgfqpoint{2.562500in}{0.622197in}}%
\pgfpathlineto{\pgfqpoint{2.659375in}{0.596737in}}%
\pgfpathlineto{\pgfqpoint{2.756250in}{0.567756in}}%
\pgfpathlineto{\pgfqpoint{2.853125in}{0.547250in}}%
\pgfpathlineto{\pgfqpoint{2.950000in}{0.566620in}}%
\pgfpathlineto{\pgfqpoint{3.046875in}{0.549634in}}%
\pgfpathlineto{\pgfqpoint{3.143750in}{0.564646in}}%
\pgfpathlineto{\pgfqpoint{3.240625in}{0.545127in}}%
\pgfpathlineto{\pgfqpoint{3.337500in}{0.546337in}}%
\pgfpathlineto{\pgfqpoint{3.434375in}{0.547306in}}%
\pgfpathlineto{\pgfqpoint{3.531250in}{0.546524in}}%
\pgfpathlineto{\pgfqpoint{3.628125in}{0.545928in}}%
\pgfpathlineto{\pgfqpoint{3.725000in}{0.547306in}}%
\pgfpathlineto{\pgfqpoint{3.821875in}{0.546710in}}%
\pgfpathlineto{\pgfqpoint{3.918750in}{0.546337in}}%
\pgfpathlineto{\pgfqpoint{4.015625in}{0.546132in}}%
\pgfpathlineto{\pgfqpoint{4.112500in}{0.546226in}}%
\pgfpathlineto{\pgfqpoint{4.209375in}{0.545946in}}%
\pgfpathlineto{\pgfqpoint{4.306250in}{0.546654in}}%
\pgfpathlineto{\pgfqpoint{4.403125in}{0.546095in}}%
\pgfpathlineto{\pgfqpoint{4.500000in}{0.545518in}}%
\pgfusepath{stroke}%
\end{pgfscope}%
\begin{pgfscope}%
\pgfpathrectangle{\pgfqpoint{0.625000in}{0.412500in}}{\pgfqpoint{3.875000in}{2.887500in}}%
\pgfusepath{clip}%
\pgfsetrectcap%
\pgfsetroundjoin%
\pgfsetlinewidth{1.003750pt}%
\definecolor{currentstroke}{rgb}{0.933333,0.200000,0.466667}%
\pgfsetstrokecolor{currentstroke}%
\pgfsetdash{}{0pt}%
\pgfpathmoveto{\pgfqpoint{0.721875in}{2.697721in}}%
\pgfpathlineto{\pgfqpoint{0.818750in}{2.310022in}}%
\pgfpathlineto{\pgfqpoint{0.915625in}{1.957467in}}%
\pgfpathlineto{\pgfqpoint{1.012500in}{1.645460in}}%
\pgfpathlineto{\pgfqpoint{1.109375in}{1.513873in}}%
\pgfpathlineto{\pgfqpoint{1.206250in}{1.279514in}}%
\pgfpathlineto{\pgfqpoint{1.303125in}{1.144072in}}%
\pgfpathlineto{\pgfqpoint{1.400000in}{1.022562in}}%
\pgfpathlineto{\pgfqpoint{1.496875in}{0.768088in}}%
\pgfpathlineto{\pgfqpoint{1.593750in}{0.698039in}}%
\pgfpathlineto{\pgfqpoint{1.690625in}{0.631789in}}%
\pgfpathlineto{\pgfqpoint{1.787500in}{0.590330in}}%
\pgfpathlineto{\pgfqpoint{1.884375in}{0.572580in}}%
\pgfpathlineto{\pgfqpoint{1.981250in}{0.572077in}}%
\pgfpathlineto{\pgfqpoint{2.078125in}{0.567700in}}%
\pgfpathlineto{\pgfqpoint{2.175000in}{0.557773in}}%
\pgfpathlineto{\pgfqpoint{2.271875in}{0.555128in}}%
\pgfpathlineto{\pgfqpoint{2.368750in}{0.554495in}}%
\pgfpathlineto{\pgfqpoint{2.465625in}{0.552111in}}%
\pgfpathlineto{\pgfqpoint{2.562500in}{0.551888in}}%
\pgfpathlineto{\pgfqpoint{2.659375in}{0.551068in}}%
\pgfpathlineto{\pgfqpoint{2.756250in}{0.550845in}}%
\pgfpathlineto{\pgfqpoint{2.853125in}{0.555780in}}%
\pgfpathlineto{\pgfqpoint{2.950000in}{0.563901in}}%
\pgfpathlineto{\pgfqpoint{3.046875in}{0.561815in}}%
\pgfpathlineto{\pgfqpoint{3.143750in}{0.563715in}}%
\pgfpathlineto{\pgfqpoint{3.240625in}{0.566192in}}%
\pgfpathlineto{\pgfqpoint{3.337500in}{0.548870in}}%
\pgfpathlineto{\pgfqpoint{3.434375in}{0.549131in}}%
\pgfpathlineto{\pgfqpoint{3.531250in}{0.548367in}}%
\pgfpathlineto{\pgfqpoint{3.628125in}{0.553527in}}%
\pgfpathlineto{\pgfqpoint{3.725000in}{0.549857in}}%
\pgfpathlineto{\pgfqpoint{3.821875in}{0.546319in}}%
\pgfpathlineto{\pgfqpoint{3.918750in}{0.549299in}}%
\pgfpathlineto{\pgfqpoint{4.015625in}{0.549820in}}%
\pgfpathlineto{\pgfqpoint{4.112500in}{0.547771in}}%
\pgfpathlineto{\pgfqpoint{4.209375in}{0.546561in}}%
\pgfpathlineto{\pgfqpoint{4.306250in}{0.547790in}}%
\pgfpathlineto{\pgfqpoint{4.403125in}{0.544717in}}%
\pgfpathlineto{\pgfqpoint{4.500000in}{0.546989in}}%
\pgfusepath{stroke}%
\end{pgfscope}%
\begin{pgfscope}%
\pgfpathrectangle{\pgfqpoint{0.625000in}{0.412500in}}{\pgfqpoint{3.875000in}{2.887500in}}%
\pgfusepath{clip}%
\pgfsetrectcap%
\pgfsetroundjoin%
\pgfsetlinewidth{1.003750pt}%
\definecolor{currentstroke}{rgb}{0.933333,0.466667,0.200000}%
\pgfsetstrokecolor{currentstroke}%
\pgfsetdash{}{0pt}%
\pgfpathmoveto{\pgfqpoint{0.721875in}{3.168750in}}%
\pgfpathlineto{\pgfqpoint{0.818750in}{3.168750in}}%
\pgfpathlineto{\pgfqpoint{0.915625in}{3.144537in}}%
\pgfpathlineto{\pgfqpoint{1.012500in}{3.135821in}}%
\pgfpathlineto{\pgfqpoint{1.109375in}{3.147145in}}%
\pgfpathlineto{\pgfqpoint{1.206250in}{3.144835in}}%
\pgfpathlineto{\pgfqpoint{1.303125in}{3.157053in}}%
\pgfpathlineto{\pgfqpoint{1.400000in}{3.150348in}}%
\pgfpathlineto{\pgfqpoint{1.496875in}{3.153720in}}%
\pgfpathlineto{\pgfqpoint{1.593750in}{3.113601in}}%
\pgfpathlineto{\pgfqpoint{1.690625in}{3.125260in}}%
\pgfpathlineto{\pgfqpoint{1.787500in}{3.089314in}}%
\pgfpathlineto{\pgfqpoint{1.884375in}{3.114607in}}%
\pgfpathlineto{\pgfqpoint{1.981250in}{3.094306in}}%
\pgfpathlineto{\pgfqpoint{2.078125in}{3.137255in}}%
\pgfpathlineto{\pgfqpoint{2.175000in}{3.127756in}}%
\pgfpathlineto{\pgfqpoint{2.271875in}{3.115203in}}%
\pgfpathlineto{\pgfqpoint{2.368750in}{3.085086in}}%
\pgfpathlineto{\pgfqpoint{2.465625in}{3.145916in}}%
\pgfpathlineto{\pgfqpoint{2.562500in}{3.121554in}}%
\pgfpathlineto{\pgfqpoint{2.659375in}{3.061488in}}%
\pgfpathlineto{\pgfqpoint{2.756250in}{3.076817in}}%
\pgfpathlineto{\pgfqpoint{2.853125in}{3.129786in}}%
\pgfpathlineto{\pgfqpoint{2.950000in}{3.092424in}}%
\pgfpathlineto{\pgfqpoint{3.046875in}{3.136156in}}%
\pgfpathlineto{\pgfqpoint{3.143750in}{3.100135in}}%
\pgfpathlineto{\pgfqpoint{3.240625in}{3.043738in}}%
\pgfpathlineto{\pgfqpoint{3.337500in}{2.986764in}}%
\pgfpathlineto{\pgfqpoint{3.434375in}{3.108833in}}%
\pgfpathlineto{\pgfqpoint{3.531250in}{3.118053in}}%
\pgfpathlineto{\pgfqpoint{3.628125in}{3.136957in}}%
\pgfpathlineto{\pgfqpoint{3.725000in}{3.128445in}}%
\pgfpathlineto{\pgfqpoint{3.821875in}{3.096522in}}%
\pgfpathlineto{\pgfqpoint{3.918750in}{3.093859in}}%
\pgfpathlineto{\pgfqpoint{4.015625in}{3.114868in}}%
\pgfpathlineto{\pgfqpoint{4.112500in}{3.157649in}}%
\pgfpathlineto{\pgfqpoint{4.209375in}{3.102277in}}%
\pgfpathlineto{\pgfqpoint{4.306250in}{3.140533in}}%
\pgfpathlineto{\pgfqpoint{4.403125in}{3.096988in}}%
\pgfpathlineto{\pgfqpoint{4.500000in}{3.142209in}}%
\pgfusepath{stroke}%
\end{pgfscope}%
\begin{pgfscope}%
\pgfsetrectcap%
\pgfsetmiterjoin%
\pgfsetlinewidth{0.803000pt}%
\definecolor{currentstroke}{rgb}{0.000000,0.000000,0.000000}%
\pgfsetstrokecolor{currentstroke}%
\pgfsetdash{}{0pt}%
\pgfpathmoveto{\pgfqpoint{0.625000in}{0.412500in}}%
\pgfpathlineto{\pgfqpoint{0.625000in}{3.300000in}}%
\pgfusepath{stroke}%
\end{pgfscope}%
\begin{pgfscope}%
\pgfsetrectcap%
\pgfsetmiterjoin%
\pgfsetlinewidth{0.803000pt}%
\definecolor{currentstroke}{rgb}{0.000000,0.000000,0.000000}%
\pgfsetstrokecolor{currentstroke}%
\pgfsetdash{}{0pt}%
\pgfpathmoveto{\pgfqpoint{0.625000in}{0.412500in}}%
\pgfpathlineto{\pgfqpoint{4.500000in}{0.412500in}}%
\pgfusepath{stroke}%
\end{pgfscope}%
\begin{pgfscope}%
\definecolor{textcolor}{rgb}{0.121569,0.466667,0.705882}%
\pgfsetstrokecolor{textcolor}%
\pgfsetfillcolor{textcolor}%
\pgftext[x=2.368750in,y=1.350010in,left,base]{\color{textcolor}\rmfamily\fontsize{14.000000}{16.800000}\selectfont Sarsa(\(\displaystyle \lambda\))}%
\end{pgfscope}%
\begin{pgfscope}%
\definecolor{textcolor}{rgb}{1.000000,0.078431,0.576471}%
\pgfsetstrokecolor{textcolor}%
\pgfsetfillcolor{textcolor}%
\pgftext[x=0.915625in,y=0.791257in,left,base]{\color{textcolor}\rmfamily\fontsize{14.000000}{16.800000}\selectfont GSP}%
\end{pgfscope}%
\begin{pgfscope}%
\definecolor{textcolor}{rgb}{1.000000,0.498039,0.054902}%
\pgfsetstrokecolor{textcolor}%
\pgfsetfillcolor{textcolor}%
\pgftext[x=2.562500in,y=2.607204in,left,base]{\color{textcolor}\rmfamily\fontsize{14.000000}{16.800000}\selectfont Approximate LAVI\(\displaystyle \)}%
\end{pgfscope}%
\end{pgfpicture}%
\makeatother%
\endgroup%

%% file: appendix.tex
\appendix

\section{Starting Simpler: Goal-Space Planning for Policy Evaluation}\label{sec_simple}

To highlight the key idea for efficient planning, we provide an example of GSP in a simpler setting: policy evaluation for learning $v^\pi$ for a fixed deterministic policy $\pi$ in a deterministic environment, assuming access to the true models. 
The key idea is to propagate values quickly across the space by updating between a subset of states that we call \emph{subgoals}, $g \in \mathcal{G} \subset \mathcal{S}$, as visualized in Figure \ref{fig:gsp_eval_illust}. (Later we extend $\mathcal{G} \not\subset \mathcal{S}$ to abstract subgoal vectors that need not correspond to any state.) To do so, we need temporally extended models between pairs $g, g'$ that may be further than one-transition apart. For policy evaluation, these models are the accumulated rewards $\rpigamma: \States \times \States \rightarrow \RR$ and discounted probabilities $\Ppigamma: \States \times \States \rightarrow [0,1]$ under $\pi$:
\begin{align*}
\rpigamma(g,g') &\defeq \mathbb{E}_{\pi}[R_{t+1} + \gamma_{g',t+1} \rpigamma(S_{t+1}, g') | S_{t} = g]\\
\Ppigamma(g, g') &\defeq \mathbb{E}_{\pi}[1(S_{t+1} = g') \gamma_{t+1} + \gamma_{g',t+1} \Ppigamma(S_{t+1}, g') | S_{t} = g]
\end{align*}
where $\gamma_{g',t+1} = 0$ if $S_{t+1} = g'$ and otherwise equals $\gamma_{t+1}$, the environment discount. If we cannot reach $g'$ from $g$ under $\pi$, then $\Ppigamma(g,g')$ will simply accumulate many zeros and be zero. 
%
We can treat $\Goals$ as our new state space and plan in this space, to get value estimates $v$ for all $g \in \Goals$
\begin{align*}
v(g) = \rpigamma(g, g') + \Ppigamma(g,g') v(g') \quad \text{ where } g' = {\textstyle \argmax_{g' \in \augGoals}} \Ppigamma(g,g') 
\end{align*}
where $\augGoals = \Goals \cup \{ \sterm\}$ if there is a terminal state (episodic problems) and otherwise $\augGoals = \Goals$.  
It is straightforward to show this converges, because $\Ppigamma$ is a substochastic matrix (see Appendix \ref{app_simpleproofs}).

Once we have these values, we can propagate these to other states, locally, again using the closest $g$ to $s$. We can do so by noticing that the above definitions can be easily extended to $\rpigamma(s,g')$ and $\Ppigamma(s, g')$, since for a pair $(s,g)$ they are about starting in the state $s$ and reaching $g$ under $\pi$. 
 \begin{equation}\label{eq_vpis}
v(s) = \vsg(s,g) + \Ppigamma(s, g) v(g) \quad \text{ where } g = {\textstyle\argmax_{g \in \augGoals}} \Ppigamma(s,g) 
.
\end{equation}
Because the rhs of this equation is fixed, we only cycle through these states once to get their values.
 
All of this might seem like a lot of work for policy evaluation; indeed, it will be more useful to have this formalism for control. But, even here
goal-space planning can be beneficial. Let assume a chain $s_1, s_2, \ldots, s_n$, where $n = 1000$ and $\Goals = \{s_{100}, s_{200}, \ldots, s_{1000}\}$. Planning over $g \in \Goals$ only requires sweeping over 10 states, rather than 1000. Further, we have taken a 1000 horizon problem and converted it into a 10 step one.\footnote{In this simplified example, we can plan efficiently by updating the value at the end in $s_n$, and then updating states backwards from the end. But, without knowing this structure, it is not a general purpose strategy. For general MDPs, we would need smart ways to do search control: the approach to pick states from one-step updates. In fact, we can leverage search control strategies to improve the goal-space planning step. Then we get the benefit of these approaches, as well as the benefit of planning over a much smaller state space.}
As a result, changes in the environment also propagate faster. If the reward at $s'$ changes, locally the reward model around $s'$ can be updated quickly, to change $\rpigamma(g, g')$ for pairs $g,g'$ where $s'$ is along the way from $g$ to $g'$. This local change quickly updates the values back to earlier $\tilde{g} \in \Goals$.   

\begin{figure}[t]
  \centering
      \includegraphics[width=0.9\textwidth]{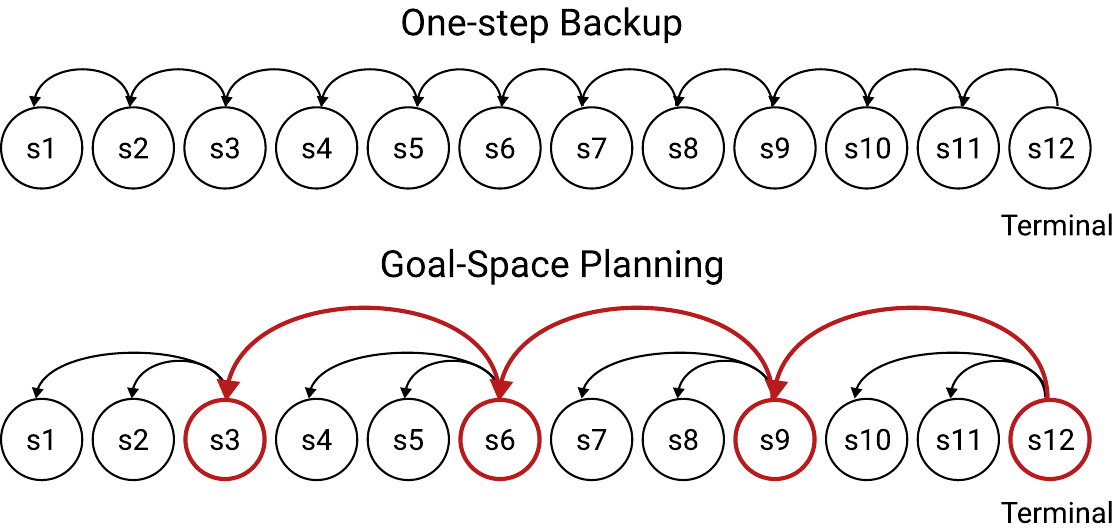}
           \caption{Comparing one-step backup with Goal-Space Planning when subgoals are concrete states. GSP first focuses planning over a smaller set of subgoals (in red), then updates the values of  individual states.}\label{fig:gsp_eval_illust}
  \end{figure}

\subsection{Proofs for the Deterministic Policy Evaluation Setting}\label{app_simpleproofs}

We provide proofs here for the deterministic policy evaluation setting. We assume throughout that the environment discount $\gamma_{t+1}$ is a constant $\gamma_c \in [0,1)$ for every step in an episode, until termination when it is zero. The below results can be extended to the case where $\gamma_c = 1$, using the standard strategy for the stochastic shortest path problem setting. 

First, we want to show that given $\rpigamma$ and $\Ppigamma$, we can guarantee that the update for the values for $\Goals$ will converge. Recall that $\augGoals = \Goals \cup \{\sterm\}$ is the augmented goal space that includes the terminal state. This terminal state is not a subgoal---since it is not a real state---but is key for appropriate planning. 
\begin{lemma}
Assume that we have a deterministic MDP, deterministic policy $\pi$, $\gamma_c < 1$, a discrete set of subgoals $\Goals \subset \States$, and that we iteratively update $v_t \in \RR^{|\augGoals|}$ with the dynamic programming update
\begin{equation}\label{eq_simple_vg}
v_t(g) = \rpigamma(g, g') + \Ppigamma(g,g') v_{t-1}(g') \quad \text{ where } g' = \argmax_{g' \in \augGoals} \Ppigamma(g,g') 
\end{equation}
for all $g \in \Goals$, starting from an arbitrary (finite) initialization $v_0 \in \RR^{|\augGoals|}$, with $v_t(s_{\text{terminal}})$ fixed at zero. Then 
then $v_t$ converges to a fixed point. 
\end{lemma}
\begin{proof}
To analyze this as a matrix update, we need to extend $\Ppigamma(g,g')$ to include an additional row transitioning from $s_{\text{terminal}}$. This row is all zeros, because the value in the terminal state is always fixed at zero. Note that there are ways to avoid introducing terminal states, using transition-based discounting \citep{white2017unifying}, but for this work it is actually simpler to explicitly reason about them and reaching them from subgoals. 

To show this we simply need to ensure that $\Ppigamma$ is a substochastic matrix. Recall that
\begin{align*}
\Ppigamma(g, g') &\defeq \mathbb{E}_{\pi}[1(S_{t+1} = g') \gamma_{t+1} + \gamma_{g',t+1} \Ppigamma(S_{t+1}, g') | S_{t} = g]
\end{align*}
where $\gamma_{g',t+1} = 0$ if $S_{t+1} = g'$ and otherwise equals $\gamma_{t+1}$, the environment discount. If it is substochastic, then $\| \Ppigamma\|_2 < 1$. Consequently, the Bellman operator 
\begin{equation*}
(Bv)(g) = \rpigamma(g, g') + \Ppigamma(g,g') \vgoal(g') \quad \text{ where } g' = \argmax_{g' \in \augGoals} \Ppigamma(g,g') 
\end{equation*}
is a contraction, because $\| Bv_1 - B v_2 \|_2 = \| \Ppigamma v_1 - \Ppigamma v_2 \|_2 \le \| \Ppigamma \|_2 \| v_1 - v_2 \|_2 < \| v_1 - v_2 \|_2$. 

Because $\gamma_c < 1$, then either $g$ immediately terminates in $g'$, giving $1(S_{t+1} = g') \gamma_{t+1} + \gamma_{g',t+1} \Ppigamma(S_{t+1}, g') = \gamma_{t+1} + 0 \le \gamma_c$. Or, it does not immediately terminate, and $1(S_{t+1} = g') \gamma_{t+1} + \gamma_{g',t+1} \Ppigamma(S_{t+1}, g') = 0 + \gamma_c \Ppigamma(S_{t+1}, g') \le \gamma_c$ because $\Ppigamma(S_{t+1}, g') \le 1$. 
Therefore, if $\gamma_c < 1$, then $\| \Ppigamma\|_2 \le \gamma_c$.

\end{proof}

\begin{proposition}
For a deterministic MDP, deterministic policy $\pi$, and a discrete set of subgoals $\Goals \subset \States$ that are all reached by $\pi$ in the MDP, given the $\vgoal(g)$ obtained from Equation \ref{eq_simple_vg}, if we set 
 \begin{equation}\label{eq_vpis2}
v(s) = \vsg(s,g) + \Ppigamma(s, g) \vgoal(g) \quad \text{ where } g = \argmax_{g \in \augGoals} \Ppigamma(s,g) 
\end{equation}
for all states $s \in \States$ then we get that $v = v_\pi$. 
\end{proposition}
\begin{proof}
For a deterministic environment and deterministic policy this result is straightforward. The term $\Ppigamma(s, g) > 0$ only if $g$ is on the trajectory from $s$ when the policy $\pi$ is executed. The term 
$\vsg(s,g)$ consists of deterministic (discounted) rewards and $\vgoal(g)$ is the true value from $g$, as shown in Lemma \ref{eq_simple_vg} (namely $\vgoal(g) = v_\pi(g)$). The subgoal $g$ is the closest state on the trajectory from $s$, and $\Ppigamma(s, g)$ is $\gamma_c^t$ where $t$ is the number of steps from $s$ to $g$. 
\end{proof}

\section{Proofs for the General Control Setting}\label{app_control}

In this section we assume that $\gamma_c < 1$, to avoid some of the additional issues for handling proper policies. The same strategies apply to the stochastic shortest path setting with $\gamma_c = 1$, with additional assumptions. 

\begin{proposition}\label{prop_vi_gsp}[Convergence of Value Iteration in Goal-Space] Assuming that $\Gammagg$ is a substochastic matrix, with $v_0 \in \RR^{|\augGoals|}$ initialized to an arbitrary value and fixing $v_{t}(\sterm) = 0$ for all $t$, then iteratively sweeping through all $g \in \Goals$ with update
\begin{equation}\label{eq_vi_gsp_theory}
v_t(g) = \max_{g' \in \augGoals: \relgg(g,g') > 0} \vgg(g, g') + \Gammagg(g,g') v_{t-1}(g')
\end{equation}
convergences to a fixed-point. 
\end{proposition}
\begin{proof}
We can use the same approach typically used for value iteration. For any $v_0 \in \RR^{|\augGoals|}$, we can define the operator
\begin{equation*}
(B^g v)(g) \defeq \max_{g' \in \augGoals: \relgg(g,g') > 0} \vgg(g, g') + \Gammagg(g,g') \vgoal(g')
\end{equation*}
First we can show that $B^g$ is a $\gamma_c$-contraction. Assume we are given any two vectors $v_1, v_2$. Notice that $\Gammagg(g,g') \le \gamma_c$, because for our problem setting the discount is either equal to $\gamma_c$ or equal to zero at termination. Then we have that for any $g \in \augGoals$
\begin{align*}
&|(B^g v_1)(g) - (B^g v_2)(g)| \\
&= \Big|\max_{g' \in \augGoals: \relgg(g,g') > 0} \vgg(g, g') + \Gammagg(g,g') v_1(g') - \max_{g' \in \augGoals: \relgg(g,g') > 0} \vgg(g, g') + \Gammagg(g,g') v_2(g') \Big|\\
&\le \max_{g' \in \augGoals: \relgg(g,g') > 0} | \vgg(g, g') + \Gammagg(g,g') v_1(g') - (\vgg(g, g') + \Gammagg(g,g') v_2(g')) |\\
&= \max_{g' \in \augGoals: \relgg(g,g') > 0} | \Gammagg(g,g') (v_1(g') - v_2(g')) |\\
&\le \max_{g' \in \augGoals: \relgg(g,g') > 0} \gamma_c | v_1(g') - v_2(g') |\\
&\le  \gamma_c \| v_1 - v_2 \|_{\infty}
\end{align*}
Since this is true for any $g$, it is true for the max over $g$, giving
\begin{equation*}
\| B^g v_1 - B^g v_2 \|_\infty \le \gamma_c \| v_1 - v_2 \|_{\infty}
.
\end{equation*}
Because the operator $B^g$ is a contraction, since $\gamma_c < 1$, we know by the Banach Fixed-Point Theorem that the fixed-point exists and is unique.
\end{proof}

\newcommand{\rmax}{r_{\text{max}}}
\newcommand{\gmax}{G_{\text{max}}}
\newcommand{\rsub}{r_{\text{sub}}}

\section{Learning the Option Policies}\label{app:opt_pol}
In this section we detail the implementation of option learning which was used in all the experiments presented in this paper. This is followed by a brief description of how these option policies could be learnt more generally across domains. Our full procedure is summarised in Figure \ref{fig:gsp_flowchart}.

\subsection{Learning our Option Policies}
\label{learn_opt_pol}

In the simplest case, it is enough to learn $\pi_g$ that makes $\rgamma(s,g)$ maximal for every relevant $s$ (i.e., $\forall \, s \in \States \, \mathrm{s.t.} \, \relsg(s,g) > 0$). For each subgoal $g$, we learn its corresponding option model $\pi_g$ by initialising the base learner in the initiation set of $g$, and terminating the episode once the learner is in a state that is a member of $g$. We used a reward of -1 per step and save the option policy once we reach a 90\% success rate, and the last 100 episodes are within some domain-dependent cut off. This cut off was 10 steps for FourRooms, and 50 steps for GridBall and PinBall.

\begin{figure}[htbp] 
\label{fig:opt_pol_eval}
 \centering
    \includegraphics[width=\textwidth]{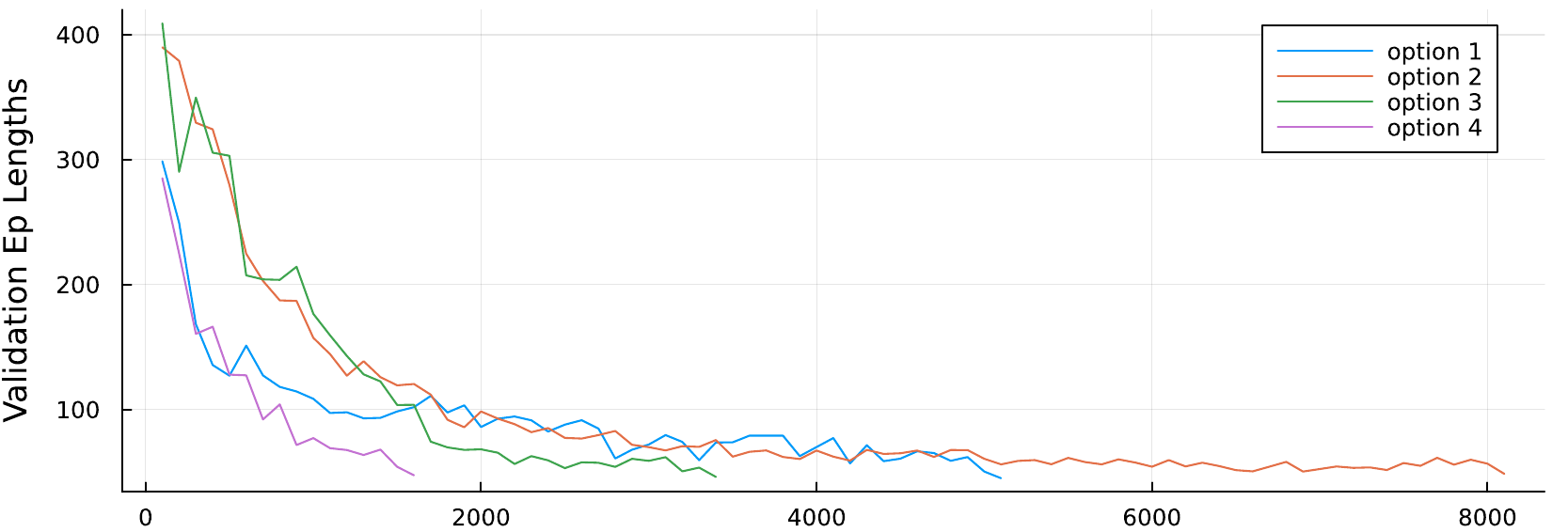}
    \caption{Evaluation of PinBall option policies by average trajectory length. Policies were saved once they were able to reach their respective subgoal in undeer 50 steps, averaged across 100 trajectories. Subgoal 2 was the hardest to learn an option policy for, due to its proximity to obstacles.}
\end{figure}

\textbf{Hyperparameters} In FourRooms, we use Sarsa(0) and Sarsa(0.9) base learners with learning rate $\alpha = 0.01$, discount factor $\gamma_c = 0.99$ and an $\epsilon = 0.02$ for its $\epsilon$-greedy policy. In GridBall, we used Sarsa(0) and Sarsa(0.9) base learners with $\alpha=0.05$, $\gamma_c=0.99$ and $\epsilon=0.1$. $\epsilon$ is decayed by 0.5\% each timestep. In the linear function approximation setting, these learners use a tilecoder with 16 tiles and 4 tilings across each of the both the GridBall dimensions. In PinBall, the Sarsa(0.9) learner was tuned to $\alpha = 0.1$, $\gamma_c=0.99$, $\epsilon=0.1$, decayed in the same manner as in GridBall. The same tile coder was used on on the 4-dimensional state space of PinBall. For the DDQN base learner, we use $\alpha=0.004$, $\gamma_c=0.99$, $\epsilon=0.1$, a buffer that holds up to $10,000$ transitions a batch size of $32$, and a target refresh rate of every $100$ steps. The Q-Network weights used Kaiming initialisation \citep{he2015delving}.

We could have also learned the action-value variant $\rgamma(s,a,g)$ using a Sarsa update, and set $\pi_g(s) = \argmax_{a \in \Actions} \rgamma(s,a,g)$, where we overloaded the definition of $\rgamma$. We can then extract $\rgamma(s,g) = \max_{a \in \Actions} \rgamma(s,a,g)$, to use in all the above updates and in planning. In our experiments, this strategy is sufficient for learning $\pi_g$.

\subsection{A General Algorithm for Option Policy Learning}

More generally, however, these approaches may be ineffective because maximizing environment reward may be at odds with reaching the subgoal in a reasonable number of steps (or at all). For example, in environments where the reward is always positive, maximizing environment reward might encourage the option policy not to terminate.\footnote{It is not always the case that positive rewards result in option policies that do not terminate. If there is a large, positive reward at the subgoal in the environment, Even if all rewards are positive, if $\gamma_c < 1$ and there is larger positive reward at the subgoal than in other nearby states, then the return is higher when reaching this subgoal sooner, since that reward is not discounted as many steps. This outcome is less nuanced for negative reward. If the rewards are always negative, on the other hand, then the option policy will terminate, trying to find the path with the best (but still negative) return. } 
However, we do want $\pi_g$ to reach $g$, while also obtaining the best return along the way to $g$. For example, if there is a lava pit along the way to a goal, even if going through the lava pit is the shortest path, we want the learned option to get to the goal by going around the lava pit.
We therefore want to be reward-respecting, as introduced for reward-respecting subtasks \citep{sutton2022rewardrespecting}, but also ensure termination. 

\newcommand{\rewardinterp}{c}

We can consider a spectrum of option policies that range from the policy that reaches the goal as fast as possible to one that focuses on environment reward. We can specify a new reward for learning the option: $\tilde{R}_{t+1} = \rewardinterp R_{t+1} + (1-\rewardinterp) (-1)$. When $\rewardinterp = 0$, we have a cost-to-goal problem, where the learned option policy should find the shortest path to the goal, regardless of reward along the way. When $\rewardinterp = 1$, the option policy focuses on environment reward, but may not terminate in $g$. 
We can start by learning the option policy that takes the shortest path with $\rewardinterp = 0$, and the corresponding $\rgamma(s,g), \GamModel(s,g)$. The constant $\rewardinterp$ can be increased until $\pi_g$ stops going to the goal, or until the discounted probability $\GamModel(s,g)$ drops below a specified threshold. 

Even without a well-specified $\rewardinterp$, the values under the option policy can still be informative. For example, it might indicate that it is difficult or dangerous to attempt to reach a goal. For this work, we propose a simple default, where we fix $\rewardinterp = 0.5$. Adaptive approaches, such as the idea described above, are left to future work.

The resulting algorithm to learn $\pi_g$ involves learning a separate value function for these rewards. We can learn action-values (or a parameterized policy) using the above reward. For example, we can learn a policy with the Q-learning update to action-values $\optionq$
\begin{align*}
\delta = \rewardinterp R_{t+1} +\rewardinterp - 1 + \gamma_{g, t+1} \max_{a'} \optionq(S_{t+1}, a', g) - \optionq(S_t, A_t, g)
\end{align*} 
Then we can set $\pi_g$ to be the greedy policy, $\pi_g(s) = \argmax_{a \in \Actions} \optionq(s,a,g)$.

\section{Learning the Subgoal Models}\label{app:sec_learning_models}

Now we need a way to learn the state-to-subgoal models, $\rgamma(s,g)$ and $\GamModel(s,g)$, still following the progression in Figure \ref{fig:gsp_flowchart}. These can both be represented as General Value Functions (GVFs) \citep{sutton2011horde}, 

\begin{equation}
    \Gamma(s, g) = \EE_{\pi_g} \left[ \sum_{k=0}^\infty \left(\prod_{k^\prime=0}^k\gamma_{t+k^\prime+1}\right )m(S_{t+1},g) \Big| S_t=s\right ], 
\end{equation}
\begin{equation}
    \rgamma(s, g) = \EE_{\pi_g} \left [ \sum_{k=0}^\infty \left(\prod_{k^\prime=0}^k\gamma_{t+k^\prime+1}\right ) R_{t+k+1} \Big| S_t=s\right],
\end{equation}

and we leverage this form to use standard algorithms in RL to learn them. 

\subsection{Model Learning}
\label{sec:model-learning}

In our experiments, the data is generated offline according to each $\pi_g$. We then use this episode dataset from each $\pi_g$ to learn the subgoal models for that subgoal $g$. This is done by ordinary least squares regression to fit a linear model in four-room, and by stochastic gradient descent with neural network models in GridBall and PinBall. Full experimental details for these methods are described in Appendix \ref{app_model_learning}.

\paragraph{Offline Model Update}
We first collect a dataset of $n$ episodes leading to a subgoal $g$,  $\mathcal{D}_g = \{\langle S_{i,1}, A_{i,1}, R_{i,1}, S_{i,1}, \dots, S_{i,T_i} \rangle\}_{i=1}^{n}$. $S_{i,t}, A_{i,t}, R_{i, t}$ represent the state, action and reward at timestep $t$ of episode $i$. $T_i$ is the length of episode $i$. $S_{i,0}$ is a randomised starting state within the initiation set of $g$, and $S_{i,T_i}$ is a state that is a member of subgoal $g$. For each $g$, we use $\mathcal{D}_g $ to generate a matrix of all visited states, $\mathbf{X} \in \mathbb{R}^{l\times|\States|}$, and a vector of all reward model returns, $\mathbf{g}_r \in \mathbb{R}^l$, and transition model returns $\mathbf{g}_\gamma \in \mathbb{R}^l$,
\begin{align*}
\mathbf{X} = \begin{pmatrix}
              S_{i,1} \\
              S_{i,2} \\
              \vdots \\
              S_{n,T_{n}}
              \end{pmatrix},
\mathbf{g}_r = \begin{pmatrix}
              R_{i,2} + \gamma \rgamma(S_{i,3}, g) \\
              R_{i,3} + \gamma \rgamma(S_{i,4}, g)\\
              \vdots \\
              R_{n,T_n}
              \end{pmatrix},
\mathbf{g}_\gamma = \begin{pmatrix}
                \gamma^{T_1 - 0} \\
                \gamma^{T_1 - 1} \\
                \vdots \\
                \gamma^{T_{n} - T_{n}}
                \end{pmatrix},
\end{align*}
where $l = \sum_{i = 1}^n T_i$ is the total number of visited states in $\mathcal{D}_g$.

This creates a system of linear equations, whose weights we can solve for numerically in the four-room domain,
\begin{align*}
  \mathbf{X}\rparams = \mathbf{g}_r \implies \rparams = \mathbf{X}^+\mathbf{g}_r, \\
  \mathbf{X}\gamparams = \mathbf{g}_\gamma \implies \gamparams = \mathbf{X}^+\mathbf{g}_\gamma,
\end{align*}
where $\rparams, \gamparams \in \mathbb{R}^{|\mathcal{S}|}$ and $\mathbf{X}^+$ is the Moore-Penrose pseudoinverse of $\mathbf{X}$ \citep{penrose1955generalized}.

For GridBall and PinBall, we used fully connected artificial neural networks for $\rgamma$ and $\Gamma$, and performed mini-batch stochastic gradient descent to solve $\rparams$ and $\gamparams$ for that subgoal $g$. We use each mini-batch of $m$ states, reward model returns and transition model returns to perform the update:
\begin{align*}
  \rparams \gets\rparams - \eta_r \sum_{j=1}^m \nabla_\rparams(\rparams^\top\mathbf{X}_{j,:} - \mathbf{g}_{r,j})^2,\\
  \gamparams \gets\gamparams - \eta_\Gamma \sum_{j=1}^m \nabla_\gamparams(\gamparams^\top\mathbf{X}_{j,:} - \mathbf{g}_{\gamma,j})^2,
\end{align*}
where $\eta_r$ and $\eta_\Gamma$ are the learning rates for the reward and discount models respectively. $\mathbf{X}_{j,:}$ is the $j^{\mathrm{th}}$ row of $\mathbf{X}$. $\mathbf{g}_{r,j}$ and $\mathbf{g}_{\gamma,j}$ are the $j^{\mathrm{th}}$ entry of $\mathbf{g}_{r}$ and $\mathbf{g}_{\gamma}$ respectively. In our experiments, we had a fully connected artificial neural network with two hidden layers of 128 units and ReLU activation for each subgoal. The network took a state $s = (x, y, \dot{x}, \dot{y})$ as input and outputted both $\rgamma(s,g)$ and $\Gamma(s,g)$. All weights were initialised using Kaiming initialisation \citep{he2015delving}. We use the Adam optimizer with $\eta=0.001$ and the other parameters set to the default ($b_1=0.9, b_2=0.999, \epsilon=10^{-8}$), mini-batches of 1024 transitions and 100 epochs.

\begin{figure}[htbp] 
  \begin{subfigure}[l]{\textwidth}
    \centering
    \begin{subfigure}[b]{0.45\textwidth} 
        \centering
        \caption{$\rgamma(s,g_1)$ and $\GamModel(s, g_1)$}
        \includegraphics[width=\textwidth]{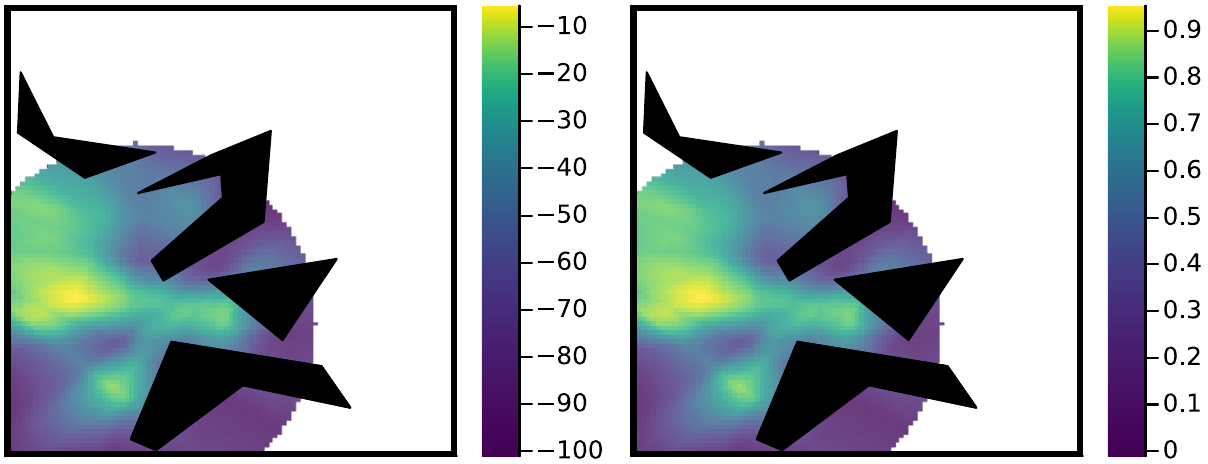} 
    \end{subfigure}
    \begin{subfigure}[b]{0.45\textwidth} 
        \centering
        \caption{$\rgamma(s,g_2)$ and $\GamModel(s, g_2)$}
        \includegraphics[width=\textwidth]{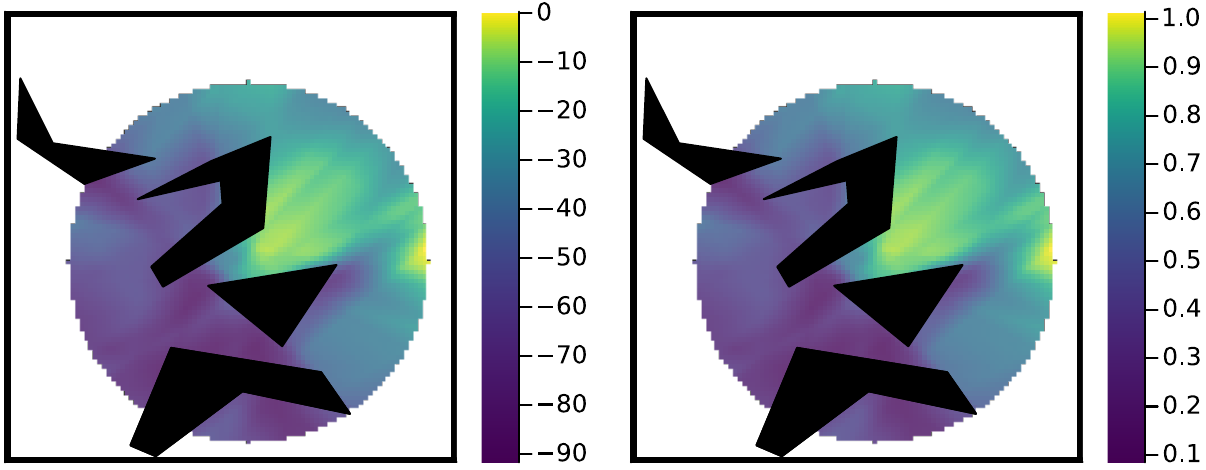} 
    \end{subfigure}

    \begin{subfigure}[b]{0.45\textwidth} 
        \centering
        \caption{$\rgamma(s,g_3)$ and $\GamModel(s, g_3)$}
        \includegraphics[width=\textwidth]{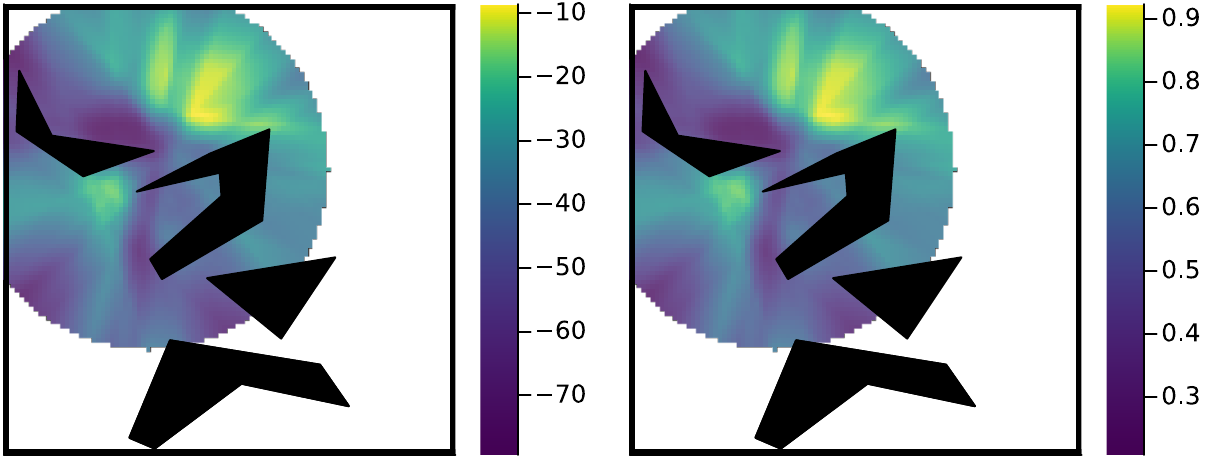} 
    \end{subfigure}
    \begin{subfigure}[b]{0.45\textwidth} 
        \centering
        \caption{$\rgamma(s,g_4)$ and $\GamModel(s, g_4)$}
        \includegraphics[width=\textwidth]{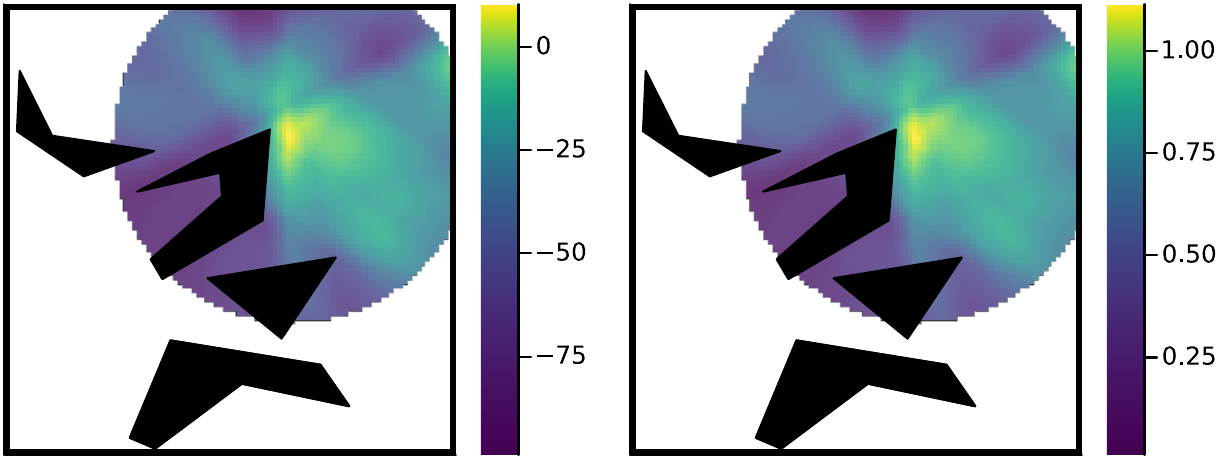} 
    \end{subfigure}
  \end{subfigure}
\caption{State-to-Subgoal models learnt by neural models after 100 epochs.}
\label{fig:s2g_models}
\end{figure}

The data could also be generated off-policy---according to some behavior $b$ rather than from $\pi_g$. We can either use importance sampling or we can learn the action-value variants of these models to avoid importance sampling. We describe both options here.

\paragraph{Off-policy Model Update using Importance Sampling}
We can update $\rgamma(\cdot, g)$ with an importance-sampled temporal difference (TD) learning update $\rho_t \delta_t \nabla \vsg(S_t,g)$ where $\rho_t = \frac{\pi_g(a | S_t)}{b(a | S_t)}$ and
\begin{align*}
\delta^r_t = R_{t+1} + \gamma_{g, t+1} \rgamma(S_{t+1},g) - \rgamma(S_t, g)
\end{align*}
The discount model $\GamModel(s,g)$ can be learned similarly, because it is also a GVF with cumulant $\indsg(S_{t+1},g) \gamma_{t+1}$ and discount $\gamma_{g,t+1}$. The TD update is $\rho_t \delta^\GamModel_t$ where
\begin{align*}
\delta^\GamModel_t = \indsg(S_{t+1},g) \gamma_{t+1} +  \gamma_{g, t+1} \GamModel(S_{t+1},g) - \GamModel(S_t,g)
\end{align*}
All of the above updates can be done using any off-policy GVF algorithm, including those using clipping of IS ratios and gradient-based methods, and can include replay. 

\paragraph{Off-policy Model Update without Importance Sampling}
Overloading notation, let us define the action-value variants $\rgamma(s,a,g)$ and $\GamModel(s,a,g)$. We get similar updates to above, now redefining
\begin{align*}
\delta^r_t = R_{t+1} + \gamma_{g, t+1} \rgamma(S_{t+1},\pi_g(S_{t+1}), g) - \rgamma(S_t, A_t, g)
\end{align*}
and using update $\delta_t \nabla \rgamma(S_t,A_t,g)$. For $\GamModel$ we have
\begin{align*}
\delta^\GamModel_t = \indsg(S_{t+1},g) \gamma_{t+1} +  \gamma_{g, t+1} \GamModel(S_{t+1},\pi_g(S_{t+1}),g) - \GamModel(S_t,A_t,g)
\end{align*}
We then define
$\rgamma(s,g) \defeq \rgamma(s,\pi_g(s), g)$ and $\GamModel(s,g) \defeq \GamModel(s,\pi_g(s), g)$ as deterministic functions of these learned functions. 

\paragraph{Restricting the Model Update to Relevant States} \label{app_restrict_model}
Recall, however, that we need only query these models where $\relsg(s,g) > 0$. We can focus our function approximation resources on those states. This idea has previously been introduced with an interest weighting for GVFs \citep{sutton2016anemphatic}, with connections made between interest and initiation sets \citep{white2017unifying}. For a large state space with many subgoals, using goal-space planning significantly expands the models that need to be learned, especially if we learn one model per subgoal. Even if we learn a model that generalizes across subgoal vectors, we are requiring that model to know a lot: values from all states to all subgoals. It is likely such a models would be hard to learn, and constraining what we learn about with $\relsg(s,g)$ is likely key for practical performance. 

The modification to the update is simple: we simply do not update $\rgamma(s,g), \Gammasg(s,g)$ in states $s$ where $\relsg(s,g) = 0$.\footnote{More generally, we might consider using \emph{emphatic weightings} \citep{sutton2016anemphatic} that allow us to incorporate such interest weightings $\relsg(s,g)$, without suffering from bootstrapping off of inaccurate values in states where $\relsg(s,g) = 0$. Incorporating this algorithm would likely benefit the whole system, but we keep things simpler for now and stick with a typical TD update.}
For the action-value variant, we do not update for state-action pairs $(s,a)$ where $\relsg(s,g) = 0$ and $\pi_g(s) \neq a$. The model will only ever be queried in $(s,a)$ where $\relsg(s,g) = 1$ and $\pi_g(s) = a$.

\paragraph{Learning the relevance model $d$}
We assume in this work that we simply have $\relsg(s,g)$, but we can at least consider ways that we could learn it. One approach is to attempt to learn $\GamModel$ for each $g$, to determine which are pertinent. Those with $\GamModel(s,g)$ closer to zero can have $\relsg(s,g) = 0$. In fact, such an approach was taken for discovering options \citep{khetarpal2020what}, where both options and such a relevance function are learned jointly. For us, they could also be learned jointly, where a larger set of goals start with $\relsg(s,g) = 1$, then if $\GamModel(s,g)$ remains small, then these may be switched to $\relsg(s,g) = 0$ and they will stop being learned in the model updates. 

\paragraph{Learning the Subgoal-to-Subgoal Models}
Finally, we need to extract the subgoal-to-subgoal models $\vgg, \Gammagg$ from $\rgamma, \GamModel$. These models were defined as means of the GVFs taken over member states of each subgoal, as specified in Equation \ref{sg_models}. The strategy involves updating towards the state-to-subgoal models, whenever a state corresponds to a subgoal. In other words, for a given $s$, if $\indsg(s,g) = 1$, then for a given $g'$ (or iterating through all of them), we can update $\vgg$ using
\begin{align*}
(\rgamma(s,g') - \vgg(g,g')) \nabla \vgg(g,g')
\end{align*}
and update $\Gammagg$ using
\begin{align*}
(\Gammasg(s,g') - \Gammagg(g,g')) \nabla \Gammagg(g,g')
.
\end{align*}
Note that these updates are not guaranteed to uniformly weight the states where $\indsg(s,g) = 1$. Instead, the implicit weighting is based on sampling $s$, such as through which states are visited and in the replay buffer. We do not attempt to correct this skew, as mentioned in the main body, we presume that this bias is minimal. An important next step is to better understand if this lack of reweighting causes convergence issues, and how to modify the algorithm to account for a potentially changing state visitation. 

\paragraph{Computing $\vsub$}
In order to compute $\vsub$, we first need a $\vgoal$ from our abstract MDP to look up abstract subgoal values. We compute $\vgoal$ by value iteration in the abstract MDP with a tolerance of $\epsilon = 10^{-8}$ and maximum of 10,000 iterations. The resulting $\vgoal$ from these subgoal models was used in the projection step to obtain $\vsub$, by iterating over relevant subgoals as described in Equation \eqref{eq:vsub}.

\begin{figure}

\centering
\includegraphics[width=0.5\textwidth]{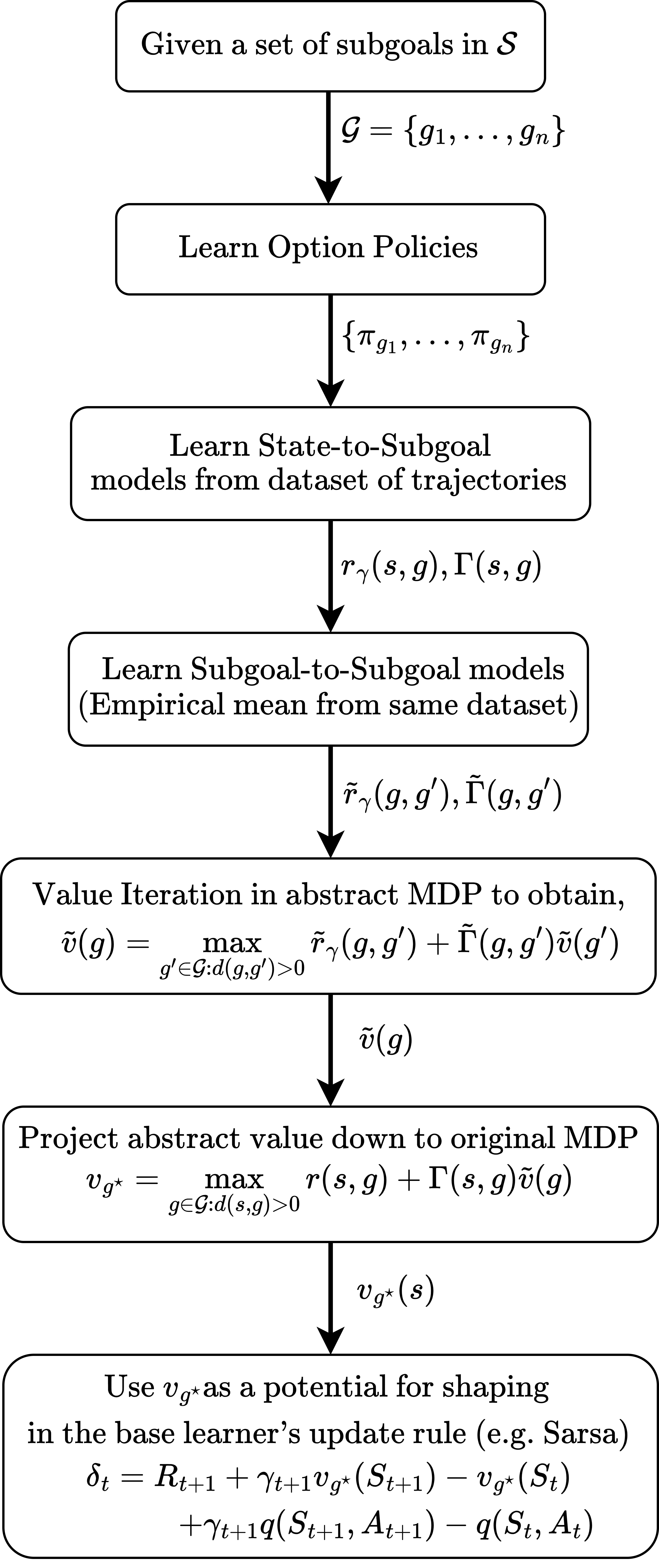}
\caption{Learning and using pre-trained models for GSP.}
\label{fig:gsp_flowchart}
\end{figure}

\section{Putting it all together}\label{app:pseudocode}

We summarize the above updates in pseudocode, specifying explicit parameters and how they are updated.
The algorithm is summarized in Algorithm \ref{alg:gsp}, with a diagram in Figure \ref{gsp_diagram}. An online update is used for the action-values for the main policy, without replay. All background computation is used for model learning using a replay buffer and for planning with those models. The pseudocode assumes a small set of subgoals, and is for episodic problems. We provide extensions to other settings in Appendix \ref{app_variants}, including using a Double DQN update for the policy update. We also discuss in-depth differences to existing related ideas, including landmark states and UVFAs.

Note that we overload the definitions of the subgoal models. We learn action-value variants $\rgamma(s,a,g; \rparams)$, with parameters $\rparams$, to avoid importance sampling corrections. 
We learn the option-policy using action-values $\optionq(s,a; \polparams)$ with parameters $\gamparams$, and so query the policy using $\pi_g(s; \polparams) \defeq \argmax_{a\in \Actions} \optionq(s,a, g; \polparams)$. The policy $\pi_g$ is not directly learned, but rather defined by $\optionq$. Similarly, we do not directly learn $\rgamma(s,g)$; instead, it is defined by $\rgamma(s,a,g; \rparams)$. Specifically, for model parameters $\mparams = (\rparams,\gamparams,\polparams)$, we set $\rgamma(s,g; \mparams) \defeq \rgamma(s,\pi_g(s; \polparams),g; \rparams)$ and $\GamModel(s,g; \mparams) \defeq \GamModel(s,\pi_g(s; \polparams),g; \gamparams)$. We query these derived functions in the pseudocode.

Finally, we assume access to a given set of subgoals. But there have been several natural ideas already proposed for option discovery, that nicely apply in our more constrained setting. One idea was to use subgoals that are often visited by the agent \citep{stolle2002learning}. Such a simple idea is likely a reasonable starting point to make a GSP algorithm that learns everything from scratch, including subgoals. Other approaches have used bottleneck states \citep{mcgovern2001automatic}.

\subsection{Pseudocode}

\begin{algorithm}[H]
  \caption{Goal-Space Planning for Episodic Problems}
  \label{alg:gsp}
\begin{algorithmic}
\State Assume given subgoals $\Goals$ and relevance function $d$
\State Initialize base learner (i.e. $\qparams, \mathbf{z} = \mathbf{0}, \mathbf{0}$ for Sarsa($\lambda$)), model parameters $\mparams = (\rparams, \gamparams, \polparams), \tilde{\mparams} = (\vggparams, \gamggparams)$
\State Sample initial state $s_0$ from the environment
  \For {$t \in 0, 1, 2, ...$}
    \State Take action $a_t$ using $q$ (e.g., $\epsilon$-greedy), observe $s_{t+1}, r_{t+1}, \gamma_{t+1}$
    \State Choose $a'$ from $s_{t+1}$ using $q$ (e.g. $\epsilon$-greedy)
    \State \subroutine{ModelUpdate}$(s_t, a_t, s_{t+1}, r_{t+1}, \gamma_{t+1})$
    \State \subroutine{Planning}$()$
    \State \subroutine{MainPolicyUpdate}$(s_t, a_t, s_{t+1}, r_{t+1}, \gamma_{t+1}, a')$ // Changes with base learner
  \EndFor
\end{algorithmic}
\end{algorithm}

\begin{algorithm}[H]
  \caption{\subroutine{MainPolicyUpdate}$(s, a, s', r, \gamma, a')$}
  \label{alg:MainPolicySarsaLambdaUpdate}
\begin{algorithmic}
\State // For a Sarsa($\lambda$) base learner
\State  $\vsub \gets \max_{g \in \augGoals: \relsg(s,g) > 0} \rgamma(s, g; \mparams) + \GamModel(s,g; \mparams) \vgoal(g)$
  \State $\delta \leftarrow r + \gamma \vsub(s') - \vsub(s) + \gamma q(s',a'; \qparams) - q(s,a; \qparams)$
  \State $\qparams \leftarrow \qparams + \alpha \delta \mathbf{z}\nabla_\qparams q(s,a; \qparams)$ 
  \State $\mathbf{z} \gets \gamma\lambda\mathbf{z} + \nabla_\qparams q(s,a; \qparams)$
\end{algorithmic}
\end{algorithm}

\begin{algorithm}[H]
  \caption{\subroutine{Planning}$()$}
  \label{alg:GoalSpacePlanning}
\begin{algorithmic}
  \For {$n$ iterations, for each $g \in \Goals$ }
    \State $\vgoal(g) \leftarrow \max_{g' \in \augGoals: d(g, g') > 0} \vgg(g, g'; \vggparams) + \Gammagg(g, g'; \gamggparams) \vgoal(g')$
  \EndFor
\end{algorithmic}
\end{algorithm}

\begin{algorithm}[H]
  \caption{\subroutine{ModelUpdate}$(s, a, s', r, \gamma)$}
  \label{alg:ModelUpdate}
\begin{algorithmic}
\State Add new transition $(s, a, s', r, \gamma)$ to buffer $B$
  \For {$g' \in \augGoals$, for multiple transitions $(s, a, r, s', \gamma)$ sampled from $B$}
  \State $\gamma_{g'} \gets \gamma (1 - \indsg(s', g'))$ 
  \State // Update option policy - e.g. by Sarsa
   \State $a' \gets \pi_{g'}(s'; \polparams)$   
    \State $\delta^\pi \leftarrow \tfrac{1}{2} (r - 1) + \gamma_{g'} \optionq(s', a', g'; \polparams) - \optionq(s, a, g'; \polparams)$
    \State $\polparams \leftarrow \polparams + \alpha^{\pi}\delta^\pi \nabla_\polparams q(s, a, g'; \polparams)$
    \State // Update reward model and discount model
    \State $\delta^r \leftarrow r + \gamma_{g'} \rgamma(s', a', g'; \rparams) - \rgamma(s, a, g'; \rparams)$
    \State $\delta^\Gamma \leftarrow \indsg(s', g)\gamma + \gamma_{g'} \GamModel(s', a', g'; \gamparams) - \GamModel(s, a, g'; \gamparams)$
    \State $\rparams \leftarrow \rparams + \alpha^{r}\delta^r \nabla_\rparams \rgamma(s, a, g'; \rparams)$
    \State $\gamparams \leftarrow \gamparams + \alpha^{\Gamma} \delta^\Gamma \nabla_\gamparams \GamModel(s, a, g'; \gamparams)$
    \State // Update goal-to-goal models using state-to-goal models
    \For {each $g$ such that $\indsg(s,g) > 0$ }
    \State $\vggparams \leftarrow \vggparams + \tilde{\alpha}^{r} (\rgamma(s, g'; \mparams) - \vgg(g, g';\vggparams)) \nabla_\rparams \vgg(g, g'; \vggparams)$
     \State $\gamggparams \leftarrow \gamggparams + \tilde{\alpha}^{\Gamma} (\GamModel(s, g'; \mparams) - \Gammagg(g, g';\vggparams)) \nabla_\gamparams
     \Gammagg(g, g'; \gamggparams)$
 \EndFor
  \EndFor
\end{algorithmic}
\end{algorithm}

\subsection{Extending GSP to Deep RL}\label{app_variants}
It is simple to extend the above pseudocode for the main policy update and the option policy update to use Double DQN \citep{van2016deep} updates with neural networks. The changes from the above pseudocode are 1) the use of a target network to stabilize learning with neural networks, 2) changing the one-step bootstrap target to the DDQN equivalent, 3) adding a replay buffer for learning the main policy, and 4) changing the update from using a single sample to using a batch update. Because the number of subgoals is discrete, the equations for learning $\vggparams$ and $\gamggparams$ does not change. We previously summarized these changes for learning the main policy in Algorithm \ref{alg:DDQN_GSP} and now detail the subgoal model learning in Algorithm \ref{alg:DDQN_model}.

\begin{algorithm}[H]
  \caption{\subroutine{DDQNModelUpdate}$(s, a, s', r, \gamma)$}
  \label{alg:DDQN_model}
\begin{algorithmic}
\State Add new transition $(s, a, s', r, \gamma)$ to buffer $D_{\mathrm{model}}$
  \For {$g' \in \augGoals$}
  \For {$n_{\mathrm{model}}$ mini-batches} 
  \State Sample batch $B_{\mathrm{model}} = \{ (s, a, r, s', \gamma )\}$ from $D_{\mathrm{model}}$
  \State $\gamma_{g'} \gets \gamma (1 - \indsg(s', g'))$ 
  \State // Update option policy
  \State $a' \leftarrow \argmax_{a' \in \Actions} \optionq(s', a', g'; \polparams)$
    \State $\delta^\pi(s, a, s', r, \gamma) \gets \tfrac{1}{2} (r - 1) + \gamma_{g'} \optionq(s', a', g'; \boldsymbol{\theta}^\pi_{\mathrm{targ}}) - q(s, a, g'; \polparams)$
    \State $\polparams \leftarrow \polparams - \alpha^{\pi}\nabla_{\polparams} \frac{1}{|B_{\mathrm{model}}|} \sum_{(s, a, r, s', \gamma) \in B_{\mathrm{model}}}(\delta^\pi)^2$

    \State $\polparams_{\mathrm{targ}} \leftarrow \rho_{\mathrm{model}} \polparams + (1 - \rho_{\mathrm{model}})\polparams_{\mathrm{targ}}$ 
    \State // Update reward model and discount model
    \State $\delta^r (s, a, r, s', \gamma) \leftarrow r + \gamma_{g'}(\gamma, s')\rgamma(s', a', g'; \rparams_{\mathrm{targ}}) - \rgamma(s, a, g'; \rparams)$
    \State $\delta^\Gamma (s, a, r, s', \gamma)  \leftarrow \indsg(s', g)\gamma + \gamma_{g'}(\gamma, s') \GamModel(s', a', g'; \gamparams_{\mathrm{targ}}) - \GamModel(s, a, g'; \gamparams)$
    \State $\rparams \leftarrow \rparams - \alpha^{r} \nabla_{\rparams} \frac{1}{|B_{\mathrm{model}}|} \sum_{(s, a, r, s', \gamma) \in B_{\mathrm{model}}} (\delta^r)^2$
    \State $\gamparams \leftarrow \gamparams - \alpha^{\Gamma} \nabla_{\gamparams} \frac{1}{|B_{\mathrm{model}}|} \sum_{(s, a, r, s', \gamma) \in B_{\mathrm{model}}} (\delta^\Gamma)^2$
    \If{$n_{\text{updates}}\% \tau == 0$}
    \State $\rparams_{\mathrm{targ}} \leftarrow \rparams$ 
    \State $\gamparams_{\mathrm{targ}} \leftarrow \gamparams$
    \EndIf
    \State $n_{\text{updates}}$ = $n_{\text{updates}}$ + 1
    \EndFor
    \State // Update goal-to-goal models using state-to-goal models
    \State \dots same as in prior pseudocode.
  \EndFor
\end{algorithmic}
\end{algorithm}

\subsection{Optimizations for GSP using Fixed Models}
It is possible to reduce computation cost of GSP when learning with a fixed model. When the subgoal models are fixed, $\vsub$ for an experience sample does not change over time as all components that are used to calculate $\vsub$ are fixed. This means that 
the agent can calculate $\vsub$ when it first receives the experience sample and save it in the buffer, and use the same calculated $\vsub$ whenever this sample is used for updating the main policy. When doing so, $\vsub$ only needs to be calculated once per sample experienced, instead of with every update. This is beneficial when training neural networks, where each sample is often used multiple times to update network weights.

An additional optimization possible on top of caching of $\vsub$ in the replay buffer is that we can batch the calculation of $\vsub$ for multiple samples together, which can be more efficient than calculating $\vsub$ for a single sample every step. To do this, we create an intermediate buffer that stores up to some number of samples. When the agent experiences a transition, it adds the sample to this intermediate buffer rather than the main buffer. When this buffer is full, the agent calculates $\vsub$ for all samples in this buffer at once and adds the samples alongside $\vsub$ to the main buffer. This intermediate buffer is then emptied and added to again every step. We set the maximum size for the intermediate buffer to 1024 in our experiments.

\section{An Alternative way of using $\vsub$}\label{app:alt_way}
As mentioned in section \ref{sec:alt_way}, this work also looked at an alternative way of incorporating $\vsub$ into the base learner's update rule. We do so by biasing the target of the TD error towards $\vsub$. This modifies the TD error,

\begin{equation}
    R_{t+1} + \gamma_{t+1} (\beta \vsub(S_{t+1}) + (1-\beta)q(S_{t+1}, A_{t+1})) - q(S_t, A_t),
\end{equation}

\noindent where $\beta \in [0,1]$ is a hyper-parameter. We can recover the base learner's update rule by setting $\beta = 0$, whereas $\beta = 1$ completely biases the updates towards the model's prediction (as in our approximation to LAVI in Section \ref{sec:alt_way}). While this allows us to control the extent of our model's influence on the learning update, we found using $\vsub$ as a potential to out perform all $\beta$ in Four-rooms, GridBall and PinBall. However, biasing the TD target in this manner does give the update a faster convergence as we reduce the effective horizon. We shall show this by analyzing the update to the main policy.

We assume we have a finite number of state-action pairs $n$, with parameterized action-values $q(\cdot; \mathbf{w}) \in \RR^n$ represented as a vector with one entry per state-action pair. 
Value iteration to find $q^*$ corresponds to updating with the Bellman optimality operator
\begin{equation}\label{eq_boperator}
(B q)(s,a) \defeq r(s,a) + \sum_{s'} P(s' | s, a)\gamma(s') \max_{a' \in \Actions} q(s',a') 
\end{equation}
On each step, for the current $q_t \defeq q(\cdot; \mathbf{w}_t)$, if we assume the parameterized function class can represent $B q_t$, then we can reason about the iterations of $\mathbf{w}_1, \mathbf{w}_2, \ldots$ obtain when minimizing distance between $q(\cdot; \mathbf{w}_{t+1})$ and $B q_t$, with
\begin{equation*}
q(s,a; \mathbf{w}_{t+1}) = (B q(\cdot; \mathbf{w}_{t}))(s,a)
\end{equation*}
Under function approximation, we do not simply update a table of values, but we can get this equality by minimizing until we have zero Bellman error. Note that $q^\star = B q^\star$, by definition.

In this \emph{realizability} regime, we can reason about the iterates produced by value iteration. The convergence rate is dictated by $\gamma_c$, as is well known, because 
\begin{equation*}
\| B q_1 - B q_2 \|_\infty \le \gamma_c \| q_1 - q_2\|_\infty
\end{equation*}
Specifically, if we assume $|r(s,a)| \le \rmax$, then we can use the fact that 1) the maximal return is no greater than $\gmax \defeq \frac{\rmax}{1-\gamma_c}$, and 2) for any initialization $q_0$ no larger in magnitude than this maximal return we have that $\| q_0 - q^\star \|_\infty \le 2 \gmax$. Therefore, we get that 
\begin{equation*}
\| B q_0 - q^\star \|_\infty = \| B q_0 - B q^\star \|_\infty \le \gamma_c \| q_0 - q^\star\|_\infty
\end{equation*}
and so after $t$ iterations we have 
\begin{equation*}
\| q_t - q^\star \|_\infty = \| B q_{t-1} - B q^\star \|_\infty \le \gamma_c \| q_{t-1} - q^\star\|_\infty \le \gamma_c^2 \| q_{t-2} - q^\star\|_\infty \ldots \le \gamma_c^t \| q_0 - q^\star\|_\infty = \gamma_c^t \gmax
\end{equation*}
We can use the exact same strategy to show convergence of value iteration, under our subgoal-value bootstrapping update. Let $r_{g^\star}(s,a) \defeq \sum_{s'} P(s' | s,a) \vsub(s')$, assuming $\vsub: \States \rightarrow [-\gmax, \gmax]$ is a given, fixed function. Then the modified Bellman optimality operator is
\begin{equation}\label{eq_boperator_beta}
(B^\beta q)(s,a) \defeq r(s,a) + \beta r_{g^\star}(s,a) + (1-\beta)\sum_{s'} P(s' | s, a) \gamma(s') \max_{a' \in \Actions} q(s',a'). 
\end{equation}

\begin{proposition}[Convergence rate of tabular value iteration for the biased update]\label{prop_gsp_rate}
The fixed point $q^\star_\beta = B^\beta q^\star_\beta$ exists and is unique.
Further, for $q_0$, and the corresponding $\mathbf{w}_0$, initialized such that $|q_0(s,a; \mathbf{w}_0)| \le \gmax$, the value iteration update with subgoal bootstrapping $q_t = B^\beta q_{t-1}$ for $t = 1, 2, \ldots$ satisfies
\begin{equation*}
\| q_t - q^\star_\beta \|_\infty \le \gamma_c^t(1-\beta)^t \frac{\rmax +  \beta\gmax}{1-\gamma_c(1-\beta)} 
\end{equation*}
\end{proposition}
\begin{proof}
First we can show that $B^\beta$ is a $\gamma_c(1-\beta)$-contraction. Assume we are given any two vectors $q_1, q_2$. Notice that $\gamma(s) \le \gamma_c$, because for our problem setting it is either equal to $\gamma_c$ or equal to zero at termination. Then we have that for any $(s,a)$
\begin{align*}
|(B^\beta) q_1(s,a) - (B^\beta q_2)(s,a)| 
&=  \Bigg|(1-\beta)\sum_{s'} P(s' | s, a) \gamma(s') [\max_{a' \in \Actions} q_1(s',a') - \max_{a' \in \Actions} q_2(s',a')]\Bigg|\\
&\le \gamma_c(1-\beta) \sum_{s'} P(s' | s, a) \big|\max_{a' \in \Actions} q_1(s',a') - \max_{a' \in \Actions} q_2(s',a')\big|\\
&\le \gamma_c(1-\beta) \sum_{s'} P(s' | s, a) \max_{a' \in \Actions} |q_1(s',a') - q_2(s',a')|\\
&\le \gamma_c(1-\beta) \sum_{s'} P(s' | s, a) \max_{s' \in \States, a' \in \Actions} |q_1(s',a') - q_2(s',a')|\\
&\le \gamma_c(1-\beta) \sum_{s'} P(s' | s, a) \| q_1 - q_2 \|_{\infty}\\
&= \gamma_c(1-\beta) \| q_1 - q_2 \|_{\infty}
\end{align*}
Since this is true for any $(s,a)$, it is true for the max, giving
\begin{equation*}
\| B^\beta q_1 - B^\beta q_2 \|_\infty \le \gamma_c(1-\beta) \| q_1 - q_2 \|_{\infty}
.
\end{equation*}
Because the operator is a contraction, since $\gamma_c(1-\beta) < 1$, we know by the Banach Fixed-Point Theorem that the fixed-point exists and is unique. 

Now we can also use contraction property for the convergence rate. Notice first that we can consider $\tilde{r}(s,a) \defeq r(s,a) + r_{g^\star}(s,a)$ as the new reward, with maximum value $\rmax + \beta\gmax$. Taking discount as $\gamma_c(1-\beta)$, the maximal return is $\frac{\rmax +  \beta\gmax}{1-\gamma_c(1-\beta)}$.
\begin{align*}
\| q_t - q^\star_\beta \|_\infty &= \| B^\beta q_{t-1} - B^\beta q^\star \|_\infty \le \gamma_c(1-\beta) \| q_{t-1} - q^\star\|_\infty \ldots \le \gamma_c^t(1-\beta)^t \| q_0 - q^\star\|_\infty\\
&\le \gamma_c^t(1-\beta)^t \frac{\rmax + \beta\gmax}{1-\gamma_c(1-\beta)} 
\end{align*}
\par\vspace{-0.9cm}
\end{proof}

This rate is dominated by $(\gamma_c(1-\beta))^t$. We can determine after how many iterations this term overcomes the increase in the upper bound on the return. In other words, we want to know how big $t$ needs to be to get
\begin{align*}
\gamma_c^t(1-\beta)^t \frac{\rmax + \beta\gmax}{1-\gamma_c(1-\beta)}  \le \gamma_c^t \gmax
. 
\end{align*}
Rearranging terms, we get that this is true for
\begin{align*}
t > \log\left(\frac{\rmax + \beta \gmax}{\gmax(1-\gamma_c(1-\beta))}\right) \Bigg/ \log\left(\frac{1}{1-\beta}\right)
. 
\end{align*}
For example if $\rmax = 1$, $\gamma_c = 0.99$ and $\beta = 0.5$, then we have that $t > 1.56$. 
If we have that $\rmax = 10$, $\gamma_c = 0.99$ and $\beta = 0.5$, then we get that $t \ge 5$. 
If we have that $\rmax = 1$, $\gamma_c = 0.99$ and $\beta = 0.1$, then we get that $t \ge 22$. 

While this increased convergence rate is present for the biased update, it does not show up when using $\vsub$ as a potential-based shaping reward. Although using $\vsub$ as a potential does not increase the convergence rate to $v^\star$, it can help quickly identify $\pi^\star$. Specifically, when $\vsub$ is $v^\star$, and the value function is constant, e.g., initialized to $0$, it only takes one application of the bellman operator in each state to find the optimal policy. We formalize this in the proposition below. 

\begin{proposition}
\label{prop:optimalstep}
    For $\vsub = v^\star$ and $v_0 = c$, for $c \in \mathbb{R}$, then the policy, $\pi_1$ derived after a single bellman update at all states will be optimal, i.e., 
    \begin{equation}
    \forall s \ \pi_1(s) \in \argmax_{a} q^\star(s,a). 
    \end{equation}
\end{proposition}
\begin{proof}
    Let the $q$ estimate for the $k^\text{th}$ iteration be 
    \begin{equation}
        q_k(s,a) = R(s,a) + \sum_{s'}P(s,a,s') \left (\gamma_c \vsub(s') - \vsub(s) + \gamma_c v_{k-1}(s') \right ). 
    \end{equation}
    The value function for iteration $k$ is $v_k = \max_a q_k(s,a)$ and 
    the policy for the $k^\text{th}$ iteration is $\pi_k(s) \in \argmax_a q_k(s,a)$. The value of $q_1$ is 
    \begin{align}
        q_1(s,a) = &\ R(s,a) + \sum_{s'}P(s,a,s') \left (\gamma_c \vsub(s') - \vsub(s) + \gamma_c v_{0}(s') \right ) \\
        = &\ R(s,a) + \sum_{s'}P(s,a,s') \left ( \gamma_c v^\star(s') - v^\star(s) + \gamma_c v_{0}(s') \right ) \\
        = &\ \underbrace{R(s,a) + \sum_{s'}P(s,a,s')\gamma_c v^\star(s')}_{=q^\star(s,a)} + \underbrace{\sum_{s'}P(s,a,s') \gamma_c v_{0}(s')}_{= \gamma_c c} - v^\star(s) \\
        = &\ q^\star(s,a) + \gamma_c c - v^\star(s)
    \end{align}
    where the last line follows because $v_{0}(s') = c$ for all $s'$.
    Then plugging this expression into $\pi_1$ yields 
    \begin{equation}
    \pi_1(s) \in \argmax_a q^\star(s,a) + \gamma_c c - v^\star(s) = \argmax_a q^\star(s,a).
    \end{equation}
\end{proof}

\begin{remark}
While having $\vsub=v^\star$ is not realistic, Proposition \ref{prop:optimalstep} means that the policy will quickly align with what is preferable under $\vsub$ before finding what is optimal for the MDP without the shaping reward. 
\end{remark}

\section{Errors in Learned Subgoal Models} \label{app_model_learning}

\begin{figure}[htbp] 
  \begin{subfigure}[l]{0.35\textwidth}
    \centering
    \begin{subfigure}[b]{0.45\textwidth} 
        \centering
        \includegraphics[width=\textwidth]{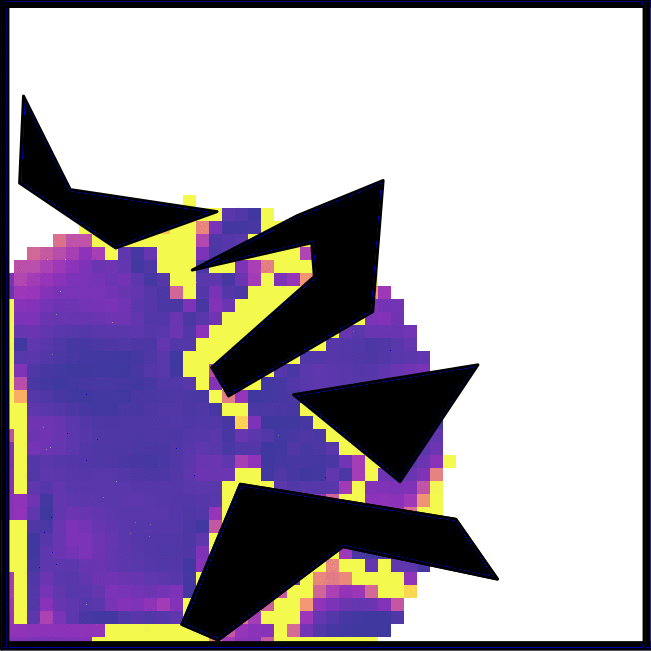} 
    \end{subfigure}
    \begin{subfigure}[b]{0.45\textwidth} 
        \centering
        \includegraphics[width=\textwidth]{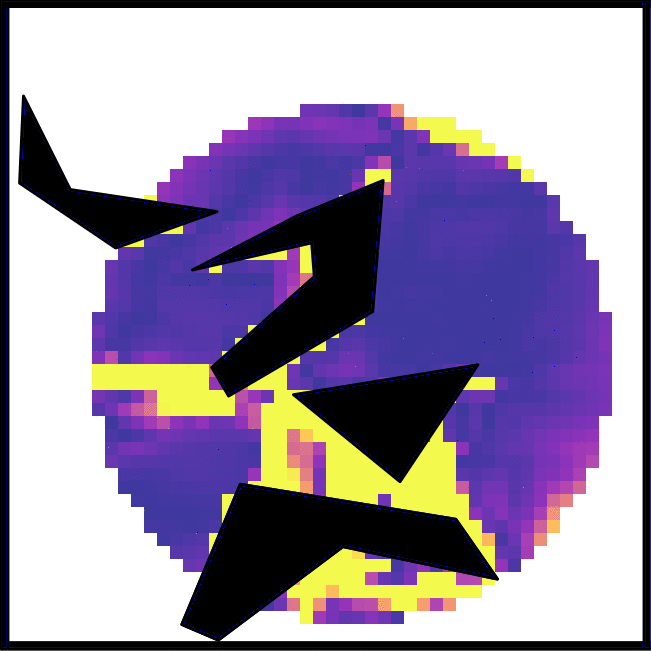} 
    \end{subfigure}

    \begin{subfigure}[b]{0.45\textwidth} 
        \centering
        \includegraphics[width=\textwidth]{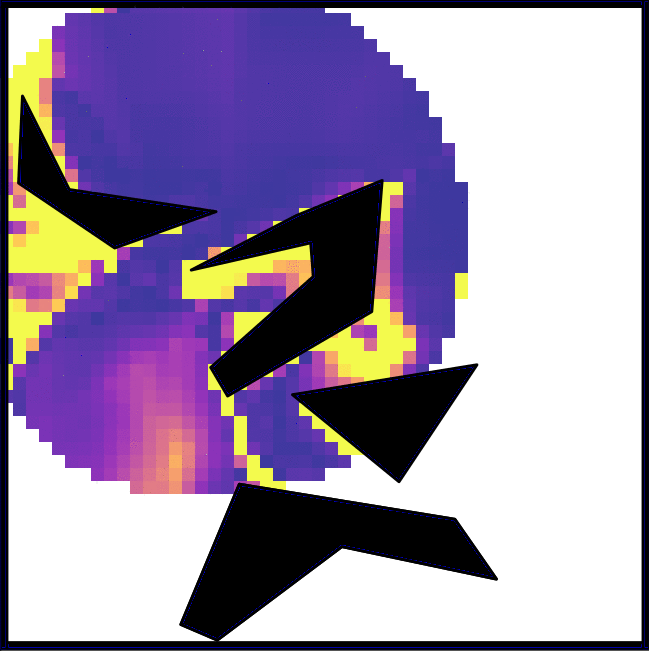} 
    \end{subfigure}
    \begin{subfigure}[b]{0.45\textwidth} 
        \centering
        \includegraphics[width=\textwidth]{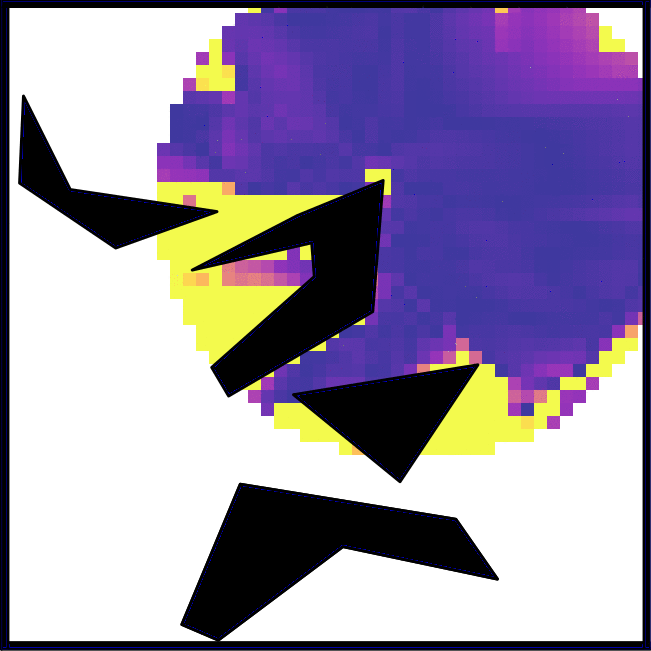} 
    \end{subfigure}
  \caption{Absolute error in $r(s,g)$}
  \label{r_model_error}
  \end{subfigure}
  \begin{subfigure}[l]{0.08\textwidth}
    \centering
    \includegraphics[width=\textwidth]{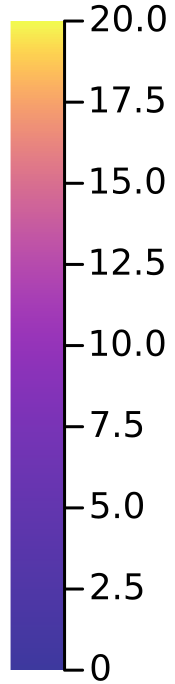} 
  \end{subfigure}
\hfill
\begin{subfigure}[r]{0.35\textwidth} 
  \centering
  \begin{subfigure}[b]{0.45\textwidth} 
      \centering
      \includegraphics[width=\textwidth]{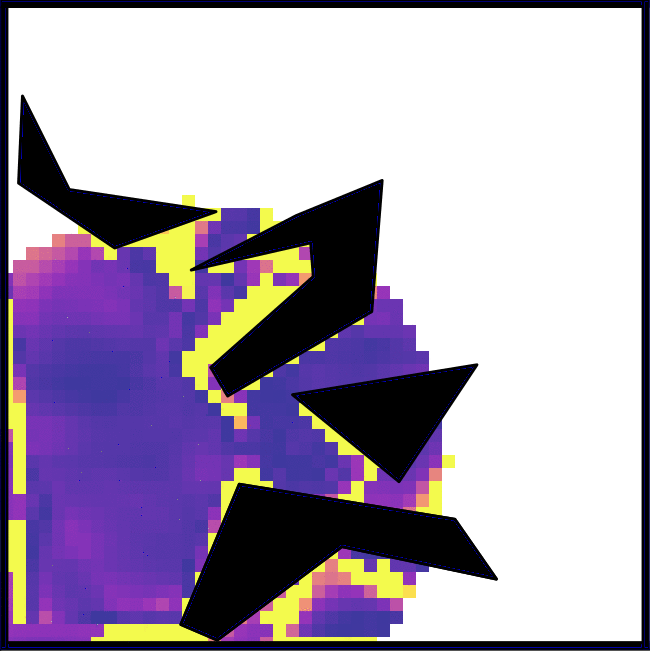} 
  \end{subfigure}
  \begin{subfigure}[b]{0.45\textwidth} 
      \centering
      \includegraphics[width=\textwidth]{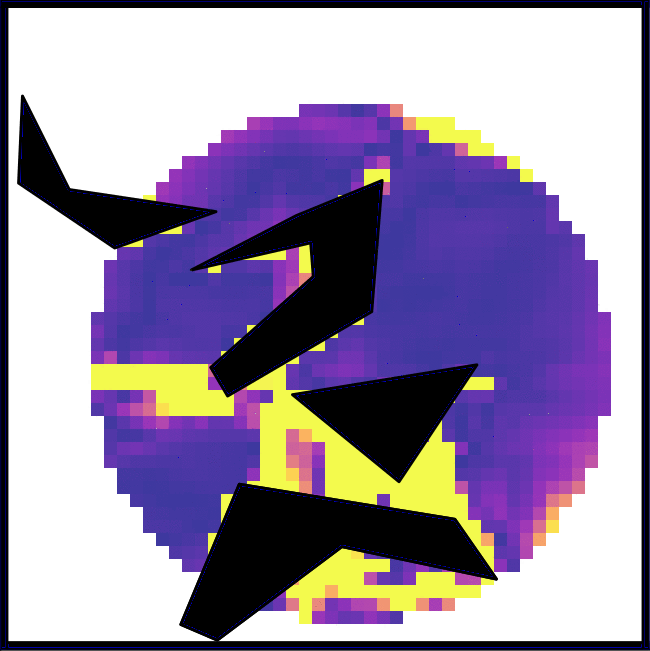} 
  \end{subfigure}

  \begin{subfigure}[b]{0.45\textwidth} 
      \centering
      \includegraphics[width=\textwidth]{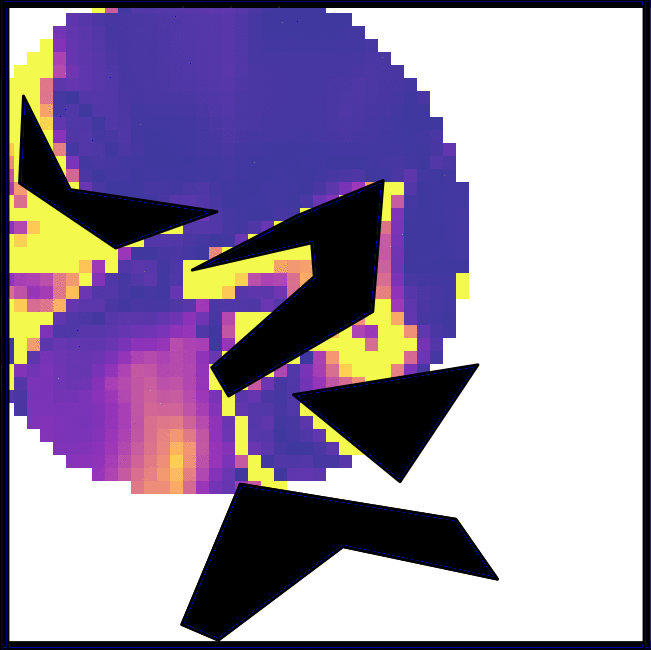} 
  \end{subfigure}
  \begin{subfigure}[b]{0.45\textwidth} 
      \centering
      \includegraphics[width=\textwidth]{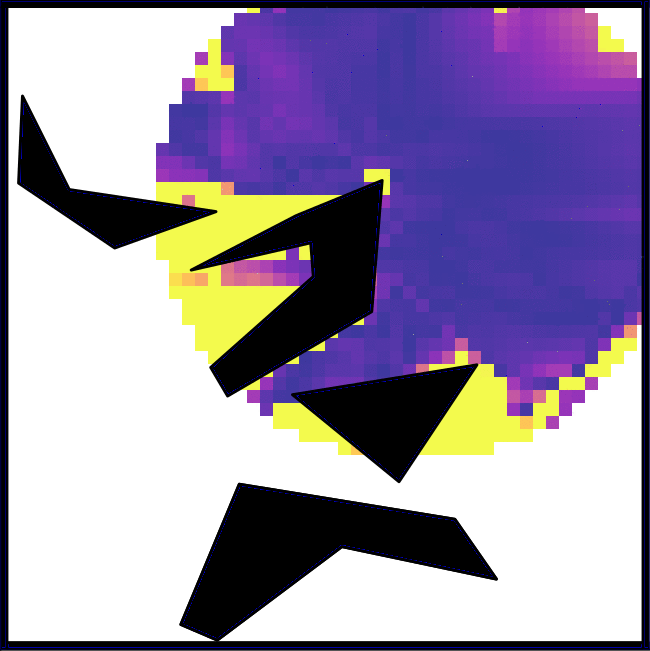} 
  \end{subfigure}
  \caption{Absolute error in $\Gamma(s,g)$}
  \label{Gamma_model_error}
\end{subfigure}
\begin{subfigure}[l]{0.09\textwidth}
  \centering
  \includegraphics[width=\textwidth]{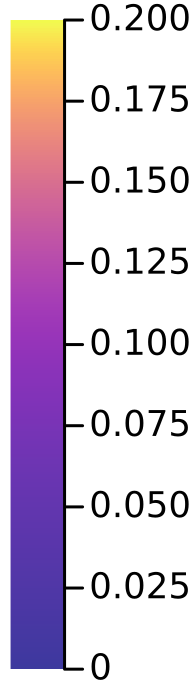} 
\end{subfigure}
\caption{Model errors in State-to-Subgoal models used in GridBall.}
\label{fig:model_errors}
\end{figure}

To better understand the accuracy of our learned subgoal models, we performed roll-outs of the learned option policy at different $(x,y)$ locations on GridBall and compared the true $\rgamma$ and $\Gamma$ with the estimated values. Figure \ref{fig:model_errors} shows a heatmap of the absolute error of the model compared to the ground truth, with the mapping of colors on the right. The error in each pixel was computed by rolling out episodes from that state and logging the actual reward and discounted probability of reaching the subgoal. The models tend to be more accurate in regions that are clear of obstacles, and less near these obstacles or near the boundary of the initiation set. The distribution of error over the initiation set is very similar for both $r$ and $\Gamma$ models. While the magnitudes of errors are not unreasonable, they are also not very near zero. This results is encouraging in that inaccuracies in the model do not prevent useful planning.

\subsection{Subgoal Placement and Region of Attraction}
\label{sec:region_of_attraction}

A counter intuitive observation from the experiments in Section \ref{sec_subgoalselection} was that the On Alternate path helped the agent quickly change its policy but $\vsub$ did not quickly change. In this section, we investigate this reason and put forth the following hypothesis:
\begin{hypothesis}
    GSP creates a region of attraction so that the agent follows the optimal path as determined by the abstract MDP.
\end{hypothesis} 

That is to say, if a single chain of subgoals is represented in the abstract MDP, then the learner will initially try and closely follow this path even if it is not the optimal path. To test this hypothesis, we want to see that the agent will occupy states similar to what is specified by the optimal path in the abstract MDP. For this experiment, we ran GSP on FourRooms (without the lava pools) with each subgoal configuration defined in the previous section. We measured how much time the agent spends in the bottom left room and the top right room. The agent should, as it learns about the environment, spend more time in the top right room and less time in the bottom left room. We would expect all agents to follow this trend, except for the one that is missing a subgoal to go through the top right room. We show the results for each configuration and Sarsa($0$) with no GSP in Figure \ref{fig:attraction}. 

\begin{figure}[tb]
  \centering
    \includegraphics[width=\textwidth]{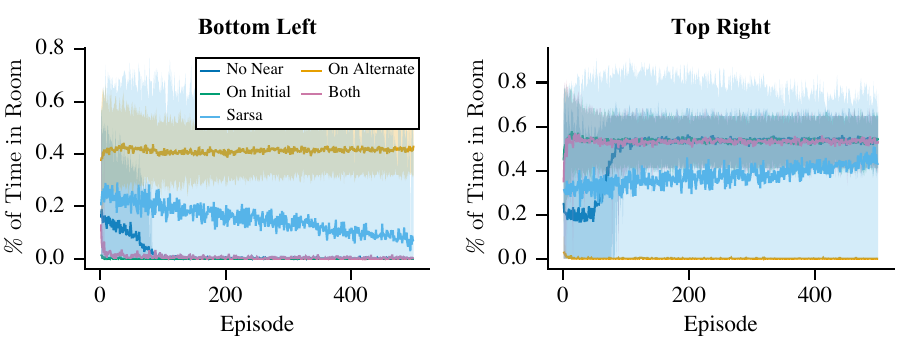}
    \caption{This figure shows the time the agent spends per episode in the bottom left and top right rooms. The lines convey the average $\%$ of time the agent spend and the shaded lines represent $(0.05,0.9)$ tolerance intervals computed from $100$ trials.} 
    \label{fig:attraction}
\end{figure}

The results in Figure \ref{fig:attraction} are clear. All methods learn to go through the top right room except for the subgoal configuration missing a subgoal on that path to the goal state. This supports our hypothesis that the agent will learn to follow the optimal path as specified by the abstract MDP. This also means that while potential-based shaping (used to propagate value information from the abstract MDP to the base learner) does not change the optimal policy, it can make it harder for the learner to find the optimal policy. 

Based on the experiments in Section \ref{sec_subgoalselection} and this one, we can conclude a few key points about GSP. The first is that $\vsub$ can only provide value information as determined by the optimal value function through the abstract MDP, which may not reflect the connectivity of the original MDP. Second, the learner's exploration through the state space will be highly impacted by the known subgoals. With the basic $\epsilon$-greedy exploration policy that GSP currently uses, GSP will quickly follow and refine the best policy found within the abstract MDP. If the optimal policy is near to the policy found by the abstract MDP, then GSP will be able to quickly discover it. However, if the optimal policy is very different than the one found by the abstract MDP (for example, if the best abstract MDP policy follows an alternate sub-optimal path), this will make the agent explore around its sub-optimal policy, and thus possibly slowing down the discovery of the optimal policy, because the basic $\epsilon$-greedy exploration policy centralizes exploration around the current best policy known by the agent. 

This is all to say that there is work to be done to improve GSP's exploration by incorporating more sophisticated exploration strategies. There are also opportunities to develop new exploration strategies that takes advantage of how GSP learns with the knowledge of subgoals within the environment. For example, one may consider leveraging an existing subgoal formulation for more directed exploration by introducing reward bonuses at other subgoals, once we know the environment has changed. Additional work to find new subgoals or refine the current subgoal configurations can also have a high impact in how well GSP can explore and adapt to changes in the environment.

\section{Additional Experiment Details}\label{app_exp_details}

This section provides additional details for the PinBall environment, the various hyperparameters used for DDQN and GSP, and the hyperparameters sweeps performed.

The pinball configuration that we used is based on the easy configuration found at \texttt{https://github.com/DecisionMakingAI/BenchmarkEnvironments.jl}, which was released under the MIT license. We have modified the environment to support additional features such as changing terminations, visualizing subgoals, and various bug fixes.

\subsection{Hyperparameter Sweep Methodology} \label{sec:hyperparam_sweep} 

For Sarsa($\lambda$), we swept it's learning rate over $[0.001, 0.01, 0.05$ $ 0.01, 0.5, 0.9]$. 0.01, 0.05 and 0.1 were found to be the best for FourRooms, GridBall and PinBall respectively. For DDQN, we swept its learning rate $\alpha$ over $[5 \times 10^{-4}, 1\times10^{-3}, 2\times10^{-3}, 4\times10^{-3}, 5\times10^{-3}]$ and target refresh rate $\tau$ over $[1, 50, 100, 200, 1000]$ as shown in Figure \ref{fig:ddqn_sweep}.

\begin{figure}[htb]
    
    \centering

    \begin{subfigure}{0.65\textwidth}
        \includegraphics[width=0.45\linewidth]{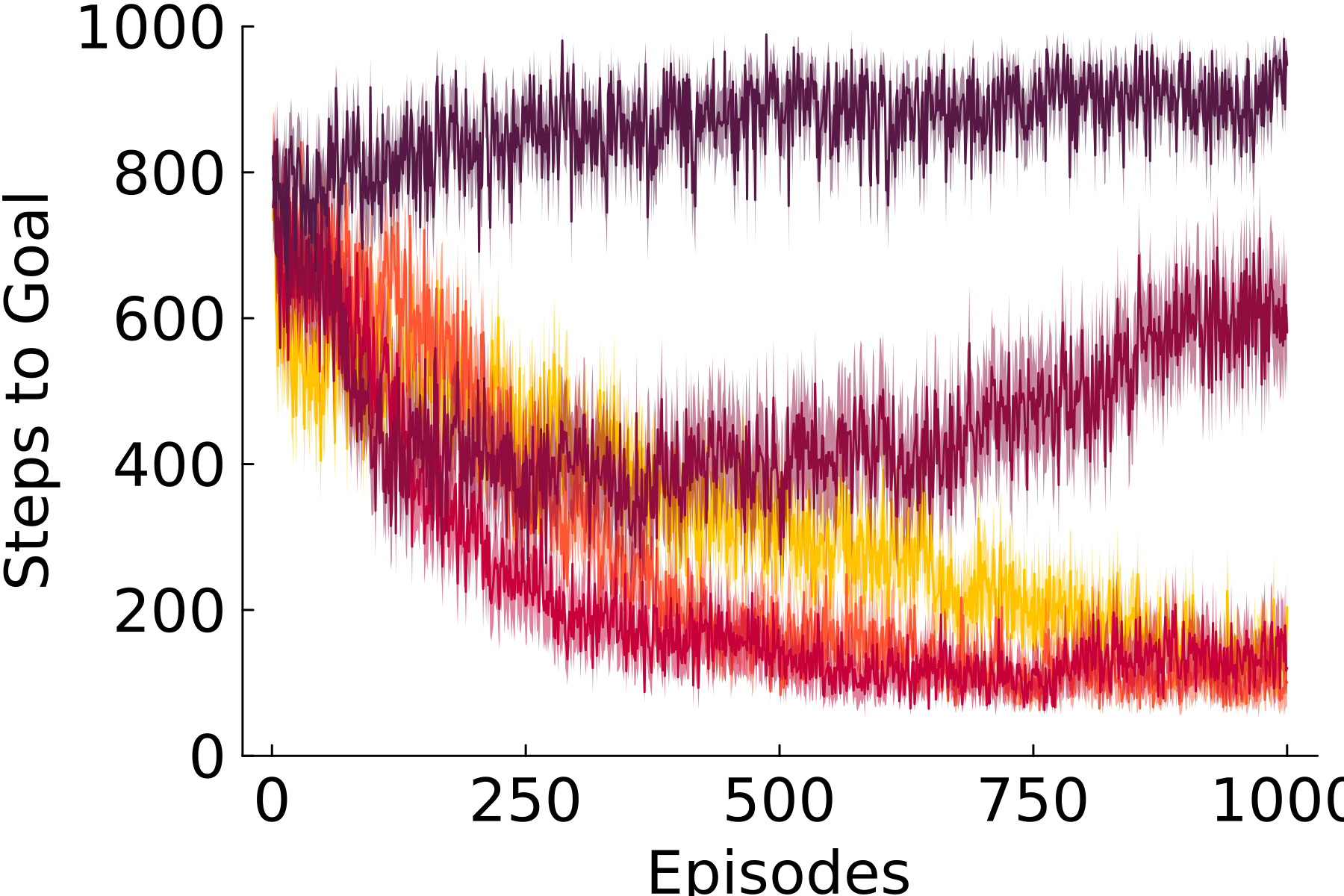}
         \includegraphics[width=0.45\linewidth]{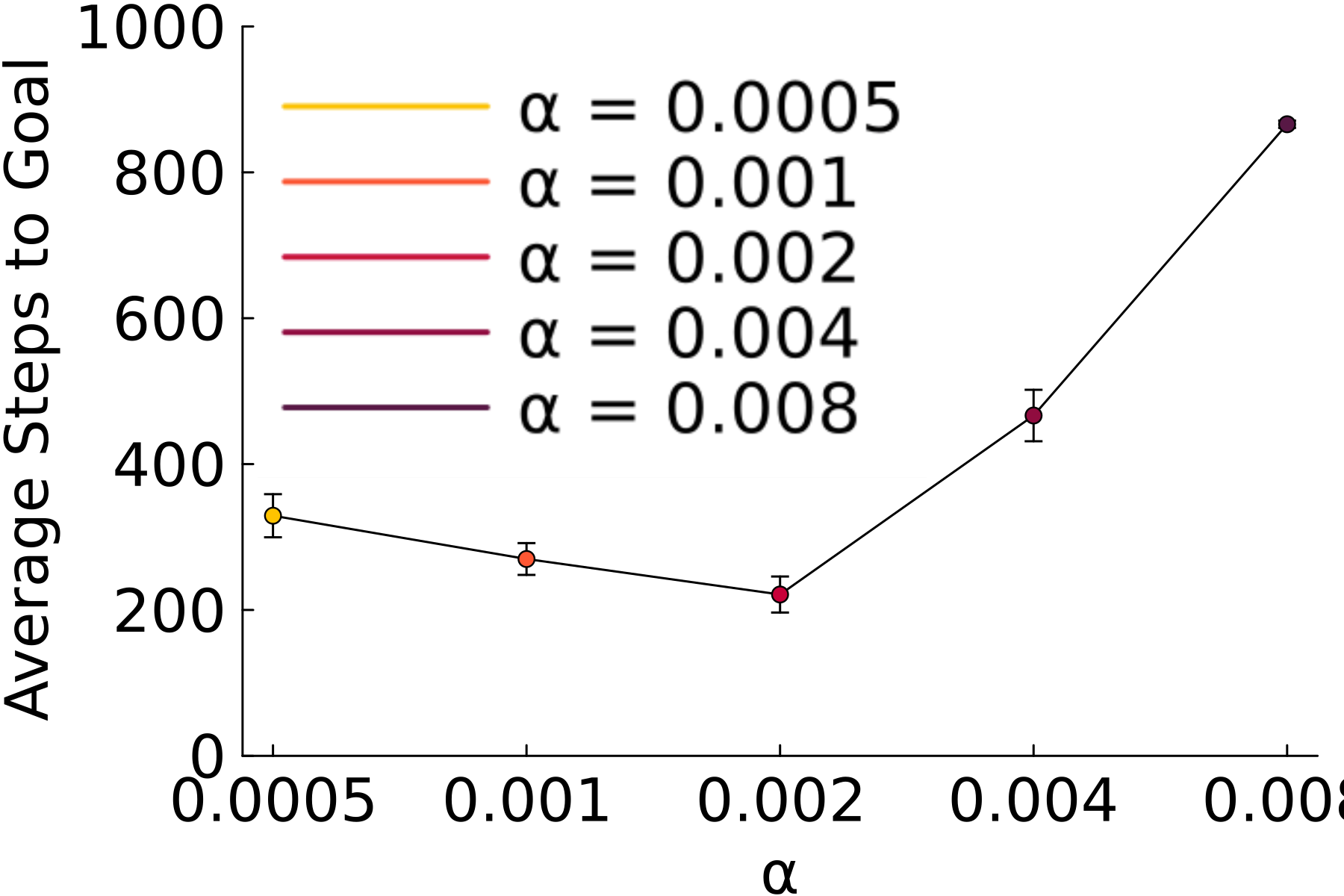}
         \caption{$\tau = 1$}
    \end{subfigure}
    \vspace{\baselineskip} 

    \begin{subfigure}{0.65\textwidth}
        \includegraphics[width=0.45\linewidth]{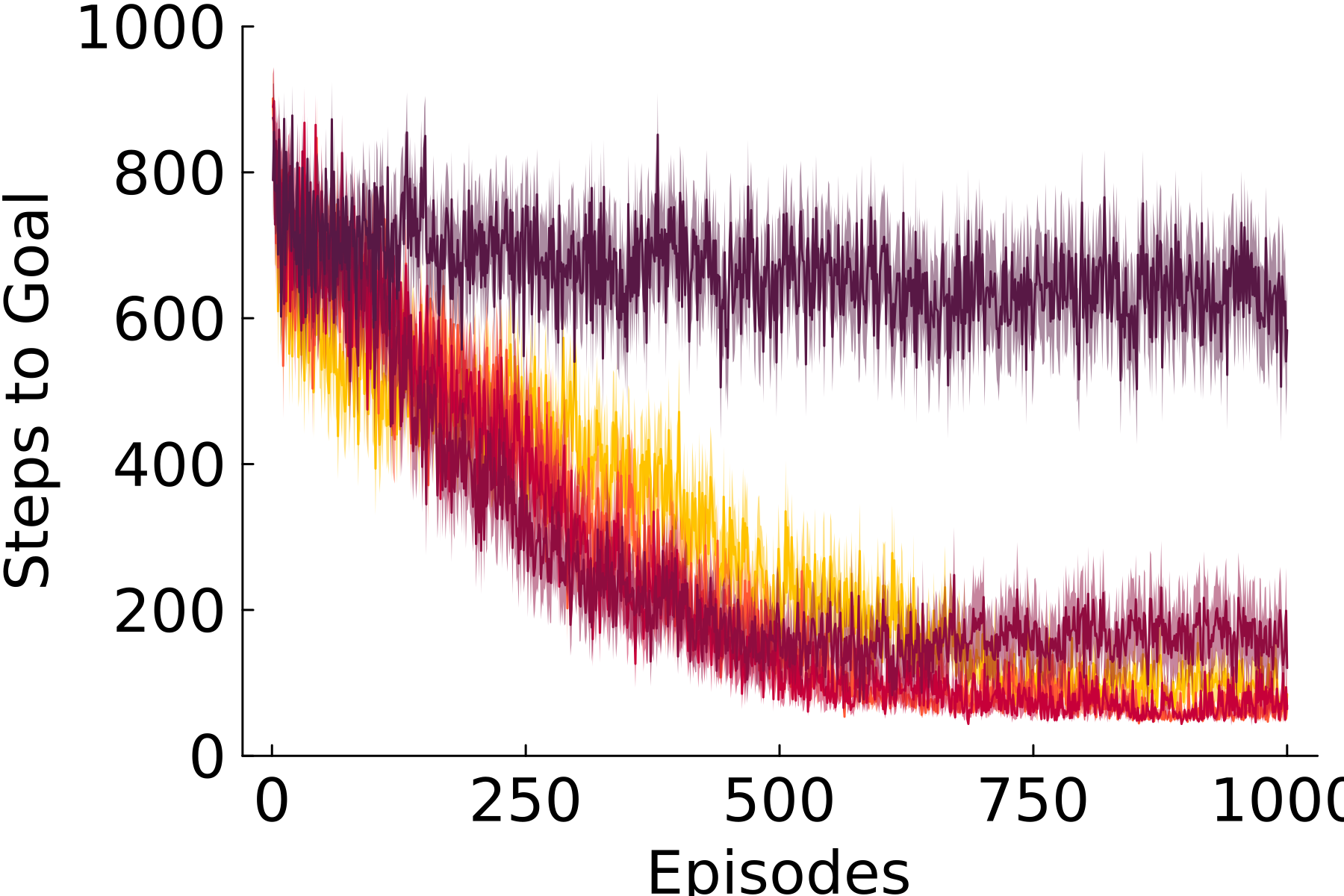}
        \includegraphics[width=0.45\linewidth]{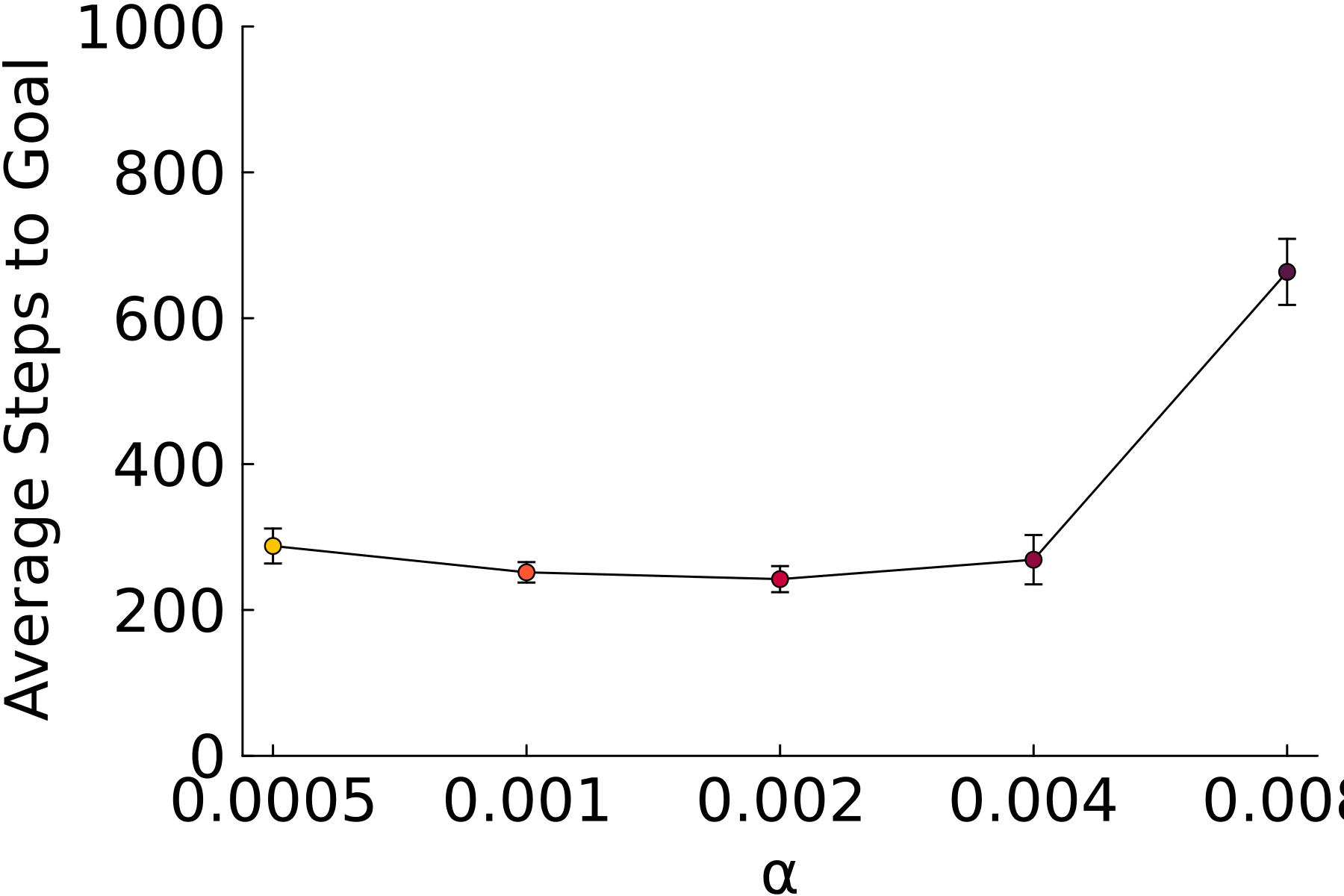}
         \caption{$\tau = 50$}
    \end{subfigure}
      \vspace{\baselineskip}
    
    \begin{subfigure}{0.65\textwidth}
        \includegraphics[width=0.45\linewidth]{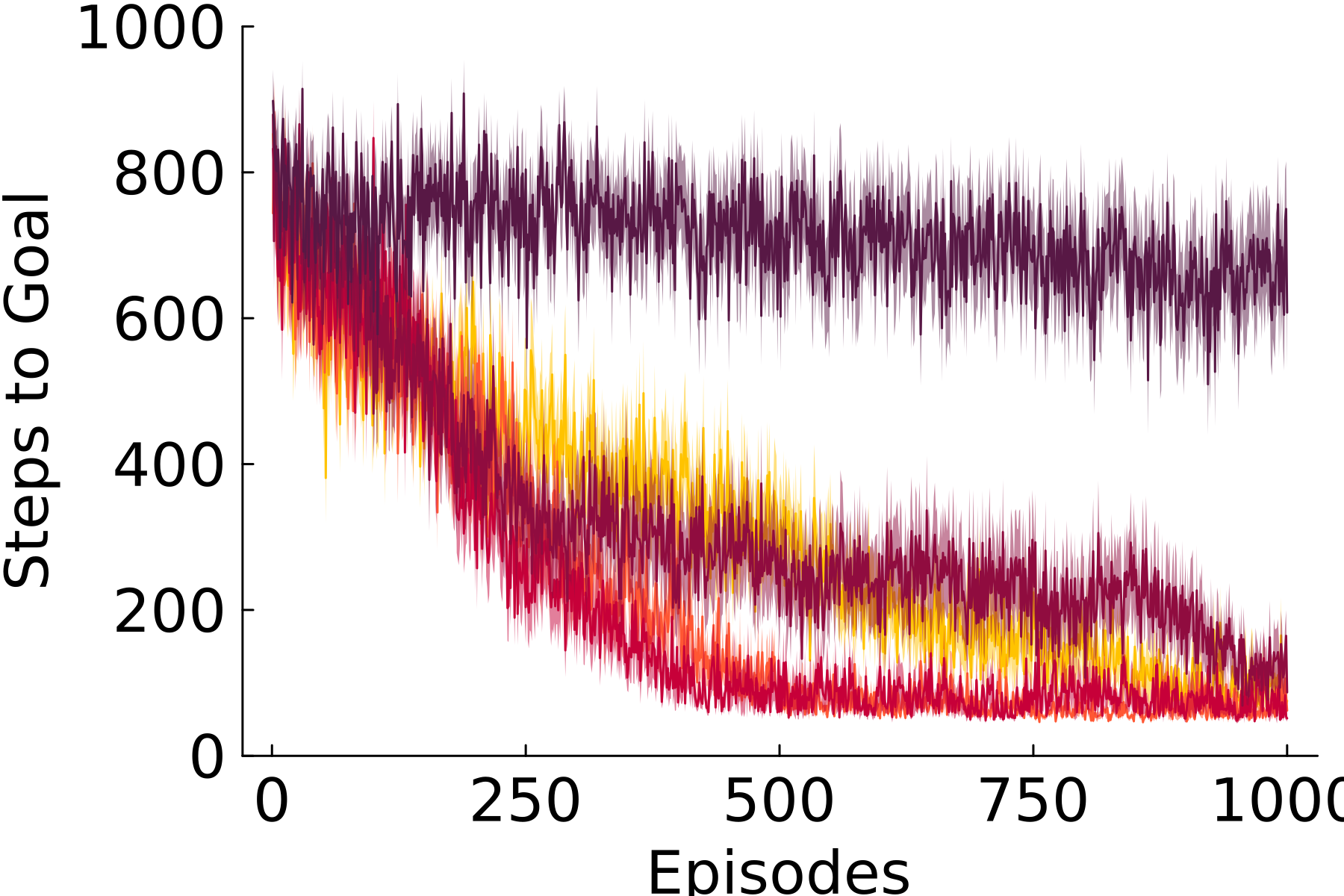}
        \includegraphics[width=0.45\linewidth]{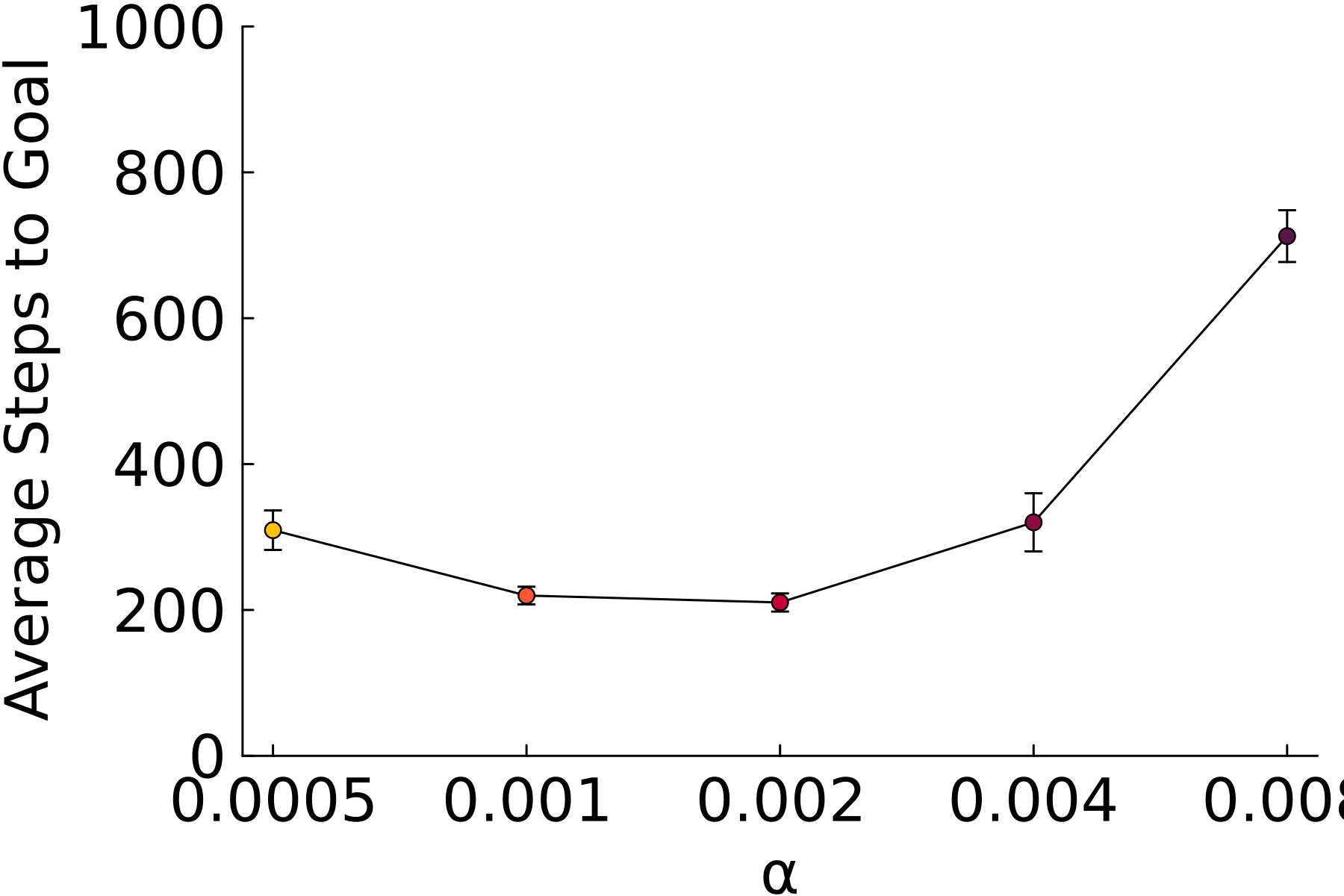}
         \caption{$\tau = 100$}
    \end{subfigure}
    \vspace{\baselineskip}
    
    \begin{subfigure}{0.65\textwidth}
        \includegraphics[width=0.45\linewidth]{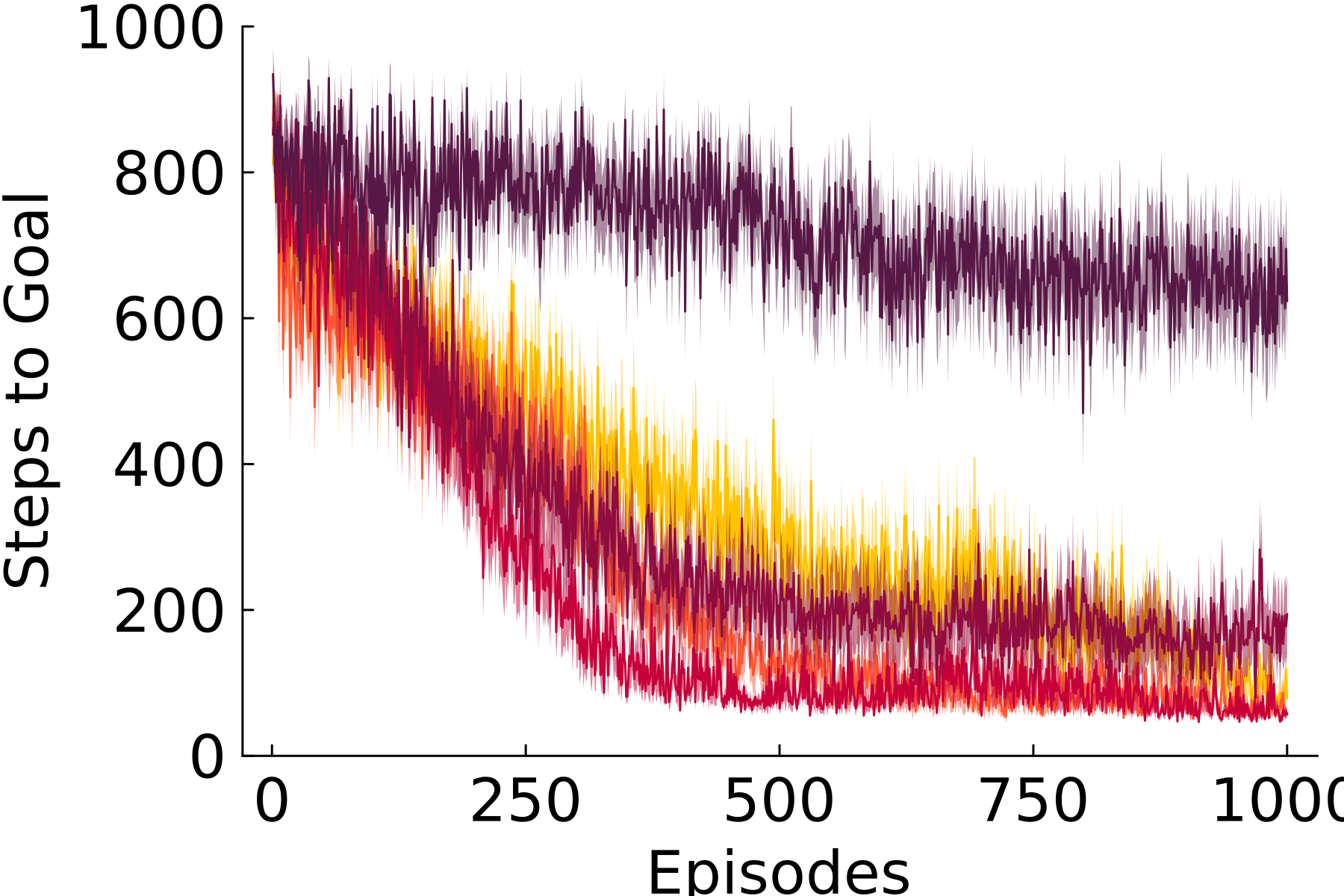}
        \includegraphics[width=0.45\linewidth]{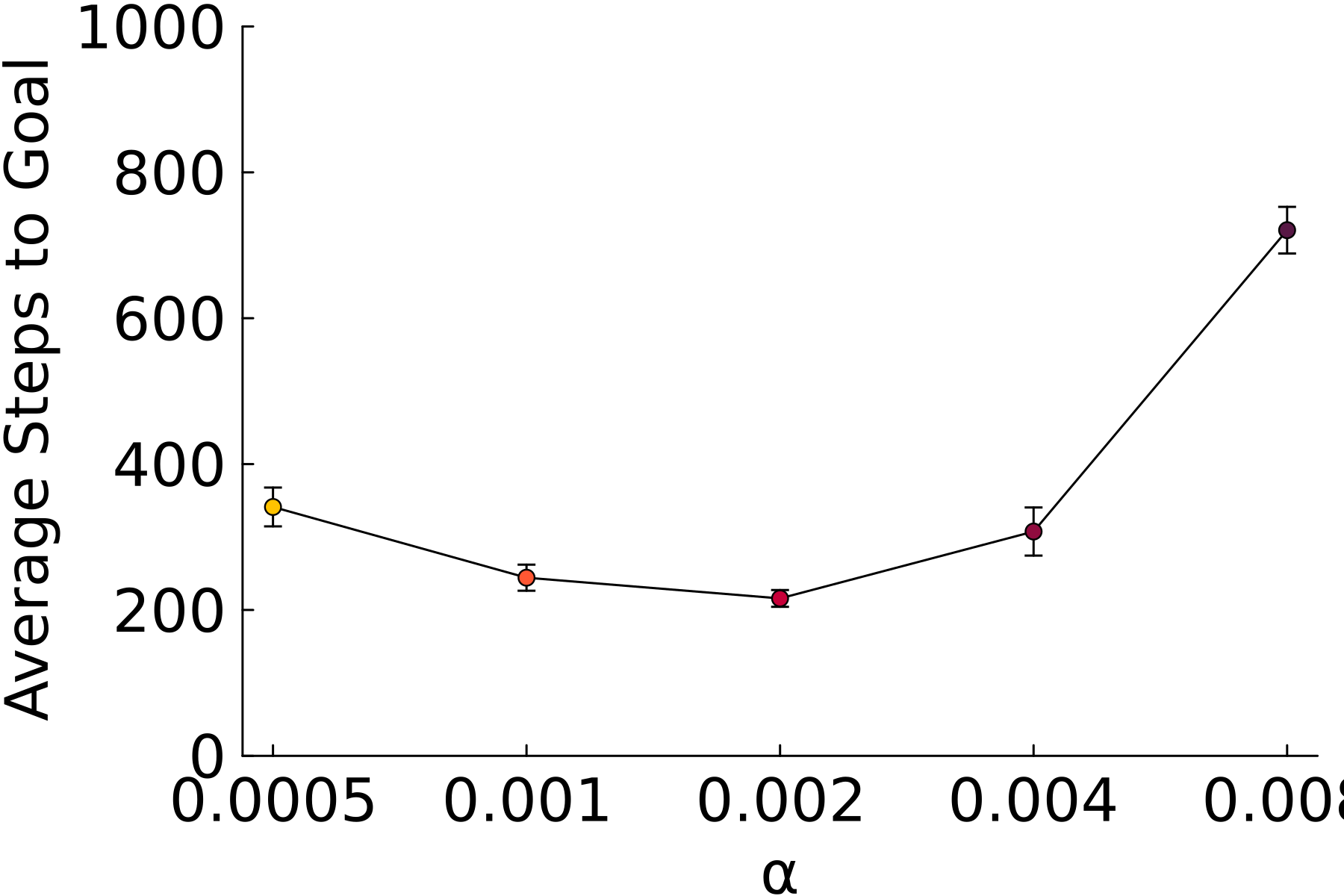}
         \caption{$\tau = 200$}
    \end{subfigure}
    \vspace{\baselineskip}
    
    \begin{subfigure}{0.65\textwidth}
        \includegraphics[width=0.45\linewidth]{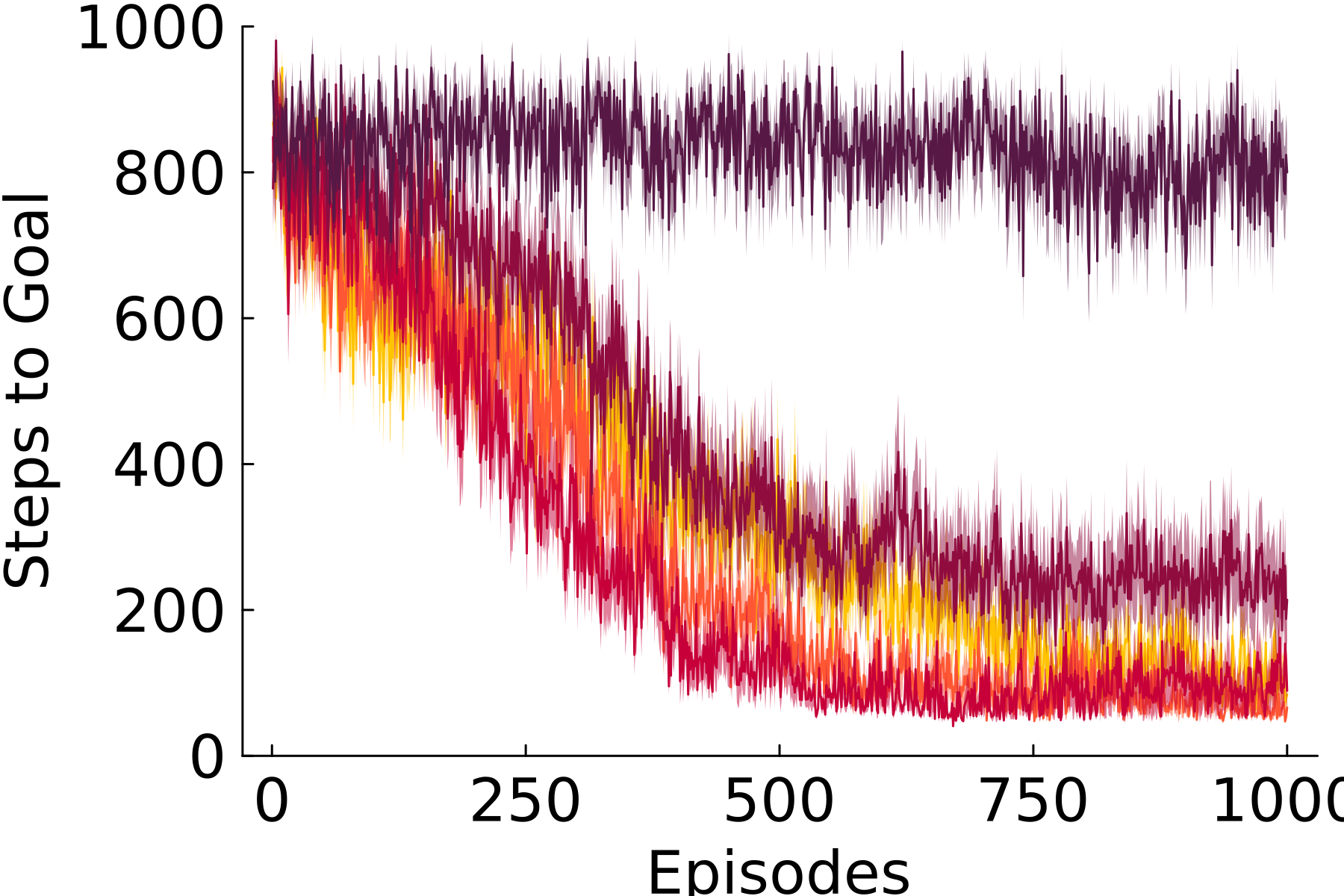}
        \includegraphics[width=0.45\linewidth]{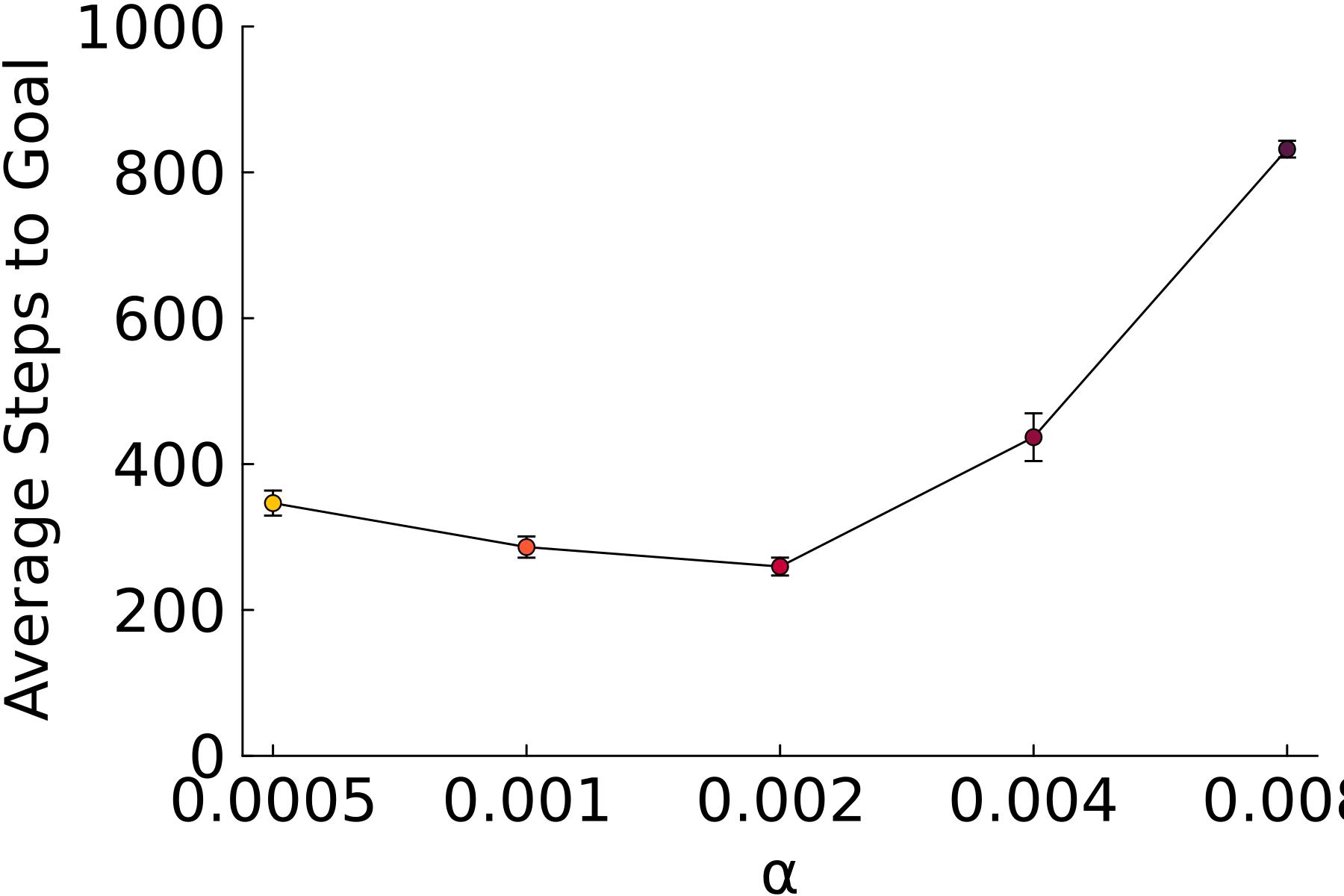}
         \caption{$\tau = 1000$}
    \end{subfigure}




    \caption{Left Column: each figure show the learning curves for five different step sizes, $\alpha$, averaged over 30 runs. Right Column: sensitivity of the DDQN base learner to different step sizes. Each dot represents the steps to goal for that learner, averaged over 30 runs and 1000 episodes. The error bars show one standard error. The refresh rate $\tau$ increases with each row.}
    \label{fig:ddqn_sweep}
\end{figure}